\documentclass{article} 
\usepackage{iclr2027_conference,times}
\usepackage{amsmath,amsfonts,amssymb,amsthm,commath,nicefrac,xfrac,mathtools,mathrsfs}
\usepackage[T1]{fontenc}
\usepackage{float}
\usepackage{bm,bbm}
\usepackage[inline]{enumitem}
\usepackage{xspace}
\usepackage{microtype}
\usepackage[dvipsnames]{xcolor}
\usepackage{colortbl}
\usepackage[toc,page]{appendix}
\usepackage{booktabs}       
\usepackage{icomma}
\usepackage{multirow}
\usepackage{hyperref}
\hypersetup{
  colorlinks=true,
  linkcolor=black,
  citecolor=blue,
  filecolor=magenta,
  urlcolor=blue
}
\usepackage{xurl}
\usepackage[capitalise]{cleveref}
\usepackage{adjustbox}
\usepackage{caption}
\usepackage{subcaption}
\crefname{assumption}{assumption}{assumptions}
\Crefname{assumption}{Assumption}{Assumptions}
\Crefname{equation}{Eq.\!}{Eqs.\!}
\Crefname{subfigure}{Figure}{Figures}
\Crefname{subtable}{Table}{Tables}

\usepackage[linesnumbered,lined,vlined,ruled,commentsnumbered,resetcount,algosection]{algorithm2e}
\SetKwComment{Comment}{// }{}
\SetArgSty{textsl}
\crefname{algocfline}{algorithm}{algorithms}
\Crefname{algocfline}{Algorithm}{Algorithms}
\counterwithin{algocfline}{section}

\theoremstyle{plain}
\newtheorem{theorem}{Theorem}[section]
\newtheorem*{theorem*}{Theorem}
\newtheorem*{lemma*}{Lemma}
\newtheorem*{corollary*}{Corollary}
\newtheorem{proposition}[theorem]{Proposition}
\newtheorem{proposition*}{Proposition}
\newtheorem{lemma}[theorem]{Lemma}
\newtheorem{corollary}[theorem]{Corollary}

\newtheorem{assumption}{Assumption}
\newtheorem*{assumption*}{Assumption}
\newtheorem{remark}[theorem]{Remark}

\usepackage{graphicx}
\usepackage{wrapfig}

\newcommand{\memcost}[1]{\textcolor{RoyalBlue}{#1}}

\usepackage{tikz}
\usetikzlibrary{arrows.meta,positioning,calc,decorations.pathreplacing,decorations.pathmorphing}
\usepackage{pgfplots}
\pgfplotsset{compat=1.18}
\usepgfplotslibrary{groupplots}

\DeclareMathOperator*{\argmax}{arg\,max}
\DeclareMathOperator*{\argmin}{arg\,min}
\DeclareMathOperator{\Cov}{Cov}
\DeclareMathOperator{\diag}{diag}
\DeclareMathOperator{\MSE}{MSE}
\DeclareMathOperator{\tr}{tr}
\DeclareMathOperator{\Proj}{Proj}
\DeclareMathOperator{\Span}{span}
\DeclareMathOperator{\vect}{vec}
\DeclareMathOperator{\Mat}{Mat}
\newcommand{\seqlen}{\mathsf{T}} 

\usepackage{titlesec}
\titlespacing*{\section}{0pt}{*1.6}{*0.8}
\titlespacing*{\subsection}{0pt}{*1.5}{*0.6}

\makeatletter
\renewcommand\paragraph{\@startsection{paragraph}{4}{\z@}%
  {0ex \@plus1ex \@minus1ex}%
  {-1em}%
  {\normalfont\normalsize\bfseries}}

\makeatother
\everypar{\looseness=-1}

\usepackage{import}
\usepackage{transparent}

\usepackage{tcolorbox}
\definecolor{takeawayaccent}{RGB}{176,108,32}   
\definecolor{intuitionaccent}{RGB}{56,98,152}   
\tcbset{paperbox/.style={%
      boxrule=0.6pt,
      arc=4pt,
      left=2pt,
      right=2pt,
      bottom=2pt,
      top=2pt,
      grow to left by=-0.01cm,
      grow to right by=-0.01cm,
      rounded corners,
      fonttitle=\bfseries,
      attach title to upper={\ },
      colframe=#1,
      coltitle=#1,
      colback=#1!5!white,
    }}
\newenvironment{takeaway}[1][]
{\begin{tcolorbox}[paperbox=takeawayaccent, title=Takeaway:]}
{\end{tcolorbox}}
\newenvironment{intuition}[1][]
{\begin{tcolorbox}[paperbox=intuitionaccent, title=Intuition:]}
{\end{tcolorbox}}
\makeatletter
\def\thm@space@setup{\thm@preskip=4pt plus 1pt minus 1pt \thm@postskip=4pt plus 1pt minus 1pt}
\makeatother
\tcbset{before skip=4pt plus 1pt,after skip=4pt plus 1pt}

\newcommand{\dashfill}{$\mathsurround=0pt \cleaders\hbox{$\mathsurround=0pt \mkern-2mu\smash-\mkern-2mu$}\hfill$}
\newcommand{\spectrumlabel}[1]{{\scriptsize\itshape #1}}
\newcommand{\spectrumrule}{{\color{intuitionaccent}%
    \leftarrowfill\ \spectrumlabel{strong: \(\mathcal{B} \uparrow\ \mathcal{V} \downarrow\)}\ \dashfill\ \spectrumlabel{data regularization}\ \dashfill\ \spectrumlabel{weak: \(\mathcal{B} \downarrow\ \mathcal{V} \uparrow\)}\ \rightarrowfill}}

\title{\textbf{Dr.\ Post-Training} \raisebox{-0.2em}{\includegraphics[height=1em]{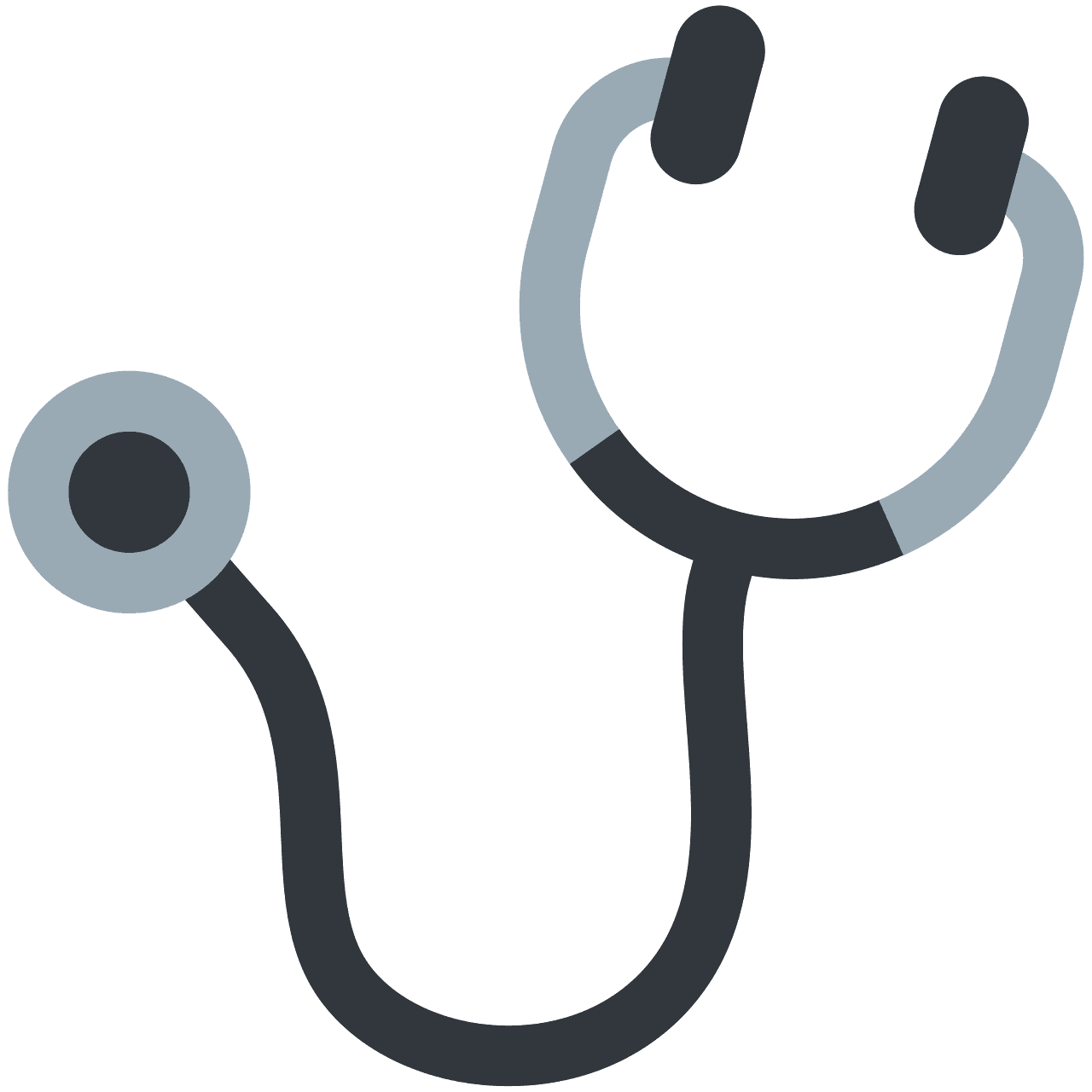}}: A \textbf{D}ata \textbf{R}egularization Perspective on LLM Post-Training}

\author{%
    Pingbang Hu\textsuperscript{\textnormal{1}}\thanks{Corresponding to: Pingbang Hu and Jiaqi W.\ Ma}\quad
    Xueshen Liu\textsuperscript{\textnormal{2}}\quad
    Z.\ Morley Mao\textsuperscript{\textnormal{2}}\quad
    Jiaqi W.\ Ma\textsuperscript{\textnormal{1}}\footnotemark[1]\\
    \textsuperscript{1}University of Illinois Urbana--Champaign\quad
    \textsuperscript{2}University of Michigan\\
    \texttt{\{pbb,jiaqima\}@illinois.edu}\quad
    \texttt{\{liuxs,zmao\}@umich.edu}
}

\iclrfinalcopy
\begin{document}

\maketitle
\lhead{Preprint}

\begin{abstract}
    Data selection methods address a critical challenge in LLM post-training: effectively leveraging scarce, high-fidelity target data alongside abundant but imperfectly aligned general training data. In this work, we move beyond the data-selection framing and introduce \textbf{Dr.\ Post-Training} (\textbf{D}ata-\textbf{R}egularized Post-Training), a novel framework that reconceptualizes general training data as a data-induced regularizer that prevents overfitting to the scarce target objective, rather than serving as a pool for selection. Specifically, our framework proposes that, at each training step, we construct a feasible set of model update directions using the general training data, and project the model update direction specified by the scarce target data onto that feasible set. Standard training and existing data selection methods arise as special cases with different choices of the data-induced regularizer, and these methods correspond to different points on a bias--variance spectrum with different regularization strengths. Building on this view, we propose a family of methods offering a richer design space and more flexible bias--variance tradeoffs. For practical LLM-scale use, we introduce careful system optimizations that realize these methods with minimal overhead. Extensive experiments across SFT, RLHF, and RLVR show that our methods outperform standard training and state-of-the-art data selection baselines, with the largest gains where the general and target data are most misaligned, and system benchmarks confirm their efficiency.
\end{abstract}
\section{Introduction}\label{sec:introduction}
\emph{Post-training} is a critical stage where pretrained large language models (LLMs) are adapted to desired downstream behaviors, including instruction following, dialogue, safety alignment, and domain-specific reasoning. It is typically carried out through supervised fine-tuning (SFT)~\citep{mishra2022cross,muennighoff2025generative} or reinforcement learning (RL) variants such as reinforcement learning from human feedback (RLHF)~\citep{christiano2017deep,ouyang2022training} and reinforcement learning with verifiable rewards (RLVR)~\citep{guo2025deepseek}. In many practical post-training settings, however, the data that most directly reflects the downstream objective is scarce: high-quality demonstrations~\citep{ouyang2022training}, domain-specific examples~\citep{ling2025domain}, preference labels~\citep{casper2023open}, or task-specific validation data~\citep{perez2021true} are often expensive to obtain. On the other hand, practitioners often have access to a large pool of general training data, such as broad instruction-tuning mixtures~\citep{wang2023far,lambert2025tulu}, which are abundant and cheaper but only partially aligned with the target objective~\citep{gururangan2020don,xia2024less,he2024what}. This creates a natural tradeoff in post-training: training on the scarce target data is faithful to the downstream objective but statistically noisy, while training on the abundant general data is more stable but may be biased by the misalignment with the target.

This tradeoff has led to a growing literature on data-centric methods for post-training, including heuristic filtering~\citep{bai2022training,ethayarajh2022understanding}, data mixture optimization~\citep{iyer2022opt,wang2023far}, offline data selection~\citep{he2024what,xia2024less}, and online data selection~\citep{wang2024greats,hu2025a}. Despite their algorithmic differences, most existing approaches are operationalized through prioritizing ``high-utility'' samples from the general training data according to some notion of utility, which makes it natural to frame the problem as \emph{data selection}~\citep{albalak2024a}.

However, the data selection perspective obscures a more fundamental question: how to effectively leverage abundant but imperfectly aligned general training data to improve a target objective defined by much scarcer data. Asking how this general dataset should influence the model's updates, rather than just which data to select, opens up a much broader design space.

In this work, we introduce \textbf{Dr.} (\textbf{D}ata-\textbf{R}egularized) Post-Training, a framework that formalizes this idea. At each training step, the target signal specifies the ideal direction for the model's update, while the general training data constrains the update by projecting it onto a feasible set supported by the available training data. From this viewpoint, general training data act as a data-induced \emph{regularizer}: they do not define the target objective itself, but restrict the admissible ways the model can move toward it, preventing overfitting to the scarce target data.

Interestingly, standard training and existing data selection methods arise as special cases of the proposed framework with different choices of the data-induced regularizer, and they can be viewed as different points on a bias--variance spectrum, corresponding to different regularization strengths. We make this spectrum precise: a feasible set incurs bias insofar as it fails to cover the target signal and variance insofar as it exposes the update to target noise, and we characterize the resulting per-step error and convergence. Building on this framework, we further propose new methods with \emph{Group-Wise Subset Update} that explicitly tune the strength of this data regularization, enabling more flexible tradeoffs between approximation bias and statistical variance.

Making this broader design space useful in practice requires addressing a second challenge: efficient implementation at LLM scale. Data-regularized updates require comparing training samples against the target signal and then assembling the final update from the selected samples. These operations introduce per-sample dependencies that are absent from standard training, where backpropagation usually computes only the aggregate batch gradient. A naive implementation would either run additional forward--backward passes or retain much more intermediate information than standard training, both of which are undesirable under the memory constraints of modern post-training pipelines.

To overcome this challenge, we conduct careful system-level optimizations to develop an efficient realization of Dr.\ Post-Training. We first develop a customized tensor lifetime scheduling strategy that selectively retains and releases intermediate quantities in the computation graph to obtain the data-regularized update within one forward--backward pass with minimal memory overhead. We further develop efficient approximate algorithms that significantly mitigate the main runtime bottleneck. We also demonstrate the compatibility between the proposed methods and modern memory-saving techniques such as LoRA finetuning or activation checkpointing.

Empirically, we validate Dr.\ Post-Training across the three post-training paradigms: SFT, RLHF, and RLVR. Our proposed methods outperform both standard training and prior state-of-the-art data-centric baselines, with gains that track the mismatch between the general and target data as the theory predicts: they are largest where the target is diluted in the general data, whereas standard training remains competitive where the two are aligned, all while incurring minimal system overhead.

In summary, our contributions are:
\begin{itemize}[leftmargin=*]
    \item \textbf{Framework.} We introduce \textbf{Dr.\ Post-Training}, a novel post-training framework that employs abundant yet imperfectly aligned general training data as data regularizers, with existing data selection methods as special cases, together with a bias--variance and convergence analysis.
    \item \textbf{Method.} Based on the Dr.\ Post-Training framework, we propose \emph{Group-Wise Subset Update}, which provides a broader design space of data regularizers with more flexible tradeoffs between approximation bias and statistical variance.
    \item \textbf{System.} We show that the proposed methods admit efficient implementations under the memory constraints of modern LLM training.
    \item \textbf{Experiments.} We demonstrate improvements of downstream performance over strong baselines across SFT, RLHF, and RLVR that track the mismatch between the general and target data as predicted, and conduct extensive system benchmarking to validate the efficiency of our implementation.
\end{itemize}
\section{Related Work}\label{sec:related-work}
\subsection{Data Optimization for LLM Post-Training}\label{subsec:data-selection-for-LLM-post-training}
Data optimization for LLM post-training aims to improve downstream performance by prioritizing high-utility training examples from large, heterogeneous corpora. We summarize several common families and refer readers to \citet{albalak2024a,deng2025survey} for a comprehensive overview.

Most existing methods approach this as a \emph{data selection} problem, leveraging downstream performance-related signals to determine which training examples the model sees. These span heuristic filtering and quality scoring~\citep{bai2022training,ethayarajh2022understanding,chen2024alpagasus,ivison2023data,lu2024instag,kung2023active}, dataset-level mixture design~\citep{cao2023instruction,wei2022finetuned,iyer2022opt,wang2023far}, and offline example-level selection via influence-style estimators, validation-based scoring, or importance weighting~\citep{xia2024less,he2024what,wang2024helpful,koh2017understanding,pruthi2020estimating,ghorbani2019data,melko2025topicreweighting}. Online methods adapt selection during training using per-step signals (e.g., loss- or gradient-based) to dynamically choose or weight examples~\citep{wang2024greats,jiao2025feasibility,hu2025a}, better tracking non-stationary utility at the cost of tighter compute constraints. In contrast, our proposed framework takes the \emph{dual} perspective, where the training examples are used to regularize the admissible downstream performance-driven updates, and unlocks a novel and broader design space.

\subsection{Memory Saving Techniques in LLM Post-Training}\label{subsec:memory-saving-techniques-in-LLM-post-training}
Modern LLM training is typically under stringent computational constraints, especially memory bottlenecks. Extensive research effort has therefore been devoted to parameter-efficient methods and memory-saving scheduling techniques that allow post-training to fit within available hardware budgets.

A common strategy restricts learning to a low-dimensional subspace. Low-Rank Adaptation (LoRA)~\citep{hu2022lora} injects trainable low-rank adapters into selected layers while keeping the pretrained parameters frozen, significantly reducing memory and compute requirements~\citep{han2024parameterefficient,renduchintala2024tied,sheng2023slora,zhang2023lora,xia2024chain,wang2023multilora,dettmers2023qlora}. Complementary to LoRA, Memory-efficient Subspace Optimization (MeSO) methods~\citep{zhao2024galore,muhamed2024grass,he2025subspace} instead compress optimizer states and gradients by projecting them into a low-dimensional representation, applying updates in that space, and mapping them back to the original parameter space for actual parameter updates. We provide further details in \Cref{adxsec:implementation-LoRA-MeSO}.

Complementary to low-dimensional subspace methods, memory-saving techniques restructure how computation is scheduled to further lower peak memory. Activation checkpointing~\citep{chen2016training} discards intermediate activations during the forward pass and recomputes them on-the-fly during the backward pass, trading additional computation for reduced memory footprint. Another technique, called gradient accumulation, is widely adopted in popular large-scale training frameworks~\citep{rasley2020deepspeed,vonwerra2022trl,sheng2024hybridflow}, partitions a large batch into micro-batches processed sequentially, accumulating gradients across micro-batches to simulate a larger effective batch size at a fraction of the peak memory cost. Both techniques preserve the optimization trajectory exactly and are widely composed with parameter-efficient methods in practice, making the resulting memory management landscape particularly intricate for methods that introduce additional per-sample computational dependencies beyond standard training.

In this work, we demonstrate that our proposed method integrates seamlessly with these techniques, showing its compatibility with common post-training pipelines (\Cref{subsec:compatibility}).

\section{Dr.\ Post-Training \texorpdfstring{\raisebox{-0.2em}{\includegraphics[height=1em]{Figures/icon.png}}}{}: A Data Regularization Framework}\label{sec:method}
We now introduce the \textbf{Dr.\ Post-Training} framework, which views general training data \emph{not} as a pool of examples to select from, but as a \emph{data regularizer} for optimizing a target objective. We consider a post-training setup (\Cref{subsec:problem-setting}) with access to two data sources: abundant general training data and limited target data. The framework is formulated through an iterative optimization procedure (\Cref{subsec:data-regularization-post-training-framework}) where each step draws a training batch and a target batch from these two sources. The target batch specifies the objective, while the training batch regularizes the optimization by defining a \emph{feasible set} of parameter update directions. Under this framework, standard training and existing data selection methods arise as special cases through different definitions of the feasible set induced by the training batch (\Cref{subsubsec:existing-methods}). Moreover, the framework suggests a broader design space that explicitly tunes the strength of data regularization, leading to a new family of methods (\Cref{subsubsec:group-wise-decomposition}). Next, we show that the methods within this framework lie on a spectrum with a bias--variance tradeoff and analyze when the resulting updates converge on the target objective (\Cref{subsec:bias-variance-tradeoffs}). Finally, we conclude this section with a remark on how the proposed framework may generalize to broader training paradigms (\Cref{subsec:beyond-two-distribution}).

\subsection{Problem Setup}\label{subsec:problem-setting}
Consider a post-training setup where a practitioner has access to two data distributions: a \emph{general training distribution}, from which abundant training data can be drawn, and a \emph{target distribution}, for which only limited data are available. The target distribution represents the downstream objective of interest, while the general training distribution provides a broader but potentially misaligned source of supervision. Our goal is to understand how to best leverage the abundant general training data to improve performance on the target objective.

This setup arises naturally in several common post-training scenarios:
\begin{enumerate*}[label=(\roman*)]
	\item general-purpose SFT, where a curated mixture of diverse tasks defines the training distribution and a specific downstream benchmark serves as the target~\citep{lambert2025tulu,grattafiori2024llama,qwen3technicalreport},
	\item domain-adaptive fine-tuning, where a broad general corpus supplements scarce domain-specific data~\citep{gururangan2020don}, and
	\item targeted instruction tuning, where representative examples from the target task guide selection from a large instruction corpus~\citep{xia2024less,he2024what,wang2024greats}.
\end{enumerate*}

\paragraph{Notation.}
We now introduce the necessary notation to formalize this setup. Let \(\ell(\theta; z)\) be a differentiable loss function over model parameters \(\theta \in \mathbb{R}^d\) and a data sample \(z\). Let the target distribution be \(\mathbb{P}_{\star}\), the corresponding population loss and the gradient are then defined as \(\mathcal{L}_\star(\theta) \coloneqq \mathbb{E}_{z \sim \mathbb{P}_\star}[\ell(\theta; z)]\) and \(g_\star(\theta) \coloneqq \nabla_\theta \mathcal{L}_\star(\theta) \in \mathbb{R}^d\). Similarly, for the general training distribution \(\mathbb{P}_{\mathrm{tr}}\), we define \(g_{\mathrm{tr}}(\theta) \coloneqq \nabla_{\theta} \mathcal{L}_{\mathrm{tr}}(\theta)\) such that \(\mathcal{L}_{\mathrm{tr}}(\theta) \coloneqq \mathbb{E}_{z \sim \mathbb{P}_{\mathrm{tr}}}[\ell(\theta ; z)]\).

\subsection{The Data-Regularization Framework for Post-Training}\label{subsec:data-regularization-post-training-framework}
We now formalize the central asymmetry in our setting: the target data determine \emph{what} objective should be optimized, while the general training data determine \emph{how} the update is allowed to move. Our framework adopts an iterative optimization view in which, at each step, the update direction is chosen to improve the target objective, but only from a set of directions induced by the training batch. In this way, the training data act as a \emph{data regularizer}: they do not specify the target objective itself, but instead constrain the admissible target-driven updates.

\subsubsection{Framework}
\paragraph{Target-driven one-step update.}
Modern post-training is carried out by iterative optimization~\citep{wei2022finetuned,schulman2017proximal,ouyang2022training,shao2024deepseekmath}, where at step \(t\), the model parameters \(\theta _t\) is updated via
\[
	\theta_{t+1} = \theta_t - \eta_t u_t,
\]
where \(\eta_t > 0\) is the step size and \(u_t \in \mathbb{R}^d\) is the update direction. Since the target objective is the ultimate quantity of interest, we adopt a one-step target-driven view and choose the update direction to minimize the next-step target loss within a prescribed \emph{feasible set} \(U_t\):
\begin{equation}\label{eq:opt}
	u_t
	= \argmin_{u \in U_t} \mathcal{L}_\star(\theta_t - \eta_t u).
\end{equation}
The role of the feasible set is central: it determines which update directions are allowed at step \(t\).

\paragraph{Training-batch-induced feasible set.}
At each step, we draw two batches:
\[
	B_t = \{z_i\}_{i=1}^n \stackrel{\mathrm{i.i.d.}}{\sim} \mathbb{P}_{\mathrm{tr}},
	\qquad
	B_t^\star = \{z_j^\star\}_{j=1}^m \stackrel{\mathrm{i.i.d.}}{\sim} \mathbb{P}_{\star}.
\]
The training batch \(B_t\) is used to construct the feasible set \(U_t\), while the target batch \(B_t^\star\) is used to evaluate which direction in \(U_t\) best decreases the target loss \(\mathcal{L}_\star\).

More concretely, for any set of data \(B\), let \(g_B(\theta_t)\) denotes the batch gradient \(\frac{1}{\lvert B \rvert } \sum_{z \in B} \nabla _{\theta } \ell (\theta _t; z)\), and for convenience, we further write
\[
	g_i(\theta_t) \coloneqq g_{\{z_i\}}(\theta_t) = \nabla_\theta \ell(\theta_t; z_i), \qquad
	g_j^\star(\theta_t) \coloneqq g_{\{z_j^{\star}\}}(\theta_t) =\nabla_\theta \ell(\theta_t; z_j^\star)
\]
denote the per-sample gradients from the training and target batches, respectively, and define
\[
	\hat{g}_{\mathrm{tr}}(\theta_t) \coloneqq g_{B_t}(\theta_t) = \frac{1}{n}\sum_{i=1}^n g_i(\theta_t), \qquad
	\hat{g}_{\star}(\theta_t) \coloneqq g_{B_t^{\star}}(\theta_t) = \frac{1}{m}\sum_{j=1}^m g_j^\star(\theta_t).
\]
We fix a step \(t\) and omit writing \(\theta_t\) when it is clear from the context hereafter, e.g., \(g_i\), \(g_j^{\star}\), \(\hat{g}_{\mathrm{tr}}\), \(\hat{g}_{\star}\), and also the population gradients \(g_\star \coloneqq g_\star(\theta_t)\) and \(g_{\mathrm{tr}} \coloneqq g_{\mathrm{tr}}(\theta_t)\).

The key modeling choice in our framework is that \(U_t\) depends only on the training batch \(B_t\), typically through the collection of training gradients \(\{g_i\}_{i=1}^n\), whereas the target batch \(B_t^\star\) enters only through the target objective via information such as \(\hat{g}_{\star}\), as we will soon see. This separation makes explicit the distinct roles of the two data sources: the target batch identifies the desired direction of improvement, while the training batch constrains which directions are admissible.

\subsubsection{The Dual Perspective of Data Selection and Data Regularization}
The formulation above admits a dual perspective of \emph{data selection} and \emph{data regularization}, as illustrated in \Cref{fig:dr-post-training}.

\begin{figure}[htpb]
	\centering
  \def\svgwidth{\linewidth}
  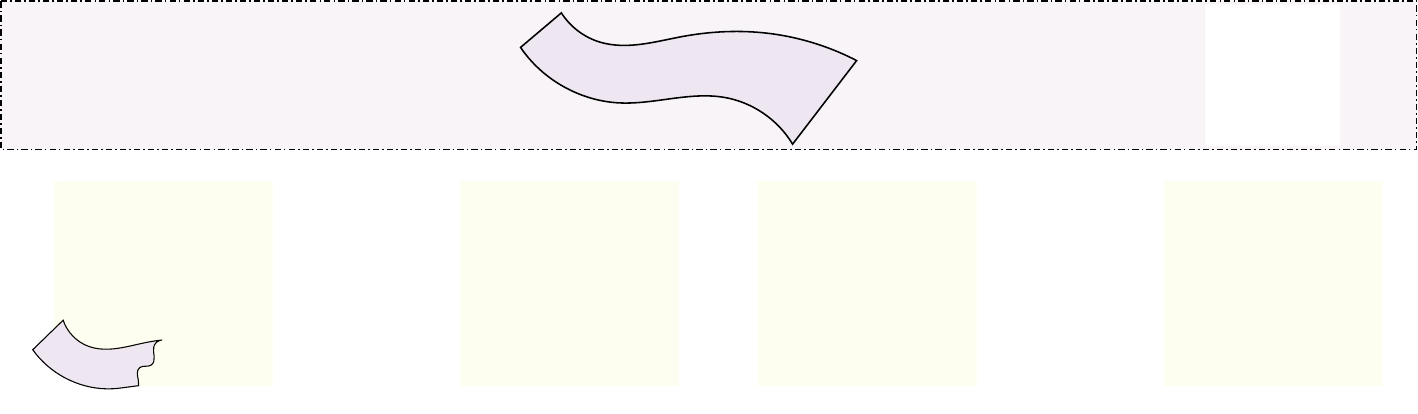

	\caption{The dual view of data selection and data regularization. }
	\label{fig:dr-post-training}
\end{figure}

\paragraph{The conventional data selection view.}
The target objective determines which training examples are most useful, and the update is formed from those selected examples. This is the perspective taken by most existing data selection methods~\citep{xia2024less,he2024what,wang2024greats}.

\paragraph{The novel data regularization view.}
The target objective specifies the ideal direction of improvement, while the training batch regularizes this update by restricting it to lie in the feasible set \(U_t\). Without such a regularization, one would attempt to optimize the target objective directly, but this can be statistically unstable when target data are scarce. Constraining the update to directions supported by the training batch can reduce this instability, at the price of introducing approximation bias when the training and target distributions are misaligned.

This viewpoint suggests that the complexity of the feasible set controls a \textbf{bias--variance tradeoff}. A smaller feasible set imposes stronger regularization: it stabilizes optimization but may bias the update away from the ideal target direction. A larger feasible set weakens this regularization: it allows more faithful target-driven updates, but may incur higher statistical variance. We formalize this tradeoff in \Cref{subsec:bias-variance-tradeoffs} after introducing some concrete instantiations of the feasible sets below.

\subsection{Concrete Instantiations of the Feasible Sets}\label{subsec:designs-of-feasible-set}
We next turn the framework into concrete methods. We begin by deriving a tractable approximation of the optimization problem in \Cref{eq:opt} (\Cref{subsubsec:tractable-approximation}). With this approximation, several existing methods arise as special cases by specifying how the training batch induces the feasible set \(U_t\) (\Cref{subsubsec:existing-methods}). We then use this perspective to motivate a broader design space of data regularizer via group-wise decomposition (\Cref{subsubsec:group-wise-decomposition}).

\subsubsection{A Tractable Approximation of \Cref{eq:opt}}\label{subsubsec:tractable-approximation}
We first derive a tractable approximation of the optimization problem in \Cref{eq:opt}. For arbitrary feasible set, the constrained optimization for \(\mathcal{L}_\star (\theta_t - \eta_t u)\) over \(u \in U_t\) is generally difficult, especially when \(U_t\) is discrete. We therefore apply the majorization--minimization principle~\citep{lange2000optimization,mairal2013optimization,lange2016mm} with a standard smoothness assumption~\citep{bottou2018optimization,nesterov2004introductory}.

\begin{assumption}[Target smoothness]\label{ass:smooth}
	The target population loss \(\mathcal{L}_\star\) is \(\beta\)-smooth for some \(\beta > 0\), and bounded below, \(\mathcal{L}_\star^{\inf} \coloneqq \inf_\theta \mathcal{L}_\star(\theta) > -\infty\).
\end{assumption}
Under \Cref{ass:smooth}, the target loss admits the following quadratic majorization upper bound:
\begin{equation}\label{eq:major}
	\mathcal{L}_{\star}(\theta_t - \eta_t u)
	\leq \mathcal{L}_{\star}(\theta_t) - \eta_t \langle g_{\star}, u \rangle + \frac{\beta \eta_t^2}{2} \lVert u \rVert^2,
\end{equation}
where \(\mathcal{L}_{\star}(\theta_t)\) is independent of \(u\) and the population gradient \(g_{\star}\) is unknown. Hence, minimizing this over \(u \in U_t\) and substituting the target batch gradient estimate \(\hat{g}_\star\) for the unknown population gradient \(g_\star\) yields a Euclidean projection \(u_t = \Proj_{U_t}(\hat{g}_{\star})\), or equivalently (see derivation in \Cref{adxsubsec:mm-derivation}):
\begin{equation}\label{eq:proj}
	u_t
	= \argmin_{u \in U_t} \lVert u - \hat{g}_\star \rVert^2.
\end{equation}
The analysis below treats \(u_t\) as the parameter step; an optimizer-aware variant for adaptive optimizers such as AdamW, which our experiments use, is noted in \Cref{rmk:optimizer-aware}. We are now ready to show how different design choices of the feasible set correspond to different methods.

\subsubsection{Existing Methods as Special Cases}\label{subsubsec:existing-methods}
We now show how several existing training paradigms can be viewed as special cases of our framework. For ease of exposition, we refer to these instantiations as the \emph{Target-Only Update}, the \emph{Full-Training Update}, and the \emph{Global Subset Update}.

\paragraph{Target-Only Update.}
The least regularized instantiation ignores the training batch when constructing the feasible set and allows the update to move in any direction in parameter space. This corresponds to choosing \(U_t^\star \coloneqq \mathbb{R}^d\).
Because there is no directional constraint, \Cref{eq:proj} gives \(u_t = \hat{g}_\star\).
Thus, the Target-Only Update recovers gradient descent on the target batch \(B_t^\star\), completely ignoring the training batch \(B_t\). In practice, this corresponds to fine-tuning exclusively on high-quality, task-specific data~\citep{zhou2023lima,gunasekar2023textbooks}. The resulting update is unbiased for \(g_\star\), but it can have high variance when the target batch size \(m\) is small.

\paragraph{Full-Training Update.}
At the opposite extreme, the strongest regularization is obtained by allowing only the aggregate training-batch gradient as the update direction. This corresponds to the singleton feasible set \(U_t^{\mathrm{tr}} \coloneqq \{\hat{g}_{\mathrm{tr}}\}\).
In this case, \Cref{eq:proj} trivially gives \(u_t = \hat{g}_{\mathrm{tr}}\).
The Full-Training Update recovers standard gradient descent on the general training distribution, as in general-purpose SFT, where the model is trained on a curated mixture without per-step target feedback~\citep{lambert2025tulu,grattafiori2024llama}. Here, the training batch imposes the strongest possible constraint on the update direction, eliminating any per-step dependence on the target batch.

\paragraph{Global Subset Update.}
Between these two extremes, existing data selection methods can be expressed by allowing the update to be the average gradient of a selected subset of training samples. For a subset size \(k\), define \(U_t^{\mathrm{glob}}(k) \coloneqq \{ g_{S} = \frac{1}{k}\sum_{i \in S} g_i \colon S \subseteq [n], \lvert S\rvert = k \}\). This feasible set contains all \(k\)-sample averages of training gradients from \(B_t\). Under \Cref{eq:proj}, the resulting update is obtained by choosing \(S_t \in \argmin_{S \subseteq [n], \lvert S\rvert = k} \lVert g_S - \hat{g}_\star \rVert^2\), and then setting \(u_t = g_{S_t}\). We refer to this instantiation as the \emph{Global} Subset Update, since a single subset \(S\) is shared across \emph{all} parameters, in contrast to the \emph{Group-Wise Subset Update} to be introduced next in \Cref{subsubsec:group-wise-decomposition}.

The Global Subset Update interpolates between the target-only and Full-Training Updates. Its feasible set \(U_t^{\mathrm{glob}}(k)\) is richer than the singleton feasible set \(U_t^{\mathrm{tr}} = \{\hat{g}_{\mathrm{tr}}\}\), but remains more restrictive than the unconstrained feasible set \(U_t^\star = \mathbb{R}^d\). However, solving for \(S\) is a combinatorial optimization problem over training subsets. Expanding the objective and dropping the term independent of \(S\), the subset optimization is equivalent to minimizing \(\lVert g_S \rVert ^2 - 2 \langle g_S, \hat{g}_\star \rangle \). Thus, the objective balances a first-order \emph{alignment} term, \(2 \langle g_S, \hat{g}_\star \rangle = \frac{2}{k}\sum_{i \in S} \langle g_i, \hat{g}_\star \rangle\), which favors samples aligned with the target gradient, against a second-order \emph{redundancy} penalty, \(\lVert g_S \rVert ^2 = \frac{1}{k^2} \sum_{i,i^{\prime} \in S} \langle g_i, g_{i^{\prime}} \rangle\), which discourages selecting highly correlated training gradients. The alignment term decomposes into independent per-sample scores
\[
	s_i \coloneqq \langle g_i, \hat{g}_\star \rangle,
\]
whereas the redundancy penalty couples all samples in \(S\), making exact optimization intractable. Existing data selection methods can therefore be viewed as tractable approximations to this subset optimization problem (\Cref{adxsubsec:global-subset-derivation}):
\begin{itemize}[leftmargin=*]
	\item \textbf{Top-\(k\).} Choose the \(k\) samples with the largest scores \(s_i\), ignoring the redundancy penalty. This is the cheapest strategy and is effective when samples are not highly correlated~\citep{han2021influence,hu2024most,he2024what}.
	\item \textbf{Thresholding.} Keep all samples whose score exceeds a fixed threshold, yielding a variable-size subset~\citep{hu2025a}. This can be viewed as a variable-cardinality relaxation of the fixed-\(k\) feasible set, and it also ignores the redundancy penalty.
	\item \textbf{Greedy construction.} Iteratively construct \(S_t\) by adding the sample that most improves the full objective, thereby accounting for redundancy at each step~\citep{wang2024greats}. This better approximates the projection problem, but incurs an additional \(O(nk)\) selection cost.
\end{itemize}

\subsubsection{New Data-Regularization Designs via Group-Wise Decomposition}\label{subsubsec:group-wise-decomposition}
The three instantiations above are only a few points in the broader space of feasible-set designs. In principle, \Cref{eq:proj} allows many possible choices of \(U_t\subseteq \mathbb{R}^d\) as long as it is independent of target data; our focus here is the ones that actually depend on training data. We introduce a new family of \emph{data regularizers}: feasible sets constructed from the training batch that are more flexible than Global Subset Update, while still restricting target-driven updates to directions supported by training samples.

The motivation is to relax the global coupling imposed by \(U_t^{\mathrm{glob}}(k)\), where a single subset \(S\) of training samples must be shared across all parameter coordinates. This can be overly restrictive when different parameter groups align with the target signal in different ways. For example, a training sample that provides a useful update for one layer or module may be less informative for another. Thus, rather than selecting one global subset for the entire model, we partition the parameter coordinates into groups and allow each group to choose its own subset of training samples. This gives a controlled relaxation of the feasible set: it weakens the data regularization imposed by the Global Subset Update, while avoiding the fully unconstrained Target-Only Update.

\paragraph{Group-Wise Subset Update.}
Let \(\mathcal{G} = \{G_p\}_{p=1}^P\) be a partition of the parameter indices \([d]\) into \(P\) disjoint groups. For any vector \(v \in \mathbb{R}^d\), let \(v^{(p)}\) denote the subvector indexed by \(G_p\). In particular, \(g_i^{(p)}\) and \(\hat{g}_\star^{(p)}\) denote the corresponding subvectors of \(g_i\) and \(\hat{g}_\star\).

For a fixed subset size \(k\), define the group-wise feasible set as
\[
	U_t^{\mathrm{grp}}(k; \mathcal{G})
	\coloneqq \left\{ u = (u^{(1)}, \dots, u^{(P)}) \colon u^{(p)} \in U_t^{(p)}(k) \text{ for all } p \in [P] \right\},
\]
where
\[
	U_t^{(p)}(k)
	\coloneqq \left\{ g_{S}^{(p)} \coloneqq \frac{1}{k}\sum_{i \in S} g_i^{(p)} \colon S \subseteq [n], \lvert S \rvert = k \right\}.
\]
In other words, each parameter group \(G_p\) is allowed to choose its own subset \(S_{t,p}\) of training samples. We refer to this instantiation as the \emph{Group-Wise Subset Update}. A practically important special case is the \emph{Layer-Wise Subset Update}, where each group corresponds to the parameters of a single layer.

The Global Subset Update is recovered as the special case in which all groups share the same subset,
\[
	S_{t,1}
	= \cdots
	= S_{t,P}.
\]
Therefore,
\[
	U_t^{\mathrm{glob}}(k)
	\subseteq U_t^{\mathrm{grp}}(k; \mathcal{G}),
\]
with the inclusion typically being strict for nontrivial partitions. Thus, the Group-Wise Subset Update enlarges the feasible set relative to the Global Subset Update. From the data-regularization perspective, this reduces the strength of the regularization imposed by requiring a single global subset, and can reduce approximation bias when that global constraint is too restrictive.

This construction also preserves a useful separability structure. Because the group-wise feasible set factorizes across parameter groups, \(U_t^{\mathrm{grp}}(k; \mathcal{G}) = \prod_{p=1}^P U_t^{(p)}(k)\), the projection problem in \Cref{eq:proj} decomposes into independent per-group problems (\Cref{adxsubsec:group-wise-derivation}):
\[
	\min_{u \in U_t^{\mathrm{grp}}(k; \mathcal{G})} \lVert u-\hat{g}_\star \rVert ^2
	= \sum_{p=1}^{P} \min_{u^{(p)} \in U_t^{(p)}(k)} \lVert u^{(p)}-\hat{g}_\star^{(p)} \rVert ^2.
\]
Equivalently, for each group \(p\), we choose
\[
	S_{t,p}
	\in \argmin_{S \subseteq [n], \lvert S\rvert = k} \left\lVert g_{S}^{(p)} - \hat{g}_\star^{(p)} \right\rVert^2,
\]
and set
\[
	u_t^{(p)}
	= g_{S_{t,p}}^{(p)}
	= \frac{1}{k}\sum_{i \in S_{t,p}} g_i^{(p)}.
\]
Thus, the same approximation strategies used for the Global Subset Update can be applied independently within each parameter group. For example, the group-wise top-\(k\) rule scores each training sample by
\[
	s_i^{(p)}
	\coloneqq \langle g_i^{(p)}, \hat{g}_\star^{(p)} \rangle
\]
and selects the \(k\) samples with the largest scores separately for each group \(G_p\). Thresholding and greedy construction can be adapted analogously.

The partition \(\mathcal{G}\), therefore, becomes a new design knob. Coarser partitions impose stronger data regularization by forcing larger portions of the model to share the same subset, while finer partitions allow more flexible target-driven updates. This produces a spectrum of feasible sets between the Global Subset Update and the Target-Only Update, which we analyze through the bias--variance tradeoff in \Cref{subsec:bias-variance-tradeoffs}.

\subsection{Bias--Variance Tradeoffs and Convergence}\label{subsec:bias-variance-tradeoffs}
We now quantify the bias--variance tradeoff of the feasible set. The feasible sets are nested, from the singleton \(U_t^{\mathrm{tr}}\) to \(U_t^\star = \mathbb{R}^d\): for fixed \(k\) and \(\mathcal{G}\),
\[
	U_t^{\mathrm{glob}}(k)
	\subseteq U_t^{\mathrm{grp}}(k; \mathcal{G})
	\subseteq U_t^\star
	= \mathbb{R}^d,
\]
and both collapse to \(U_t^{\mathrm{tr}}\) at \(k = n\), so a larger set means weaker data regularization. We bound the error of one update by two quantities of \(U_t\) that also govern convergence; proofs are in \Cref{adxsubsec:bias-variance-proof,adxsubsec:trajectory-proofs}.

\paragraph{Two sources of error.}
Fix \(\theta_t\) and take expectations over both batches. For an update \(u\), we measure its error by the conditional mean squared error
\[
	\MSE(u)
	\coloneqq \mathbb{E}\left[\lVert u - g_\star \rVert^2 \middle| \theta_t\right].
\]
The following descent lemma reduces the progress of one step to this error.

\begin{lemma}[Descent lemma]\label{lma:majorization}
	Assume \Cref{ass:smooth}. Fix \(\theta_t\) and let \(0 < \eta_t \leq 1/\beta\). For any update \(u\) with \(\mathbb{E}[\lVert u \rVert ^2 \mid \theta_t] < \infty\) and \(\theta_{t+1} = \theta_t - \eta_t u\),
	\[
		\mathbb{E}\left[\mathcal{L}_\star(\theta_{t+1}) \mid \theta_t\right]
		\leq \mathcal{L}_\star(\theta_t) - \frac{\eta_t}{2} \lVert g_\star \rVert ^2 + \frac{\eta_t}{2} \MSE(u).
	\]
\end{lemma}

One step with update \(u\) thus decreases the expected target loss by at least \(\frac{\eta_t}{2}(\lVert g_\star \rVert^2 - \MSE(u))\), so we quantify \(\MSE(u_t)\) for \(u_t = \Proj_{U_t}(\hat{g}_\star)\), which can miss \(g_\star\) for two reasons:
\begin{itemize}
	\item \textbf{Bias.} Restricted to training-supported directions, \(U_t\) may contain nothing near \(g_\star\).
	\item \textbf{Variance.} The projection uses \(\hat{g}_\star = g_\star + \xi\), so the target-batch noise \(\xi\) perturbs its solution.
\end{itemize}

\begin{intuition}
	Only the noise along candidate differences matters, since the choice between \(u, u^{\prime} \in U_t\) depends on \(\xi\) only through \(\langle u - u^{\prime}, \xi \rangle\); noise orthogonal to them does not.
\end{intuition}

\begin{wrapfigure}[14]{r}{0.3\linewidth}
	\centering
	\vspace{-1\intextsep}
  \def\svgwidth{\linewidth}
  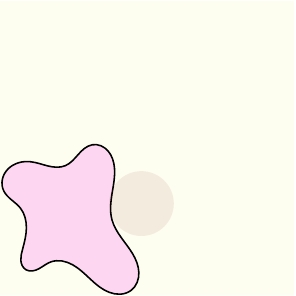

	\caption{Geometry of the one-step error for a chosen \(U_t\).}
	\label{fig:one-step-error}
\end{wrapfigure}

Formally, for a feasible set \(U_t\), we define the \emph{bias} and the \emph{exposed variance} as
\[
	\begin{aligned}
		\mathcal{B}(U_t)
		 & \coloneqq \mathbb{E}\left[\inf_{u \in U_t} \lVert u - g_\star \rVert^2 \middle| \theta_t\right], \\
		\mathcal{V}(U_t)
		 & \coloneqq \mathbb{E}\left[\lVert \Proj_{D_t}(\xi) \rVert^2 \middle| \theta_t\right],
	\end{aligned}
\]
where \(D_t \coloneqq \Span\{u - u^{\prime} \colon u, u^{\prime} \in U_t\}\) is the update-difference subspace. Given the projected update \(u_t\) of \Cref{eq:proj}, we write
\begin{equation}\label{eq:mse-decomposition}
	\MSE(u_t)
	= \mathcal{B}(U_t) + \mathcal{E}(u_t),
\end{equation}
where \(\mathcal{E}(u_t) \coloneqq \MSE(u_t) - \mathcal{B}(U_t)\) is the \emph{excess error} of selecting within \(U_t\) with the noisy estimate \(\hat{g}_\star\) rather than with \(g_\star\).

Let \(\Sigma_{\mathrm{tr}} \coloneqq \Cov(g_i \mid \theta_t)\) and \(\Sigma_\star \coloneqq \Cov(g_j^\star \mid \theta_t)\). For the two endpoints of the spectrum the decomposition is exact.

\begin{proposition}[Bias and excess error at the endpoints]\label{prop:bias-variance}
	Fix \(\theta_t\) and let \(\eta_t = 1/\beta\). Then:
	\begin{enumerate}[label=(\roman*),leftmargin=*]
		\item\label{prop:bias-variance-i} \textbf{Full-Training Update.}
		      For \(U_t^{\mathrm{tr}}=\{\hat{g}_{\mathrm{tr}}\}\),
		      \[
			      \mathcal{B}(U_t^{\mathrm{tr}})
			      = \lVert g_{\mathrm{tr}} - g_\star \rVert^2 + \frac{\tr(\Sigma_{\mathrm{tr}})}{n},\qquad
			      \mathcal{E}(u_t^{\mathrm{tr}})
			      = 0.
		      \]
		\item\label{prop:bias-variance-ii} \textbf{Target-Only Update.}
		      For \(U_t^\star=\mathbb{R}^d\),
		      \[
			      \mathcal{B}(U_t^{\star})
			      = 0,\qquad
			      \mathcal{E}(u_t^{\star})
			      = \frac{\tr(\Sigma_\star)}{m}.
		      \]
	\end{enumerate}
\end{proposition}

\begin{remark}[Bias includes the fluctuation of the feasible set]\label{rmk:classical-bias-variance}
	\(\mathcal{B}(U_t)\) is the expected squared distance from \(g_\star\) to the random feasible set, so it includes the sampling fluctuation of \(U_t\) itself: for Full-Training Update it contains the training-batch term \(\tr(\Sigma_{\mathrm{tr}}) / n\), which a classical bias--variance decomposition of the estimator \(\hat{g}_{\mathrm{tr}}\) would call variance. We count it as bias because it persists with a noiseless target (\(m \to \infty\)), whereas the excess error \(\mathcal{E}(u_t)\) vanishes in that limit (\Cref{thm:projection-oracle}).
\end{remark}

Between the endpoints, the excess error is controlled by the exposed variance alone. For the bounds in terms of the per-direction noise level we use a standard sub-Gaussian noise assumption~\citep{lan2020first,liu2023high}.

\begin{assumption}[Sub-Gaussian noise]\label{ass:sub-Gaussian-noise}
	The target gradient noise \(\xi \coloneqq \hat{g}_\star - g_\star\) is sub-Gaussian with parameter \(\sigma/\sqrt m\) for some \(\sigma > 0\), i.e., \(\mathbb{E}[\exp(\lambda \langle \xi, v \rangle) \mid \theta_t] \leq \exp(\lambda^2 \sigma^2 \lVert v \rVert^2 / (2m))\) for all \(v \in \mathbb{R}^d\) and \(\lambda \in \mathbb{R}\).
\end{assumption}

\begin{theorem}[Noisy projection oracle inequality]\label{thm:projection-oracle}
	Fix \(\theta_t\), and let \(U_t\) be any nonempty closed feasible set determined by \(B_t\). The update \(u_t\) of \Cref{eq:proj} satisfies
	\[
		\sqrt{\MSE(u_t)} \leq \sqrt{\mathcal{B}(U_t)} + 2\sqrt{\mathcal{V}(U_t)},
		\qquad \text{equivalently} \qquad
		\mathcal{E}(u_t) \leq 4\sqrt{\mathcal{B}(U_t)\,\mathcal{V}(U_t)} + 4\mathcal{V}(U_t),
	\]
	where \(\mathcal{V}(U_t) = \mathbb{E}[\tr(\Proj_{D_t} \Sigma_\star) \mid \theta_t] / m\). Moreover:
	\begin{enumerate}[label=(\roman*),leftmargin=*]
		\item\label{thm:projection-oracle-i} \textbf{Orderings.} If \(U_t \subseteq U_t^{\prime}\) are two feasible sets determined by \(B_t\), then \(\mathcal{B}(U_t^{\prime}) \leq \mathcal{B}(U_t)\) and \(\mathcal{V}(U_t) \leq \mathcal{V}(U_t^{\prime})\). In particular, for any \(k \in [n]\) and partition \(\mathcal{G}\), \(\mathcal{B}(U_t^{\mathrm{grp}}(k; \mathcal{G})) \leq \mathcal{B}(U_t^{\mathrm{glob}}(k))\) and \(\mathcal{V}(U_t^{\mathrm{glob}}(k)) \leq \mathcal{V}(U_t^{\mathrm{grp}}(k; \mathcal{G}))\), while \(\mathcal{V}(U_t^{\mathrm{tr}}) = 0\) and \(\mathcal{V}(U_t^{\star}) = \tr(\Sigma_\star)/m\).
		\item\label{thm:projection-oracle-ii} \textbf{Exposed-variance bounds.} Under \Cref{ass:sub-Gaussian-noise}, for any \(k \in [n]\) and partition \(\mathcal{G}\) with \(P\) groups, \(\mathcal{V}(U_t^{\mathrm{glob}}(k)) \leq (n-1)\sigma^2/m\) and \(\mathcal{V}(U_t^{\mathrm{grp}}(k; \mathcal{G})) \leq \sigma^2 \sum_{p=1}^{P} \min\{\lvert G_p \rvert, n-1\} / m \leq P(n-1)\sigma^2/m\).
	\end{enumerate}
\end{theorem}

\Cref{fig:one-step-error} illustrates the geometry: enlarging \(U_t\) can only decrease the bias and increase the exposed variance. For target noise with variance at most \(\sigma^2 / m\) per direction, \Cref{tab:bias-exposure} lists both quantities. Full-Training Update pays the full mismatch but sees no target noise, Target-Only Update is unbiased but exposed to all \(d\) noise directions, and refining the partition lowers the bias while its exposed-variance bound grows at most \(P\)-fold. Since \(\mathcal{V}(U_t) = O(1/m)\) while the bias does not depend on \(m\), the bound on the excess error \(\mathcal{E}(u_t)\) decays as \(1/m\) when \(\mathcal{B}(U_t) = 0\) but only as \(1/\sqrt{m}\) otherwise, the rate set by the cross term (\Cref{adxsubsubsec:rates}). The classical finite-class bound through the cardinality of \(U_t\), its comparison with \Cref{thm:projection-oracle}, and a unified bound that pays per group the smaller of the two complexity measures are given in \Cref{adxsubsubsec:two-bounds}.

\begin{table}[t]
	\centering
	\caption{Bias and exposed variance of the four updates (\Cref{thm:projection-oracle}).}
	\label{tab:bias-exposure}
	\small
	\setlength{\tabcolsep}{6.9pt}
\begin{tabular}{l cccc}
	\toprule
	\rowcolor{intuitionaccent!8}
	& Full-Training & Global Subset & Group-Wise Subset & Target-Only \\
	\addlinespace[1pt]
	& \multicolumn{4}{c}{\spectrumrule} \\
	\addlinespace[1pt]
	\(\mathcal{B}(U_t)\) & \(\lVert g_{\mathrm{tr}} - g_\star \rVert^2 + \tr(\Sigma_{\mathrm{tr}}) / n\) & \(\mathcal{B}(U_t^{\mathrm{glob}}(k))\) & \(\leq \mathcal{B}(U_t^{\mathrm{glob}}(k))\) & \(0\) \\
	\(\mathcal{V}(U_t)\) & \(0\) & \(\leq (n - 1)\sigma^2 / m\) & \(\leq P(n - 1)\sigma^2 / m\) & \(\tr(\Sigma_\star) / m \leq d\sigma^2 / m\) \\
	\bottomrule
\end{tabular}

\end{table}

\paragraph{From one step to the trajectory.}
With \Cref{thm:projection-oracle} in hand, we ask when repeated updates converge on the target objective. By \Cref{lma:majorization}, step \(t\) decreases the target loss by at least \(\frac{\eta_t}{2}(\lVert g_\star(\theta_t) \rVert^2 - \MSE(u_t))\); summed over steps, this bounds the average target gradient by the average one-step error, so convergence hinges on the error relative to the current gradient. An absolute bound on \(\MSE(u_t)\) yields progress only while \(\lVert g_\star(\theta_t) \rVert^2\) exceeds it, whereas a relative bound \(\MSE(u_t) \leq q^2 \lVert g_\star(\theta_t) \rVert^2\) with \(q < 1\) retains at every step a \((1 - q^2)\) fraction of the decrease of an exact target-gradient step. We say that the feasible sets \emph{cover} the target gradient along the trajectory if, for every \(t\),
\begin{equation}\label{eq:coverage-condition}
	\sqrt{\mathcal{B}(U_t)} + 2\sqrt{\mathcal{V}(U_t)}
	\leq q \lVert g_\star(\theta_t) \rVert + \phi_t
\end{equation}
with \(q < 1\): some feasible direction approximates the target gradient to within a fraction \(q\) of its length while exposing comparable noise, up to a floor \(\phi_t \geq 0\) from the training--target mismatch and the target noise. We state this condition separately for the two quantities.

\begin{assumption}[Coverage and directional noise along the trajectory]\label{ass:coverage}
	There are constants \(\rho \in [0, 1)\), \(\omega \geq 0\) and deterministic \(b_t, \varsigma_t \geq 0\) such that, almost surely for every \(t\), \(\mathcal{B}(U_t) \leq \rho \lVert g_\star(\theta_t) \rVert^2 + b_t^2\) and \(\mathcal{V}(U_t) \leq \omega \lVert g_\star(\theta_t) \rVert^2 + \varsigma_t^2\), where \(\mathcal{B}(U_t)\) and \(\mathcal{V}(U_t)\) are the bias and the exposed variance defined above, conditional on the history up to step \(t\).
\end{assumption}

\Cref{ass:coverage} implies the combined coverage condition \Cref{eq:coverage-condition} with \(q \coloneqq \sqrt{\rho} + 2\sqrt{\omega}\) and \(\phi_t \coloneqq b_t + 2\varsigma_t\) (\Cref{lma:relative-mse}), and the proofs use only this combined form. For the convergence of the iterates themselves we additionally use the Polyak--\L{}ojasiewicz condition.

\begin{assumption}[Polyak--\L{}ojasiewicz condition~\citep{polyak1963gradient,karimi2016linear}]\label{ass:PL}
	There is \(\mu > 0\) such that \(\lVert g_\star(\theta) \rVert^2 \geq 2\mu \bigl(\mathcal{L}_\star(\theta) - \mathcal{L}_\star^{\inf}\bigr)\) for all \(\theta \in \mathbb{R}^d\).
\end{assumption}

Under the coverage condition, summing the descent bound over steps yields the following guarantee. We write \(g_{\star, t} \coloneqq g_\star(\theta_t)\), \(\Lambda_T \coloneqq \sum_{t < T} \eta_t\), and \(\Delta_t \coloneqq \mathbb{E}[\mathcal{L}_\star(\theta_t)] - \mathcal{L}_\star^{\inf}\), so that \(\Delta_0 = \mathcal{L}_\star(\theta_0) - \mathcal{L}_\star^{\inf}\).

\begin{theorem}[Trajectory bound]\label{thm:trajectory}
	Under \Cref{ass:smooth,ass:coverage} with \(q < 1\) and deterministic step sizes \(0 < \eta_t \leq 1/\beta\), for every \(T \geq 1\),
	\[
		\frac{\sum_{t < T} \eta_t\, \mathbb{E}[\lVert g_{\star, t} \rVert^2]}{\Lambda_T}
		\leq \frac{4 \Delta_0}{(1 - q^2) \Lambda_T}
		+ \frac{2(1 + q^2)}{(1 - q^2)^2} \cdot \frac{\sum_{t < T} \eta_t (b_t + 2\varsigma_t)^2}{\Lambda_T}.
	\]
\end{theorem}

\begin{corollary}[Exact and whole-sequence convergence]\label{cor:exact-convergence}
	Under \Cref{ass:smooth,ass:coverage} with \(b_t = \varsigma_t = 0\), \(q < 1\), and a constant step size \(0 < \eta \leq 1/\beta\), \(\frac{1}{T} \sum_{t < T} \mathbb{E}[\lVert g_{\star, t} \rVert^2] \leq 2 \Delta_0 / (\eta (1 - q^2) T)\) and \(g_{\star, t} \to 0\) almost surely. If in addition \Cref{ass:PL} holds, then \(\Delta_T \leq (1 - \mu \eta (1 - q^2))^T \Delta_0\), and \(\theta_t\) converges almost surely, with finite path length \(\sum_t \lVert \theta_{t+1} - \theta_t \rVert < \infty\), to a random \(\theta_\infty\) with \(\mathcal{L}_\star(\theta_\infty) = \mathcal{L}_\star^{\inf}\) (\Cref{thm:finite-length}).
\end{corollary}

The first term of \Cref{thm:trajectory} is the optimization term, the initial suboptimality over the total step size, which vanishes as \(\Lambda_T \to \infty\); the second is a stationarity floor set by the part of the one-step error that does not shrink with the gradient (\(\phi_t = b_t + 2\varsigma_t\)), such as distribution mismatch and target noise. The average squared target gradient is thus driven to the floor, at rate \(1/(\eta T)\) when the floor vanishes and \(\eta_t \equiv \eta\); \Cref{cor:exact-convergence} adds almost sure convergence, and under \Cref{ass:PL} linear convergence of the target loss and convergence of the iterates. Refining the partition lowers the bias part of the coverage condition but raises its variance part, as the group-wise set covers the target gradient better but exposes more noise; a finer partition is certified when the coverage gained outweighs the noise exposed, and since exposed noise scales as \(1/m\), larger target batches certify finer ones (\Cref{adxsubsec:trajectory-proofs}). Which update dominates for a given target batch size and training--target mismatch is sketched in \Cref{adxsubsubsec:regimes}.

\begin{takeaway}
	\emph{Larger feasible sets reduce bias but raise exposed variance}: the update gains freedom to align with \(g_\star\) but \textbf{also with the noise} \(\xi\). The same two quantities govern convergence along the trajectory. Full-Training is stable but biased, Target-Only unbiased but unstable at small \(m\), and the subset updates interpolate between them; which is best depends on the mismatch (\Cref{sec:experiment}).
\end{takeaway}

\subsection{Generalization Beyond Two-Distribution Setting}\label{subsec:beyond-two-distribution}
Although we developed the framework in the two-distribution setting above, the underlying principle is more general. At its core, the data regularization principle rests on the following structure:
\begin{enumerate}[label=(\roman*)]
	\item \emph{Training signal} that can be optimized reliably but might introduce bias, and
	\item \emph{Target signal} that is too noisy or costly to optimize directly.
\end{enumerate}
In the two-distribution setting, the target batch provides the noisy signal for optimizing the target objective, while the training batch induces the feasible set that stabilizes the update.

Other common post-training paradigms also admit this structure: for instance, in reinforcement learning, the target signal of interest is the expected reward, but directly optimizing it via policy gradient is high-variance. Training signals such as PPO's clipped loss~\citep{schulman2017proximal} or GRPO's group-relative estimator~\citep{shao2024deepseekmath}, on the other hand, stabilize training at the cost of optimization bias. We demonstrate and apply these principles in the experiments (\Cref{subsec:RLHF,subsec:RLVR}).
\section{Efficient Realization of Dr.\ Post-Training Under Memory Constraints}\label{sec:system}
Modern LLM training infrastructures are heavily constrained by memory bottlenecks. Efficiently realizing the proposed method under such tight memory limits, therefore, requires nontrivial system-level optimizations: the proposed methods introduce extra computational dependencies on \emph{per-sample} quantities that standard backpropagation does not preserve, which a naive implementation can resolve only at the cost of either an extra forward-backward pass or retaining the entire computation graph, both prohibitive at LLM scale. In this section, we discuss the optimizations that make Dr.\ Post-Training practical and compatible with modern infrastructures.

Specifically, we will start by revisiting the memory footprint of standard training (\Cref{subsec:standard-training}) that grounds the later discussions in three aspects. First, in \Cref{subsec:one-pass-tensor-lifetime-scheduling}, we address the computational overhead of data-regularized updates with a carefully designed tensor lifetime schedule that can efficiently reuse the intermediate variables in standard forward and backward passes with negligible additional memory cost. Second, in \Cref{subsec:efficient-scoring}, we investigate the computational costs of the key per-sample scoring step incurred by the data-regularized updates, develop several implementations with different compute--memory trade-offs, and identify the most efficient choice for different regimes. Third, in \Cref{subsec:compatibility}, we show that the resulting design remains compatible with modern memory-saving training techniques such as LoRA, MeSO, activation checkpointing, and gradient accumulation. Throughout, we focus on layer-aligned partitions, in which every group \(G_p \in \mathcal{G}\) consists of whole layers; the extension to general partitions, a two-pass alternative for tight memory budgets, the exact complexity derivations, the details of the LoRA and MeSO integration, and the treatment of the embedding and normalization layers are deferred to \Cref{adxsec:system}.

\subsection{Preliminary: Memory Footprint in Standard Training}\label{subsec:standard-training}
Here, we provide a quick overview of the memory footprint in a standard training step that consists of a forward and a backward pass, where we simplify the transformer model into a stack of \(L\) MLP linear layers for discussion convenience.\footnote{Throughout this section, ``layer'' refers to an individual linear module (e.g., \texttt{nn.Linear} in PyTorch), not a full transformer block. A single transformer block contains multiple linear modules (query, key, value, output projections, and MLP layers). Data regularization targets these linear modules, which account for the vast majority of transformer parameters; the treatment of embedding layers and other non-linear modules is discussed in \Cref{adxsubsec:beyond-linear}.} Let the model with \(d\) total parameters \(\theta_t = (\theta_t^{(1)}, \dots, \theta_t^{(L)}) \in \mathbb{R} ^d\), with each layer \(\theta_t ^{(l)}\) having \(d/L\) parameters and activations of dimension \(\sqrt{d/L}\) for simplicity. We use both the vector form \(\theta_t^{(l)} \in \mathbb{R}^{d/L}\) and the matrix form \(W_t^{(l)} \in \mathbb{R}^{\sqrt{d/L} \times \sqrt{d/L}}\) interchangeably, related by \(\theta_t^{(l)} = \vect(W_t^{(l)})\); the same applies to weight gradients in matrix form \(G^{(l)} \in \mathbb{R}^{\sqrt{d/L} \times \sqrt{d/L}}\) with \(g^{(l)} = \vect(G^{(l)})\).

\subsubsection{Forward and Backward Pass}
Given a training batch \(B_t = \{z_i\}_{i=1}^n\) of sequences of length \(\seqlen\), a standard training step first performs a forward pass to evaluate the batch loss \(\ell \coloneqq \sum_{i=1}^{n} \ell_i\) with \(\ell_i \coloneqq \ell(\theta_t ; z_i)\), caching intermediate activations, then a backward pass to compute the batch gradient \(\hat{g}_{\mathrm{tr}} = (\hat{g}_{\mathrm{tr}}^{(1)}, \dots, \hat{g}_{\mathrm{tr}}^{(L)})\). \Cref{fig:standrad} illustrates both passes at a single layer \(l\) along with the resulting memory footprint. For simplicity, we omit the dimension of sequence length in the illustration.

\begin{figure}[htpb]
	\centering
  \def\svgwidth{\columnwidth}
  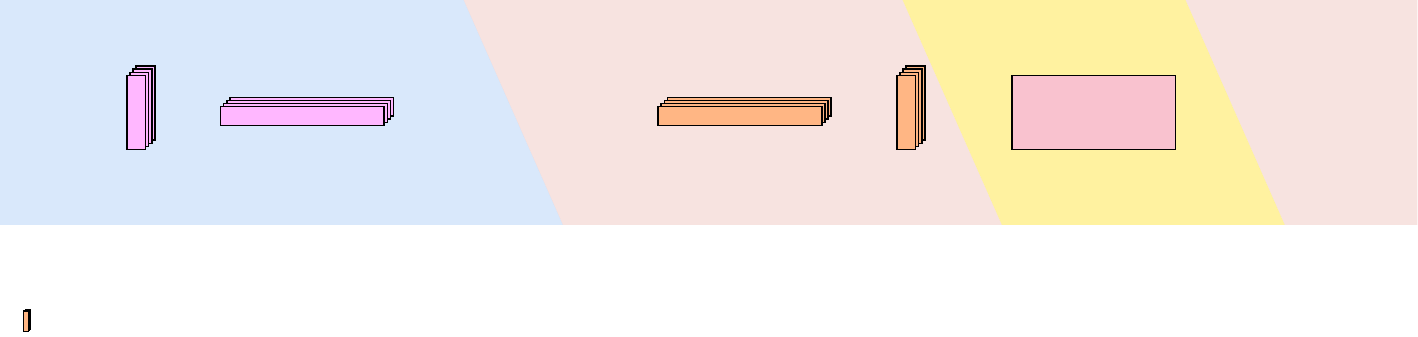

	\caption[Forward/backward pass and memory footprint.]{\textbf{Top:} Forward pass and backward pass (activation and weight gradient) computation graph: both \(\longrightarrow\) and \(\dashrightarrow\) denote \emph{computational dependency}; \(\dashrightarrow\) indicates that the dependency runs across the boundary between forward and backward. \textbf{Bottom:} Memory footprint: a tensor can be released once no remaining operation depends on it.}
	\label{fig:standrad}
\end{figure}

\paragraph{Forward pass.}
Each layer \(l\) computes the output (pre-activation) of the layer \(e_{i,\tau}^{(l)} = W_t^{(l)} a_{i,\tau}^{(l)}\) from the input activation \(a_{i,\tau}^{(l)} \in \mathbb{R}^{\sqrt{d/L}}\) for each sample \(i \in [n]\) and token \(\tau \in [\seqlen]\), then applies an activation function \(\sigma \): \(a_{i,\tau}^{(l+1)} = \sigma(e_{i,\tau}^{(l)})\). In practice, all \(n\seqlen\) vectors are batched as columns of an activation matrix, so each layer reduces to a single matrix multiplication. All activations and pre-activations \(\{a_{i,\tau}^{(l)}, e^{(l)}_{i, \tau}\}_{i,\tau}\) are \emph{cached} for later use.

\paragraph{Backward pass.}
The activation gradient propagates backward layer by layer. At layer \(l\), the pre-activation gradient \(\partial \ell / \partial e_{i,\tau}^{(l)}\) arrives for each sample \(i\) and token \(\tau\). The activation gradient propagates to layer \(l-1\) via \(\partial \ell / \partial e_{i,\tau}^{(l-1)} = \sigma^{\prime}(e_{i,\tau}^{(l-1)}) \odot \partial \ell / \partial a_{i,\tau}^{(l)}\) where \(\partial \ell / \partial a_{i,\tau}^{(l)} = W_t^{(l)\top} \partial \ell / \partial e_{i,\tau}^{(l)}\) and \(\odot\) denotes the element-wise product. The batch weight gradient is then computed as \(\hat{g}_{\mathrm{tr}}^{(l)} = \frac{1}{n}\sum_{i=1}^{n}\sum_{\tau=1}^{\seqlen} (\partial \ell / \partial e_{i,\tau}^{(l)}) \otimes a_{i,\tau}^{(l)} \in \mathbb{R}^{d/L}\), where \(\otimes\) denotes the Kronecker product; for column vectors \(b\) and \(a\) we identify \(b \otimes a\) with the rank-one matrix \(b a^{\top}\) (its matrix form), so that \(g_i^{(l)} = \sum_{\tau} (\partial \ell / \partial e_{i,\tau}^{(l)}) (a_{i,\tau}^{(l)})^{\top}\) in matrix form, and we switch between the two forms freely as above. For later reference, we write the per-sample weight gradient as \(g_i^{(l)} \coloneqq \sum_{\tau=1}^{\seqlen} (\partial \ell / \partial e_{i,\tau}^{(l)}) \otimes a_{i,\tau}^{(l)}\), so that \(\hat{g}_{\mathrm{tr}}^{(l)} = \frac{1}{n}\sum_{i=1}^{n} g_i^{(l)}\). In practice, however, standard backpropagation does not explicitly form each \(g_i^{(l)}\). Instead, it stacks all token-level activations and pre-activation gradients across the batch and evaluates the above double sum with a single batched matrix multiplication, which directly returns the aggregated batch gradient \(\hat{g}_{\mathrm{tr}}^{(l)}\). Thus, standard training computes the \emph{batch gradient} without materializing separate \emph{per-sample weight gradients} in memory.

\subsubsection{Memory Footprint of Key Tensors during Training}
During each step of the standard training update, a tensor can be released as soon as no remaining computation depends on it. As the goal is to compute \(\hat{g}_{\mathrm{tr}}\), a key observation is that during backward pass at layer \(l-1\), both \(a^{(l)} \) and \(\partial \ell / \partial e^{(l)}\) become free once the batch gradient \(\hat{g}_{\mathrm{tr}}^{(l)}\) is assembled (\Cref{fig:standrad}). Moreover, \(\partial \ell / \partial e^{(l+1)}\) is released before \(\partial \ell / \partial e^{(l)}\) is allocated, so its buffer is reused and the memory footprint of pre-activation gradients remains constant across layers.

\Cref{algo:standard-training} presents the pseudocode of one standard training step with memory management, which recovers the Full-Training Update. For simplicity, we hide the computation of \(a_i^{(l)}\), \(e_i^{(l)}\), their gradient, and also their dependency on each other in the backward function call. The creation and release of key tensors of interest are annotated with \memcost{memory cost} (\(+\) allocation, \(-\) release) to be compared with later algorithms. We note that when the training data is \(B_t^{\star}\) instead of \(B_t\), we recover Target-Only Update.

\begin{algorithm}[htpb]\label{algo:standard-training}
	\DontPrintSemicolon{}
	\caption{Standard Training Update}
	\KwData{Model \(\theta_t\), training data \(\{z_i\}_{i=1}^n\), learning rate \(\eta_t\)}
	\KwResult{Updated model parameters \(\theta_{t+1}\)}
	\SetKwFunction{Forward}{Forward}
	\SetKwFunction{Backward}{Backward}
	\SetKwFunction{Release}{Release}

	\BlankLine

	(\(\ell\), \(\{a_{i}^{(l)}, e_{i}^{(l)}\}_{l,i}\))\(\gets\)\Forward{\(\theta_t\), \(\{z_i\}_{i=1}^{n}\)}\Comment*[r]{\memcost{\(+ O(2n\seqlen \sqrt{dL})\)}}
	\(u_t \gets 0 \in \mathbb{R}^d\)\Comment*[r]{\memcost{\(+ O(d)\)}}
	\For{\(l = L, \ldots, 1\)}{
		\tcc{Allocate \(\{\partial \ell / \partial e_{i}^{(l)}\}_{i=1}^{n}\), release \(\{e_{i}^{(l)}\}_{i=1}^{n}\); memory remains constant}
		\(\{\partial \ell / \partial e_{i}^{(l)}\}_{i=1}^{n} \gets\)\Backward{\(\ell\), \(\theta_t^{(l)}\)}\Comment*[r]{\memcost{\(+O(n\seqlen\sqrt{d/L}) - O(n\seqlen\sqrt{d/L})\)}}
		\(u_t^{(l)} \gets \frac{1}{n} \sum_{i=1}^{n}\sum_{\tau=1}^{\seqlen} (\partial \ell / \partial e_{i,\tau}^{(l)}) \otimes a_{i,\tau}^{(l)}\)\;
		\Release{\(a_i^{(l)}\), \(\partial \ell / \partial e_i^{(l)}\)} for all \(i\in [n]\)\Comment*[r]{\memcost{\(- O(2n\seqlen \sqrt{d/L})\)}}
	}
	\(\theta_{t+1} \gets \theta_t - \eta_t u_t\)\;
	\Release{\(u_t\)}\Comment*[r]{\memcost{\(- O(d)\)}}
	\Return{\(\theta_{t+1}\)}\;
\end{algorithm}

\subsection{Customized Tensor Lifetime Scheduling for Efficient Data-Regularized Update}\label{subsec:one-pass-tensor-lifetime-scheduling}
We now turn to a main system challenge introduced by the additional computational dependencies in data-regularized updates. Recall that, unlike standard training, our method must first determine the final update direction by solving a subset-selection problem based on per-sample scores, and only then assemble the corresponding gradient update. This creates additional computational dependencies on intermediate per-sample quantities that are not preserved in standard backpropagation.

\paragraph{Naive implementations.}
A naive implementation is a \emph{two-pass} approach: use one forward-backward pass to compute the per-sample scores and determine the selected subset, and then run a second forward-backward pass to compute the final gradient update using only the selected subset.\footnote{For example, an existing online data selection method~\citep{wang2024greats} adopts a two-pass implementation. See their implementation in \url{https://github.com/Jiachen-T-Wang/GREATS/} and an in-depth discussion in \Cref{adxsec:two-pass-implementation}.} This approach is simple and memory-efficient because each pass follows the standard training schedule, but it substantially increases runtime by effectively duplicating the backward computation.

Another naive implementation is a \emph{retain-graph} approach: retain the computation graph (e.g., by calling \texttt{retain\_graph=True} in \texttt{PyTorch}) during backpropagation and keep all intermediate tensors alive so that the same computation graph can be reused for both scoring and gradient assembly. While this avoids the second pass, it is impractical at LLM scale, since retaining the full graph together with the per-sample quantities needed by our method leads to prohibitive memory overhead.

\subsubsection{Customized Tensor Lifetime Scheduling}
We develop a customized tensor lifetime scheduling to obtain the data-regularized gradient update within \emph{one pass} while avoiding significant memory overhead.

For convenience, we consider \emph{layer-aligned} partition in the rest of the discussion, where each group contains one or more layers. We direct interested readers to \Cref{adxsubsec:general-partition} for a discussion on general partitions, where each group might contain partial parameters of a layer.

\paragraph{A starting trick: merging training and target batch.}
We start by addressing a problem that the training gradients and target gradients are typically calculated in separate passes, which prevents effective tensor lifetime scheduling in one pass. We address this problem by realizing that these gradients can be obtained from a \emph{merged batch} with both training and target samples. Specifically, suppose a training step uses a training batch \(B_t = \{z_i\}_{i=1}^n\) and a target batch \(B_t^{\star} = \{z_j^{\star}\}_{j=1}^m\). Rather than processing them in separate passes, we concatenate them into one batch of size \(N = n+m\) and define the merged loss
\[
	\ell
	\coloneqq \sum_{i=1}^n \ell(\theta_t; z_i) + \sum_{j=1}^m \ell(\theta_t; z_j^\star).
\]
Because different samples do not share activations, a single forward-backward pass on this merged loss produces all \emph{per-sample} backward signals for both the training samples and the target samples. From these, we can extract the quantities needed for constructing the data-regularized update direction.

\begin{remark}
	Doing forward and backward pass on the merged-batch loss \(\ell \coloneqq \sum_{i=1}^{n} \ell_i + \sum_{j=1}^{m} \ell_j^{\star}\) is an existing trick in the literature~\citep{pandya2025sidda,wang2024greats}. We note that this extends to non-loss-gradient target signals for an arbitrary \(\theta_t\)-differentiable target signal \(f(\{z_j^\star\}_{j=1}^{m}; \theta_t)\) (cf.\ \Cref{subsec:beyond-two-distribution}): backward pass on the merged batch with \(\ell \coloneqq \sum_{i=1}^{n} \ell_i + f(\{z_j^\star\}_{j=1}^{m}; \theta_t)\) yields both the per-sample training gradients and the target gradient \(\nabla_{\theta} f\) in a single pass; however, when \(f\) depends on the training samples \(z_i\), the gradient information is mixed and separate passes are required.
\end{remark}

\paragraph{The computational dependencies in data-regularized update.}
Next, given the merged batch where we can obtain the per-sample quantities from one pass, we carefully examine the computational dependencies for obtaining the data-regularized updates. For both Global Subset Update and Group-Wise Subset Update, the algorithm can be decomposed into three parts:
\begin{enumerate}[leftmargin=*]
	\item \textbf{Scoring:} compute the score of each training sample against the target signal (at group level);
	\item \textbf{Subset selection:} determine the selected subset \(S_t\) (or \(S_{t,p}\) for each group \(G_p\));
	\item \textbf{Update assembly:} assemble the final update direction by averaging the gradients over the selected subset.
\end{enumerate}
The key issue is that both \emph{scoring} and \emph{update assembly} depend on per-sample quantities, while standard training only materializes the aggregated batch gradient \(\hat g_{\mathrm{tr}}\). In particular, the per-sample gradient \(g_i^{(l)} = \sum_{\tau=1}^{\seqlen} (\partial \ell/\partial e_{i,\tau}^{(l)})\otimes a_{i,\tau}^{(l)}\) is never explicitly stored in standard backpropagation. Yet for our method, it enters twice: first in the score computation, and later in the assembly of the selected subset. This is exactly the source of the extra dependencies that need to be handled.

A crucial observation is that both scoring and update assembly ultimately depend on the same local information at each layer, namely the pair \((a^{(l)}, \partial \ell / \partial e^{(l)})\). Once this pair is available, one can either materialize the corresponding per-sample gradients \(g_i^{(l)}\) explicitly, or compute the required scores and selected averages directly from it. This suggests that the fundamental scheduling question is not whether to preserve the entire computation graph, but rather how to retain this minimal local information long enough for all downstream computations to be completed.

\paragraph{Tensor lifetime scheduling.}
Given the computational dependencies, we now introduce the proposed tensor lifetime scheduling.

The key idea is to \emph{swap} the forward-cached tensor \(e^{(l)}\) with the backward-generated tensor \(\partial \ell / \partial e^{(l)}\) at each layer. Recall from \Cref{subsec:standard-training} that in standard training the forward pass caches \((a^{(l)}, e^{(l)})\), and during backward, \(e^{(l)}\) is consumed to produce \(\partial \ell / \partial e^{(l)}\). After this point, \(e^{(l)}\) is no longer needed. Therefore, instead of releasing everything immediately after the layer gradient is computed, we retain the pair \((a^{(l)}, \partial \ell/\partial e^{(l)})\) for later scoring and update assembly, while discarding \(e^{(l)}\). Since \(\partial \ell / \partial e^{(l)}\) has the same shape as \(e^{(l)}\), this replacement preserves the per-layer memory footprint up to lower-order bookkeeping costs.

Applying this swap iteratively across layers yields a one-pass schedule with the following structure:
\begin{enumerate}[leftmargin=*]
	\item run one forward pass on the merged batch and cache \((a^{(l)}, e^{(l)})\) as in standard training;
	\item during backward, replace each cached \(e^{(l)}\) with \(\partial \ell / \partial e^{(l)}\), so that after the backward pass the retained tensors are \((a^{(l)}, \partial \ell / \partial e^{(l)})\) for all layers;
	\item use these retained pairs to compute scores, determine the selected subset(s), and assemble the final update direction;
	\item release the retained tensors once all groups depending on them have been resolved.
\end{enumerate}

This schedule differs fundamentally from \texttt{retain\_graph=True}: we do \emph{not} keep the full autograd graph alive, but only the minimal local tensors needed for later computation. At the same time, unlike the two-pass implementation, we do \emph{not} rerun the backward pass. The result is a one-pass implementation that stays close to the peak memory cost of standard training while avoiding a near-doubling of runtime.

\paragraph{When can a group be resolved?}
The remaining question is when a group can be finalized and its tensors released. This depends on the partition granularity. For a general partition \(\mathcal{G} = \{G_p\}_{p=1}^{P}\), the subset \(S_{t,p}\) for group \(G_p\) cannot be determined until all score contributions from the layers intersecting \(G_p\) have been accumulated. Only after that can the corresponding update \(u_t^{(p)}\) be assembled. Therefore, a group can be released only after all of its constituent layers have finished both scoring and update assembly.

This reveals a direct connection between the statistical design and the systems design: finer partitions not only give more flexible update rules, but also allow earlier release of tensors and hence a memory footprint closer to standard training.\footnote{It is worth noting that the peak memory cost remains similar regardless of the partition granularity, since the peak memory is achieved right after the forward pass. However, finer granularity becomes helpful when considering other memory-saving techniques, such as the activation checkpointing discussed in \Cref{subsec:compatibility}.}

\subsubsection{Case Study: Group Granularity}
We illustrate the above scheduling with two extreme cases.

\paragraph{Global Subset Update.}
For Global Subset Update, all parameters share a single global subset \(S_t\). Consequently, the score of each training sample must aggregate contributions from \emph{all} layers before \(S_t\) can be determined. This means the retained pairs \((a^{(l)}, \partial \ell / \partial e^{(l)})\) must remain alive across the entire model until scoring is complete. After the global subset \(S_t\) is determined, the final update is assembled by revisiting each retained layer and averaging the per-sample gradients of the selected samples (See \Cref{algo:global-subset-update-one-pass}). Compared with standard training, the peak memory remains close, but the high-memory period now spans almost the entire backward-and-selection phase (Middle row in \Cref{fig:compare}).

\begin{algorithm}[htpb]\label{algo:global-subset-update-one-pass}
	\DontPrintSemicolon{}
	\caption{Global Subset Update}
	\KwData{Model \(\theta_t\), training data \(\{z_i\}_{i=1}^n\), target data \(\{z^\star_j\}_{j=1}^m\), learning rate \(\eta_t\)}
	\KwResult{Updated model parameters \(\theta_{t+1}\)}
	\SetKwFunction{Forward}{Forward}
	\SetKwFunction{Backward}{Backward}
	\SetKwFunction{Select}{Select-S}
	\SetKwFunction{Release}{Release}

	\BlankLine

	(\(\ell\), \(\{a_{i}^{(l)}, e_{i}^{(l)}\}_{l,i} \cup \{a_{j}^{\star(l)}, e_{j}^{\star(l)}\}_{l,j}\))\(\gets\)\Forward{\(\theta_t\), \(\{z_i\}_{i=1}^{n} \cup \{z^\star_j\}_{j=1}^{m}\)}\Comment*[r]{\memcost{\(+ O(2N\seqlen \sqrt{dL})\)}}
	\For(\Comment*[f]{Retain all pairs via swapping}){\(l = L, \ldots, 1\)}{
		\(\{\partial \ell / \partial e_{i}^{(l)}\}_{i=1}^{n} \cup \{\partial \ell / \partial e_{j}^{\star(l)}\}_{j=1}^{m} \gets\)\Backward{\(\ell\), \(\theta_t^{(l)}\)}\;
	}
	\tcc{Post-hoc: scoring from retained activations}
	\For(\Comment*[f]{\memcost{Peak: \(2N\seqlen \sqrt{dL}\)}}){\(l = L, \ldots, 1\)}{
		\(g_i^{(l)} \gets \sum_{\tau=1}^{\seqlen} (\partial \ell / \partial e_{i,\tau}^{(l)}) \otimes a_{i,\tau}^{(l)}\) for each \(i \in [n]\)\Comment*[r]{\memcost{\(+ O(n d/L)\)}}
		\(\hat{g}_\star^{(l)} \gets \frac{1}{m} \sum_{j=1}^{m}\sum_{\tau=1}^{\seqlen} (\partial \ell / \partial e_{j,\tau}^{\star(l)}) \otimes a_{j,\tau}^{\star(l)}\)\Comment*[r]{\memcost{\(+ O(d/L)\)}}
		\Release{\(a_j^{\star(l)}\), \(\partial \ell / \partial e_j^{\star(l)}\)} for all \(j \in [m]\)\Comment*[r]{\memcost{\(- O(2m\seqlen \sqrt{d/L})\)}}
		\(s_i^{(l)} \gets \langle \hat{g}_\star^{(l)}, g_i^{(l)} \rangle\) for each \(i \in [n]\)\;
		\Release{\(g_i^{(l)}\), \(\hat{g} _{\star}^{(l)}\)} for all \(i \in [n]\)\Comment*[r]{\memcost{\(- O((n+1) d/L)\)}}
	}
	\(s_i \gets \sum_{l=1}^L s_i^{(l)}\) for all \(i \in [n]\)\Comment*[r]{Global scores}
	\(S_t \gets\)\Select{\(\{s_i\}_{i=1}^n\), \(k\)}\;
	\BlankLine
	\tcc{Post-hoc: assemble gradients from retained activations}
	\(u_t \gets 0 \in \mathbb{R}^d\)\Comment*[r]{\memcost{\(+ O(d)\)}}
	\For(){\(l = L, \ldots, 1\)}{
		\(u_t^{(l)} \gets \frac{1}{k} \sum_{i \in S_t}\sum_{\tau=1}^{\seqlen} (\partial \ell / \partial e_{i,\tau}^{(l)}) \otimes a_{i,\tau}^{(l)}\)\;
		\Release{\(a_i^{(l)}\), \(\partial \ell / \partial e_i^{(l)}\)} for all \(i \in [n]\)\Comment*[r]{\memcost{\(- O(2n\seqlen \sqrt{d/L})\)}}
	}
	\(\theta_{t+1} \gets \theta_t - \eta_t u_t\)\;
	\Release{\(u_t\)}\Comment*[r]{\memcost{\(- O(d)\)}}
	\Return{\(\theta_{t+1}\)}\;
\end{algorithm}

\paragraph{Layer-Wise Subset Update.}
At the other extreme, for Layer-Wise Subset Update, each layer forms its own group. In this case, the score computation, subset selection, and update assembly for layer \(l\) depend only on the local tensors of that same layer. Therefore, once the backward pass reaches layer \(l\), we can immediately compute the layer-wise scores, determine its subset \(S_{t,l}\), assemble the update \(u_t^{(l)}\), and release all retained tensors for that layer before moving to layer \(l-1\) (See \Cref{algo:layer-wise-subset-update}). As a result, the tensor lifetime schedule almost exactly matches that of standard training, and the memory footprint is closest to the standard baseline (Bottom row in \Cref{fig:compare}).

\begin{algorithm}[htpb]\label{algo:layer-wise-subset-update}
	\DontPrintSemicolon{}
	\caption{Layer-Wise Subset Update}
	\KwData{Model \(\theta_t\), training data \(\{z_i\}_{i=1}^n\), target data \(\{z^\star_j\}_{j=1}^m\), learning rate \(\eta_t\)}
	\KwResult{Updated model parameters \(\theta_{t+1}\)}
	\SetKwFunction{Forward}{Forward}
	\SetKwFunction{Backward}{Backward}
	\SetKwFunction{Release}{Release}
	\SetKwFunction{Select}{Select-S}

	\BlankLine

	(\(\ell\), \(\{a_{i}^{(l)}, e_{i}^{(l)}\}_{l,i} \cup \{a_{j}^{\star(l)}, e_{j}^{\star(l)}\}_{l,j}\))\(\gets\)\Forward{\(\theta_t\), \(\{z_i\}_{i=1}^{n} \cup \{z^\star_j\}_{j=1}^{m}\)}\Comment*[r]{\memcost{\(+ O(2N\seqlen \sqrt{dL})\)}}
	\(u_t \gets 0 \in \mathbb{R}^d\)\Comment*[r]{\memcost{\(+ O(d)\)}}
	\For{\(l = L, \ldots, 1\)}{
		\(\{\partial \ell / \partial e_{i}^{(l)}\}_{i=1}^{n} \cup \{\partial \ell / \partial e_{j}^{\star(l)}\}_{j=1}^{m} \gets\)\Backward{\(\ell \), \(\theta_t ^{(l)}\)}\;
		\(g_i^{(l)} \gets \sum_{\tau=1}^{\seqlen} (\partial \ell / \partial e_{i,\tau}^{(l)}) \otimes a_{i,\tau}^{(l)}\) for each \(i \in [n]\)\Comment*[r]{\memcost{\(+ O(n d/L)\)}}
		\(\hat{g}_\star^{(l)} \gets \frac{1}{m} \sum_{j=1}^{m} \sum_{\tau=1}^{\seqlen} (\partial \ell / \partial e_{j,\tau}^{\star(l)}) \otimes a_{j,\tau}^{\star(l)}\)\Comment*[r]{\memcost{\(+ O(d/L)\)}}
		\(s_i^{(l)} \gets \langle \hat{g}_\star^{(l)}, g_i^{(l)} \rangle\) for each \(i \in [n]\)\;
		\(S_{t,l} \gets \)\Select{\(\{s_i^{(l)}\}_{i=1}^n\), \(k\)}\;
		\(u_t^{(l)} \gets \frac{1}{k} \sum_{i \in S_{t,l}} g_i^{(l)}\)\;
		\Release{\(g_i^{(l)}\), \(\hat{g} _{\star}^{(l)}\)} for all \(i \in [n]\)\Comment*[r]{\memcost{\(- O((n+1) d/L)\)}}
		\Release{\(a_i^{(l)}\), \(\partial \ell / \partial e_i^{(l)}\), \(a_j^{\star(l)}\), \(\partial \ell / \partial e_j^{\star(l)}\)} for all \(i \in [n], j \in [m]\)\Comment*[r]{\memcost{\(- O(2N\seqlen \sqrt{d/L})\)}}
	}
	\(\theta_{t+1} \gets \theta_t - \eta_t u_t\)\;
	\Release{\(u_t\)}\Comment*[r]{\memcost{\(- O(d)\)}}
	\Return{\(\theta_{t+1}\)}\;
\end{algorithm}

\begin{figure}[htpb]
	\centering
  \def\svgwidth{\columnwidth}
  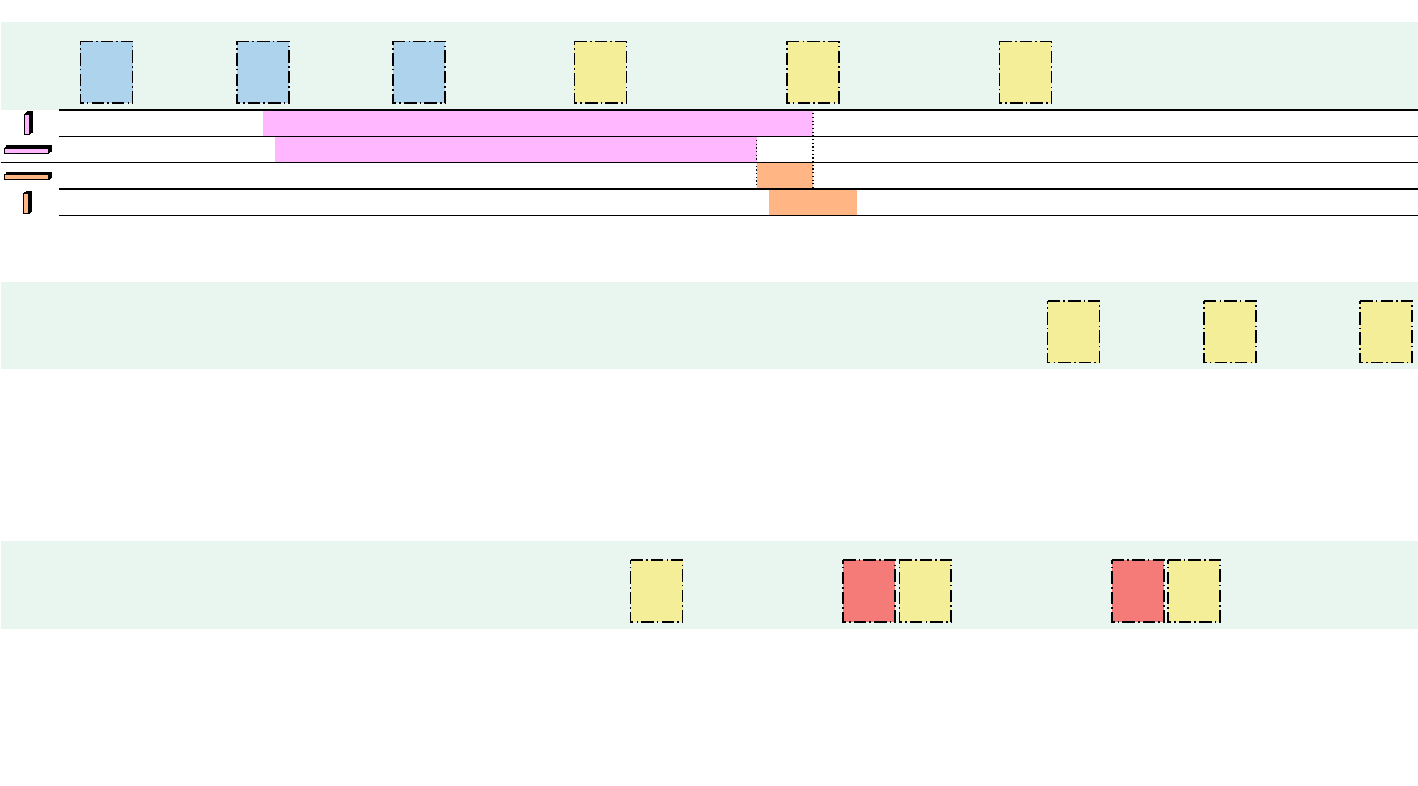

	\caption{Memory footprint comparison. \textbf{Top:} Standard Training Update (\Cref{algo:standard-training}). \textbf{Middle:} Global Subset Update (\Cref{algo:global-subset-update-one-pass}). \textbf{Bottom:} Layer-Wise Subset Update (\Cref{algo:layer-wise-subset-update}). All three methods achieve comparable peak memory per layer, but Global Subset Update must retain both \(a^{(l)}\) and \(\partial \ell / \partial e^{(l)}\) for all \(l\) throughout the backward pass until the global subset is determined, while standard training and Layer-Wise Subset Update release them on-the-fly.}
	\label{fig:compare}
\end{figure}

\paragraph{Intermediate granularities.}
More generally, if each group contains a small number of consecutive layers, then a group's tensors only need to be retained until the backward pass finishes those layers. This gives a smooth interpolation between the two extremes above. In particular, smaller groups allow earlier release and lower sustained memory usage, whereas larger groups delay release and make the schedule resemble Global Subset Update.

\begin{remark}
	One benefit of more aggressive memory release (as in Layer-Wise Subset Update) is that, in practice, when layer sizes are not uniform, later in the backward pass, encountering a larger layer provides enough free memory to process it.
\end{remark}

In summary, the proposed customized tensor lifetime schedule resolves a core system challenge of data-regularized updates: it avoids the runtime overhead of a two-pass implementation and the prohibitive memory cost of retaining the full graph, while enabling a practical one-pass realization under the memory constraints of modern LLM training.

\subsection{Efficient Implementations of Per-Sample Scoring}\label{subsec:efficient-scoring}
We next study efficient implementations of per-sample scoring, which constitutes a major computational overhead introduced by data-regularized updates. While the tensor lifetime schedule in \Cref{subsec:one-pass-tensor-lifetime-scheduling} resolves the memory-management challenge of preserving the necessary intermediate quantities within one pass, the overall efficiency of the method still depends critically on how the per-sample scores are computed.

In this subsection, we revisit two existing methods for score computation and introduce two new variants. We then provide a systematic complexity analysis in terms of the key training hyperparameters. This analysis reveals that the relative advantages of different implementations depend strongly on the hyperparameter regime, and that this regime-dependent comparison has not been fully characterized in the existing literature. Among these methods, one of our proposed variants---an approximate scoring method---achieves the best computational efficiency across all regimes.

Throughout this subsection, we focus on the score restricted to one particular layer, \(s_i^{(l)}\), since when a parameter group spans multiple layers, the total score is simply the sum of the scores from those layers. \Cref{tab:scoring-comparison} summarizes the computational and memory complexity of each method; the exact derivations are given in \Cref{adxsubsec:scoring-complexity}.

\begin{table}[htpb]
	\centering
	\caption{Computational overhead of scoring per layer for \(n\) training samples and \(m\) target samples, where \(N = n + m\), \(\seqlen\) denotes the sequence length, and \(\kappa\) is the compressed gradient dimension (\(\kappa \ll d/L\)). All costs listed below represent additional overhead beyond standard backpropagation. We assume fully batched execution in calculating the complexity; micro-batching over samples or tokens can trade parallelism for lower memory, which further favors the proposed PIP method.}
	\label{tab:scoring-comparison}
	\begin{adjustbox}{max width=\linewidth}
\begin{tabular}{lcc}
	\toprule
	\textbf{Scoring method} & \textbf{FLOPs per layer}        & \textbf{Memory per layer (entries)} \\
	\midrule
	Direct                  & \(2N\seqlen(d/L) + n(d/L)\)           & \((n+1)(d/L)\)                      \\
	GIP                     & \(4nm\seqlen^2 \sqrt{d/L}\)           & \(2nm\seqlen^2\)                          \\
	PIP                     & \(2N\seqlen(d/L) + n\seqlen\sqrt{d/L}\)     & \(d/L + n\seqlen\sqrt{d/L}\)              \\
	Compressed              & \(O(N\seqlen \kappa) + (2n+m)\kappa\) & \((n+1) \kappa\)                    \\
	\bottomrule
\end{tabular}

	\end{adjustbox}
\end{table}

\subsubsection{Existing Methods: Direct and Ghost Inner Product (GIP)}\label{subsubsec:direct-gip}
We first revisit two existing methods, \emph{Direct} and \emph{Ghost Inner Product} (GIP).

\paragraph{Direct.}
Direct is a naive approach for the score computation. It directly materializes each per-sample gradient \(g_i^{(l)} = \sum_\tau (\partial \ell / \partial e_{i,\tau}^{(l)}) \otimes a_{i,\tau}^{(l)}\) and the target gradient \(\hat{g}_\star^{(l)}\) first, then computes the score as \(s_i^{(l)} = \langle \hat{g}_\star^{(l)}, g_i^{(l)} \rangle\). Existing literature criticizes this naive approach for materializing \(O(n)\) per-sample gradient vectors and incurring high computational costs, which led to the GIP introduced below~\citep{wang2024greats}. However, we demonstrate that the comparison between Direct and GIP is more nuanced than the literature suggested when considering the sequence length \(\seqlen\) and the number of target samples \(m\).

\paragraph{Ghost Inner Product (GIP)~\citep{wang2024greats}.}
GIP avoids the full gradient materialization (thus the name ``Ghost'') by evaluating the same inner product in an alternative contraction order from \(a^{(l)}\) and \(\partial \ell / \partial e^{(l)}\): contracting over the model dimension instead of the token dimension yields \(\seqlen \times \seqlen\) cross-correlations. Treating \(\partial \ell / \partial e^{(l)}, a^{(l)} \in \mathbb{R}^{\sqrt{d/L} \times \seqlen}\) as matrices whose columns are the per-token vectors, the score is
\[
	s_i^{(l)}
	= \frac{1}{m}\sum_{j=1}^{m} \left\langle \Bigg(\frac{\partial \ell}{\partial e_i^{(l)}}\Bigg)^{\top} \frac{\partial \ell}{\partial e_j^{\star(l)}}, (a_i^{(l)})^{\top} a_j^{\star(l)} \right\rangle,
\]
where each factor is a \(\seqlen \times \seqlen\) matrix product computed over all \(nm\) training-target pairs. While this avoids gradient materialization, it introduces quadratic dependence on \(\seqlen\) and linear dependence on \(m\).

\paragraph{Comparison between Direct and GIP.}
As shown in \Cref{tab:scoring-comparison}, GIP is competitive only when \(m\) and \(\seqlen\) are relatively small. By comparing the leading terms in the computational complexity, GIP is preferable to Direct only when \(m\seqlen \lesssim \sqrt{d/L}/2\). A similar trend holds for the memory cost comparison as well.

\subsubsection{Per-Token Inner Product (PIP)}\label{subsubsec:per-token-inner-product}
The comparison shows that GIP suffers from long context (large \(\seqlen\)) while Direct suffers from large parameter size per layer (large \(d/L\)). We now introduce a third method, \emph{Per-Token Inner Product} (PIP), that interpolates between these two methods and provides a more favorable computation cost trade-off in an important regime.

Observing that the GIP differs from Direct by swapping the order of some linear operators (sum and inner product) in calculating the score \(s_i^{(l)}\), our proposed PIP applies a different order. Let \(\hat{G}_\star^{(l)} \in \mathbb{R}^{\sqrt{d/L}\times \sqrt{d/L}}\) denote the matrix form of the \emph{aggregated} target gradient \(\hat{g}_\star^{(l)}\). Then the score can be written as
\[
	s_i^{(l)}
	= \sum_{\tau=1}^{\seqlen} \Bigg(\frac{\partial \ell}{\partial e_{i,\tau}^{(l)}}\Bigg)^\top \hat{G}_\star^{(l)} a_{i,\tau}^{(l)}.
\]
This expression factors the computation into a matrix-vector product \(\hat{G}_\star^{(l)} a_{i,\tau}^{(l)} \in \mathbb{R}^{\sqrt{d/L}}\), followed by a dot product with \(\partial \ell / \partial e_{i,\tau}^{(l)}\). Importantly, this avoids materializing the full per-sample training gradient \(g_i^{(l)}\): the score is accumulated directly from token-level quantities that are already available under the tensor lifetime schedule.

\paragraph{Comparison with existing methods.}
As shown in \Cref{tab:scoring-comparison}, the complexities of PIP remain linear in the sequence length \(\seqlen\), and therefore avoid the quadratic dependence on \(\seqlen\) that makes GIP unattractive for long-context training. Compared with Direct, PIP has the same dominant FLOPs term \(2N\seqlen(d/L)\), but replaces the lower-order term \(n(d/L)\) with \(n\seqlen\sqrt{d/L}\). Therefore, when \(\seqlen \lesssim \sqrt{d/L}\), PIP is more efficient than Direct in both FLOPs and memory. Intuitively, in this regime, it is cheaper to work directly with token-level quantities than to first materialize \(n\) full per-sample gradients. When \(\seqlen\) becomes large, however, this advantage disappears in the FLOP count: the additional \(n\seqlen\sqrt{d/L}\) term dominates, and Direct becomes preferable, although in practice the gap stays small because a fused PIP kernel keeps that intermediate in registers while Direct's \(n(d/L)\) term is materialized in memory (\Cref{tab:score-cost}). Thus, PIP occupies an intermediate regime between GIP and Direct: it improves substantially over GIP for long sequences, and improves over Direct when the context length is moderate relative to the per-layer hidden dimension.

\subsubsection{Compressed Scoring}\label{subsubsec:compressed-scoring}
All three methods above compute the score exactly, and therefore their cost necessarily depends on the gradient dimension \(d/L\). We now introduce a fourth method, \emph{Compressed Scoring}, which reduces this cost by approximately computing the score in a low-dimensional compressed space.

Let \(\Pi^{(l)} : \mathbb{R}^{d/L} \to \mathbb{R}^{\kappa}\) be a compression map with \(\kappa \ll d/L\). Instead of directly computing the full per-sample training gradient \(g_i^{(l)}\) and target gradient \(\hat g_\star^{(l)}\), we first compute their compressed versions
\[
	\widetilde{g}_i^{(l)} \coloneqq \Pi^{(l)}(g_i^{(l)}) \in \mathbb{R}^{\kappa},
	\qquad
	\widetilde{g}_\star^{(l)} \coloneqq \Pi^{(l)}(\hat g_\star^{(l)}) \in \mathbb{R}^{\kappa},
\]
and then approximate the score by
\[
	s_i^{(l)} \approx \langle \widetilde{g}_\star^{(l)}, \widetilde{g}_i^{(l)} \rangle.
\]

The key point is that \(\widetilde{g}_i^{(l)}\) and \(\widetilde{g}_\star^{(l)}\) can be computed directly from the retained token-level quantities \((a^{(l)}, \partial \ell / \partial e^{(l)})\), without ever materializing the full \(d/L\)-dimensional gradients. In particular, since \(g_i^{(l)} = \sum_{\tau=1}^{\seqlen} (\partial \ell / \partial e_{i,\tau}^{(l)} )\otimes a_{i,\tau}^{(l)}\), state-of-the-art gradient compression methods in data attribution literature~\citep{hu2025grass,choe2025your} can exploit this outer-product structure and apply the projection \(\Pi^{(l)}\) in a factorized manner (see \Cref{adxsubsec:efficient-projection} for more details). As a result, the compressed per-sample gradient can be obtained from the same retained local tensors used throughout \Cref{subsec:one-pass-tensor-lifetime-scheduling}, but with cost scaling only in the compressed dimension \(\kappa\).

\paragraph{Comparison with exact methods.}
As shown in \Cref{tab:scoring-comparison}, Compressed Scoring has per-layer cost \(O(N\seqlen\kappa) + (2n+m)\kappa\) and memory cost \((n+1)\kappa\), which removes the dependence on the original gradient dimension \(d/L\) from both FLOPs and memory. Since \(\kappa \ll d/L\), this makes Compressed Scoring strictly more efficient than all three exact methods across all practical regimes in terms of the leading dimension dependence, and \Cref{tab:score-cost} confirms that it is the cheapest method at every \(\seqlen\) and \(m\) we test.

While the gain in efficiency of Compressed Scoring comes at the price of approximation, empirically, we find the approximate scores remain a reliable surrogate for ranking or filtering training samples~\citep{hu2025grass,hu2026unified}. In our implementation, both backends are available: the system benchmarks use exact GIP scoring, the backend of the question-answering experiments (\Cref{subsec:system-efficiency}), Compressed Scoring with \(\kappa^{(l)} = 64 \times 64\) is the choice for long contexts and many target examples, where its cost advantage is largest (\Cref{tab:score-cost}), and under LoRA, whose per-sample gradients are already small, we use exact PIP. \Cref{adxsubsubsec:sft-ablations,adxsubsubsec:rlhf-scoring-backend} compare the two backends.

\subsection{Compatibility with Memory-Saving Techniques in Modern Infrastructure}\label{subsec:compatibility}
Modern LLM training pipelines employ a range of memory-saving techniques to mitigate the tight memory constraints. Since our methods introduce additional per-sample computational dependencies beyond standard training, a key system question is whether the proposed methods remain compatible with these techniques, including parameter-efficient training, activation checkpointing, and gradient accumulation.

In this subsection, we show that the proposed methods, especially Layer-Wise Subset Update, can be largely integrated with common memory-saving strategies given our system optimizations.

\subsubsection{Compatibility with Parameter-Efficient Training}\label{subsubsec:low-dimensional-parameter-training}
We first consider two widely-used parameter-efficient training methods, LoRA~\citep{hu2022lora} and MeSO~\citep{zhao2024galore}. Both reduce memory consumption by restricting or compressing the trainable update space, and both are naturally compatible with our framework.

\paragraph{LoRA.}
Recall from \Cref{adxsubsec:LoRA} that LoRA reparametrizes each linear layer as \(W_t^{(l)} + B_t^{(l)} A_t^{(l)}\), where \(B_t^{(l)} \in \mathbb{R}^{\sqrt{d/L} \times r}\) and \(A_t^{(l)} \in \mathbb{R}^{r \times \sqrt{d/L}}\) with rank \(r \ll \sqrt{d/L}\), and only the adapter parameters are updated. Since \(A_t^{(l)}\) and \(B_t^{(l)}\) are themselves linear maps, per-sample gradients retain the same outer-product factorization as in the full-parameter case, with the adapter dimension \(2r\sqrt{d/L}\) replacing the full dimension \(d/L\). All scoring and gradient computation carry over directly at proportionally lower cost (details in \Cref{adxsubsec:lora-integration}).

\paragraph{MeSO.}
MeSO methods reduce memory in a different way: instead of restricting the trainable parameters, they maintain optimizer states and updates in a compressed subspace (\Cref{adxsubsec:MeSO}). This is also compatible with our framework, since data regularization only changes how the update direction is \emph{chosen}, not how the resulting gradient is \emph{applied}. In particular, any scoring method from \Cref{subsec:efficient-scoring} can be used to determine the selected subset, after which the resulting gradient update is passed to the MeSO optimizer in the usual compressed form.

There is also a tighter integration between MeSO and the Compressed Scoring method in \Cref{subsubsec:compressed-scoring}. Since both rely on low-dimensional gradient representations, the same compressed per-sample gradients can be used both for score computation and for the optimizer update, avoiding redundant computation. This makes MeSO especially well aligned with our framework: it reduces the memory cost of optimization itself, while also supporting efficient approximate scoring in the same compressed space. We detail the full space of design choices in \Cref{adxsubsec:meso-integration}.

\subsubsection{Compatibility with Memory-Saving Scheduling Techniques}\label{subsubsec:memory-saving-techniques}
We now turn to two techniques that reduce memory by changing the temporal structure of training: activation checkpointing and gradient accumulation. Unlike LoRA and MeSO, these methods affect \emph{when} intermediate tensors are available, and therefore interact more directly with our tensor lifetime schedule.

\paragraph{Activation checkpointing.}
Activation checkpointing reduces memory by discarding some forward activations and recomputing them on demand during backward. Our one-pass tensor lifetime schedule is compatible with checkpointing whenever the parameter groups align with the checkpointing schedule, so that all computations associated with a group can be resolved within a checkpoint segment. Layer-Wise Subset Update satisfies this condition naturally, since each layer forms its own group and can therefore be resolved within its own backward step.

When a group spans multiple checkpoint segments, however, the one-pass schedule can no longer preserve all required local tensors across segment boundaries without defeating the purpose of checkpointing. In this case, one must revert to the two-pass implementation discussed in \Cref{adxsec:two-pass-implementation}: the first pass computes scores and determines the subsets, and the second pass assembles the update (see \Cref{adxsubsubsec:compatability-with-memory-saving-techniques}). Thus, the compatibility with checkpointing depends on whether the partition granularity is aligned with the memory-saving schedule.

\paragraph{Gradient accumulation.}
Gradient accumulation reduces peak memory by splitting a large batch into micro-batches that are processed sequentially, and the activations and pre-activation gradients associated with an earlier micro-batch are typically released before the next one is processed for saving memory. This conflicts with our one-pass implementation, which relies on retaining the pairs \((a_i^{(l)}, \partial \ell / \partial e_i^{(l)})\) for all samples in the full batch in order to compute scores and assemble the final update. When the rule for determining \(S_{t,p}\) decomposes across micro-batches, however, this issue resolves: each micro-batch can be scored, solved, and assembled independently during its own forward and backward pass, and the resulting updates can then be accumulated across micro-batches. Thresholding is one such example. In our RL experiments each micro-batch is scored, filtered, and normalized by its own retained item count, and the micro-batch gradients are then averaged with the standard accumulation weights; the update is therefore the average of per-micro-batch retained means rather than the mean over all retained samples of the mini-batch, mirroring the per-micro-batch normalization that standard training uses in the same pipeline. In contrast, rules such as top-\(k\) depend on statistics computed over the entire batch, so the micro-batches cannot be resolved independently. In this case, we again resort to the two-pass implementation described in \Cref{adxsubsubsec:compatability-with-memory-saving-techniques}.

Overall, Dr.\ Post-Training is not only algorithmically appealing but also practical: it fits the memory budget of modern post-training and, as \Cref{subsec:system-efficiency} shows, runs at a per-step cost on par with standard training.

\section{Experiments}\label{sec:experiment}
We evaluate Dr.\ Post-Training across three post-training paradigms: supervised fine-tuning (SFT; \Cref{subsec:SFT}), reinforcement learning from human feedback (RLHF; \Cref{subsec:RLHF}), and reinforcement learning with verifiable rewards (RLVR; \Cref{subsec:RLVR}). In every setting, we compare standard training (Full-Training Update) and the state-of-the-art online data selection baseline of that paradigm (a variant of Global Subset Update) with the proposed Layer-Wise Subset Update. We then benchmark the efficiency of the end-to-end training pipeline and of the scoring mechanisms in \Cref{subsec:system-efficiency}.

Unless stated otherwise, all results are averaged over \(5\) seeds and reported as mean \(\pm\) standard error, evaluation is performed on held-out test sets disjoint from the data used for regularization or training, every arm shares the hyperparameters of standard training, and every curated arm uses exact scoring. Datasets, training configurations, and the data-regularization protocol of every setting are given in \Cref{adxsec:experiment}.

\subsection{Supervised Fine-Tuning}\label{subsec:SFT}
We first compare the update rules on question answering, and then scale to reasoning and instruction following. Throughout SFT, each step draws \(n = 8\) candidate training examples and \(m = 1\) target example, and the curated updates keep \(k = 4\) of the candidates by top-\(k\) selection.

\subsubsection{Question Answering}\label{subsubsec:SFT-QA}
We fine-tune \textsc{Llama-3.2-1B}~\citep{grattafiori2024llama} with full-parameter, LoRA, and MeSO fine-tuning on four general/target pairs: \texttt{alpaca}~\citep{alpaca}/\texttt{samsum}~\citep{gliwa2019samsum}, \texttt{less-mix}~\citep{xia2024less}/\texttt{tydiqa}~\citep{clark2020tydi}, \texttt{triviaqa}~\citep{joshi2017triviaqa}/\texttt{nq\_open}~\citep{kwiatkowski2019natural}, and \texttt{less-mix}/\texttt{squad}~\citep{rajpurkar2016squad} (closed-book). The first two pairs are taken from GREATS~\citep{wang2024greats} (Global Subset Update), every pair uses a target set of only \(16\) examples, and alignment scores are computed with exact GIP.

\begin{table}[htpb]
	\centering
	\caption{Question-answering SFT (LoRA): final-checkpoint Rouge-L (\texttt{samsum}) / F1 (\%).}
	\label{tab:SFT-downstream}
	\begin{adjustbox}{width=\linewidth}
\begin{tabular}{l c c c c}
	\toprule
	\textbf{Method} & \texttt{alpaca}/\texttt{samsum} & \texttt{less-mix}/\texttt{tydiqa} & \texttt{triviaqa}/\texttt{nq\_open} & \texttt{less-mix}/\texttt{squad} \\
	\midrule
	Full-Training & \(22.4\,{\scriptstyle\pm 0.3}\) & \(4.9\,{\scriptstyle\pm 0.3}\) & \(12.7\,{\scriptstyle\pm 0.2}\) & \(9.0\,{\scriptstyle\pm 0.2}\) \\
	Global Subset & \(22.7\,{\scriptstyle\pm 0.5}\) & \(13.7\,{\scriptstyle\pm 1.5}\) & \(13.5\,{\scriptstyle\pm 0.3}\) & \(9.1\,{\scriptstyle\pm 0.2}\) \\
	Layer-Wise Subset & \(29.3\,{\scriptstyle\pm 1.4}\) & \(48.3\,{\scriptstyle\pm 1.7}\) & \(\mathbf{14.0\,{\scriptstyle\pm 0.3}}\) & \(\mathbf{10.0\,{\scriptstyle\pm 0.4}}\) \\
	Target-Only & \(\mathbf{32.8\,{\scriptstyle\pm 0.6}}\) & \(\mathbf{51.1\,{\scriptstyle\pm 2.9}}\) & \(9.3\,{\scriptstyle\pm 0.4}\) & \(8.7\,{\scriptstyle\pm 0.4}\) \\
	\bottomrule
\end{tabular}

	\end{adjustbox}
\end{table}

\paragraph{Results.}
\Cref{tab:SFT-downstream} reports the LoRA runs' downstream performance of all four methods; \Cref{adxsubsubsec:sft-qa} adds held-out perplexities and training dynamics. Layer-Wise Subset Update outperforms Global Subset Update, which improves over or matches Full-Training; the gap is largest under the broad \texttt{less-mix} pool, the most misaligned with its targets, where data regularization matters most. Target-Only Update (early-stopped, \Cref{adxsubsubsec:sft-qa}) splits the pairs: it leads or ties on \texttt{samsum} and \texttt{tydiqa}, whose answers are summaries or spans of the given passage, so a handful of target examples fixes the answer style, and never comes first on the closed-book targets, which test knowledge they cannot supply. It thus measures what the target set alone carries, not how much of the pool to use, the question the curated updates answer against Full-Training.

\subsubsection{Case Study}\label{adxsubsubsec:case-study}
To understand why Global Subset Update falls short, we record per-layer scores during a Full-Training Update across all four general/target pairs. At each training step, we separately score the same batch of \(n = 8\) samples as in Layer-Wise Subset Update and Global Subset Update; to remove score compression as a confounder, the recorded scores are exact PIP scores (\Cref{subsubsec:per-token-inner-product}). We compute
\begin{enumerate}[label=(\roman*)]
	\item \textbf{mean absolute score} per layer (magnitude), and
	\item \textbf{Spearman rank correlation} \(\rho_S\) between each layer's per-sample scores and the subset's global ranking.
\end{enumerate}
We aggregate the results by \emph{transformer blocks} and \emph{layer types}: each transformer block contains seven linear layers, namely four attention projections (query (Q), key (K), value (V), and output (O)) and three MLP projections (gate, up, and down).

The first column of \Cref{fig:case-study} shows the per-layer scores on the \texttt{alpaca}/\texttt{samsum} setting. The layers disagree about which candidates to keep (bottom panel): the ranking of \texttt{down\_proj} agrees with the global ranking (\(\rho_S = 0.88\) at its peak), whereas the rankings of the Q and K projections are nearly uncorrelated with it (\(\rho_S = 0.12\) and \(0.11\) averaged over blocks, and at most \(0.24\) at any block), so the two groups of layers prefer different samples. Which of them the shared ranking follows is decided by magnitude (top panel): \texttt{down\_proj} at block \(1\) produces scores \(32\times\) larger than the second-largest layer type at that block, the V projection, and the K projection, which shares V's input--output shape (\(512 \times 2048\)), is smaller still by more than an order of magnitude, so the sum over layers follows \texttt{down\_proj}. Under Global Subset Update these \(32\) layers (out of \(112\) block-level linear layers) therefore train on samples chosen for \texttt{down\_proj}; Layer-Wise Subset Update lets each layer follow its own ranking.

The remaining columns of \Cref{fig:case-study} extend the score analysis to the three QA settings. The same qualitative pattern holds across all four settings: Q and K projections remain weakly correlated with the global subset ranking (\(\rho_S \leq 0.22\) averaged over blocks), while a single \texttt{down\_proj} layer at an early block (block \(0\) or \(1\)) produces scores \(22\times\)--\(38\times\) larger than the second-largest layer type at the same block, again the V projection, and thereby decides the shared ranking (\(\rho_S = 0.90\)--\(0.94\) at its peak). The disagreement between layers is thus not specific to any particular general/target pair or training-pool composition. Whether rescaling the layers' scores suffices to repair the shared ranking is tested in \Cref{subsubsec:SFT-ablations}.

\begin{figure}[htbp]
	\centering
	\begin{subfigure}[t]{0.245\linewidth}
		\subcaption{\texttt{alpaca}/\texttt{samsum}}
		\includegraphics[width=\linewidth]{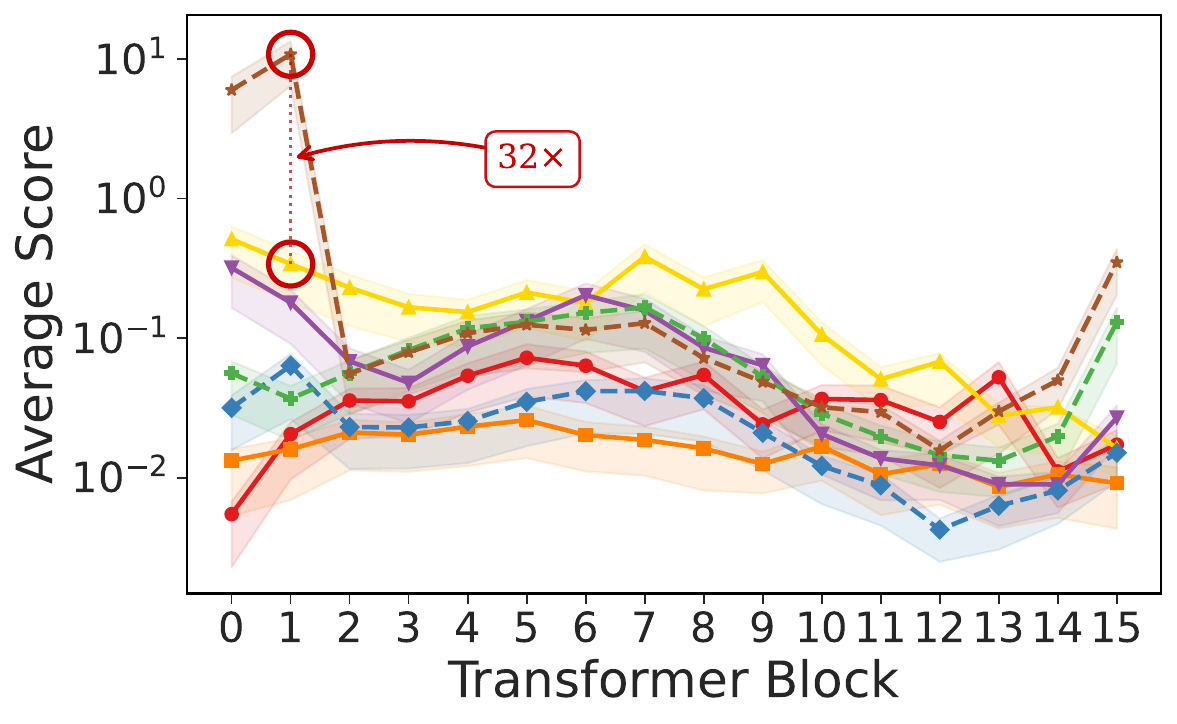}
	\end{subfigure}\hfill
	\begin{subfigure}[t]{0.245\linewidth}
		\subcaption{\texttt{less-mix}/\texttt{tydiqa}}
		\includegraphics[width=\linewidth]{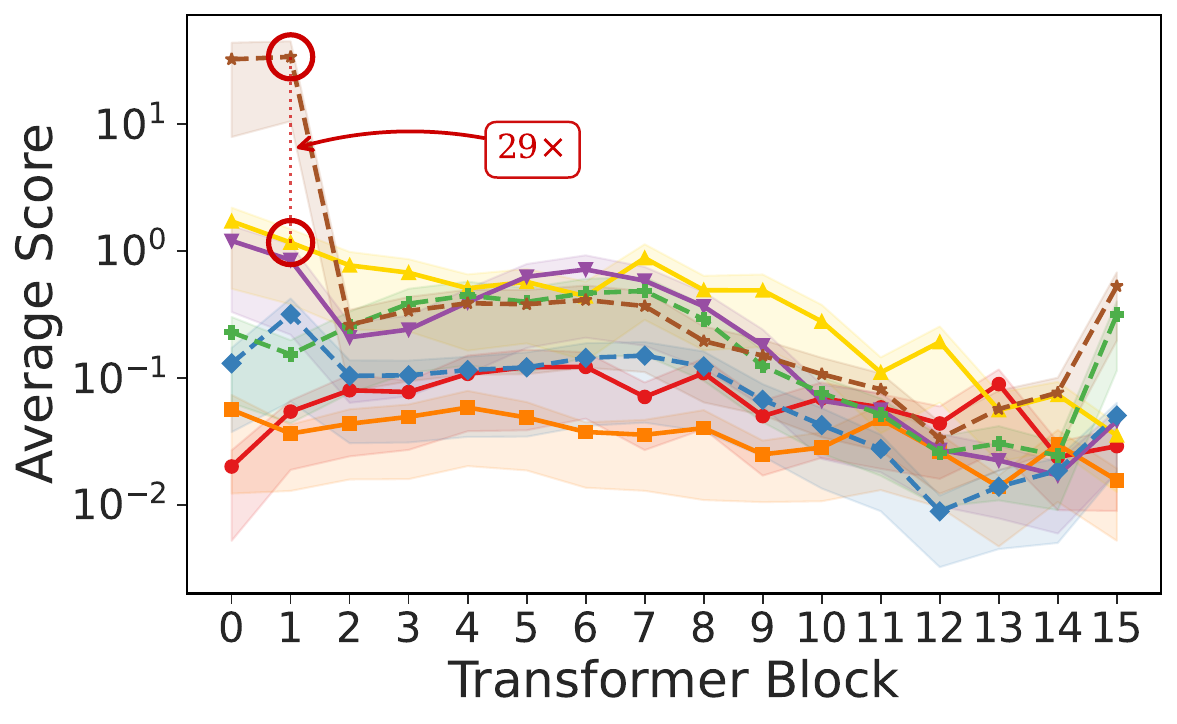}
	\end{subfigure}\hfill
	\begin{subfigure}[t]{0.245\linewidth}
		\subcaption{\texttt{triviaqa}/\texttt{nq\_open}}
		\includegraphics[width=\linewidth]{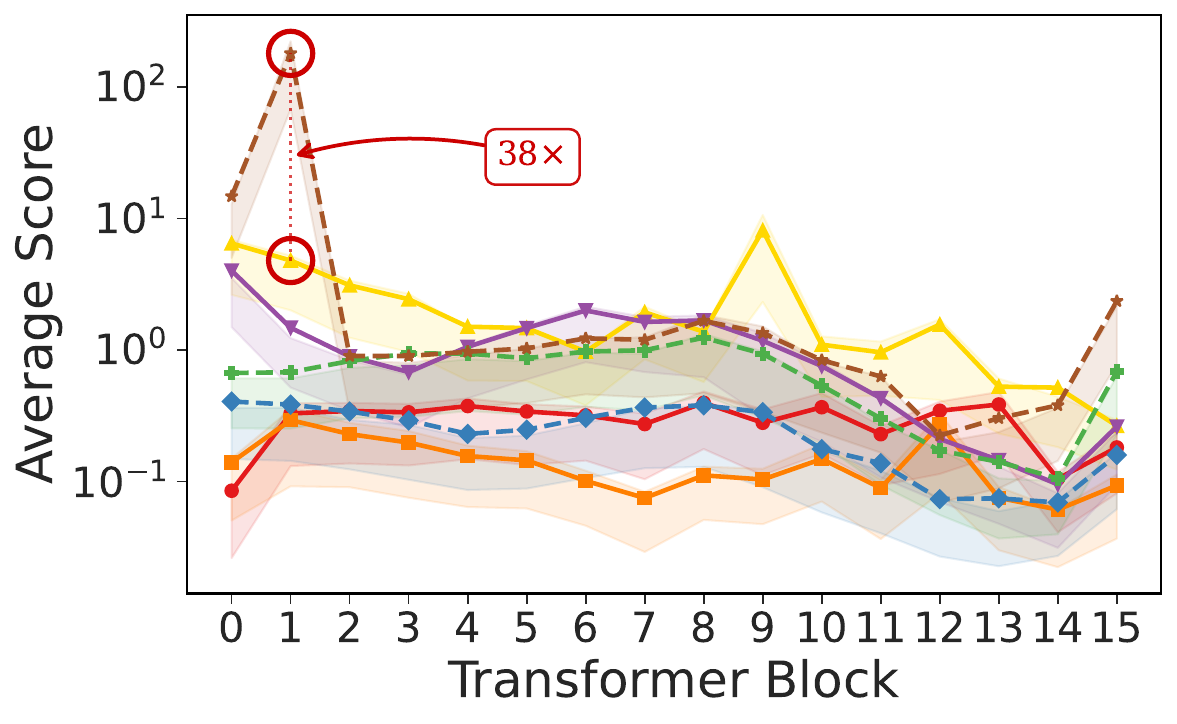}
	\end{subfigure}\hfill
	\begin{subfigure}[t]{0.245\linewidth}
		\subcaption{\texttt{less-mix}/\texttt{squad}}
		\includegraphics[width=\linewidth]{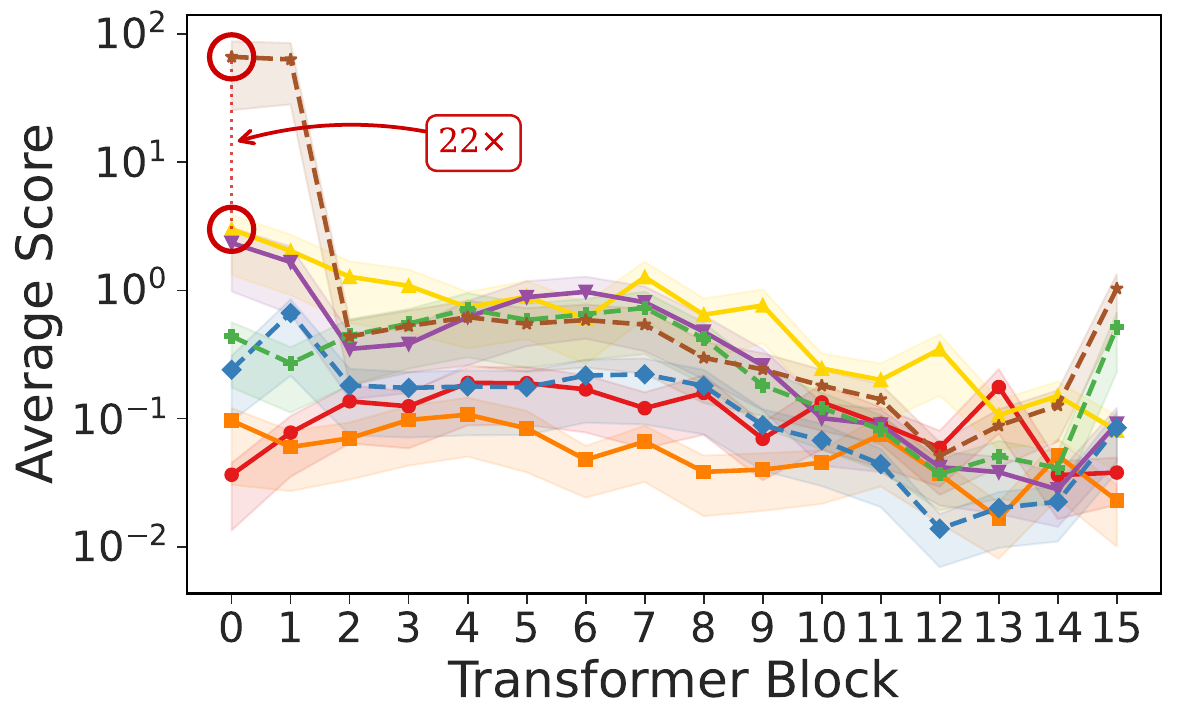}
	\end{subfigure}
	\\[0.4em]
	\includegraphics[width=0.245\linewidth]{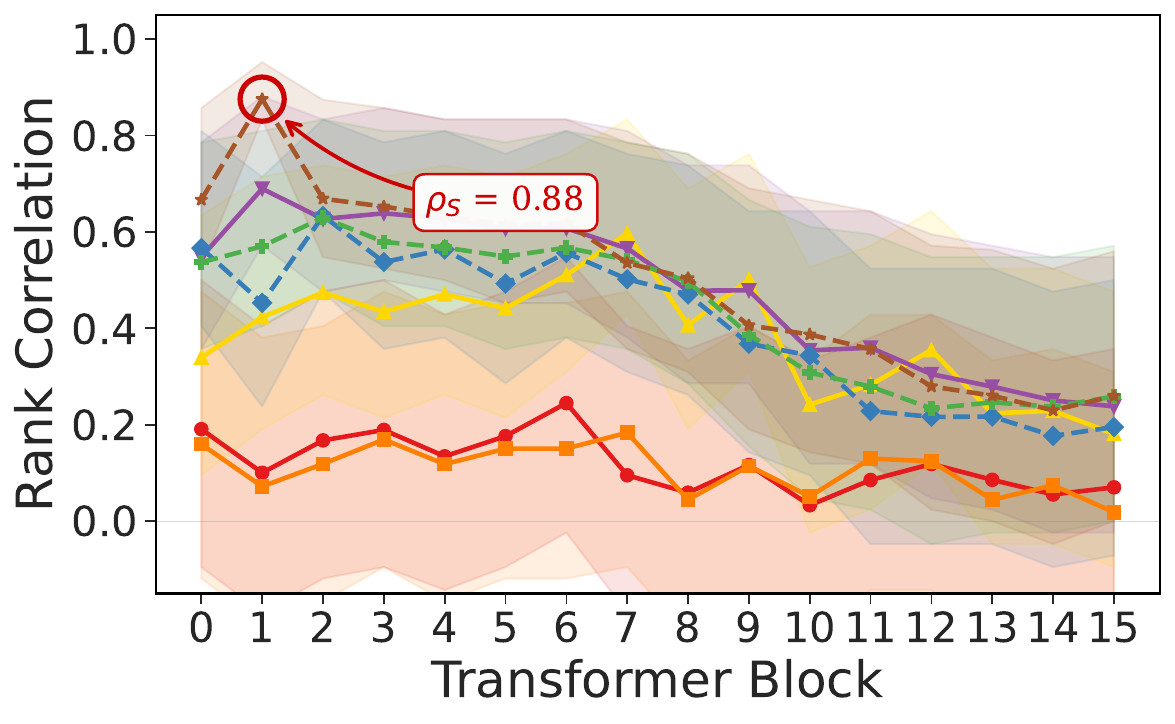}\hfill
	\includegraphics[width=0.245\linewidth]{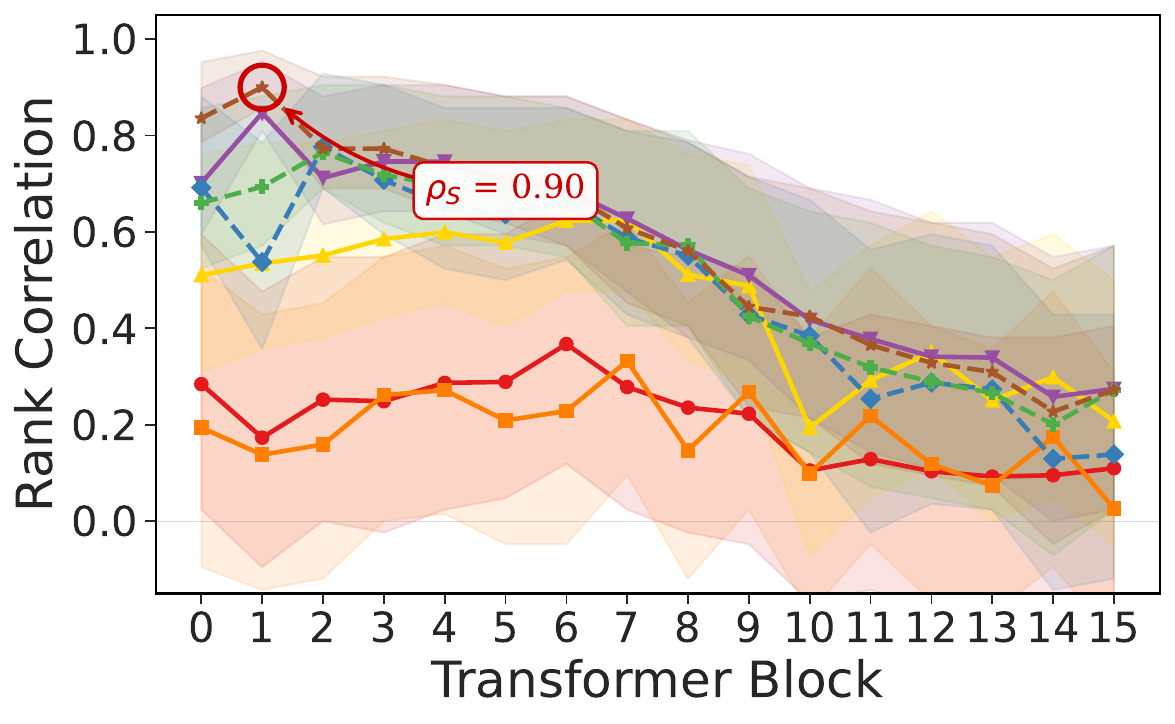}\hfill
	\includegraphics[width=0.245\linewidth]{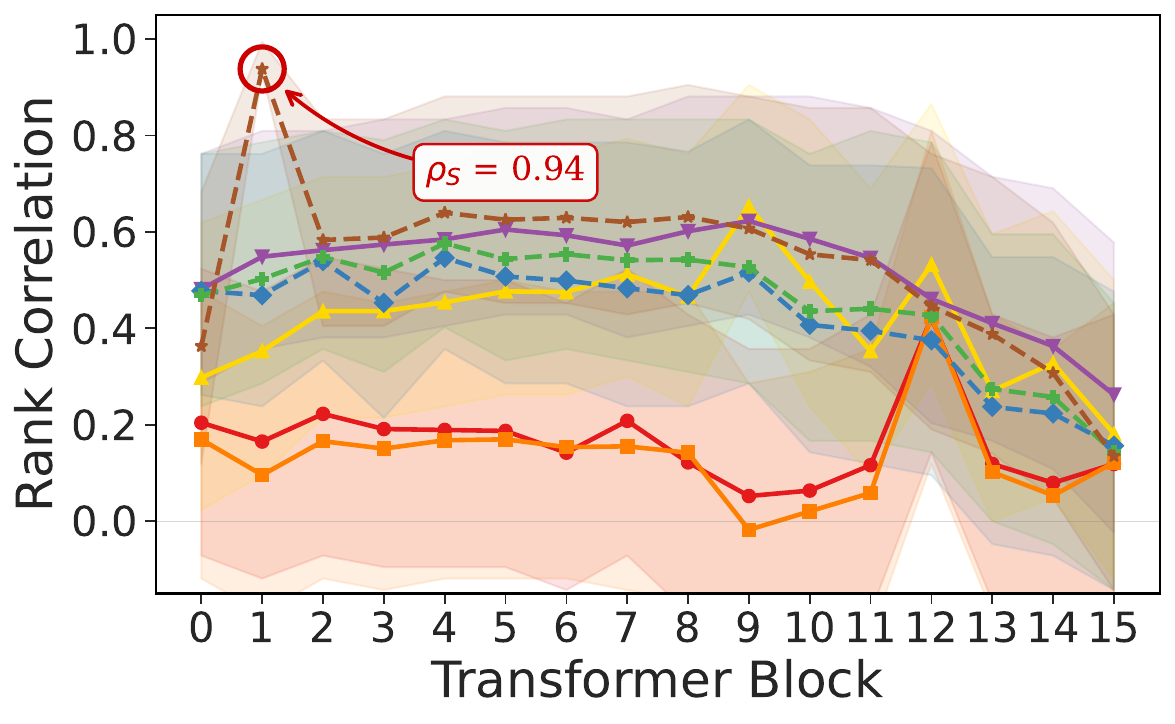}\hfill
	\includegraphics[width=0.245\linewidth]{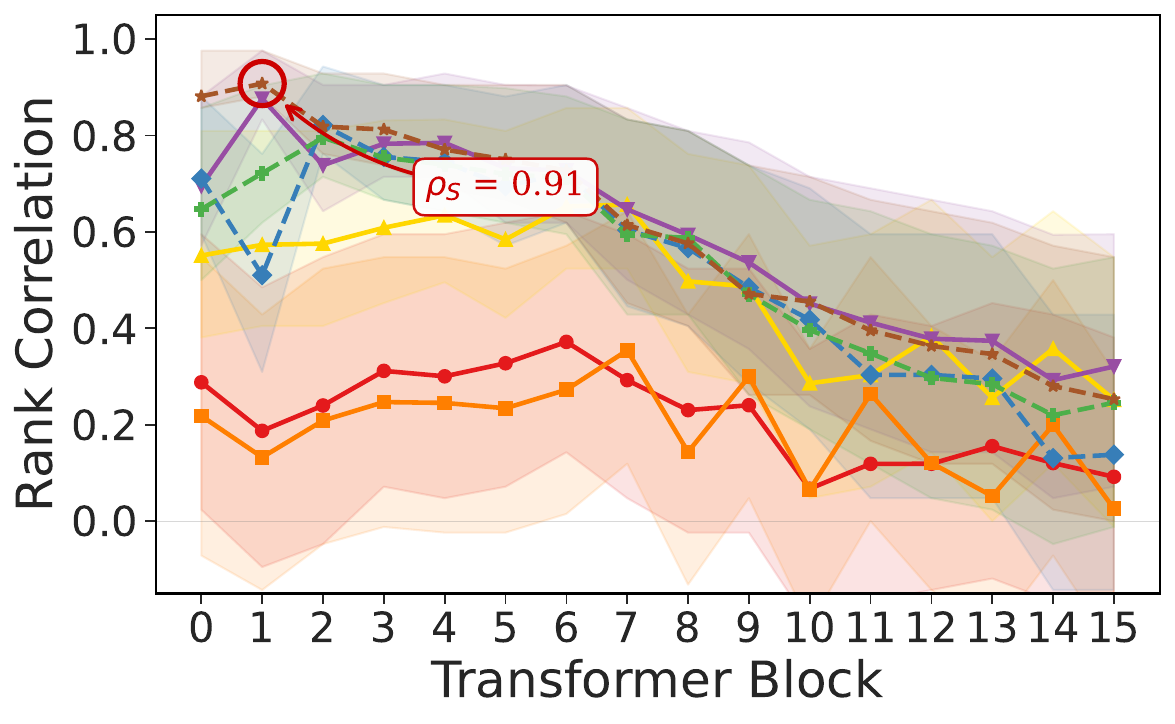}
	\par\smallskip
	\includegraphics[width=0.9\linewidth]{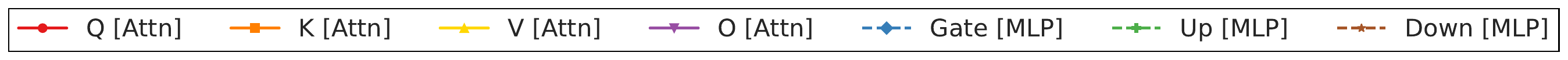}
	\caption{Case study on the four general/target pairs (columns): per-layer alignment scores recorded during Full-Training Update, aggregated by transformer block and layer type; lines are means over training steps and seeds, bands the \(25^{\text{th}}\)--\(75^{\text{th}}\) percentiles. \textbf{Top:} mean absolute score (log scale); the annotation marks the \texttt{down\_proj} peak and its ratio to the second-largest layer type at that block. \textbf{Bottom:} Spearman rank correlation \(\rho_S\) between each layer's per-sample scores and the shared ranking of Global Subset Update; the annotation marks \(\rho_S\) at the \texttt{down\_proj} peak.}
	\label{fig:case-study}
\end{figure}

\subsubsection{Ablations}\label{subsubsec:SFT-ablations}
\Cref{tab:SFT-downstream} compares the two endpoints of the group-wise family, the global and the per-layer partition. On the same four settings and under the same protocol, we ablate the partition granularity and the scoring backend; \Cref{adxsubsubsec:sft-ablations} gives the full setup, repeats every arm under full-parameter fine-tuning and MeSO, and reports two controls suggested by the case study, a layer-normalized variant of Global Subset Update and Random Subset Updates.
\begin{itemize}[leftmargin=*,nosep]
	\item \emph{Partition granularity.} Between the two endpoints, \emph{Block-Wise} and \emph{Sublayer-Wise} Subset Update group the scored layers by transformer block or by attention/MLP sublayer (group counts in \Cref{adxtab:partition-groups}); each group sums its layers' scores and keeps its own top-\(k\) subset.
	\item \emph{Scoring backend.} Every curated arm is scored both exactly (GIP) and with compressed scoring at \(\kappa^{(l)} = 64 \times 64\) (\Cref{subsubsec:compressed-scoring}).
\end{itemize}

\begin{table}[htpb]
	\centering
	\caption{Partition granularity and scoring backend on the question-answering settings (LoRA): downstream metric (\%) at the final checkpoint.}
	\label{adxtab:sft-qa-matrix-lora}
	\small
	\begin{adjustbox}{max width=\linewidth}
\begin{tabular}{l c c c c c c c c}
	\toprule
	\multirow{2}{*}{\textbf{Method}} & \multicolumn{2}{c}{\texttt{alpaca}/\texttt{samsum}} & \multicolumn{2}{c}{\texttt{less-mix}/\texttt{tydiqa}} & \multicolumn{2}{c}{\texttt{triviaqa}/\texttt{nq\_open}} & \multicolumn{2}{c}{\texttt{less-mix}/\texttt{squad}} \\
	\cmidrule(lr){2-3}\cmidrule(lr){4-5}\cmidrule(lr){6-7}\cmidrule(lr){8-9}
	 & Exact & Compressed & Exact & Compressed & Exact & Compressed & Exact & Compressed \\
	\midrule
	Full-Training & \multicolumn{2}{c}{\(22.4\,{\scriptstyle\pm 0.3}\)} & \multicolumn{2}{c}{\(4.9\,{\scriptstyle\pm 0.3}\)} & \multicolumn{2}{c}{\(12.7\,{\scriptstyle\pm 0.2}\)} & \multicolumn{2}{c}{\(9.0\,{\scriptstyle\pm 0.2}\)} \\
	\midrule
	Global Subset & \(22.7\,{\scriptstyle\pm 0.5}\) & \(22.3\,{\scriptstyle\pm 0.2}\) & \(13.7\,{\scriptstyle\pm 1.5}\) & \(12.8\,{\scriptstyle\pm 1.5}\) & \(13.5\,{\scriptstyle\pm 0.3}\) & \(\mathbf{13.8\,{\scriptstyle\pm 0.2}}\) & \(9.1\,{\scriptstyle\pm 0.2}\) & \(9.0\,{\scriptstyle\pm 0.2}\) \\
	Block-Wise Subset & \(25.9\,{\scriptstyle\pm 0.8}\) & \(23.0\,{\scriptstyle\pm 0.8}\) & \(33.5\,{\scriptstyle\pm 2.6}\) & \(25.8\,{\scriptstyle\pm 3.0}\) & \(13.7\,{\scriptstyle\pm 0.2}\) & \(13.2\,{\scriptstyle\pm 0.4}\) & \(8.6\,{\scriptstyle\pm 0.2}\) & \(8.3\,{\scriptstyle\pm 0.4}\) \\
	Sublayer-Wise Subset & \(28.3\,{\scriptstyle\pm 1.1}\) & \(24.7\,{\scriptstyle\pm 1.1}\) & \(42.2\,{\scriptstyle\pm 1.1}\) & \(29.0\,{\scriptstyle\pm 2.4}\) & \(13.2\,{\scriptstyle\pm 0.2}\) & \(13.2\,{\scriptstyle\pm 0.3}\) & \(8.8\,{\scriptstyle\pm 0.2}\) & \(8.4\,{\scriptstyle\pm 0.4}\) \\
	Layer-Wise Subset & \(\mathbf{29.3\,{\scriptstyle\pm 1.4}}\) & \(\mathbf{29.2\,{\scriptstyle\pm 0.7}}\) & \(\mathbf{48.3\,{\scriptstyle\pm 1.7}}\) & \(\mathbf{46.9\,{\scriptstyle\pm 0.6}}\) & \(\mathbf{14.0\,{\scriptstyle\pm 0.3}}\) & \(13.3\,{\scriptstyle\pm 0.4}\) & \(\mathbf{10.0\,{\scriptstyle\pm 0.4}}\) & \(\mathbf{9.5\,{\scriptstyle\pm 0.3}}\) \\
	\midrule
	Target-Only (same budget) & \multicolumn{2}{c}{\(30.9\,{\scriptstyle\pm 0.3}\)} & \multicolumn{2}{c}{\(51.8\,{\scriptstyle\pm 3.7}\)} & \multicolumn{2}{c}{\(7.0\,{\scriptstyle\pm 0.5}\)} & \multicolumn{2}{c}{\(5.5\,{\scriptstyle\pm 0.4}\)} \\
	Target-Only (early stop) & \multicolumn{2}{c}{\(32.8\,{\scriptstyle\pm 0.6}\)} & \multicolumn{2}{c}{\(51.1\,{\scriptstyle\pm 2.9}\)} & \multicolumn{2}{c}{\(9.3\,{\scriptstyle\pm 0.4}\)} & \multicolumn{2}{c}{\(8.7\,{\scriptstyle\pm 0.4}\)} \\
	\bottomrule
\end{tabular}

	\end{adjustbox}
\end{table}

\paragraph{Results.}
Two trends stand out in \Cref{adxtab:sft-qa-matrix-lora}. First, finer partitions help, and most under mismatch: with exact scoring, performance improves from the global to the block-wise, sublayer-wise, and layer-wise partition wherever the arms separate, by a wide margin under the misaligned \texttt{less-mix} pool and within a point or two on the closed-book settings, as \Cref{thm:projection-oracle,thm:trajectory} predict. Second, finer partitions need exact scores: Global Subset Update is insensitive to the backend, whereas Layer-Wise Subset Update loses several points under compressed scoring on the misaligned settings, with the intermediate partitions in between. The controls in \Cref{adxsubsubsec:sft-ablations} rule out two alternative explanations of the gap, since rescaling the layers' scores recovers only part of it and random subsets stay within a couple of points of Full-Training; the same ordering holds under full-parameter fine-tuning and, on the misaligned settings, under MeSO.

\subsubsection{Reasoning and Instruction Following}\label{subsubsec:SFT-capability}
We next test the regime dependence predicted in \Cref{subsec:bias-variance-tradeoffs} by varying the pool--target mismatch and the construction of the target set. We fine-tune \textsc{Qwen3-1.7B-Base}~\citep{qwen3technicalreport} with full-parameter fine-tuning and exact PIP scoring on four \(32\)K-example pools: the instruction, reasoning, and mixed pools of \texttt{Dolci-Instruct}~\citep{olmo3} and the \texttt{Tulu 3} SFT mixture~\citep{lambert2025tulu}. Each target set has \(32\) examples and is evaluated on OLMo 3's benchmarks for precise instruction following (\texttt{IFEval}~\citep{zhou2023instruction}, \texttt{IFBench}~\citep{pyatkin2025generalizing}) and mathematics (\texttt{MATH500}~\citep{lightman2024lets}, \texttt{GSM8K}~\citep{cobbe2021training}): \emph{Precise IF (\texttt{Dolci})} and \emph{Math (\texttt{Dolci})} are held-out rows of \texttt{Dolci}'s own sources, whereas \emph{Math (curated)} consists of \texttt{MATH}-train problems~\citep{hendrycks2021measuring} re-solved by \textsc{Qwen3-32B} with verified final answers, mimicking a small, carefully curated target set.

\begin{table}[thpb]
	\centering
	\caption{Reasoning and instruction-following SFT: downstream metric (\%) at the final checkpoint. \textbf{Top:} accuracy on \texttt{IFEval} and \texttt{IFBench}. \textbf{Bottom:} accuracy on \texttt{MATH500} and \texttt{GSM8K}.}
	\label{tab:SFT-qwen}
	\small\setlength{\tabcolsep}{3pt}\renewcommand{\arraystretch}{0.92}
\begin{adjustbox}{max width=\linewidth}
\begin{tabular}{@{}l@{\hspace{\tabcolsep}}*{24}{@{}>{\centering\arraybackslash}p{21.5pt}}@{}}
\toprule
\multirow{2}{*}[-0.5\dimexpr\aboverulesep+\belowrulesep+\cmidrulewidth\relax]{\textbf{Method}} & \multicolumn{8}{@{}c@{}}{\makebox[172pt]{\texttt{Dolci} instruction/Precise IF (\texttt{Dolci})}} & \multicolumn{8}{@{}c@{}}{\makebox[172pt]{\texttt{Dolci} mixed/Precise IF (\texttt{Dolci})}} & \multicolumn{8}{@{}c@{}}{\makebox[172pt]{\texttt{Tulu 3}/Precise IF (\texttt{Dolci})}} \\
\cmidrule(lr){2-9} \cmidrule(lr){10-17} \cmidrule(lr){18-25}
 & \multicolumn{4}{@{}c@{}}{\makebox[86pt]{\texttt{IFEval}}} & \multicolumn{4}{@{}c@{}}{\makebox[86pt]{\texttt{IFBench}}} & \multicolumn{4}{@{}c@{}}{\makebox[86pt]{\texttt{IFEval}}} & \multicolumn{4}{@{}c@{}}{\makebox[86pt]{\texttt{IFBench}}} & \multicolumn{4}{@{}c@{}}{\makebox[86pt]{\texttt{IFEval}}} & \multicolumn{4}{@{}c@{}}{\makebox[86pt]{\texttt{IFBench}}} \\
\midrule
Full-Training & \multicolumn{4}{@{}c@{}}{\makebox[86pt]{\(42.7\,{\scriptstyle\pm 0.6}\)}} & \multicolumn{4}{@{}c@{}}{\makebox[86pt]{\(17.9\,{\scriptstyle\pm 0.2}\)}} & \multicolumn{4}{@{}c@{}}{\makebox[86pt]{\(40.9\,{\scriptstyle\pm 0.6}\)}} & \multicolumn{4}{@{}c@{}}{\makebox[86pt]{\(\mathbf{17.9\,{\scriptstyle\pm 0.3}}\)}} & \multicolumn{4}{@{}c@{}}{\makebox[86pt]{\(32.5\,{\scriptstyle\pm 0.2}\)}} & \multicolumn{4}{@{}c@{}}{\makebox[86pt]{\(16.2\,{\scriptstyle\pm 0.3}\)}} \\
Global Subset & \multicolumn{4}{@{}c@{}}{\makebox[86pt]{\(\mathbf{44.5\,{\scriptstyle\pm 0.2}}\)}} & \multicolumn{4}{@{}c@{}}{\makebox[86pt]{\(18.5\,{\scriptstyle\pm 0.3}\)}} & \multicolumn{4}{@{}c@{}}{\makebox[86pt]{\(\mathbf{41.8\,{\scriptstyle\pm 0.3}}\)}} & \multicolumn{4}{@{}c@{}}{\makebox[86pt]{\(16.5\,{\scriptstyle\pm 0.4}\)}} & \multicolumn{4}{@{}c@{}}{\makebox[86pt]{\(\mathbf{35.3\,{\scriptstyle\pm 0.3}}\)}} & \multicolumn{4}{@{}c@{}}{\makebox[86pt]{\(16.0\,{\scriptstyle\pm 0.3}\)}} \\
\midrule
Layer-Wise Subset & \multicolumn{4}{@{}c@{}}{\makebox[86pt]{\(44.1\,{\scriptstyle\pm 0.4}\)}} & \multicolumn{4}{@{}c@{}}{\makebox[86pt]{\(\mathbf{18.7\,{\scriptstyle\pm 0.5}}\)}} & \multicolumn{4}{@{}c@{}}{\makebox[86pt]{\(41.7\,{\scriptstyle\pm 0.4}\)}} & \multicolumn{4}{@{}c@{}}{\makebox[86pt]{\(17.1\,{\scriptstyle\pm 0.3}\)}} & \multicolumn{4}{@{}c@{}}{\makebox[86pt]{\(\mathbf{35.3\,{\scriptstyle\pm 0.4}}\)}} & \multicolumn{4}{@{}c@{}}{\makebox[86pt]{\(\mathbf{16.4\,{\scriptstyle\pm 0.9}}\)}} \\
\bottomrule
\noalign{\vskip 3pt}
\toprule
\multirow{2}{*}[-0.5\dimexpr\aboverulesep+\belowrulesep+\cmidrulewidth\relax]{\textbf{Method}} & \multicolumn{6}{@{}c@{}}{\makebox[129pt]{\texttt{Dolci} reasoning/Math (\texttt{Dolci})}} & \multicolumn{6}{@{}c@{}}{\makebox[129pt]{\texttt{Dolci} mixed/Math (\texttt{Dolci})}} & \multicolumn{6}{@{}c@{}}{\makebox[129pt]{\texttt{Tulu 3}/Math (\texttt{Dolci})}} & \multicolumn{6}{@{}c@{}}{\makebox[129pt]{\texttt{Dolci} reasoning/Math (curated)}} \\
\cmidrule(lr){2-7} \cmidrule(lr){8-13} \cmidrule(lr){14-19} \cmidrule(lr){20-25}
 & \multicolumn{3}{@{}c@{}}{\makebox[64.5pt]{\texttt{MATH500}}} & \multicolumn{3}{@{}c@{}}{\makebox[64.5pt]{\texttt{GSM8K}}} & \multicolumn{3}{@{}c@{}}{\makebox[64.5pt]{\texttt{MATH500}}} & \multicolumn{3}{@{}c@{}}{\makebox[64.5pt]{\texttt{GSM8K}}} & \multicolumn{3}{@{}c@{}}{\makebox[64.5pt]{\texttt{MATH500}}} & \multicolumn{3}{@{}c@{}}{\makebox[64.5pt]{\texttt{GSM8K}}} & \multicolumn{3}{@{}c@{}}{\makebox[64.5pt]{\texttt{MATH500}}} & \multicolumn{3}{@{}c@{}}{\makebox[64.5pt]{\texttt{GSM8K}}} \\
\midrule
Full-Training & \multicolumn{3}{@{}c@{}}{\makebox[64.5pt]{\(63.6\,{\scriptstyle\pm 0.5}\)}} & \multicolumn{3}{@{}c@{}}{\makebox[64.5pt]{\(79.9\,{\scriptstyle\pm 0.2}\)}} & \multicolumn{3}{@{}c@{}}{\makebox[64.5pt]{\(\mathbf{64.2\,{\scriptstyle\pm 0.3}}\)}} & \multicolumn{3}{@{}c@{}}{\makebox[64.5pt]{\(78.8\,{\scriptstyle\pm 0.1}\)}} & \multicolumn{3}{@{}c@{}}{\makebox[64.5pt]{\(57.1\,{\scriptstyle\pm 0.5}\)}} & \multicolumn{3}{@{}c@{}}{\makebox[64.5pt]{\(78.8\,{\scriptstyle\pm 0.2}\)}} & \multicolumn{3}{@{}c@{}}{\makebox[64.5pt]{\(63.6\,{\scriptstyle\pm 0.5}\)}} & \multicolumn{3}{@{}c@{}}{\makebox[64.5pt]{\(79.9\,{\scriptstyle\pm 0.2}\)}} \\
Global Subset & \multicolumn{3}{@{}c@{}}{\makebox[64.5pt]{\(\mathbf{64.4\,{\scriptstyle\pm 0.3}}\)}} & \multicolumn{3}{@{}c@{}}{\makebox[64.5pt]{\(\mathbf{80.7\,{\scriptstyle\pm 0.2}}\)}} & \multicolumn{3}{@{}c@{}}{\makebox[64.5pt]{\(\mathbf{64.2\,{\scriptstyle\pm 0.4}}\)}} & \multicolumn{3}{@{}c@{}}{\makebox[64.5pt]{\(\mathbf{80.4\,{\scriptstyle\pm 0.3}}\)}} & \multicolumn{3}{@{}c@{}}{\makebox[64.5pt]{\(58.3\,{\scriptstyle\pm 0.4}\)}} & \multicolumn{3}{@{}c@{}}{\makebox[64.5pt]{\(79.4\,{\scriptstyle\pm 0.2}\)}} & \multicolumn{3}{@{}c@{}}{\makebox[64.5pt]{\(64.0\,{\scriptstyle\pm 0.2}\)}} & \multicolumn{3}{@{}c@{}}{\makebox[64.5pt]{\(79.9\,{\scriptstyle\pm 0.1}\)}} \\
\midrule
Layer-Wise Subset & \multicolumn{3}{@{}c@{}}{\makebox[64.5pt]{\(62.8\,{\scriptstyle\pm 0.2}\)}} & \multicolumn{3}{@{}c@{}}{\makebox[64.5pt]{\(80.3\,{\scriptstyle\pm 0.2}\)}} & \multicolumn{3}{@{}c@{}}{\makebox[64.5pt]{\(63.9\,{\scriptstyle\pm 0.4}\)}} & \multicolumn{3}{@{}c@{}}{\makebox[64.5pt]{\(\mathbf{80.4\,{\scriptstyle\pm 0.3}}\)}} & \multicolumn{3}{@{}c@{}}{\makebox[64.5pt]{\(\mathbf{60.6\,{\scriptstyle\pm 0.6}}\)}} & \multicolumn{3}{@{}c@{}}{\makebox[64.5pt]{\(\mathbf{79.8\,{\scriptstyle\pm 0.3}}\)}} & \multicolumn{3}{@{}c@{}}{\makebox[64.5pt]{\(\mathbf{65.1\,{\scriptstyle\pm 0.5}}\)}} & \multicolumn{3}{@{}c@{}}{\makebox[64.5pt]{\(\mathbf{81.0\,{\scriptstyle\pm 0.2}}\)}} \\
\bottomrule
\end{tabular}
\end{adjustbox}

\end{table}

\paragraph{Results.}
\Cref{tab:SFT-qwen} reports the results. Selection helps most where the target skill is diluted: on \texttt{Tulu 3}, both curated updates improve clearly over Full-Training on \texttt{IFEval}, and Layer-Wise Subset Update also on \texttt{MATH500}, whereas on the in-distribution \texttt{Dolci} pools the gains are smaller and the strongly regularized Full-Training and Global Subset Updates are competitive or best, consistent with the proposed tradeoff. Math gains also hinge on the target set: Math (\texttt{Dolci}) yields only modest \texttt{GSM8K} gains, whereas Math (curated) lets Layer-Wise Subset Update gain on both \texttt{MATH500} and \texttt{GSM8K}. Additional ablations over \(k\) and the target-set size are reported in \Cref{adxsubsubsec:capability-setting}.

\subsection{Reinforcement Learning from Human Feedback}\label{subsec:RLHF}
We next apply Dr.\ Post-Training to RLHF, going beyond the two-distribution setting (\Cref{subsec:beyond-two-distribution}): the PPO surrogate now acts as the stabilizing training signal for the high-variance policy gradient. We follow the standard PPO~\citep{schulman2017proximal} detoxification pipeline~\citep{huggingface2023detoxifying,hu2025a} with \texttt{TRL}~\citep{vonwerra2022trl} on \textsc{GPT-Neo-2.7B}~\citep{gpt-neo} with LoRA and exact PIP scoring, using prompts from \texttt{real-toxicity-prompts}~\citep{gehman2020realtoxicityprompts}. Continuations are rewarded by \texttt{LFTW R4 Target}~\citep{vidgen2021lftw} and judged by the independent \texttt{da-electra-hatespeech-detection} classifier. We compare standard PPO (Full-Training), Iterative Influence-based Filtering (IIF)~\citep{hu2025a} (offline selection once per PPO round), Global Subset Update, and Layer-Wise Subset Update (online selection at every mini-batch step), all with \emph{negative score filtering}. We ablate the \emph{target source} (self-referencing the rollout batch vs.\ fresh rollouts on held-out prompts) and the \emph{target loss} (expected reward vs.\ training loss). The target gradient is computed once per PPO round and reused across its mini-batches: one extra backward pass per round, plus one target rollout batch for the held-out source.

\begin{figure}[htbp]
	\centering
	\begin{subfigure}[t]{0.24\linewidth}
		\caption{Self-ref, reward}\label{fig:RLHF-sr-reward}
		\includegraphics[width=\linewidth]{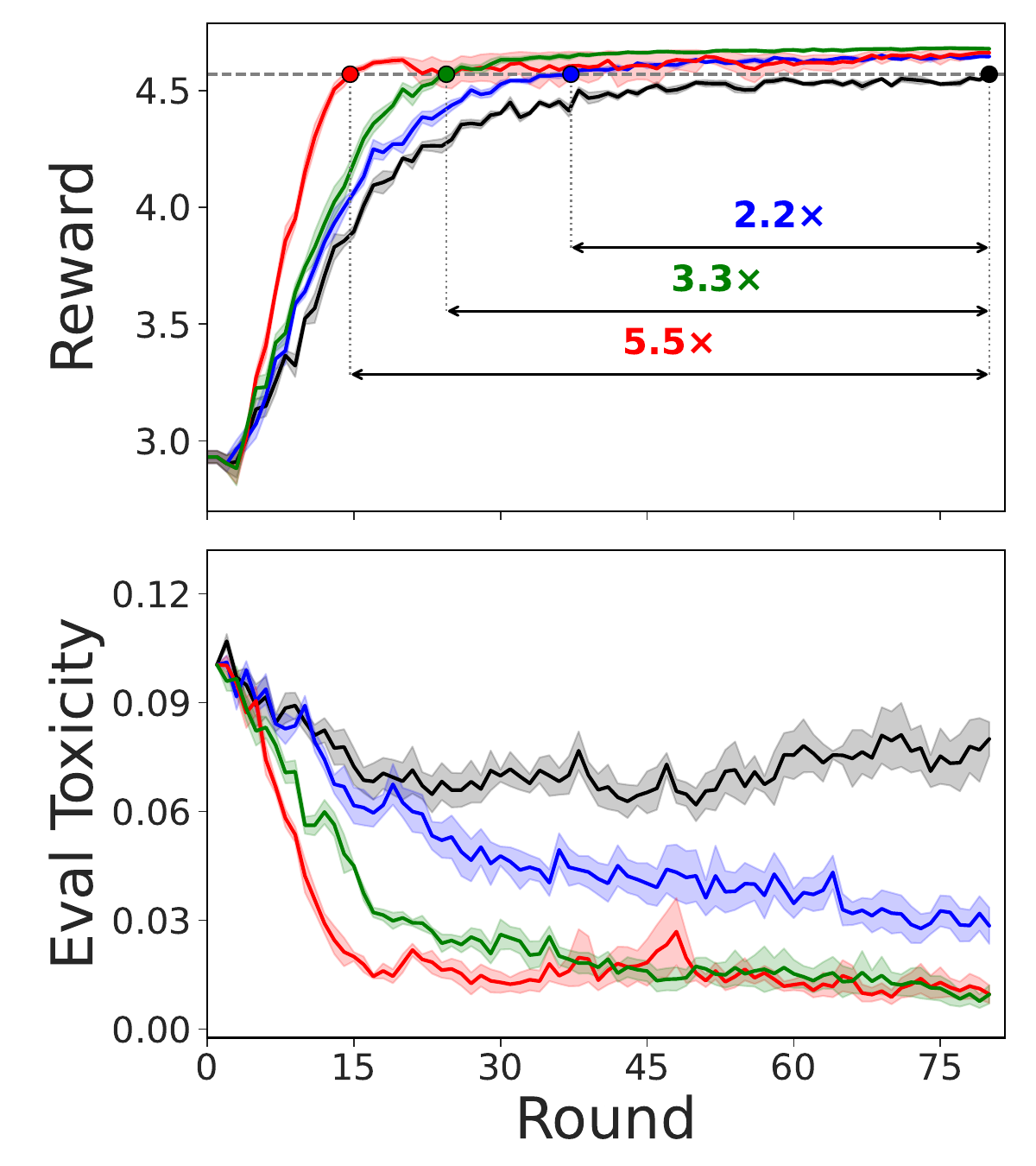}
	\end{subfigure}\hfill
	\begin{subfigure}[t]{0.24\linewidth}
		\caption{Self-ref, train loss}\label{fig:RLHF-sr-tloss}
		\includegraphics[width=\linewidth]{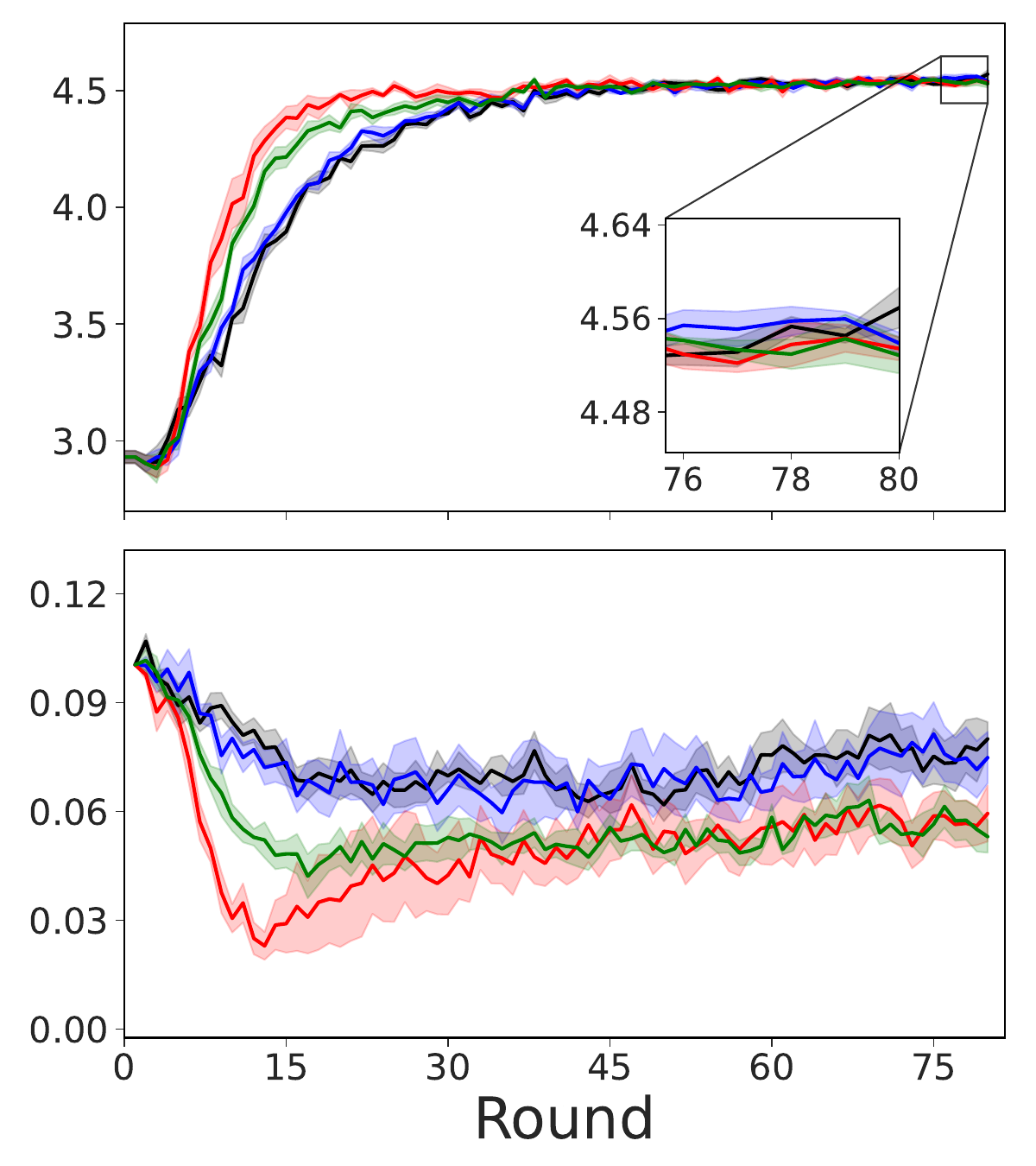}
	\end{subfigure}\hfill
	\begin{subfigure}[t]{0.24\linewidth}
		\caption{Held-out, reward}\label{fig:RLHF-ho-reward}
		\includegraphics[width=\linewidth]{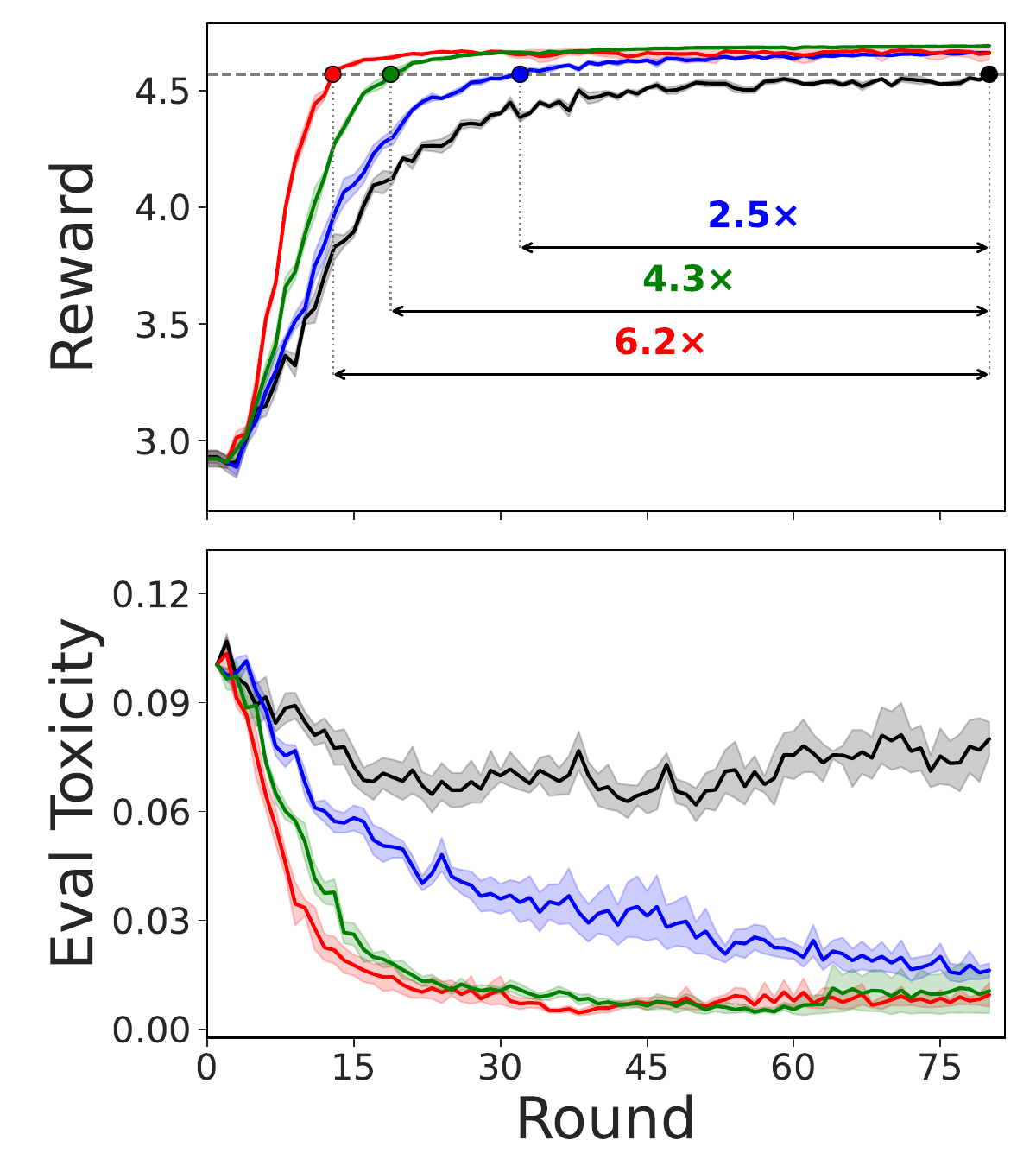}
	\end{subfigure}\hfill
	\begin{subfigure}[t]{0.24\linewidth}
		\caption{Held-out, train loss}\label{fig:RLHF-ho-tloss}
		\includegraphics[width=\linewidth]{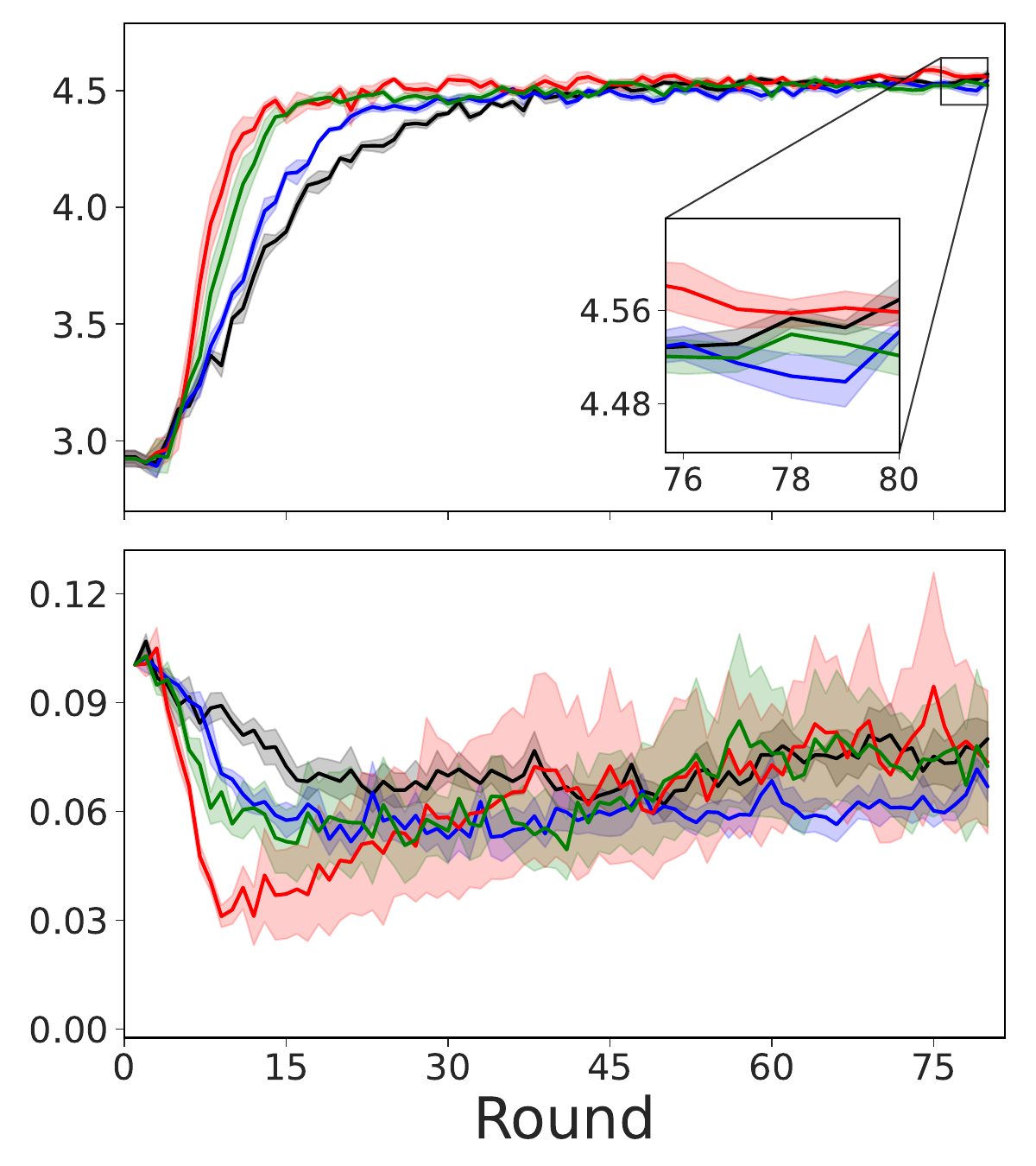}
	\end{subfigure}
	\par\smallskip
	\includegraphics[width=0.5\linewidth]{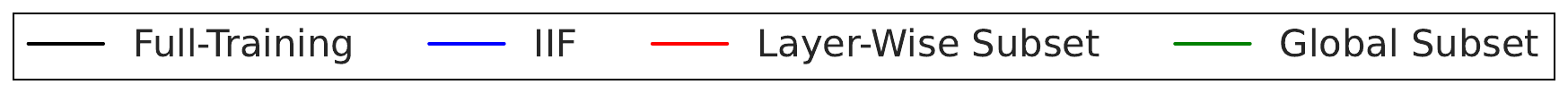}
	\caption{Detoxification RLHF for two target sources (self-reference/held-out) and two target losses (expected reward/training loss). \textbf{Top:} training reward. \textbf{Bottom:} evaluation toxicity.}
	\label{fig:RLHF}
\end{figure}

\paragraph{Results.}
Three patterns stand out in \Cref{fig:RLHF}. First, the expected-reward target (\Cref{fig:RLHF-sr-reward,fig:RLHF-ho-reward}) is far more effective than the PPO training loss (\Cref{fig:RLHF-sr-tloss,fig:RLHF-ho-tloss}), and under it held-out rollouts converge faster than self-referencing. Second, the ordering Layer-Wise, Global, IIF, standard PPO holds early in training in every configuration, in both reward and toxicity: under the reward target the curated updates reach standard PPO's reward at step \(80\) about five, three to four, and two to three times sooner, and Layer-Wise and Global Subset Updates end an order of magnitude below standard PPO in toxicity, with IIF in between. Third, the reward target separates the methods at convergence in both metrics, whereas under the training-loss target the early speed-up vanishes at convergence (insets of \Cref{fig:RLHF-sr-tloss,fig:RLHF-ho-tloss}), as expected. The gains also hold under three other toxicity judges and with compressed scoring (\Cref{adxsubsubsec:rlhf-judges,adxsubsubsec:rlhf-scoring-backend}).

\subsection{Reinforcement Learning with Verifiable Rewards}\label{subsec:RLVR}
Finally, we consider RLVR, where an automatic verifier provides the reward, using \texttt{Verl}~\citep{sheng2024hybridflow} with GRPO~\citep{shao2024deepseekmath} on \texttt{MATH}~\citep{hendrycks2021measuring} and full-parameter fine-tuning of \textsc{Qwen3-1.7B-Base}. We instantiate the framework as \emph{in-run cross-validation}: every fourth round draws a small training batch as the target batch and estimates its expected verifier reward by REINFORCE (as in \Cref{subsec:RLHF}); the resulting target gradient is reused until the next refresh, and exact PIP scores drive negative score filtering.

\begin{wrapfigure}[13]{r}{0.38\textwidth}
	\centering
	\includegraphics[width=\linewidth]{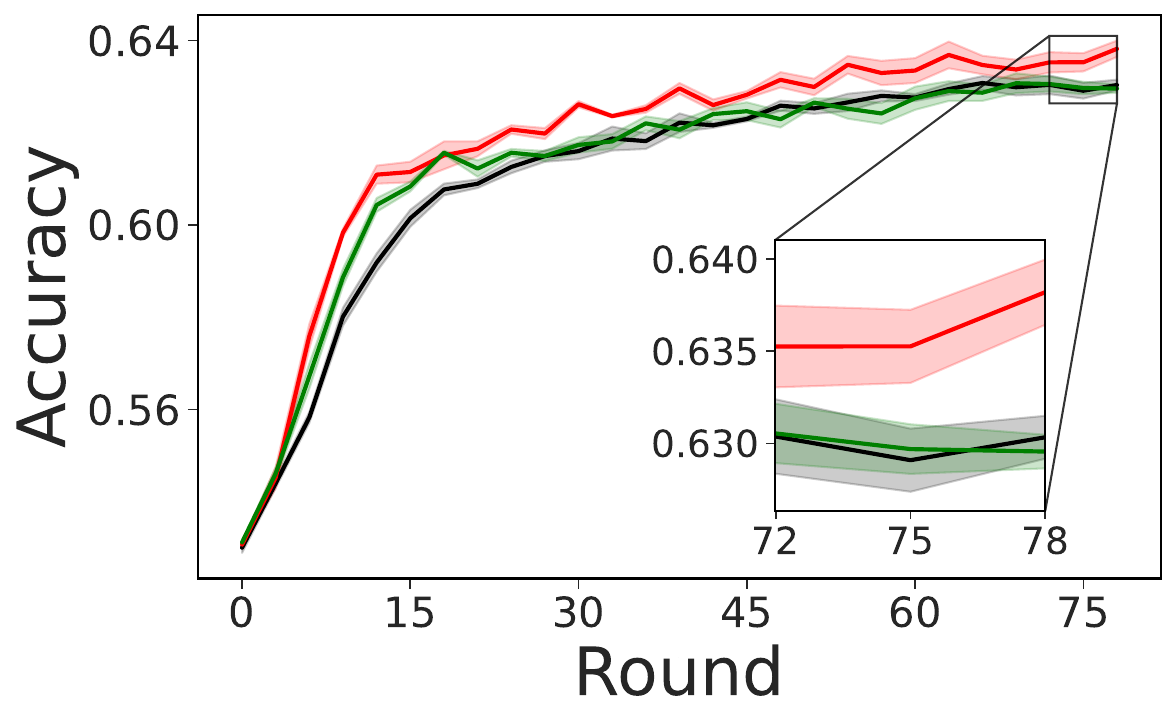}
	\includegraphics[width=\linewidth]{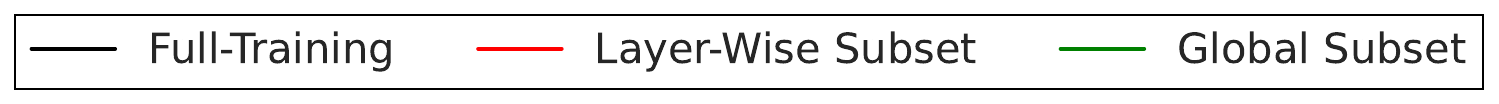}
	\caption{Math RLVR: evaluation accuracy over training rounds.}
	\label{fig:RLVR-math}
\end{wrapfigure}

\paragraph{Results.}
\Cref{fig:RLVR-math} shows the evaluation accuracy during the training. We see that Layer-Wise Subset Update performs best: Global Subset Update accelerates early training but ends near standard GRPO, whereas Layer-Wise sustains its advantage throughout and ends consistently higher (see \Cref{adxtab:rlvr}). Both curated arms pay for the reward-based target signal with extra target rollouts, but as in RLHF the target gradient is captured once and reused, here across four rounds, so the amortized target cost is an eighth of the training rollouts (\Cref{adxsubsec:details-of-rlvr}). Layer-Wise Subset Update thus converges to a higher final accuracy than standard GRPO at a marginal overhead.

\subsection{System Efficiency}\label{subsec:system-efficiency}
We benchmark training overhead and scoring cost on three models (\textsc{SmolLM2-360M}~\citep{allal2025smollm2}, \textsc{TinyLlama-1.1B}~\citep{zhang2024tinyllama}, \textsc{Llama-3.2-3B}~\citep{grattafiori2024llama}) on one \texttt{A40} GPU, using synthetic batches of fixed sequence length \(\seqlen\) (\Cref{tab:system-efficiency}).

\begin{table}[htpb]
	\centering
	\caption{System efficiency on one \texttt{A40} GPU (\(n=8\)).}
	\label{tab:system-efficiency}
	\renewcommand{\arraystretch}{0.92}
	\begin{subtable}{\linewidth}
		\centering
		\subcaption{Per-step component breakdown (ms) at \(\seqlen=512\), \(m=1\), \(k=4\) with exact GIP scoring.}
		\label{tab:system-overhead}
		\begin{adjustbox}{max width=\linewidth}
\begin{tabular}{l ccc ccc ccc}
\toprule
	& \multicolumn{3}{c}{\textsc{SmolLM2-360M}} & \multicolumn{3}{c}{\textsc{TinyLlama-1.1B}} & \multicolumn{3}{c}{\textsc{Llama-3.2-3B}} \\
\cmidrule(lr){2-4}\cmidrule(lr){5-7}\cmidrule(lr){8-10}
	\textbf{Component} & Full-Training & Layer-Wise Subset & Global Subset & Full-Training & Layer-Wise Subset & Global Subset & Full-Training & Layer-Wise Subset & Global Subset \\
\midrule
	\textbf{Forward} & 78.7 & 90.5 & 89.1 & 134.0 & 145.7 & 146.2 & 340.0 & 387.6 & 387.4 \\
	\textbf{Backward} & 155.2 & 194.1 & 191.6 & 278.5 & 306.6 & 306.3 & 821.0 & 752.4 & 758.8 \\
\rowcolor{black!6}
	\quad \texttt{a.grad} & 32.6 & 33.5 & 33.3 & 78.3 & 81.2 & 81.4 & 226.2 & 254.7 & 252.4 \\
\rowcolor{black!6}
	\quad \texttt{scoring} & --- & 32.2 & 32.5 & --- & 42.0 & 42.3 & --- & 85.5 & 85.1 \\
\rowcolor{black!6}
	\quad \texttt{w.grad} & 27.5 & 21.6 & 21.2 & 75.9 & 44.4 & 44.4 & 338.1 & 126.2 & 136.0 \\
\rowcolor{black!6}
	\quad \texttt{autograd} & 95.1 & 106.8 & 104.5 & 124.3 & 139.1 & 138.1 & 256.7 & 286.0 & 285.3 \\
	\textbf{Optimizer} & 26.6 & 26.7 & 26.6 & 80.7 & 80.7 & 80.8 & 235.5 & 235.4 & 235.5 \\
\midrule
	\textbf{Total} & \textbf{260} & \textbf{311} \scriptsize(+20\%) & \textbf{307} \scriptsize(+18\%) & \textbf{493} & \textbf{533} \scriptsize(+8\%) & \textbf{533} \scriptsize(+8\%) & \textbf{1397} & \textbf{1375} \scriptsize(-2\%) & \textbf{1382} \scriptsize(-1\%) \\
\bottomrule
\end{tabular}

		\end{adjustbox}
	\end{subtable}
	\par\smallskip
	\begin{subtable}{\linewidth}
		\centering
		\subcaption{Standalone scoring cost (ms/step) sweeping \(\seqlen\) at \(m=1\) and \(m\) at \(\seqlen=512\); cheapest exact method in bold.}
		\label{tab:score-cost}
		\begin{adjustbox}{max width=\linewidth}
\begin{tabular}{l rrr rrr rrr rrr rrr rrr}
\toprule
	& \multicolumn{6}{c}{\textsc{SmolLM2-360M}} & \multicolumn{6}{c}{\textsc{TinyLlama-1.1B}} & \multicolumn{6}{c}{\textsc{Llama-3.2-3B}} \\
\cmidrule(lr){2-7}\cmidrule(lr){8-13}\cmidrule(lr){14-19}
	& \multicolumn{3}{c}{sweep \(\seqlen\)} & \multicolumn{3}{c}{sweep \(m\)} & \multicolumn{3}{c}{sweep \(\seqlen\)} & \multicolumn{3}{c}{sweep \(m\)} & \multicolumn{3}{c}{sweep \(\seqlen\)} & \multicolumn{3}{c}{sweep \(m\)} \\
\cmidrule(lr){2-4}\cmidrule(lr){5-7}\cmidrule(lr){8-10}\cmidrule(lr){11-13}\cmidrule(lr){14-16}\cmidrule(lr){17-19}
	\textbf{Scoring} & 512 & 2048 & 8192 & 1 & 4 & 8 & 512 & 2048 & 8192 & 1 & 4 & 8 & 512 & 2048 & 8192 & 1 & 4 & 8 \\
\midrule
	Direct & 56 & 153 & 575 & 56 & 66 & 80 & 132 & 366 & 1353 & 132 & 161 & 198 & 414 & 1128 & 4091 & 414 & 485 & 587 \\
	GIP & \textbf{29} & 356 & 5950 & \textbf{29} & 95 & 182 & \textbf{39} & 519 & 10855 & \textbf{39} & 138 & 267 & \textbf{73} & 1001 & 18777 & \textbf{73} & \textbf{260} & 508 \\
	PIP & 40 & \textbf{122} & \textbf{463} & 40 & \textbf{52} & \textbf{70} & 88 & \textbf{335} & \textbf{1351} & 88 & \textbf{121} & \textbf{160} & 251 & \textbf{961} & \textbf{3880} & 251 & 325 & \textbf{429} \\
\midrule
	Compressed & 29 & 47 & 162 & 29 & 29 & 29 & 25 & 57 & 218 & 25 & 25 & 25 & 39 & 106 & 412 & 39 & 39 & 40 \\
\bottomrule
\end{tabular}

		\end{adjustbox}
	\end{subtable}
\end{table}

\paragraph{Results.}
\Cref{tab:system-overhead} reports the per-step breakdown: the curated updates add one target example and its scoring but skip half of the weight gradient (only the \(k=4\) kept samples enter it). On the two small, launch-bound models the fixed cost dominates (\(+18\)--\(20\%\) and \(+8\%\)); once compute-bound, the saving pays for it, and on \textsc{Llama-3.2-3B} the curated step is \(1\)--\(2\%\) faster than Full-Training; peak memory is close to Full-Training (\Cref{adxtab:peak-memory}).

\Cref{tab:score-cost} isolates the scoring cost, the only component of \Cref{tab:system-overhead} that the backend changes: GIP wins only at short contexts with \(m = 1\), PIP is the cheapest exact method beyond them (by \(\seqlen = 2048\) or \(m = 4\)--\(8\)) and \(5\)--\(13\times\) cheaper than GIP at \(\seqlen = 8192\), and compressed scoring is cheaper still, at the accuracy cost of \Cref{adxsubsubsec:sft-ablations}. The complete per-step sweep, including the small-batch configuration and the two-pass implementation of Global Subset Update, together with the peak GPU memory with and without activation checkpointing, is reported in \Cref{adxsubsec:system-efficiency}.

\section{Discussion}\label{sec:discussion}
In this section, we discuss several aspects related to our proposed framework.

\subsection{Classical Regularization}
Both data regularization and classical regularization (e.g., weight decay, LoRA's low-rank constraint) ultimately constrain the parameter update, but they differ in the \emph{source} of the constraint. Classical regularization is usually derived from a \emph{data-agnostic} complexity measure on \(\theta_t\), restricting the update through intrinsic properties of the parameters themselves. Data regularization, by contrast, imposes a \emph{data-induced} constraint in the parameter-update space via the feasible set \(U_t\), whose admissible directions are shaped by the available training samples.

\subsection{Designing New Feasible Sets}
The Dr.\ Post-Training framework provides a blueprint for designing new data-centric methods:
\begin{enumerate*}[label=(\roman*)]
    \item define a feasible set \(U_t\),
    \item design an efficient algorithm that solves \(u_t \in U_t\) based on \Cref{eq:proj}.
\end{enumerate*}
Several natural extensions of the feasible set follow this template:
\begin{enumerate}[leftmargin=*]
    \item \textbf{Soft weighting.} Replace hard subset constraints with continuous weights: \(U_t^{\mathrm{soft}} = \{\sum_i w_i g_i \colon w \in \Delta_{n-1}\}\), where \(\Delta_{n-1}\) is the probability simplex. The projection \Cref{eq:proj} becomes a quadratic program over the simplex, solvable in closed form or via efficient Frank-Wolfe steps~\citep{nikdan2025efficient}. This can be further extended to the same group-wise decomposition across parameter groups. This relaxation sidesteps the combinatorial search over \(\binom{n}{k}\) subsets required previously.
    \item \textbf{Token-level granularity.} For sequential models, instead of operating at the sample level, operate at the token level within each sample, yielding a finer-grained feasible set. Since per-token gradients are naturally available during backpropagation, this can be implemented within the existing Group-Wise Subset Update pipeline.
\end{enumerate}

\subsection{Connection to Domain Adaptation}\label{subsec:domain-transfer-and-distribution-shift}
The mismatch between general-purpose pretraining corpora and a narrow target task is a recurring theme in transfer learning and domain adaptation~\citep{pan2010survey}. A long line of work has studied when and how such transfer succeeds, highlighting source--target divergence and the risk of negative transfer as central concerns~\citep{ben2010theory}. Classical remedies operate at the instance level: reweighting examples under covariate shift~\citep{sugiyama2007covariate}, adapting feature representations~\citep{daume2007frustratingly}, aligning domains adversarially~\citep{ganin2016domain}, or combining multiple source domains~\citep{guo2018multi}. In the LLM era, continued in-domain or task-adaptive pretraining has become the de facto approach for reducing this mismatch~\citep{gururangan2020don}.

Our framework addresses the same general training--target gap, but at a fundamentally different level of abstraction. Where instance reweighting corrects the empirical risk via density ratios and representation alignment enforces domain invariance in hidden states, both are specific algorithmic choices for bridging the domain gap. Dr.\ Post-Training instead formulates general training data as a \emph{regularizer} on the target-driven update direction, yielding an explicit per-step bias--variance tradeoff governed by the feasible set \(U\). This perspective not only subsumes existing instance-level strategies as special cases of the feasible set but also opens the door to \emph{beyond-instance} designs such as Group-Wise Subset Update, where different parameter groups draw on different training subsets within a single optimization step---a capability that falls outside the scope of standard domain-adaptation formulations.
\section{Conclusion}\label{sec:conclusion}
In this work, we propose \textbf{Dr.\ Post-Training}, a data regularization framework for LLM post-training that leverages the abundant yet imperfectly aligned general training data as a regularizer that constrains target-driven updates. This framework provides an alternative view of existing data selection methods and further leads to natural generalization, broadening the design space and shedding light on their respective tradeoffs and convergence. With extensive system optimization, we offer an efficient implementation of our proposed methods under strict memory constraints in modern LLM training pipelines. Experiments across SFT, RLHF, and RLVR demonstrated improvements over strong baselines that track the mismatch between the general and target data, as the theory predicts, with minimal system overhead.

\subsubsection*{Acknowledgments}
The authors would like to thank Joe Melkonian for helpful discussions at the early stage of this project. P.\ Hu and J.W.\ Ma are partially supported by a gift grant from Google.

\bibliography{reference}

@article{ajalloeian2020convergence,
  author  = {Ajalloeian, Ahmad and Stich, Sebastian U},
  title   = {On the convergence of {SGD} with biased gradients},
  journal = {arXiv preprint arXiv:2008.00051},
  year    = {2020}
}

@article{albalak2024a,
  author  = {Alon Albalak and Yanai Elazar and Sang Michael Xie and Shayne Longpre and Nathan Lambert and Xinyi Wang and Niklas Muennighoff and Bairu Hou and Liangming Pan and Haewon Jeong and Colin Raffel and Shiyu Chang and Tatsunori Hashimoto and William Yang Wang},
  title   = {A Survey on Data Selection for Language Models},
  issn    = {2835-8856},
  journal = {Transactions on Machine Learning Research},
  note    = {Survey Certification},
  url     = {https://openreview.net/forum?id=XfHWcNTSHp},
  year    = {2024}
}

@article{allal2025smollm2,
  author  = {Allal, Loubna Ben and Lozhkov, Anton and Bakouch, Elie and Bl{\'a}zquez, Gabriel Mart{\'\i}n and Penedo, Guilherme and Tunstall, Lewis and Marafioti, Andr{\'e}s and Kydl{\'\i}{\v{c}}ek, Hynek and Lajar{\'\i}n, Agust{\'\i}n Piqueres and Srivastav, Vaibhav and others},
  title   = {SmolLM2: When Smol Goes Big--Data-Centric Training of a Small Language Model},
  journal = {arXiv preprint arXiv:2502.02737},
  year    = {2025}
}

@misc{alpaca,
  author       = {Rohan Taori and Ishaan Gulrajani and Tianyi Zhang and Yann Dubois and Xuechen Li and Carlos Guestrin and Percy Liang and Tatsunori B. Hashimoto },
  title        = {Stanford Alpaca: An Instruction-following LLaMA model},
  howpublished = {\url{https://github.com/tatsu-lab/stanford_alpaca}},
  journal      = {GitHub repository},
  publisher    = {GitHub},
  year         = {2023}
}

@article{ba2016layer,
  author  = {Ba, Jimmy Lei and Kiros, Jamie Ryan and Hinton, Geoffrey E},
  title   = {Layer normalization},
  journal = {arXiv preprint arXiv:1607.06450},
  year    = {2016}
}

@article{bai2022training,
  author  = {Bai, Yuntao and Jones, Andy and Ndousse, Kamal and Askell, Amanda and Chen, Anna and DasSarma, Nova and Drain, Dawn and Fort, Stanislav and Ganguli, Deep and Henighan, Tom and others},
  title   = {Training a helpful and harmless assistant with reinforcement learning from human feedback},
  journal = {arXiv preprint arXiv:2204.05862},
  year    = {2022}
}

@article{ben2010theory,
  author    = {Ben-David, Shai and Blitzer, John and Crammer, Koby and Kulesza, Alex and Pereira, Fernando and Vaughan, Jennifer Wortman},
  title     = {A theory of learning from different domains},
  journal   = {Machine Learning},
  number    = {1-2},
  pages     = {151--175},
  publisher = {Springer},
  volume    = {79},
  year      = {2010}
}

@article{bottou2018optimization,
  author    = {Bottou, L{\'e}on and Curtis, Frank E and Nocedal, Jorge},
  title     = {Optimization methods for large-scale machine learning},
  journal   = {SIAM review},
  number    = {2},
  pages     = {223--311},
  publisher = {SIAM},
  volume    = {60},
  year      = {2018}
}

@article{cao2023instruction,
  author  = {Cao, Yihan and Kang, Yanbin and Wang, Chi and Sun, Lichao},
  title   = {Instruction mining: Instruction data selection for tuning large language models},
  journal = {arXiv preprint arXiv:2307.06290},
  year    = {2023}
}

@article{casper2023open,
  author  = {Casper, Stephen and Davies, Xander and Shi, Claudia and Gilbert, Thomas Krendl and Scheurer, J{\'e}r{\'e}my and Rando, Javier and Freedman, Rachel and Korbak, Tomasz and Lindner, David and Freire, Pedro and others},
  title   = {Open problems and fundamental limitations of reinforcement learning from human feedback},
  journal = {arXiv preprint arXiv:2307.15217},
  year    = {2023}
}

@article{chen2016training,
  author  = {Chen, Tianqi and Xu, Bing and Zhang, Chiyuan and Guestrin, Carlos},
  title   = {Training deep nets with sublinear memory cost},
  journal = {arXiv preprint arXiv:1604.06174},
  year    = {2016}
}

@inproceedings{chen2024alpagasus,
  author    = {Chen, Lichang and Li, Shiyang and Yan, Jun and Wang, Hai and Gunaratna, Kalpa and Yadav, Vikas and Tang, Zheng and Srinivasan, Vijay and Zhou, Tianyi and Huang, Heng and others},
  title     = {AlpaGasus: Training a Better Alpaca with Fewer Data},
  booktitle = {The Twelfth International Conference on Learning Representations},
  year      = {2024}
}

@inproceedings{choe2025your,
  author    = {Choe, Sang Keun and Ahn, Hwijeen and Bae, Juhan and Zhao, Kewen and Kang, Minsoo and Chung, Youngseog and Pratapa, Adithya and Neiswanger, Willie and Strubell, Emma and Mitamura, Teruko and others},
  title     = {What is your data worth to gpt? llm-scale data valuation with influence functions},
  booktitle = {Advances in Neural Information Processing Systems},
  year      = {2025}
}

@article{christiano2017deep,
  author  = {Christiano, Paul F and Leike, Jan and Brown, Tom and Martic, Miljan and Legg, Shane and Amodei, Dario},
  title   = {Deep reinforcement learning from human preferences},
  journal = {Advances in neural information processing systems},
  volume  = {30},
  year    = {2017}
}

@article{clark2020tydi,
  author  = {Clark, Jonathan H and Choi, Eunsol and Collins, Michael and Garrette, Dan and Kwiatkowski, Tom and Nikolaev, Vitaly and Palomaki, Jennimaria},
  title   = {{TyDi QA}: A benchmark for information-seeking question answering in typologically diverse languages},
  journal = {Transactions of the Association for Computational Linguistics},
  pages   = {454--470},
  volume  = {8},
  year    = {2020}
}

@article{cobbe2021training,
  title   = {Training verifiers to solve math word problems},
  author  = {Cobbe, Karl and Kosaraju, Vineet and Bavarian, Mohammad and Chen, Mark and Jun, Heewoo and Kaiser, Lukasz and Plappert, Matthias and Tworek, Jerry and Hilton, Jacob and Nakano, Reiichiro and others},
  journal = {arXiv preprint arXiv:2110.14168},
  year    = {2021}
}

@inproceedings{daume2007frustratingly,
  author    = {Daum{\'e} III, Hal},
  title     = {Frustratingly easy domain adaptation},
  booktitle = {Proceedings of the 45th annual meeting of the association of computational linguistics},
  pages     = {256--263},
  year      = {2007}
}

@article{deng2025survey,
  author  = {Deng, Junwei and Hu, Yuzheng and Hu, Pingbang and Li, Ting-Wei and Liu, Shixuan and Wang, Jiachen T. and Ley, Dan and Dai, Qirun and Huang, Benhao and Huang, Jin and Jiao, Cathy and Just, Hoang Anh and Pan, Yijun and Shen, Jingyan and Tu, Yiwen and Wang, Weiyi and Wang, Xinhe and Zhang, Shichang and Zhang, Shiyuan and Jia, Ruoxi and Lakkaraju, Himabindu and Peng, Hao and Tang, Weijing and Xiong, Chenyan and Zhao, Jieyu and Tong, Hanghang and Zhao, Han and Ma, Jiaqi W.},
  title   = {A Survey of Data Attribution: Methods, Applications, and Evaluation in the Era of Generative AI},
  doi     = {10.2139/ssrn.5451054},
  journal = {SSRN},
  note    = {Available at SSRN: \url{https://ssrn.com/abstract=5451054}},
  year    = {2025}
}

@article{dettmers2023qlora,
  author  = {Dettmers, Tim and Pagnoni, Artidoro and Holtzman, Ari and Zettlemoyer, Luke},
  title   = {Qlora: Efficient finetuning of quantized llms},
  journal = {Advances in neural information processing systems},
  pages   = {10088--10115},
  volume  = {36},
  year    = {2023}
}

@inproceedings{ethayarajh2022understanding,
  author       = {Ethayarajh, Kawin and Choi, Yejin and Swayamdipta, Swabha},
  title        = {Understanding dataset difficulty with $\mathcal{V}$-Usable Information},
  booktitle    = {International Conference on Machine Learning},
  organization = {PMLR},
  pages        = {5988--6008},
  year         = {2022}
}

@article{ganin2016domain,
  author  = {Ganin, Yaroslav and Ustinova, Evgeniya and Ajakan, Hana and Germain, Pascal and Larochelle, Hugo and Laviolette, Fran{\c{c}}ois and March, Mario and Lempitsky, Victor},
  title   = {Domain-adversarial training of neural networks},
  journal = {Journal of machine learning research},
  number  = {59},
  pages   = {1--35},
  volume  = {17},
  year    = {2016}
}

@article{gehman2020realtoxicityprompts,
  author  = {Gehman, Samuel and Gururangan, Suchin and Sap, Maarten and Choi, Yejin and Smith, Noah A},
  title   = {Realtoxicityprompts: Evaluating neural toxic degeneration in language models},
  journal = {arXiv preprint arXiv:2009.11462},
  year    = {2020}
}

@article{ghadimi2013stochastic,
  author  = {Ghadimi, Saeed and Lan, Guanghui},
  title   = {Stochastic first- and zeroth-order methods for nonconvex stochastic programming},
  journal = {SIAM Journal on Optimization},
  volume  = {23},
  number  = {4},
  pages   = {2341--2368},
  year    = {2013}
}

@inproceedings{ghorbani2019data,
  author       = {Ghorbani, Amirata and Zou, James},
  title        = {Data shapley: Equitable valuation of data for machine learning},
  booktitle    = {International conference on machine learning},
  organization = {PMLR},
  pages        = {2242--2251},
  year         = {2019}
}

@inproceedings{gliwa2019samsum,
  author    = {Gliwa, Bogdan and Mochol, Iwona and Biesek, Maciej and Wawer, Aleksander},
  title     = {{SAMSum} corpus: A human-annotated dialogue dataset for abstractive summarization},
  booktitle = {Proceedings of the 2nd Workshop on New Frontiers in Summarization},
  pages     = {70--79},
  publisher = {Association for Computational Linguistics},
  year      = {2019}
}

@misc{gpt-neo,
  author    = {Black, Sid and Leo, Gao and Wang, Phil and Leahy, Connor and Biderman, Stella},
  title     = {{GPT-Neo: Large Scale Autoregressive Language Modeling with Mesh-Tensorflow}},
  doi       = {10.5281/zenodo.5297715},
  month     = mar,
  note      = {{If you use this software, please cite it using these metadata.}},
  publisher = {Zenodo},
  url       = {https://doi.org/10.5281/zenodo.5297715},
  version   = {1.0},
  year      = {2021}
}

@article{grattafiori2024llama,
  author  = {Grattafiori, Aaron and Dubey, Abhimanyu and Jauhri, Abhinav and Pandey, Abhinav and Kadian, Abhishek and Al-Dahle, Ahmad and Letman, Aiesha and Mathur, Akhil and Schelten, Alan and Vaughan, Alex and others},
  title   = {The llama 3 herd of models},
  journal = {arXiv preprint arXiv:2407.21783},
  year    = {2024}
}

@article{gunasekar2023textbooks,
  author  = {Gunasekar, Suriya and Zhang, Yi and Aneja, Jyoti and Mendes, Caio C{\'e}sar Teodoro and Del Giorno, Allie and Gopi, Sivakanth and Javaheripi, Mojan and Kauffmann, Piero and de Rosa, Gustavo and Saarikivi, Olli and others},
  title   = {Textbooks are all you need},
  journal = {arXiv preprint arXiv:2306.11644},
  year    = {2023}
}

@inproceedings{guo2018multi,
  author    = {Guo, Jiang and Shah, Darsh and Barzilay, Regina},
  title     = {Multi-source domain adaptation with mixture of experts},
  booktitle = {Proceedings of the 2018 conference on empirical methods in natural language processing},
  pages     = {4694--4703},
  year      = {2018}
}

@article{guo2025deepseek,
  author  = {Guo, Daya and Yang, Dejian and Zhang, Haowei and Song, Junxiao and Wang, Peiyi and Zhu, Qihao and Xu, Runxin and Zhang, Ruoyu and Ma, Shirong and Bi, Xiao and others},
  title   = {Deepseek-r1: Incentivizing reasoning capability in llms via reinforcement learning},
  journal = {arXiv preprint arXiv:2501.12948},
  year    = {2025}
}

@inproceedings{gururangan2020don,
  author    = {Gururangan, Suchin and Marasovi{\'c}, Ana and Swayamdipta, Swabha and Lo, Kyle and Beltagy, Iz and Downey, Doug and Smith, Noah A},
  title     = {Don’t stop pretraining: Adapt language models to domains and tasks},
  booktitle = {Proceedings of the 58th annual meeting of the association for computational linguistics},
  pages     = {8342--8360},
  year      = {2020}
}

@inproceedings{han2021influence,
  author    = {Han, Xiaochuang  and Tsvetkov, Yulia},
  title     = {Influence Tuning: Demoting Spurious Correlations via Instance Attribution and Instance-Driven Updates},
  address   = {Punta Cana, Dominican Republic},
  booktitle = {Findings of the Association for Computational Linguistics: EMNLP 2021},
  doi       = {10.18653/v1/2021.findings-emnlp.374},
  editor    = {Moens, Marie-Francine  and Huang, Xuanjing  and Specia, Lucia  and Yih, Scott Wen-tau},
  month     = nov,
  pages     = {4398--4409},
  publisher = {Association for Computational Linguistics},
  url       = {https://aclanthology.org/2021.findings-emnlp.374/},
  year      = {2021}
}

@article{han2024parameterefficient,
  author  = {Zeyu Han and Chao Gao and Jinyang Liu and Jeff Zhang and Sai Qian Zhang},
  title   = {Parameter-Efficient Fine-Tuning for Large Models: A Comprehensive Survey},
  issn    = {2835-8856},
  journal = {Transactions on Machine Learning Research},
  url     = {https://openreview.net/forum?id=lIsCS8b6zj},
  year    = {2024}
}

@article{hazan2016introduction,
  author    = {Hazan, Elad},
  title     = {Introduction to online convex optimization},
  journal   = {Foundations and Trends in Optimization},
  number    = {3-4},
  pages     = {157--325},
  publisher = {Now Publishers},
  volume    = {2},
  year      = {2016}
}

@inproceedings{he2024what,
  author    = {Luxi He and Mengzhou Xia and Peter Henderson},
  title     = {What is in Your Safe Data? Identifying Benign Data that Breaks Safety},
  booktitle = {First Conference on Language Modeling},
  url       = {https://openreview.net/forum?id=Hi8jKh4HE9},
  year      = {2024}
}

@inproceedings{he2025subspace,
  author    = {He, Yutong and Li, Pengrui and Hu, Yipeng and Chen, Chuyan and Yuan, Kun},
  title     = {Subspace Optimization for Large Language Models with Convergence Guarantees},
  booktitle = {Forty-second International Conference on Machine Learning},
  year      = {2025}
}

@inproceedings{hendrycks2021measuring,
  title     = {Measuring Mathematical Problem Solving With the {MATH} Dataset},
  author    = {Dan Hendrycks and Collin Burns and Saurav Kadavath and Akul Arora and Steven Basart and Eric Tang and Dawn Song and Jacob Steinhardt},
  booktitle = {Thirty-fifth Conference on Neural Information Processing Systems Datasets and Benchmarks Track (Round 2)},
  year      = {2021},
  url       = {https://openreview.net/forum?id=7Bywt2mQsCe}
}

@inproceedings{hu2022lora,
  title     = {Lo{RA}: Low-Rank Adaptation of Large Language Models},
  author    = {Edward J Hu and Yelong Shen and Phillip Wallis and Zeyuan Allen-Zhu and Yuanzhi Li and Shean Wang and Lu Wang and Weizhu Chen},
  booktitle = {International Conference on Learning Representations},
  year      = {2022},
  url       = {https://openreview.net/forum?id=nZeVKeeFYf9}
}

@inproceedings{hu2024most,
  author    = {Hu, Yuzheng and Hu, Pingbang and Zhao, Han and Ma, Jiaqi},
  title     = {Most Influential Subset Selection: Challenges, Promises, and Beyond},
  booktitle = {Advances in Neural Information Processing Systems},
  editor    = {A. Globerson and L. Mackey and D. Belgrave and A. Fan and U. Paquet and J. Tomczak and C. Zhang},
  pages     = {119778--119810},
  publisher = {Curran Associates, Inc.},
  url       = {https://proceedings.neurips.cc/paper_files/paper/2024/file/d8684e49752e06ac5e4b554b60ad212a-Paper-Conference.pdf},
  volume    = {37},
  year      = {2024}
}

@inproceedings{hu2025a,
  author    = {Yuzheng Hu and Fan Wu and Haotian Ye and David Forsyth and James Zou and Nan Jiang and Jiaqi W. Ma and Han Zhao},
  title     = {A Snapshot of Influence: A Local Data Attribution Framework for Online Reinforcement Learning},
  booktitle = {The Thirty-ninth Annual Conference on Neural Information Processing Systems},
  url       = {https://openreview.net/forum?id=sYK4yPDuT1},
  year      = {2025}
}

@inproceedings{hu2025grass,
  author    = {Pingbang Hu and Joseph Melkonian and Weijing Tang and Han Zhao and Jiaqi W. Ma},
  title     = {Gra{SS}: Scalable Data Attribution with Gradient Sparsification and Sparse Projection},
  booktitle = {The Thirty-ninth Annual Conference on Neural Information Processing Systems},
  url       = {https://openreview.net/forum?id=o0HgWRmyY1},
  year      = {2025}
}

@inproceedings{hu2026unified,
  author    = {Pingbang Hu and Yuzheng Hu and Jiaqi W. Ma and Han Zhao},
  title     = {A Unified Theory of Random Projection for Influence Functions},
  booktitle = {ICLR 2026 Workshop on Navigating and Addressing Data Problems for Foundation Models},
  year      = {2026}
}

@misc{huggingface2023detoxifying,
  author       = {{Hugging Face}},
  title        = {Detoxifying a Language Model using PPO},
  howpublished = {\url{https://huggingface.co/docs/trl/en/detoxifying_a_lm}},
  note         = {TRL documentation (v0.17.0), accessed May 8, 2025},
  year         = {2023}
}

@article{hutchinson1989stochastic,
  author    = {Hutchinson, Michael F},
  title     = {A stochastic estimator of the trace of the influence matrix for Laplacian smoothing splines},
  journal   = {Communications in Statistics-Simulation and Computation},
  number    = {3},
  pages     = {1059--1076},
  publisher = {Taylor \& Francis},
  volume    = {18},
  year      = {1989}
}

@inproceedings{ivison2023data,
  author    = {Ivison, Hamish and Smith, Noah A and Hajishirzi, Hannaneh and Dasigi, Pradeep},
  title     = {Data-efficient finetuning using cross-task nearest neighbors},
  booktitle = {Findings of the Association for Computational Linguistics: ACL 2023},
  pages     = {9036--9061},
  year      = {2023}
}

@article{iyer2022opt,
  author  = {Iyer, Srinivasan and Lin, Xi Victoria and Pasunuru, Ramakanth and Mihaylov, Todor and Simig, Daniel and Yu, Ping and Shuster, Kurt and Wang, Tianlu and Liu, Qing and Koura, Punit Singh and others},
  title   = {Opt-iml: Scaling language model instruction meta learning through the lens of generalization},
  journal = {arXiv preprint arXiv:2212.12017},
  year    = {2022}
}

@inproceedings{jiao2025feasibility,
  author    = {Jiao, Cathy  and Gao, Weizhen  and Raghunathan, Aditi  and Xiong, Chenyan},
  title     = {On the Feasibility of In-Context Probing for Data Attribution},
  address   = {Albuquerque, New Mexico},
  booktitle = {Findings of the Association for Computational Linguistics: NAACL 2025},
  doi       = {10.18653/v1/2025.findings-naacl.286},
  editor    = {Chiruzzo, Luis  and Ritter, Alan  and Wang, Lu},
  isbn      = {979-8-89176-195-7},
  month     = apr,
  pages     = {5140--5155},
  publisher = {Association for Computational Linguistics},
  url       = {https://aclanthology.org/2025.findings-naacl.286/},
  year      = {2025}
}

@inproceedings{joshi2017triviaqa,
  author    = {Joshi, Mandar and Choi, Eunsol and Weld, Daniel S and Zettlemoyer, Luke},
  title     = {Triviaqa: A large scale distantly supervised challenge dataset for reading comprehension},
  booktitle = {Proceedings of the 55th Annual Meeting of the Association for Computational Linguistics (Volume 1: Long Papers)},
  pages     = {1601--1611},
  year      = {2017}
}

@inproceedings{karimi2016linear,
  author    = {Karimi, Hamed and Nutini, Julie and Schmidt, Mark},
  title     = {Linear convergence of gradient and proximal-gradient methods under the {P}olyak-{\L}ojasiewicz condition},
  booktitle = {Joint European Conference on Machine Learning and Knowledge Discovery in Databases},
  pages     = {795--811},
  year      = {2016}
}

@inproceedings{katharopoulos2018not,
  author       = {Katharopoulos, Angelos and Fleuret, Fran{\c{c}}ois},
  title        = {Not all samples are created equal: Deep learning with importance sampling},
  booktitle    = {International conference on machine learning},
  organization = {PMLR},
  pages        = {2525--2534},
  year         = {2018}
}

@inproceedings{koh2017understanding,
  author       = {Koh, Pang Wei and Liang, Percy},
  title        = {Understanding black-box predictions via influence functions},
  booktitle    = {International conference on machine learning},
  organization = {PMLR},
  pages        = {1885--1894},
  year         = {2017}
}

@inproceedings{kopiczko2026data,
  title     = {Data Repetition Beats Data Scaling in Long-CoT Supervised Fine-Tuning},
  author    = {Dawid Jan Kopiczko and Sagar Vaze and Tijmen Blankevoort and Yuki M Asano},
  booktitle = {Third Conference on Language Modeling},
  year      = {2026},
  url       = {https://openreview.net/forum?id=SbmQuvutVt}
}

@inproceedings{kung2023active,
  author    = {Kung, Po-Nien and Yin, Fan and Wu, Di and Chang, Kai-Wei and Peng, Nanyun},
  title     = {Active Instruction Tuning: Improving Cross-Task Generalization by Training on Prompt Sensitive Tasks},
  booktitle = {Proceedings of the 2023 Conference on Empirical Methods in Natural Language Processing},
  pages     = {1813--1829},
  year      = {2023}
}

@article{kwiatkowski2019natural,
  author    = {Kwiatkowski, Tom and Palomaki, Jennimaria and Redfield, Olivia and Collins, Michael and Parikh, Ankur and Alberti, Chris and Epstein, Danielle and Polosukhin, Illia and Devlin, Jacob and Lee, Kenton and others},
  title     = {Natural questions: a benchmark for question answering research},
  journal   = {Transactions of the Association for Computational Linguistics},
  pages     = {453--466},
  publisher = {MIT Press},
  volume    = {7},
  year      = {2019}
}

@misc{kydlicek2025mathverify,
  author       = {Hynek Kydl{\'\i}{\v{c}}ek},
  title        = {{Math-Verify}: Math Verification Library},
  year         = {2025},
  howpublished = {\url{https://github.com/huggingface/Math-Verify}},
  note         = {Version 0.9.0}
}

@inproceedings{lambert2025tulu,
  title     = {Tulu 3: Pushing Frontiers in Open Language Model Post-Training},
  author    = {Nathan Lambert and Jacob Morrison and Valentina Pyatkin and Shengyi Huang and Hamish Ivison and Faeze Brahman and Lester James Validad Miranda and Alisa Liu and Nouha Dziri and Xinxi Lyu and Yuling Gu and Saumya Malik and Victoria Graf and Jena D. Hwang and Jiangjiang Yang and Ronan Le Bras and Oyvind Tafjord and Christopher Wilhelm and Luca Soldaini and Noah A. Smith and Yizhong Wang and Pradeep Dasigi and Hannaneh Hajishirzi},
  booktitle = {Second Conference on Language Modeling},
  year      = {2025},
  url       = {https://openreview.net/forum?id=i1uGbfHHpH}
}

@book{lan2020first,
  author    = {Lan, Guanghui},
  title     = {First-order and stochastic optimization methods for machine learning},
  publisher = {Springer},
  volume    = {1},
  year      = {2020}
}

@article{lange2000optimization,
  author    = {Lange, Kenneth and Hunter, David R and Yang, Ilsoon},
  title     = {Optimization transfer using surrogate objective functions},
  journal   = {Journal of computational and graphical statistics},
  number    = {1},
  pages     = {1--20},
  publisher = {Taylor \& Francis},
  volume    = {9},
  year      = {2000}
}

@book{lange2016mm,
  author    = {Lange, Kenneth},
  title     = {MM optimization algorithms},
  publisher = {SIAM},
  year      = {2016}
}

@inproceedings{lightman2024lets,
  title     = {Let's Verify Step by Step},
  author    = {Hunter Lightman and Vineet Kosaraju and Yuri Burda and Harrison Edwards and Bowen Baker and Teddy Lee and Jan Leike and John Schulman and Ilya Sutskever and Karl Cobbe},
  booktitle = {The Twelfth International Conference on Learning Representations},
  year      = {2024},
  url       = {https://openreview.net/forum?id=v8L0pN6EOi}
}

@article{ling2025domain,
  author    = {Ling, Chen and Zhao, Xujiang and Lu, Jiaying and Deng, Chengyuan and Zheng, Can and Wang, Junxiang and Chowdhury, Tanmoy and Li, Yun and Cui, Hejie and Zhang, Xuchao and others},
  title     = {Domain specialization as the key to make large language models disruptive: A comprehensive survey},
  journal   = {ACM Computing Surveys},
  number    = {3},
  pages     = {1--39},
  publisher = {ACM New York, NY},
  volume    = {58},
  year      = {2025}
}

@inproceedings{liu2023high,
  author       = {Liu, Zijian and Nguyen, Ta Duy and Nguyen, Thien Hang and Ene, Alina and Nguyen, Huy},
  title        = {High probability convergence of stochastic gradient methods},
  booktitle    = {International conference on machine learning},
  organization = {PMLR},
  pages        = {21884--21914},
  year         = {2023}
}

@article{loshchilov2015online,
  author  = {Loshchilov, Ilya and Hutter, Frank},
  title   = {Online batch selection for faster training of neural networks},
  journal = {arXiv preprint arXiv:1511.06343},
  year    = {2015}
}

@inproceedings{loshchilov2018decoupled,
  author    = {Ilya Loshchilov and Frank Hutter},
  title     = {Decoupled Weight Decay Regularization},
  booktitle = {International Conference on Learning Representations},
  url       = {https://openreview.net/forum?id=Bkg6RiCqY7},
  year      = {2019}
}

@inproceedings{lu2024instag,
  author    = {Keming Lu and Hongyi Yuan and Zheng Yuan and Runji Lin and Junyang Lin and Chuanqi Tan and Chang Zhou and Jingren Zhou},
  title     = {\#InsTag: Instruction Tagging for Analyzing Supervised Fine-tuning of Large Language Models},
  booktitle = {The Twelfth International Conference on Learning Representations},
  url       = {https://openreview.net/forum?id=pszewhybU9},
  year      = {2024}
}

@inproceedings{mairal2013optimization,
  author       = {Mairal, Julien},
  title        = {Optimization with first-order surrogate functions},
  booktitle    = {International conference on machine learning},
  organization = {PMLR},
  pages        = {783--791},
  year         = {2013}
}

@misc{melko2025topicreweighting,
  author = {Joseph Melkonian},
  title  = {Learning Which Data To Learn: The TerRIFIC Meta-Optimizer},
  year   = {2025},
  url    = {https://joemelko.github.io/blog.html?post=TopicReweighting}
}

@inproceedings{mindermann2022prioritized,
  author    = {Mindermann, S{\"o}ren and Brauner, Jan M and Razzak, Muhammed T and Sharma, Mrinank and Kirsch, Andreas and Xu, Winnie and H{\"o}ltgen, Benedikt and Gomez, Aidan N and Morisot, Adrien and Farquhar, Sebastian and Gal, Yarin},
  title     = {Prioritized Training on Points that are Learnable, Worth Learning, and not yet Learnt},
  booktitle = {Proceedings of the 39th International Conference on Machine Learning},
  series    = {Proceedings of Machine Learning Research},
  volume    = {162},
  pages     = {15630--15649},
  publisher = {PMLR},
  year      = {2022}
}

@inproceedings{mishra2022cross,
  author    = {Mishra, Swaroop and Khashabi, Daniel and Baral, Chitta and Hajishirzi, Hannaneh},
  title     = {Cross-task generalization via natural language crowdsourcing instructions},
  booktitle = {Proceedings of the 60th Annual Meeting of the Association for Computational Linguistics (Volume 1: Long Papers)},
  pages     = {3470--3487},
  year      = {2022}
}

@inproceedings{muennighoff2025generative,
  title     = {Generative representational instruction tuning},
  author    = {Muennighoff, Niklas and Su, Hongjin and Wang, Liang and Yang, Nan and Wei, Furu and Yu, Tao and Singh, Amanpreet and Kiela, Douwe},
  booktitle = {International Conference on Learning Representations},
  year      = {2025}
}

@inproceedings{muhamed2024grass,
  author    = {Muhamed, Aashiq and Li, Oscar and Woodruff, David and Diab, Mona and Smith, Virginia},
  title     = {GRASS: Compute Efficient Low-Memory LLM Training with Structured Sparse Gradients},
  booktitle = {Proceedings of the 2024 Conference on Empirical Methods in Natural Language Processing},
  pages     = {14978--15003},
  year      = {2024}
}

@article{nayak2026critical,
  title   = {A Critical Look at Targeted Instruction Selection: Disentangling What Matters (and What Doesn't)},
  author  = {Nayak, Nihal V and Rodriguez-Diaz, Paula and Hulkund, Neha and Beery, Sara and Alvarez-Melis, David},
  journal = {arXiv preprint arXiv:2602.14696},
  year    = {2026}
}

@book{nesterov2004introductory,
  author    = {Nesterov, Yurii},
  title     = {Introductory lectures on convex optimization: A basic course},
  publisher = {Kluwer Academic Publishers},
  series    = {Applied Optimization},
  volume    = {87},
  year      = {2004}
}

@article{olmo3,
  title   = {Olmo 3},
  author  = {{Olmo Team}},
  journal = {arXiv preprint arXiv:2512.13961},
  year    = {2025}
}

@article{ouyang2022training,
  author  = {Ouyang, Long and Wu, Jeffrey and Jiang, Xu and Almeida, Diogo and Wainwright, Carroll and Mishkin, Pamela and Zhang, Chong and Agarwal, Sandhini and Slama, Katarina and Ray, Alex and others},
  title   = {Training language models to follow instructions with human feedback},
  journal = {Advances in neural information processing systems},
  pages   = {27730--27744},
  volume  = {35},
  year    = {2022}
}

@article{pan2010survey,
  title     = {A survey on transfer learning},
  author    = {Pan, Sinno Jialin and Yang, Qiang},
  journal   = {IEEE Transactions on Knowledge and Data Engineering},
  volume    = {22},
  number    = {10},
  pages     = {1345--1359},
  year      = {2010},
  publisher = {IEEE}
}

@article{pandya2025sidda,
  author    = {Pandya, Sneh and Patel, Purvik and Nord, Brian D and Walmsley, Mike and {\'C}iprijanovi{\'c}, Aleksandra},
  title     = {SIDDA: SInkhorn Dynamic Domain Adaptation for image classification with equivariant neural networks},
  journal   = {Machine Learning: Science and Technology},
  number    = {3},
  pages     = {035032},
  publisher = {IOP Publishing},
  volume    = {6},
  year      = {2025}
}

@article{perez2021true,
  author  = {Perez, Ethan and Kiela, Douwe and Cho, Kyunghyun},
  title   = {True few-shot learning with language models},
  journal = {Advances in neural information processing systems},
  pages   = {11054--11070},
  volume  = {34},
  year    = {2021}
}

@article{polyak1963gradient,
  author  = {Polyak, Boris T},
  title   = {Gradient methods for the minimisation of functionals},
  journal = {USSR Computational Mathematics and Mathematical Physics},
  volume  = {3},
  number  = {4},
  pages   = {864--878},
  year    = {1963}
}

@article{pruthi2020estimating,
  author  = {Pruthi, Garima and Liu, Frederick and Kale, Satyen and Sundararajan, Mukund},
  title   = {Estimating training data influence by tracing gradient descent},
  journal = {Advances in Neural Information Processing Systems},
  pages   = {19920--19930},
  volume  = {33},
  year    = {2020}
}

@inproceedings{pyatkin2025generalizing,
  title     = {Generalizing Verifiable Instruction Following},
  author    = {Valentina Pyatkin and Saumya Malik and Victoria Graf and Hamish Ivison and Shengyi Huang and Pradeep Dasigi and Nathan Lambert and Hannaneh Hajishirzi},
  booktitle = {The Thirty-ninth Annual Conference on Neural Information Processing Systems Datasets and Benchmarks Track},
  year      = {2025},
  url       = {https://openreview.net/forum?id=yfYgwjj5F8}
}

@misc{qwen3technicalreport,
  author        = {{Qwen Team}},
  title         = {Qwen3 Technical Report},
  archiveprefix = {arXiv},
  eprint        = {2505.09388},
  primaryclass  = {cs.CL},
  url           = {https://arxiv.org/abs/2505.09388},
  year          = {2025}
}

@inproceedings{rajpurkar2016squad,
  author    = {Rajpurkar, Pranav and Zhang, Jian and Lopyrev, Konstantin and Liang, Percy},
  title     = {Squad: 100,000+ questions for machine comprehension of text},
  booktitle = {Proceedings of the 2016 conference on empirical methods in natural language processing},
  pages     = {2383--2392},
  year      = {2016}
}

@inproceedings{rasley2020deepspeed,
  author    = {Rasley, Jeff and Rajbhandari, Samyam and Ruwase, Olatunji and He, Yuxiong},
  title     = {Deepspeed: System optimizations enable training deep learning models with over 100 billion parameters},
  booktitle = {Proceedings of the 26th ACM SIGKDD international conference on knowledge discovery \& data mining},
  pages     = {3505--3506},
  year      = {2020}
}

@inproceedings{reimers2019sentence,
  author    = {Reimers, Nils and Gurevych, Iryna},
  title     = {Sentence-{BERT}: Sentence Embeddings using {S}iamese {BERT}-Networks},
  booktitle = {Proceedings of the 2019 Conference on Empirical Methods in Natural Language Processing and the 9th International Joint Conference on Natural Language Processing (EMNLP-IJCNLP)},
  pages     = {3982--3992},
  publisher = {Association for Computational Linguistics},
  year      = {2019}
}

@inproceedings{renduchintala2024tied,
  author    = {Renduchintala, Adithya and Konuk, Tugrul and Kuchaiev, Oleksii},
  title     = {Tied-lora: Enhancing parameter efficiency of lora with weight tying},
  booktitle = {Proceedings of the 2024 Conference of the North American Chapter of the Association for Computational Linguistics: Human Language Technologies (Volume 1: Long Papers)},
  pages     = {8694--8705},
  year      = {2024}
}

@article{schulman2017proximal,
  author  = {Schulman, John and Wolski, Filip and Dhariwal, Prafulla and Radford, Alec and Klimov, Oleg},
  title   = {Proximal policy optimization algorithms},
  journal = {arXiv preprint arXiv:1707.06347},
  year    = {2017}
}

@article{schulman2025lora,
  author  = {John Schulman and Thinking Machines Lab},
  title   = {LoRA Without Regret},
  doi     = {10.64434/tml.20250929},
  journal = {Thinking Machines Lab: Connectionism},
  note    = {https://thinkingmachines.ai/blog/lora/},
  year    = {2025}
}

@book{shalev2014understanding,
  author    = {Shalev-Shwartz, Shai and Ben-David, Shai},
  title     = {Understanding machine learning: From theory to algorithms},
  publisher = {Cambridge university press},
  year      = {2014}
}

@article{shao2024deepseekmath,
  author  = {Shao, Zhihong and Wang, Peiyi and Zhu, Qihao and Xu, Runxin and Song, Junxiao and Bi, Xiao and Zhang, Haowei and Zhang, Mingchuan and Li, YK and Wu, Yang and others},
  title   = {Deepseekmath: Pushing the limits of mathematical reasoning in open language models},
  journal = {arXiv preprint arXiv:2402.03300},
  year    = {2024}
}

@article{shen2026geometry,
  title   = {On the Geometry of On-Policy Distillation},
  author  = {Shen, Zhennan and Li, Yanshu and Yin, Qingyu and Leong, Chak Tou and Wang, Zhilin and Chen, Yanxu and Han, Rongduo and Lee, Sunbowen and Fung, Yi R},
  journal = {arXiv preprint arXiv:2606.07082},
  year    = {2026}
}

@article{sheng2023slora,
  author  = {Sheng, Ying and Cao, Shiyi and Li, Dacheng and Hooper, Coleman and Lee, Nicholas and Yang, Shuo and Chou, Christopher and Zhu, Banghua and Zheng, Lianmin and Keutzer, Kurt and Gonzalez, Joseph E. and Stoica, Ion},
  title   = {S-LoRA: Serving Thousands of Concurrent LoRA Adapters},
  journal = {arXiv preprint arXiv:2311.03285},
  year    = {2023}
}

@article{sheng2024hybridflow,
  author  = {Guangming Sheng and Chi Zhang and Zilingfeng Ye and Xibin Wu and Wang Zhang and Ru Zhang and Yanghua Peng and Haibin Lin and Chuan Wu},
  title   = {HybridFlow: A Flexible and Efficient RLHF Framework},
  journal = {arXiv preprint arXiv: 2409.19256},
  year    = {2024}
}

@article{sugiyama2007covariate,
  author  = {Sugiyama, Masashi and Krauledat, Matthias and M{\"u}ller, Klaus-Robert},
  title   = {Covariate shift adaptation by importance weighted cross validation},
  journal = {Journal of Machine Learning Research},
  volume  = {8},
  number  = {35},
  pages   = {985--1005},
  year    = {2007}
}

@inproceedings{vaswani2019fast,
  title        = {Fast and faster convergence of sgd for over-parameterized models and an accelerated perceptron},
  author       = {Vaswani, Sharan and Bach, Francis and Schmidt, Mark},
  booktitle    = {The 22nd international conference on artificial intelligence and statistics},
  pages        = {1195--1204},
  year         = {2019},
  organization = {PMLR}
}

@book{vershynin2018high,
  author    = {Vershynin, Roman},
  title     = {High-Dimensional Probability: An Introduction with Applications in Data Science},
  publisher = {Cambridge University Press},
  year      = {2018}
}

@inproceedings{vidgen2021lftw,
  author    = {Bertie Vidgen and Tristan Thrush and Zeerak Waseem and Douwe Kiela},
  title     = {Learning from the Worst: Dynamically Generated Datasets to Improve Online Hate Detection},
  booktitle = {ACL},
  year      = {2021}
}

@misc{vonwerra2022trl,
  author       = {Leandro von Werra and Younes Belkada and Lewis Tunstall and Edward Beeching and Tristan Thrush and Nathan Lambert and Shengyi Huang and Kashif Rasul and Quentin Gallouédec},
  title        = {TRL: Transformer Reinforcement Learning},
  howpublished = {\url{https://github.com/huggingface/trl}},
  journal      = {GitHub repository},
  publisher    = {GitHub},
  year         = {2020}
}

@article{wang2023far,
  author  = {Wang, Yizhong and Ivison, Hamish and Dasigi, Pradeep and Hessel, Jack and Khot, Tushar and Chandu, Khyathi and Wadden, David and MacMillan, Kelsey and Smith, Noah A and Beltagy, Iz and others},
  title   = {How far can camels go? exploring the state of instruction tuning on open resources},
  journal = {Advances in Neural Information Processing Systems},
  pages   = {74764--74786},
  volume  = {36},
  year    = {2023}
}

@article{wang2023multilora,
  author  = {Wang, Yiming and Lin, Yu and Zeng, Xiaodong and Zhang, Guannan},
  title   = {Multilora: Democratizing lora for better multi-task learning},
  journal = {arXiv preprint arXiv:2311.11501},
  year    = {2023}
}

@article{wang2024greats,
  author  = {Wang, Jiachen Tianhao and Wu, Tong and Song, Dawn and Mittal, Prateek and Jia, Ruoxi},
  title   = {Greats: Online selection of high-quality data for llm training in every iteration},
  journal = {Advances in Neural Information Processing Systems},
  pages   = {131197--131223},
  volume  = {37},
  year    = {2024}
}

@inproceedings{wang2024helpful,
  author    = {Jingtan Wang and Xiaoqiang Lin and Rui Qiao and Chuan-Sheng Foo and Bryan Kian Hsiang Low},
  title     = {Helpful or Harmful Data? Fine-tuning-free Shapley Attribution for Explaining Language Model Predictions},
  booktitle = {Forty-first International Conference on Machine Learning},
  url       = {https://openreview.net/forum?id=WSpPC1Jm0p},
  year      = {2024}
}

@inproceedings{wang2025nice,
  author    = {Wang, Jingtan and Lin, Xiaoqiang and Qiao, Rui and Koh, Pang Wei and Foo, Chuan-Sheng and Low, Bryan Kian Hsiang},
  title     = {{NICE} Data Selection for Instruction Tuning in {LLM}s with Non-differentiable Evaluation Metric},
  booktitle = {Proceedings of the 42nd International Conference on Machine Learning},
  series    = {Proceedings of Machine Learning Research},
  volume    = {267},
  pages     = {63662--63689},
  publisher = {PMLR},
  year      = {2025}
}

@inproceedings{wei2022finetuned,
  author    = {Jason Wei and Maarten Bosma and Vincent Zhao and Kelvin Guu and Adams Wei Yu and Brian Lester and Nan Du and Andrew M. Dai and Quoc V Le},
  title     = {Finetuned Language Models are Zero-Shot Learners},
  booktitle = {International Conference on Learning Representations},
  url       = {https://openreview.net/forum?id=gEZrGCozdqR},
  year      = {2022}
}

@article{xia2024chain,
  author  = {Xia, Wenhan and Qin, Chengwei and Hazan, Elad},
  title   = {Chain of lora: Efficient fine-tuning of language models via residual learning},
  journal = {arXiv preprint arXiv:2401.04151},
  year    = {2024}
}

@inproceedings{xia2024less,
  author       = {Xia, Mengzhou and Malladi, Sadhika and Gururangan, Suchin and Arora, Sanjeev and Chen, Danqi},
  title        = {LESS: Selecting Influential Data for Targeted Instruction Tuning},
  booktitle    = {International Conference on Machine Learning},
  organization = {PMLR},
  pages        = {54104--54132},
  year         = {2024}
}

@article{zhang2019root,
  author  = {Zhang, Biao and Sennrich, Rico},
  title   = {Root mean square layer normalization},
  journal = {Advances in neural information processing systems},
  volume  = {32},
  year    = {2019}
}

@article{zhang2023lora,
  author  = {Zhang, Longteng and Zhang, Lin and Shi, Shaohuai and Chu, Xiaowen and Li, Bo},
  title   = {{LoRA-FA}: Memory-efficient low-rank adaptation for large language models fine-tuning},
  journal = {arXiv preprint arXiv:2308.03303v1},
  year    = {2023}
}

@misc{zhang2024tinyllama,
  author        = {Peiyuan Zhang and Guangtao Zeng and Tianduo Wang and Wei Lu},
  title         = {TinyLlama: An Open-Source Small Language Model},
  archiveprefix = {arXiv},
  eprint        = {2401.02385},
  primaryclass  = {cs.CL},
  year          = {2024}
}

@inproceedings{zhao2024galore,
  author       = {Zhao, Jiawei and Zhang, Zhenyu and Chen, Beidi and Wang, Zhangyang and Anandkumar, Anima and Tian, Yuandong},
  title        = {GaLore: Memory-Efficient LLM Training by Gradient Low-Rank Projection},
  booktitle    = {International Conference on Machine Learning},
  organization = {PMLR},
  pages        = {61121--61143},
  year         = {2024}
}

@article{zhou2023instruction,
  title   = {Instruction-Following Evaluation for Large Language Models},
  author  = {Zhou, Jeffrey and Lu, Tianjian and Mishra, Swaroop and Brahma, Siddhartha and Basu, Sujoy and Luan, Yi and Zhou, Denny and Hou, Le},
  journal = {arXiv preprint arXiv:2311.07911},
  year    = {2023}
}

@inproceedings{zhou2023lima,
  author    = {Chunting Zhou and Pengfei Liu and Puxin Xu and Srini Iyer and Jiao Sun and Yuning Mao and Xuezhe Ma and Avia Efrat and Ping Yu and LILI YU and Susan Zhang and Gargi Ghosh and Mike Lewis and Luke Zettlemoyer and Omer Levy},
  title     = {{LIMA}: Less Is More for Alignment},
  booktitle = {Thirty-seventh Conference on Neural Information Processing Systems},
  url       = {https://openreview.net/forum?id=KBMOKmX2he},
  year      = {2023}
}

@inproceedings{nikdan2025efficient,
  author    = {Mahdi Nikdan and Vincent Cohen-Addad and Dan Alistarh and Vahab Mirrokni},
  title     = {Efficient Data Selection at Scale via Influence Distillation},
  booktitle = {The Thirty-ninth Annual Conference on Neural Information Processing Systems},
  url       = {https://openreview.net/forum?id=E6ZdfjtoiX},
  year      = {2025}
}
\bibliographystyle{iclr2027_conference}

\newpage
\appendix
\crefalias{section}{appendix}
\crefalias{subsection}{appendix}
\crefalias{subsubsection}{appendix}
\section{Implementation Details for LoRA and MeSO}\label{adxsec:implementation-LoRA-MeSO}
We briefly review the two low-dimensional parameter training methods mentioned in \Cref{sec:related-work}, LoRA and MeSO.

We use the notation of \Cref{subsec:standard-training}: \(l \in [L]\) indexes the linear layers, \(a_{i,\tau}^{(l)}\) and \(e_{i,\tau}^{(l)}\) are the input activation and the pre-activation of sample \(i\) at token \(\tau \in [\seqlen]\), and \texttt{Forward}, \texttt{Backward}, and \texttt{Release} are the primitives of \Cref{algo:standard-training}. For a linear layer \(l\), we write its weight matrix as \(W^{(l)} \in \mathbb{R}^{d_{\text{out}}^{(l)} \times d_{\text{in}}^{(l)}}\), where \(d_{\text{in}}^{(l)}\) and \(d_{\text{out}}^{(l)}\) denote the input and output dimensions; unlike \Cref{subsec:standard-training}, we do not assume \(d_{\text{in}}^{(l)} = d_{\text{out}}^{(l)} = \sqrt{d/L}\).

\subsection{LoRA}\label{adxsubsec:LoRA}
LoRA~\citep{hu2022lora} reparameterizes each linear layer \(l\) as \(W^{(l)} + B^{(l)}A^{(l)}\), where \(B^{(l)} \in \mathbb{R}^{d_{\text{out}}^{(l)} \times r}\) and \(A^{(l)} \in \mathbb{R}^{r \times d_{\text{in}}^{(l)}}\) are low-rank adapter matrices with \emph{rank} \(r \ll \min(d_{\text{in}}^{(l)}, d_{\text{out}}^{(l)})\). During training, the pretrained weight \(W^{(l)}\) is kept frozen, and only the adapter parameters \(A^{(l)}\) and \(B^{(l)}\) are updated. We collectively refer to these trainable parameters as the \emph{adapters}. Because the number of trainable parameters is much smaller than in full fine-tuning, LoRA substantially reduces memory usage: gradients and optimizer states (which can be several times larger than the gradients themselves) need to be stored only for the adapters rather than for the full weight matrix.

\subsection{MeSO}\label{adxsubsec:MeSO}
MeSO~\citep{zhao2024galore} maintains optimizer states in a compressed subspace of dimension \(\sum_l \kappa^{(l)}\) rather than the full parameter space \(\mathbb{R}^{d}\). At each layer \(l\), MeSO uses a projection \(\Pi^{(l)} \in \mathbb{R}^{\kappa^{(l)} \times d^{(l)}}\) (where \(d^{(l)} = d_{\text{in}}^{(l)} d_{\text{out}}^{(l)}\) is the number of parameters, \(\kappa^{(l)}\) is the compressed dimension for layer \(l\), and \(d = \sum_{l} d^{(l)}\)) to compress the vectorized gradient \(\vect(g^{(l)})\). The standard MeSO update is given in \Cref{algo:meso-standard}.

\begin{algorithm}[htpb]\label{algo:meso-standard}
	\DontPrintSemicolon{}
	\caption{Standard MeSO Training}
	\KwData{Model \(\theta_t\), training data \(\{z_i\}_{i=1}^n\), learning rate \(\eta_t\), projections \(\{\Pi^{(l)}\}_{l \in [L]}\)}
	\KwResult{Updated model parameters \(\theta_{t+1}\)}
	\SetKwFunction{Forward}{Forward}
	\SetKwFunction{Backward}{Backward}
	\SetKwFunction{Release}{Release}
	\BlankLine
	(\(\ell\), \(\{a_{i}^{(l)}, e_{i}^{(l)}\}_{l,i}\))\(\gets\)\Forward{\(\theta_t\), \(\{z_i\}_{i=1}^{n}\)}\;
	\For{\(l = L, \ldots, 1\)}{
		\(\{\partial \ell / \partial e_{i}^{(l)}\}_{i=1}^{n} \gets\)\Backward{\(\ell\), \(\theta_t^{(l)}\)}\;
		\(u_t^{(l)} \gets \frac{1}{n} \sum_{i=1}^{n} \sum_{\tau=1}^{\seqlen} (\partial \ell / \partial e_{i,\tau}^{(l)}) \otimes a_{i,\tau}^{(l)}\)\;
		\Release{\(a_i^{(l)}\), \(\partial \ell / \partial e_i^{(l)}\)} for all \(i\)\;
		\(\widetilde{u}_t^{(l)} \gets \Pi^{(l)}(u_t^{(l)})\)\Comment*[r]{Compress to \(\mathbb{R}^{\kappa^{(l)}}\)}
		\(\theta_{t+1}^{(l)} \gets \theta_t^{(l)} - \eta_t \cdot ((\Pi^{(l)})^\top \widetilde{u}_t^{(l)})\)\Comment*[r]{Back-project and update}
		\Release{\(\widetilde{u}_t^{(l)}\)}\;
	}
	\Return{\(\theta_{t+1}\)}\;
\end{algorithm}

Compared to LoRA, MeSO achieves memory savings through a different mechanism. LoRA reduces memory by restricting training to a small set of adapter parameters, so gradients and optimizer states are stored only for those parameters. In contrast, MeSO keeps the original parameterization of the model, but performs optimization in a low-dimensional subspace. As a result, gradient-related quantities such as optimizer states need only be maintained in the compressed space rather than in the full parameter space. This can lead to even greater memory savings, especially when the subspace dimension is smaller than the number of trainable parameters used by LoRA.\footnote{Technically, it is possible to combine LoRA and MeSO. However, this leads to instability and is not a common practice.}

\subsubsection{Subspace Refreshing}
We note that in many MeSO methods, each projection \(\Pi^{(l)}\) is periodically refreshed from \(\Pi_{\text{old}}\) to \(\Pi_{\text{new}}\) (we drop the layer index in what follows) to allow updating different subspaces, which requires transferring the optimizer's momentum states to the new subspace. Prior works omit this nuance by setting the momentum states to zero for simplicity~\citep{zhao2024galore}. Here, we provide a strategy for AdamW~\citep{loshchilov2018decoupled} to mitigate this.

\paragraph{First moment.}
For the \emph{first moment} \(\hat{m}_{\text{old}} \in \mathbb{R}^{\kappa^{(l)}}\), the transfer is straightforward: back-project to full space via \(\Pi^\top_{\text{old}}\), then forward-project via \(\Pi_{\text{new}}\).

\paragraph{Second moment.}
For the \emph{second moment} \(\hat{v}_{\text{old}} \in \mathbb{R}^{\kappa^{(l)}}\), a naive linear transformation is incorrect since variance is a quadratic quantity. In our implementation we consider a diagonal covariance approximation \(\hat{\Sigma}_{\text{old}} \approx \diag(\hat{v}_{\text{old}})\). This allows us to transform the second moment as:
\[
	\hat{v}_{\text{new}}
	= \diag(M \diag(\hat{v}_{\text{old}}) M^\top) = (M \odot M) \hat{v}_{\text{old}},
\]
where \(M = \Pi_{\text{new}} \Pi^\top_{\text{old}}\). When \(M\) is too large to form explicitly, this diagonal can be estimated stochastically via Hutchinson's method~\citep{hutchinson1989stochastic}: \(\hat{v}_{\text{new}} \approx \frac{1}{N_{\mathrm{probe}}} \sum_{j=1}^{N_{\mathrm{probe}}} r_j \odot (M \diag(\hat{v}_{\text{old}}) M^\top r_j)\) with Rademacher probes \(r_j \in \{-1,1\}^{\kappa^{(l)}}\).

\section{Omitted Details from \texorpdfstring{\Cref{sec:method}}{Section 3}}\label{adxsec:method}
This section collects the derivations and proofs deferred from \Cref{sec:method}. \Cref{adxsubsec:mm-derivation} derives the projection formulation (\Cref{eq:proj}) from the majorization--minimization relaxation and notes its optimizer-aware variant (\Cref{rmk:optimizer-aware}), \Cref{adxsubsec:global-subset-derivation} derives the objective of Global Subset Update, and \Cref{adxsubsec:group-wise-derivation} shows that the group-wise projection separates across groups. \Cref{adxsubsec:bias-variance-proof} proves the one-step results of \Cref{subsec:bias-variance-tradeoffs}: the descent lemma (\Cref{lma:majorization}) that reduces the comparison of methods to their MSE, the exact endpoints (\Cref{prop:bias-variance}), and the noisy projection oracle inequality (\Cref{thm:projection-oracle}); \Cref{adxsubsubsec:rates} derives the rates in the target batch size, \Cref{adxsubsubsec:two-bounds} compares the oracle inequality with the classical finite-class bound (\Cref{thm:bias-variance}) and unifies the two (\Cref{prop:unified}), and \Cref{adxsubsubsec:regimes} sketches the resulting regimes. \Cref{adxsubsec:trajectory-proofs} proves the trajectory results of \Cref{subsec:bias-variance-tradeoffs} (\Cref{thm:trajectory,cor:exact-convergence}) from a generic stochastic-descent argument.

For convenience, we restate the assumptions of \Cref{sec:method} and introduce the one additional assumption used only in this appendix.

\begin{assumption*}[Target smoothness, \Cref{ass:smooth}]
	The target population loss \(\mathcal{L}_\star\) is \(\beta\)-smooth for some \(\beta > 0\), and bounded below, \(\mathcal{L}_\star^{\inf} \coloneqq \inf_\theta \mathcal{L}_\star(\theta) > -\infty\).
\end{assumption*}

\begin{assumption*}[Sub-Gaussian noise, \Cref{ass:sub-Gaussian-noise}]
	The target gradient noise \(\xi \coloneqq \hat{g}_\star - g_\star\) is sub-Gaussian with parameter \(\sigma/\sqrt m\) for some \(\sigma > 0\), i.e., \(\mathbb{E}[\exp(\lambda \langle \xi, v \rangle) \mid \theta_t] \leq \exp(\lambda^2 \sigma^2 \lVert v \rVert^2 / (2m))\) for all \(v \in \mathbb{R}^d\) and \(\lambda \in \mathbb{R}\).
\end{assumption*}

\begin{assumption*}[Coverage and directional noise along the trajectory, \Cref{ass:coverage}]
	There are constants \(\rho \in [0, 1)\), \(\omega \geq 0\) and deterministic \(b_t, \varsigma_t \geq 0\) such that, almost surely for every \(t\), \(\mathcal{B}(U_t) \leq \rho \lVert g_\star(\theta_t) \rVert^2 + b_t^2\) and \(\mathcal{V}(U_t) \leq \omega \lVert g_\star(\theta_t) \rVert^2 + \varsigma_t^2\), where \(\mathcal{B}(U_t)\) and \(\mathcal{V}(U_t)\) are the bias and the exposed variance of \Cref{subsec:bias-variance-tradeoffs}, conditional on the history up to step \(t\).
\end{assumption*}

\Cref{subsec:bias-variance-tradeoffs} states this condition in the combined form \(\sqrt{\mathcal{B}(U_t)} + 2\sqrt{\mathcal{V}(U_t)} \leq q \lVert g_\star(\theta_t) \rVert + \phi_t\) with \(q \coloneqq \sqrt{\rho} + 2\sqrt{\omega}\) and \(\phi_t \coloneqq b_t + 2\varsigma_t\); \Cref{ass:coverage} implies it (\Cref{lma:relative-mse}), and the trajectory proofs use only this combined form.

\begin{assumption*}[Polyak--\L{}ojasiewicz condition~\citep{polyak1963gradient,karimi2016linear}, \Cref{ass:PL}]
	There is \(\mu > 0\) such that \(\lVert g_\star(\theta) \rVert^2 \geq 2\mu \bigl(\mathcal{L}_\star(\theta) - \mathcal{L}_\star^{\inf}\bigr)\) for all \(\theta \in \mathbb{R}^d\).
\end{assumption*}

\begin{assumption}[Bounded gradients]\label{ass:bounded}
	There exists \(C > 0\) such that \(\lVert g_i \rVert \leq C\) for all \(i \in [n]\).
\end{assumption}

\Cref{ass:smooth,ass:sub-Gaussian-noise} are used by the one-step results of \Cref{adxsubsec:bias-variance-proof}, \Cref{ass:coverage,ass:PL} only by the trajectory results of \Cref{adxsubsec:trajectory-proofs}, and \Cref{ass:bounded} only by the classical finite-class bound (\Cref{thm:bias-variance}).

\subsection{Derivation of the Projection Formulation}\label{adxsubsec:mm-derivation}
By \Cref{ass:smooth}, the target loss admits the quadratic majorization
\[
	\mathcal{L}_{\star}(\theta_t - \eta_t u)
	\leq \mathcal{L}_{\star}(\theta_t) - \eta_t \langle g_{\star}, u \rangle + \frac{\beta \eta_t^2}{2} \lVert u \rVert^2.
\]
Minimizing the upper bound on the right-hand side over \(u \in U_t\), dropping the constant \(\mathcal{L}_\star(\theta_t)\), and substituting the target batch estimate \(\hat{g}_\star\) for the unknown \(g_\star\) yields
\[
	u_t
	= \argmax_{u \in U_t} \left\{ \langle \hat{g}_\star, u \rangle - \frac{\beta \eta_t}{2} \lVert u \rVert^2 \right\}
	\iff u_t
	\in \argmin_{u \in U_t} \left\lVert u - \frac{\hat{g}_\star}{\beta \eta_t} \right\rVert^2,
\]
where the equivalence follows from completing the square:
\[
	\langle \hat{g}_\star, u \rangle - \frac{\beta\eta_t}{2}\lVert u \rVert^2
	= -\frac{\beta\eta_t}{2} \left\lVert u - \frac{1}{\beta\eta_t} \hat{g}_\star \right\rVert^2 + \frac{\lVert \hat{g}_\star \rVert^2}{2\beta\eta_t}.
\]
Setting \(\eta_t = 1/\beta\)~\citep{bottou2018optimization} absorbs the scaling, giving \Cref{eq:proj}.

\begin{remark}[Optimizer-aware selection]\label{rmk:optimizer-aware}
	The analysis treats \(u_t\) as the parameter step, whereas the experiments pass \(u_t\) to AdamW, as the compared selection methods do, so all arms are compared under a common optimizer. Freezing the moment estimates within a step linearizes the AdamW update to \(\eta_t P_t u_t\) with a diagonal preconditioner \(P_t\), and the same majorization then replaces \Cref{eq:proj} by \(\argmin_{u \in U_t} \lVert P_t u - \hat{g}_\star \rVert^2\), whose first-order scores are \(\langle g_i, P_t \hat{g}_\star \rangle\): optimizer-aware selection reuses the scoring machinery of \Cref{sec:system} with a preconditioned target gradient. We leave this simple extension as a future work.
\end{remark}

\subsection{Derivation of the Global Subset Objective}\label{adxsubsec:global-subset-derivation}
For the Global Subset Update with feasible set \(U_t^{\mathrm{glob}}(k) = \{ g_S = \frac{1}{k}\sum_{i \in S} g_i \colon S \subseteq [n], \lvert S\rvert = k \}\), \Cref{eq:proj} reduces to
\[
	S_t \in \argmin_{S \subseteq [n], \lvert S\rvert = k} \lVert g_S - \hat{g}_\star \rVert^2,
	\qquad u_t = g_{S_t}.
\]
Expanding the squared norm and dropping the term \(\lVert \hat{g}_\star \rVert^2\) (independent of \(S\)) gives the equivalent objective
\[
	\min_{S \subseteq [n], \lvert S\rvert = k} \lVert g_S \rVert^2 - 2\langle g_S, \hat{g}_\star \rangle .
\]
The two terms admit natural interpretations:
\begin{itemize}[leftmargin=*]
	\item \emph{Alignment term.} \(2\langle g_S, \hat{g}_\star \rangle = \frac{2}{k}\sum_{i \in S} \langle g_i, \hat{g}_\star \rangle\), favoring samples whose per-sample gradient aligns with the target gradient. This term decomposes additively into independent per-sample scores \(s_i \coloneqq \langle g_i, \hat{g}_\star \rangle\).
	\item \emph{Redundancy penalty.} \(\lVert g_S \rVert^2 = \frac{1}{k^2} \sum_{i, i^{\prime} \in S} \langle g_i, g_{i^{\prime}} \rangle\), discouraging selecting highly correlated training gradients.
\end{itemize}
The redundancy penalty couples all samples in \(S\), making exact optimization intractable; the selection rules used in practice (\Cref{subsubsec:existing-methods}) are tractable approximations of this combinatorial problem.

\subsection{Separability of the Group-Wise Projection}\label{adxsubsec:group-wise-derivation}
Because the group-wise feasible set factorizes as \(U_t^{\mathrm{grp}}(k; \mathcal{G}) = \prod_{p=1}^{P} U_t^{(p)}(k)\) and the squared Euclidean norm decomposes as \(\lVert v \rVert ^2 = \sum_{p=1}^{P} \lVert v^{(p)} \rVert ^2\), the projection problem in \Cref{eq:proj} decomposes into independent per-group subproblems:
\[
	\min_{u \in U_t^{\mathrm{grp}}(k; \mathcal{G})} \lVert u - \hat{g}_\star \rVert^2
	= \sum_{p=1}^{P} \min_{u^{(p)} \in U_t^{(p)}(k)} \lVert u^{(p)} - \hat{g}_\star^{(p)} \rVert^2.
\]
Equivalently, for each group \(p\), one chooses
\[
	S_{t,p} \in \argmin_{S \subseteq [n], \lvert S\rvert = k} \lVert g_{S}^{(p)} - \hat{g}_\star^{(p)} \rVert^2,
	\qquad
	u_t^{(p)} = g_{S_{t,p}}^{(p)} = \frac{1}{k}\sum_{i \in S_{t,p}} g_i^{(p)}.
\]
Because each per-group subproblem is structurally identical to the Global Subset Update objective (\Cref{adxsubsec:global-subset-derivation}), the same selection rules apply independently within each group, with the per-group scores \(s_i^{(p)} \coloneqq \langle g_i^{(p)}, \hat{g}_\star^{(p)} \rangle\).

\subsection{One-Step Analysis: Proofs}\label{adxsubsec:bias-variance-proof}
This section proves the one-step results of \Cref{subsec:bias-variance-tradeoffs}. Throughout, \(\theta_t\) is fixed and \(\MSE(u) \coloneqq \mathbb{E}[\lVert u - g_\star \rVert^2 \mid \theta_t]\), with expectation over \(B_t\) and \(B_t^\star\). We first prove the descent lemma that reduces one-step progress to the MSE.

\begin{lemma*}[Descent lemma, \Cref{lma:majorization}]
	Assume \Cref{ass:smooth}. Fix \(\theta_t\) and let \(0 < \eta_t \leq 1/\beta\). For any update \(u\) with \(\mathbb{E}[\lVert u \rVert ^2 \mid \theta_t] < \infty\) and \(\theta_{t+1} = \theta_t - \eta_t u\),
	\[
		\mathbb{E}\left[\mathcal{L}_\star(\theta_{t+1}) \mid \theta_t\right]
		\leq \mathcal{L}_\star(\theta_t) - \frac{\eta_t}{2} \lVert g_\star \rVert ^2 + \frac{\eta_t}{2} \MSE(u).
	\]
\end{lemma*}
\begin{proof}
	By \Cref{ass:smooth} and linearity of expectation,
	\[
		\mathbb{E}[\mathcal{L}_\star(\theta_{t+1}) \mid \theta_t]
		\leq \mathcal{L}_\star(\theta_t) - \eta_t \langle g_\star, \mathbb{E}[u \mid \theta_t] \rangle + \frac{\beta \eta_t^2}{2} \mathbb{E}[\lVert u \rVert ^2 \mid \theta_t].
	\]
	Expanding \(\MSE(u) = \mathbb{E}[\lVert u \rVert ^2 \mid \theta_t] - 2\langle g_\star, \mathbb{E}[u \mid \theta_t]\rangle + \lVert g_\star \rVert ^2\) and rearranging gives
	\[
		-\langle g_\star, \mathbb{E}[u \mid \theta_t] \rangle + \frac{1}{2} \mathbb{E}[\lVert u \rVert ^2 \mid \theta_t]
		= -\frac{1}{2} \lVert g_\star \rVert ^2 + \frac{1}{2} \MSE(u).
	\]
	Substituting and using \(\eta_t \leq 1/\beta\) to drop the nonpositive remainder \((\frac{\beta \eta_t^2}{2} - \frac{\eta_t}{2})\mathbb{E}[\lVert u \rVert ^2 \mid \theta_t] \leq 0\) yields the result.
\end{proof}

Recall the decomposition \(\MSE(u_t) = \mathcal{B}(U_t) + \mathcal{E}(u_t)\) of \Cref{subsec:bias-variance-tradeoffs}, where \(\mathcal{B}(U_t) \coloneqq \mathbb{E}[\inf_{u \in U_t} \lVert u - g_\star \rVert^2 \mid \theta_t]\) is the bias and \(\mathcal{E}(u_t) \coloneqq \MSE(u_t) - \mathcal{B}(U_t)\) is the excess error of selecting within \(U_t\) with the noisy estimate \(\hat{g}_\star\), and \(\Sigma_{\mathrm{tr}} \coloneqq \Cov(g_i \mid \theta_t)\) and \(\Sigma_\star \coloneqq \Cov(g_j^\star \mid \theta_t)\). For the two endpoints the decomposition is exact.

\begin{proposition*}[Bias and excess error at the endpoints, \Cref{prop:bias-variance}]
	Fix \(\theta_t\) and let \(\eta_t = 1/\beta\). Then:
	\begin{enumerate}[label=(\roman*),leftmargin=*]
		\item \textbf{Full-Training Update.}
		      For \(U_t^{\mathrm{tr}}=\{\hat{g}_{\mathrm{tr}}\}\),
		      \[
			      \mathcal{B}(U_t^{\mathrm{tr}})
			      = \lVert g_{\mathrm{tr}} - g_\star \rVert^2 + \frac{\tr(\Sigma_{\mathrm{tr}})}{n},\qquad
			      \mathcal{E}(u_t^{\mathrm{tr}})
			      = 0.
		      \]
		\item \textbf{Target-Only Update.}
		      For \(U_t^\star=\mathbb{R}^d\),
		      \[
			      \mathcal{B}(U_t^{\star})
			      = 0,\qquad
			      \mathcal{E}(u_t^{\star})
			      = \frac{\tr(\Sigma_\star)}{m}.
		      \]
	\end{enumerate}
\end{proposition*}
\begin{proof}
	We prove this one by one.
	\begin{enumerate}[label=(\roman*),leftmargin=*]
		\item \emph{Part~\labelcref{prop:bias-variance-i} (Full-Training Update).}
		      Since \(u_t^{\mathrm{tr}} = \hat{g}_{\mathrm{tr}}\) has mean \(g_{\mathrm{tr}}\) and covariance \(\Sigma_{\mathrm{tr}} / n\),
		      \[
			      \MSE(u_t^{\mathrm{tr}})
			      = \lVert \mathbb{E}[\hat{g}_{\mathrm{tr}} - g_\star \mid \theta_t] \rVert^2 + \tr(\Cov(\hat{g}_{\mathrm{tr}} - g_\star \mid \theta_t))
			      = \lVert g_{\mathrm{tr}} - g_\star \rVert ^2 + \frac{1}{n} \tr(\Sigma_{\mathrm{tr}}).
		      \]
		      Furthermore, note that since \(U_t^{\mathrm{tr}} = \{\hat{g}_{\mathrm{tr}}\}\), we also have \(\mathcal{B}(U_t^{\mathrm{tr}}) = \MSE(u_t^{\mathrm{tr}})\).
		\item \emph{Part~\labelcref{prop:bias-variance-ii} (Target-Only Update).}
		      With \(u_t^{\star} = \hat{g}_\star\) and \(\hat{g}_\star - g_\star\) has mean zero and covariance \(\Sigma_\star / m\),
		      \[
			      \MSE(u_t^{\star})
			      = \mathbb{E}[\lVert \hat{g}_\star - g_\star\rVert ^2 \mid \theta_t]
			      = \tr(\Cov(\hat{g}_\star \mid \theta_t))
			      = \frac{1}{m} \tr(\Sigma_\star).
		      \]
		      We conclude the proof by noting that \(\mathcal{B}(U_t^{\star}) = 0\) as \(U_t^{\star} = \mathbb{R}^d\).
	\end{enumerate}
\end{proof}

We next prove the noisy projection oracle inequality, with the update-difference subspace \(D_t = D(U_t)\), where \(D(U) \coloneqq \Span\{u - u^{\prime} \colon u, u^{\prime} \in U\}\) for any \(U \subseteq \mathbb{R}^d\), and the exposed variance \(\mathcal{V}(U_t)\) as defined in \Cref{subsec:bias-variance-tradeoffs}. Throughout, condition on \(\theta_t\); the training batch \(B_t\) determines \(U_t\) and hence \(D(U_t)\), and the target batch \(B_t^\star\) is independent of \(B_t\) with \(\mathbb{E}[\xi \mid B_t, \theta_t] = 0\) and \(\Cov(\xi \mid B_t, \theta_t) = \Sigma_\star / m\). The key observation is that the projection is blind to noise outside the update-difference subspace.

\begin{lemma}\label{lma:orthogonal-noise}
	For any nonempty closed \(U \subseteq \mathbb{R}^d\) and any \(g, \xi \in \mathbb{R}^d\),
	\[
		\argmin_{u \in U} \lVert u - g - \xi \rVert^2
		= \argmin_{u \in U} \lVert u - g - \Proj_{D(U)}(\xi) \rVert^2 .
	\]
\end{lemma}
\begin{proof}
	For \(u, v \in U\), \(\lVert u - g - \xi \rVert^2 - \lVert v - g - \xi \rVert^2 = \lVert u - g \rVert^2 - \lVert v - g \rVert^2 - 2\langle u - v, \xi \rangle\), and \(\langle u - v, \xi \rangle = \langle u - v, \Proj_{D(U)}(\xi) \rangle\) because \(u - v \in D(U)\). Every pairwise comparison of the two objectives is therefore unchanged when \(\xi\) is replaced by \(\Proj_{D(U)}(\xi)\), so the two minimizer sets coincide.
\end{proof}

\begin{theorem*}[Noisy projection oracle inequality, \Cref{thm:projection-oracle}]
	Fix \(\theta_t\), and let \(U_t\) be any nonempty closed feasible set determined by \(B_t\). The update \(u_t\) of \Cref{eq:proj} satisfies
	\[
		\sqrt{\MSE(u_t)} \leq \sqrt{\mathcal{B}(U_t)} + 2\sqrt{\mathcal{V}(U_t)},
		\qquad \text{equivalently} \qquad
		\mathcal{E}(u_t) \leq 4\sqrt{\mathcal{B}(U_t)\,\mathcal{V}(U_t)} + 4\mathcal{V}(U_t),
	\]
	where \(\mathcal{V}(U_t) = \mathbb{E}[\tr(\Proj_{D(U_t)} \Sigma_\star) \mid \theta_t] / m\). Moreover:
	\begin{enumerate}[label=(\roman*),leftmargin=*]
		\item \textbf{Orderings.} If \(U_t \subseteq U_t^{\prime}\) are two feasible sets determined by \(B_t\), then \(\mathcal{B}(U_t^{\prime}) \leq \mathcal{B}(U_t)\) and \(\mathcal{V}(U_t) \leq \mathcal{V}(U_t^{\prime})\). In particular, for any \(k \in [n]\) and partition \(\mathcal{G}\), \(\mathcal{B}(U_t^{\mathrm{grp}}(k; \mathcal{G})) \leq \mathcal{B}(U_t^{\mathrm{glob}}(k))\) and \(\mathcal{V}(U_t^{\mathrm{glob}}(k)) \leq \mathcal{V}(U_t^{\mathrm{grp}}(k; \mathcal{G}))\), while \(\mathcal{V}(U_t^{\mathrm{tr}}) = 0\) and \(\mathcal{V}(U_t^{\star}) = \tr(\Sigma_\star)/m\).
		\item \textbf{Exposed-variance bounds.} Under \Cref{ass:sub-Gaussian-noise}, for any \(k \in [n]\) and partition \(\mathcal{G}\) with \(P\) groups, \(\mathcal{V}(U_t^{\mathrm{glob}}(k)) \leq (n-1)\sigma^2/m\) and \(\mathcal{V}(U_t^{\mathrm{grp}}(k; \mathcal{G})) \leq \sigma^2 \sum_{p=1}^{P} \min\{\lvert G_p \rvert, n-1\} / m \leq P(n-1)\sigma^2/m\).
	\end{enumerate}
\end{theorem*}
\begin{proof}
	Fix \(B_t\) (hence \(U_t\)) and \(\xi\), and write \(\Proj \coloneqq \Proj_{D(U_t)}\) and \(\zeta \coloneqq \Proj(\xi)\). Let \(u^\dagger \in \argmin_{u \in U_t} \lVert u - g_\star \rVert^2\) be the oracle choice. By \Cref{lma:orthogonal-noise}, \(u_t \in \argmin_{u \in U_t} \lVert u - g_\star - \zeta \rVert^2\), so \(\lVert u_t - g_\star - \zeta \rVert \leq \lVert u^\dagger - g_\star - \zeta \rVert\) because \(u^\dagger \in U_t\) is a competitor; combined with the triangle inequality on both sides,
	\[
		\lVert u_t - g_\star \rVert
		\leq \lVert u_t - g_\star - \zeta \rVert + \lVert \zeta \rVert
		\leq \lVert u^\dagger - g_\star - \zeta \rVert + \lVert \zeta \rVert
		\leq \lVert u^\dagger - g_\star \rVert + 2 \lVert \zeta \rVert .
	\]
	This holds for every realization of \((B_t, B_t^\star)\). For a random vector \(X\), write \(\lVert X \rVert_{L^2} \coloneqq \sqrt{\mathbb{E}[\lVert X \rVert^2 \mid \theta_t]}\) for its root-mean-square size over both batches. This map is monotone (\(0 \leq A \leq B\) pointwise implies \(\lVert A \rVert_{L^2} \leq \lVert B \rVert_{L^2}\)) and satisfies the triangle inequality (Minkowski's inequality), so applying it to both sides of the display gives
	\[
		\sqrt{\MSE(u_t)}
		= \lVert u_t - g_\star \rVert_{L^2}
		\leq \lVert u^\dagger - g_\star \rVert_{L^2} + 2 \lVert \zeta \rVert_{L^2}
		= \sqrt{\mathcal{B}(U_t)} + 2\sqrt{\mathcal{V}(U_t)} ,
	\]
	where \(\mathbb{E}[\lVert u^\dagger - g_\star \rVert^2 \mid \theta_t] = \mathbb{E}[\inf_{u \in U_t} \lVert u - g_\star \rVert^2 \mid \theta_t] = \mathcal{B}(U_t)\) because \(u^\dagger\) attains the infimum, and \(\mathbb{E}[\lVert \zeta \rVert^2 \mid \theta_t] = \mathcal{V}(U_t)\) by definition. Squaring both (nonnegative) sides expands the right-hand side to \(\mathcal{B}(U_t) + 4\sqrt{\mathcal{B}(U_t)\mathcal{V}(U_t)} + 4\mathcal{V}(U_t)\), and subtracting \(\mathcal{B}(U_t)\) gives \(\mathcal{E}(u_t) = \MSE(u_t) - \mathcal{B}(U_t) \leq 4\sqrt{\mathcal{B}(U_t)\mathcal{V}(U_t)} + 4\mathcal{V}(U_t)\); adding \(\mathcal{B}(U_t)\) back and taking square roots returns the first form, so the two are equivalent. For the covariance form of \(\mathcal{V}(U_t)\), condition on \(B_t\): \(D(U_t)\) is then a fixed subspace, so \(\Proj\) is a fixed matrix with \(\Proj^{\top} = \Proj\) and \(\Proj^2 = \Proj\), while \(\xi\) has conditional mean zero and conditional covariance \(\Sigma_\star / m\) because the target batch is independent of \(B_t\). Hence
	\begin{align*}
		\mathbb{E}[\lVert \Proj(\xi) \rVert^2 \mid B_t, \theta_t]
		 & = \mathbb{E}[\tr(\Proj \xi \xi^{\top} \Proj) \mid B_t, \theta_t]
		= \tr(\Proj\, \mathbb{E}[\xi \xi^{\top} \mid B_t, \theta_t]\, \Proj) \\
		 & = \tr(\Proj \Sigma_\star \Proj) / m
		= \tr(\Proj \Sigma_\star) / m ,
	\end{align*}
	using \(\lVert v \rVert^2 = \tr(v v^{\top})\), linearity of the trace and of the expectation, and cyclicity of the trace together with \(\Proj^2 = \Proj\); taking the expectation over \(B_t\) yields \(\mathcal{V}(U_t) = \mathbb{E}[\tr(\Proj_{D(U_t)} \Sigma_\star) \mid \theta_t] / m\).

	\emph{Part~\labelcref{thm:projection-oracle-i}.} For \(U_t \subseteq U_t^{\prime}\), the bias ordering holds because the infimum over a larger set is smaller, and \(D(U_t) \subseteq D(U_t^{\prime})\), so the difference of the two orthogonal projectors is itself an orthogonal projector, hence positive semidefinite, and \(\tr\bigl((\Proj_{D(U_t^{\prime})} - \Proj_{D(U_t)}) \Sigma_\star\bigr) \geq 0\) as the trace of a product of two positive semidefinite matrices. The global and group-wise sets are nested by \Cref{subsubsec:group-wise-decomposition}. For the singleton \(U_t^{\mathrm{tr}}\) we have \(D(U_t^{\mathrm{tr}}) = \{0\}\), and for \(U_t^{\star} = \mathbb{R}^d\) we have \(\Proj = I\).

	\emph{Part~\labelcref{thm:projection-oracle-ii}.} Conditional on \(B_t\), let \(v_1, \dots, v_{\dim D(U_t)}\) be an orthonormal basis of \(D(U_t)\). By \Cref{ass:sub-Gaussian-noise} each \(\langle \xi, v_j \rangle\) is sub-Gaussian with variance proxy \(\sigma^2/m\), and a sub-Gaussian variable has second moment at most its variance proxy, so \(\mathbb{E}[\lVert \Proj(\xi) \rVert^2 \mid B_t, \theta_t] = \sum_{j=1}^{\dim D(U_t)} \mathbb{E}[\langle \xi, v_j \rangle^2 \mid B_t, \theta_t] \leq \dim D(U_t) \cdot \sigma^2/m\). It remains to bound \(\dim D(U_t)\). Any two elements of \(U_t^{\mathrm{glob}}(k)\) differ by \(g_S - g_{S^{\prime}} = \sum_{i=1}^{n} c_i g_i\) with \(\sum_i c_i = 0\), so \(D(U_t^{\mathrm{glob}}(k)) \subseteq \Span\{g_i - g_1 \colon 2 \leq i \leq n\}\), whose dimension is at most \(n - 1\). For the product set, \(U_t^{\mathrm{grp}}(k; \mathcal{G}) - U_t^{\mathrm{grp}}(k; \mathcal{G}) = \prod_{p} \bigl(U_t^{(p)}(k) - U_t^{(p)}(k)\bigr)\), so \(D(U_t^{\mathrm{grp}}(k; \mathcal{G})) = \bigoplus_{p=1}^{P} D(U_t^{(p)}(k))\) with each summand supported on the coordinates \(G_p\), and the same argument bounds its dimension by \(\min\{\lvert G_p \rvert, n-1\}\).
\end{proof}

\begin{remark}[Non-convex feasible sets]\label{rmk:nonconvex}
	The proof of \Cref{thm:projection-oracle} compares \(u_t\) with the oracle choice \(u^\dagger\) only through the minimality of \(u_t\), so it applies to the discrete, non-convex feasible sets of \Cref{subsec:designs-of-feasible-set}. For these sets the projection need not be unique and is not \(1\)-Lipschitz in its input, since an arbitrarily small change of \(\hat{g}_\star\) can switch the selected subset; the standard argument through the non-expansiveness of projections onto convex sets is therefore not available.
\end{remark}

\paragraph{Bias ordering.}
The bias inequality \(\mathcal{B}(U_t^{\mathrm{grp}}(k; \mathcal{G})) \leq \mathcal{B}(U_t^{\mathrm{glob}}(k))\) follows directly from \(U_t^{\mathrm{glob}}(k) \subseteq U_t^{\mathrm{grp}}(k; \mathcal{G})\). More generally, for two partitions \(\mathcal{G}, \mathcal{G}^{\prime}\) where \(\mathcal{G}^{\prime}\) refines \(\mathcal{G}\), we have \(U_t^{\mathrm{grp}}(k; \mathcal{G}) \subseteq U_t^{\mathrm{grp}}(k; \mathcal{G}^{\prime})\), and hence \(\mathcal{B}(U_t^{\mathrm{grp}}(k; \mathcal{G}^{\prime})) \leq \mathcal{B}(U_t^{\mathrm{grp}}(k; \mathcal{G}))\). For different subset sizes \(k \neq k^{\prime}\), no inclusion relation generally holds between \(U_t^{\mathrm{glob}}(k)\) and \(U_t^{\mathrm{glob}}(k^{\prime})\) (or between \(U_t^{\mathrm{grp}}(k; \mathcal{G})\) and \(U_t^{\mathrm{grp}}(k^{\prime}; \mathcal{G})\)), so bias is generally not monotonic in \(k\). Intuitively, however, larger feasible sets give lower bias: when \(k = n\), \(U_t^{\mathrm{tr}} = U_t^{\mathrm{glob}}(n) = U_t^{\mathrm{grp}}(n; \mathcal{G})\) degenerates to a singleton; when \(k = n/2\), \(U_t^{\mathrm{glob}}(k)\) admits \(\binom{n}{n/2}\) candidates for approximating \(g_\star\), which boosts the chance of a good approximation and hence yields a lower bias.

\subsubsection{Rates in the Target Batch Size}\label{adxsubsubsec:rates}
Squaring the oracle inequality of \Cref{thm:projection-oracle} bounds the excess error by
\[
	\mathcal{E}(u_t)
	= \MSE(u_t) - \mathcal{B}(U_t)
	\leq 4\sqrt{\mathcal{B}(U_t)\,\mathcal{V}(U_t)} + 4\mathcal{V}(U_t) .
\]
The two quantities on the right scale differently with \(m\). Under \Cref{ass:sub-Gaussian-noise}, the proof of \Cref{thm:projection-oracle}\labelcref{thm:projection-oracle-ii} gives \(\mathcal{V}(U_t) \leq d_t \sigma^2 / m\) with \(d_t \coloneqq \mathbb{E}[\dim D(U_t) \mid \theta_t] \leq \sum_{p} \min\{\lvert G_p \rvert, n - 1\}\), so \(\mathcal{V}(U_t) = O(1/m)\); this bound depends on the feasible set only through the dimension of its update-difference subspace, not on \(k\) or on the cardinality \(\binom{n}{k}\), and it vanishes at \(k = n\), where the feasible set is a singleton. The bias \(\mathcal{B}(U_t)\), in contrast, does not depend on \(m\) at all: it is determined by the training batch and the training--target mismatch. In the realizable case \(\mathcal{B}(U_t) = 0\) the cross term vanishes and \(\mathcal{E}(u_t) \leq 4\mathcal{V}(U_t) = O(1/m)\), the rate of Target-Only Update in \Cref{prop:bias-variance}. When \(\mathcal{B}(U_t) > 0\), the cross term \(4\sqrt{\mathcal{B}(U_t)\mathcal{V}(U_t)}\) is \(O(1/\sqrt{m})\) whereas \(4\mathcal{V}(U_t)\) is \(O(1/m)\), so the cross term is the larger of the two as soon as \(\mathcal{V}(U_t) < \mathcal{B}(U_t)\), i.e., once \(m\) exceeds about \(d_t \sigma^2 / \mathcal{B}(U_t)\), and the bound on the excess error decays only as \(1/\sqrt{m}\). This slower rate is a property of the bound rather than of the excess error itself: when the candidates in \(U_t\) are separated from the oracle choice by a fixed margin, a mis-selection requires the noise to overturn that margin, an event of probability exponentially small in \(m\); the \(1/\sqrt{m}\) rate is the worst case over near-ties, which is the relevant regime when many per-sample gradients are similar.

\subsubsection{Two Bounds on the Excess Error}\label{adxsubsubsec:two-bounds}
\Cref{fig:one-step-error} illustrates the excess error: the update chosen with the noisy estimate \(\hat{g}_\star\) may differ from the oracle choice within \(U_t\), and a larger feasible set gives the update more freedom to align with \(g_\star\) but also with the noise \(\xi\). The classical finite-class argument controls this freedom through the cardinality of the feasible set under a gradient-norm bound (\Cref{ass:bounded})~\citep{shalev2014understanding,hazan2016introduction,lan2020first,liu2023high}, whereas \Cref{thm:projection-oracle} uses only the noise components inside the update-difference subspace, whose energy is bounded by the dimension of that subspace. We state the classical bound, compare the two, and then show that the oracle inequality can be refined to pay, group by group, the smaller of the two complexity measures.

\begin{theorem}[Finite-class excess-error bounds]\label{thm:bias-variance}
	Assume \Cref{ass:smooth,ass:bounded,ass:sub-Gaussian-noise} and fix \(\theta_t\) and let \(\eta_t = 1/\beta\). Then for any \(k \in [n]\):
	\begin{enumerate}[label=(\roman*),leftmargin=*]
		\item\label{thm:bias-variance-i} \textbf{Global Subset Update.}
		      For \(U_t^{\mathrm{glob}}(k)\),
		      \[
			      \mathcal{E}(u_t^{\mathrm{glob}}(k))
			      \leq \frac{4C\sigma}{\sqrt m} \sqrt{2\log\left(2\tbinom{n}{k}\right)}.
		      \]
		\item\label{thm:bias-variance-ii} \textbf{Group-Wise Subset Update.}
		      For \(U_t^{\mathrm{grp}}(k; \mathcal{G})\) where \(\mathcal{G} = \{G_p\}_{p=1}^{P}\) contains \(P\) groups,
		      \[
			      \mathcal{E}(u_t^{\mathrm{grp}}(k; \mathcal{G}))
			      \leq \frac{4CP\sigma}{\sqrt m} \sqrt{2\log\left(2\tbinom{n}{k}\right)}.
		      \]
	\end{enumerate}
\end{theorem}
\begin{proof}
	We prove this one by one.
	\begin{enumerate}[label=(\roman*),leftmargin=*]
		\item \emph{Part~\labelcref{thm:bias-variance-i} (Global Subset Update).}
		      Fix \(k\) and write \(U \coloneqq U_t^{\mathrm{glob}}(k)\). Let \(\xi \coloneqq \hat{g}_\star - g_\star\). Define
		      \[
			      u^{\ast} \coloneqq u_t^{\mathrm{glob}} \in \argmin_{u \in U} \lVert u - \hat{g}_\star \rVert ^2, \qquad
			      u^\dagger \in \argmin_{u \in U} \lVert u - g_\star \rVert ^2.
		      \]
		      By the projection property, \(\lVert u^\ast - \hat{g}_\star \rVert ^2 \leq \lVert u^\dagger - \hat{g}_\star \rVert ^2\). Expanding both sides via \(\hat{g}_\star = g_\star + \xi\):
		      \[
			      \lVert u^\ast - g_\star \rVert ^2 - 2\langle u^\ast - g_\star, \xi \rangle
			      \leq \lVert u^\dagger - g_\star \rVert ^2 - 2\langle u^\dagger - g_\star, \xi \rangle.
		      \]
		      Rearranging, we get
		      \[
			      \lVert u^\ast - g_\star \rVert ^2
			      \leq \lVert u^\dagger - g_\star \rVert ^2 + 2\langle u^\ast - u^\dagger, \xi \rangle
			      \leq \lVert u^\dagger - g_\star \rVert ^2 + 4 \sup_{u \in U} \lvert \langle \xi, u \rangle \rvert .
		      \]
		      Taking expectation over \(B_t\) and \(B_t^{\star}\) (recall that \(\MSE\) is the expectation over both), we have
		      \[
			      \MSE(u_t^{\mathrm{glob}})
			      \leq \mathbb{E}[\lVert u^\dagger - g_\star \rVert ^2 \mid \theta_t] + 4 \mathbb{E}\left[\sup_{u \in U} \lvert \langle \xi, u \rangle \rvert \middle| \theta_t\right],
		      \]
		      The first term is simply \(\mathcal{B}(U) \coloneqq \mathbb{E}[\inf_{u \in U} \lVert u - g_\star \rVert^2 \mid \theta_t ]\). Hence, we have
		      \[
			      \mathcal{E}(u_t^{\mathrm{glob}})
			      = \MSE(u_t^{\mathrm{glob}}) - \mathcal{B}(U)
			      \leq 4 \mathbb{E}\left[\sup_{u \in U} \lvert \langle \xi, u \rangle \rvert \middle| \theta_t\right],
		      \]
		      and it remains to bound \(\mathbb{E}[\sup_{u \in U} \lvert \langle \xi, u \rangle \rvert \mid \theta_t]\). Condition on \(B_t\), so that \(U\) is a fixed finite set independent of \(\xi\). By the sub-Gaussian assumption on the target noise, for each fixed \(u\) the random variable \(\langle \xi, u \rangle\) is sub-Gaussian with variance proxy \(\sigma^2 \lVert u \rVert ^2 / m\), i.e.,
		      \[
			      \mathbb{E}\left[\exp(\lambda \langle \xi, u \rangle) \middle| \theta_t\right]
			      \leq \exp \left(\frac{\lambda^2 \sigma^2}{2m} \lVert u \rVert ^2\right), \quad \forall \lambda \in \mathbb{R}.
		      \]
		      By \Cref{ass:bounded} and the triangle inequality (every \(u \in U\) is an average of \(k\) per-sample gradients), \(\lVert u \rVert \leq C\) for all \(u \in U\). The sub-Gaussian maximal inequality \citep{vershynin2018high} states that for \(\lvert U \rvert\) mean-zero sub-Gaussian random variables \(X_1, \ldots, X_{\lvert U \rvert}\) with common variance proxy \(\nu^2\), \(\mathbb{E}[\max_i \lvert X_i \rvert] \leq \nu \sqrt{2 \log(2 \lvert U \rvert)}\). Applying this conditionally on \(B_t\) with \(\lvert U \rvert \leq \binom{n}{k}\) and \(\nu^2 = \sigma^2 C^2 / m\):
		      \[
			      \mathbb{E}\left[\sup_{u \in U} \lvert\langle \xi, u \rangle\rvert \middle| B_t, \theta_t\right]
			      \leq \frac{\sigma C}{\sqrt{m}} \sqrt{2 \log(2\lvert U \rvert)}.
		      \]
		      Taking expectation over \(B_t\) finishes the proof.

		\item \emph{Part~\labelcref{thm:bias-variance-ii} (Group-Wise Subset Update).}
		      The argument follows Part~\labelcref{thm:bias-variance-i}. Write \(U \coloneqq U_t^{\mathrm{grp}}(k; \mathcal{G})\). By \Cref{ass:bounded}, each per-group component satisfies \(\lVert u^{(p)}\rVert \leq C\), so \(\lVert u \rVert ^2 \leq PC^2\). Hence, the same argument yields
		      \[
			      \mathcal{E}(u_t^{\mathrm{grp}})
			      \leq \frac{4C\sqrt{P} \sigma}{\sqrt{m}} \sqrt{2 \log(2\lvert U \rvert)}.
		      \]
		      The cardinality bound follows from the product structure of \(U\): each group \(p\) chooses its own \(S_{t,p} \subset [n]\) with \(\lvert S_{t,p}\rvert = k\), so \(\lvert U\rvert \leq \binom{n}{k}^P\). Substituting: \(\log(2\binom{n}{k}^P) = \log 2 + P\log\binom{n}{k} \leq P\log(2\binom{n}{k})\) for \(P \geq 1\), so \(C\sqrt{P} \cdot \sqrt{2P\log(2\binom{n}{k})} = CP\sqrt{2\log(2\binom{n}{k})}\).
	\end{enumerate}
\end{proof}

\paragraph{Comparison.}
\Cref{thm:projection-oracle,thm:bias-variance} bound the same excess error \(\mathcal{E}(u_t)\) through two different complexity measures of the feasible set, the dimension of its update-difference subspace and the logarithm of its cardinality, and neither bound dominates the other. The finite-class bound has no cross term, but it needs the gradient-norm bound \(C\) and decays as \(1/\sqrt{m}\) even when \(\mathcal{B}(U_t) = 0\). The oracle inequality needs no gradient-norm bound and decays as \(1/m\) in that case, but it carries the cross term \(4\sqrt{\mathcal{B}(U_t)\mathcal{V}(U_t)}\); by Young's inequality the cross term can be traded for an inflated bias coefficient, \(\MSE(u_t) \leq (1 + \epsilon)\mathcal{B}(U_t) + 4(1 + \epsilon^{-1})\mathcal{V}(U_t)\) for any \(\epsilon > 0\), which is an additive bound of the same form as \Cref{thm:bias-variance}. The two complexity measures are not ordered either: per group, the dimension is at most \(n - 1\) while the log-cardinality term of \Cref{prop:unified} below is \(2\log\binom{n}{k} + 2\), so the dimension is the smaller measure for our \(n = 8\), \(k = 4\) (\(7\) versus about \(10.5\)), whereas the log-cardinality is smaller for large candidate batches with small subsets.

\paragraph{Unification.}
The oracle-inequality argument can be refined so that every group pays the smaller of the two measures.

\begin{proposition}[Unified excess-error bound]\label{prop:unified}
	Under \Cref{ass:sub-Gaussian-noise}, for \(U_t = U_t^{\mathrm{grp}}(k; \mathcal{G})\) with \(P\) groups (\(P = 1\) for Global Subset Update),
	\[
		\sqrt{\MSE(u_t)}
		\leq \sqrt{\mathcal{B}(U_t)} + 2\sqrt{\frac{\sigma^2}{m} \sum_{p=1}^{P} \min\Bigl\{ \mathbb{E}[\dim D(U_t^{(p)}(k)) \mid \theta_t],\ 2 \log \tbinom{n}{k} + 2 \Bigr\}} .
	\]
\end{proposition}
\begin{proof}
	For each group \(p\), let \(v_p\) be the oracle choice in \(U_t^{(p)}(k)\), and let \(Z_p\) be the largest normalized noise component along a candidate difference,
	\[
		Z_p
		\coloneqq \max\left\{ 0,\ \max_{w \in U_t^{(p)}(k) \setminus \{v_p\}} \frac{\langle \xi^{(p)}, w - v_p \rangle}{\lVert w - v_p \rVert} \right\} .
	\]
	Per-group nearest-point optimality of \(u_t^{(p)}\) against \(v_p\), expanded as in the proof of \Cref{thm:bias-variance}, gives
	\begin{align*}
		\lVert u_t^{(p)} - g_\star^{(p)} \rVert^2 - \lVert v_p - g_\star^{(p)} \rVert^2
		 & \leq 2 \langle u_t^{(p)} - v_p, \xi^{(p)} \rangle                                                      \\
		 & \leq 2 Z_p \lVert u_t^{(p)} - v_p \rVert                                                               \\
		 & \leq 2 Z_p \bigl( \lVert u_t^{(p)} - g_\star^{(p)} \rVert + \lVert v_p - g_\star^{(p)} \rVert \bigr) ,
	\end{align*}
	and dividing by \(\lVert u_t^{(p)} - g_\star^{(p)} \rVert + \lVert v_p - g_\star^{(p)} \rVert\) (the case \(u_t^{(p)} = v_p\) being trivial) yields
	\[
		\lVert u_t^{(p)} - g_\star^{(p)} \rVert
		\leq \lVert v_p - g_\star^{(p)} \rVert + 2 Z_p .
	\]
	Summing the squares over groups, Minkowski's inequality in \(\ell_2\) and then the \(L^2\) step from the proof of \Cref{thm:projection-oracle} give
	\[
		\sqrt{\MSE(u_t)}
		\leq \sqrt{\mathcal{B}(U_t)} + 2\sqrt{\mathbb{E}\left[\sum_{p=1}^{P} Z_p^2 \Bigm| \theta_t\right]} ,
	\]
	so it remains to bound \(\mathbb{E}[Z_p^2 \mid B_t, \theta_t]\), which we do in two ways. First, since \(w - v_p \in D(U_t^{(p)}(k))\), the Cauchy--Schwarz inequality gives \(Z_p \leq \lVert \Proj_{D(U_t^{(p)}(k))}(\xi^{(p)}) \rVert\), and as in Part~\labelcref{thm:projection-oracle-ii} of \Cref{thm:projection-oracle},
	\[
		\mathbb{E}[Z_p^2 \mid B_t, \theta_t]
		\leq \dim D(U_t^{(p)}(k))\, \frac{\sigma^2}{m} .
	\]
	Second, each \(\langle \xi^{(p)}, w - v_p \rangle / \lVert w - v_p \rVert\) is sub-Gaussian with variance proxy \(\sigma^2 / m\) by \Cref{ass:sub-Gaussian-noise}, so a union bound over the \(N_p \coloneqq \lvert U_t^{(p)}(k) \rvert \leq \binom{n}{k}\) candidates gives \(\Pr(Z_p \geq z \mid B_t, \theta_t) \leq N_p \exp(-m z^2 / (2\sigma^2))\) for \(z > 0\), and integrating the tail,
	\begin{align*}
		\mathbb{E}[Z_p^2 \mid B_t, \theta_t]
		 & = \int_0^\infty \Pr(Z_p^2 \geq s \mid B_t, \theta_t)\, \mathrm{d}s
		\leq \int_0^\infty \min\bigl\{ 1, N_p e^{-m s / (2\sigma^2)} \bigr\}\, \mathrm{d}s \\
		 & = \frac{\sigma^2}{m} \bigl( 2 \log N_p + 2 \bigr) .
	\end{align*}
	Taking the smaller of the two bounds for each group and then the expectation over \(B_t\) completes the proof.
\end{proof}

\Cref{prop:unified} recovers the \(\log\binom{n}{k}\) dependence of \Cref{thm:bias-variance} whenever it is the smaller measure. The exposed-variance form of \Cref{subsec:bias-variance-tradeoffs} therefore subsumes the log-cardinality dependence of the classical analysis, while retaining the \(1/m\) rate when \(\mathcal{B}(U_t) = 0\) and requiring no gradient-norm bound.

\subsubsection{Regimes}\label{adxsubsubsec:regimes}
Which update dominates depends on the target batch size, the training--target mismatch, and the feasible-set complexity. \Cref{fig:spectrum-view} sketches three representative regimes with synthetic parameter choices for the bounds of \Cref{prop:bias-variance,thm:projection-oracle,thm:bias-variance}:
\begin{itemize}[leftmargin=*]
	\item \textbf{Small mismatch.} All four methods occupy distinct regimes. As \(m\) grows, the preferred update transitions from the Full-Training Update to the Global Subset Update, then to the Group-Wise Subset Update, and finally to the Target-Only Update.
	\item \textbf{Moderate mismatch.} The Group-Wise Subset Update dominates the Global Subset Update across a wider range of \(m\), yielding a transition from the Full-Training Update to the Group-Wise Subset Update and then to the Target-Only Update.
	\item \textbf{Large mismatch.} The bias of full-training and Global Subset Updates is large enough that the Group-Wise Subset Update dominates even for relatively small \(m\), before eventually giving way to the Target-Only Update when enough target data are available.
\end{itemize}
The pools of \Cref{subsubsec:SFT-capability} span these regimes: the in-distribution \texttt{Dolci} instruction and reasoning pools, the mixed pool, and \texttt{Tulu 3}, in which the target skill is most diluted.

\begin{figure}[htpb]
	\centering
	\begin{tikzpicture}

    \pgfplotsset{
        spectrumaxis/.style={
                width=0.24\linewidth,
                height=0.16\linewidth,
                scale only axis,
                xmin=0, xmax=300,
                ymin=-0.05, ymax=2.6,
                grid=both,
                major grid style={opacity=0.22},
                minor grid style={opacity=0.10},
                tick align=outside,
                xticklabels=\empty,
                yticklabels=\empty,
                major tick length=2pt,
                label style={font=\small},
                xlabel style={at={(ticklabel cs:0.5)}, yshift=8pt},
                ylabel style={at={(ticklabel cs:0.5)}, yshift=-8pt},
                title style={font=\scriptsize, yshift=-4pt},
                title style={font=\scriptsize},
                clip=true,
                unbounded coords=discard,
                every axis plot/.append style={
                        thick,
                        line join=round,
                        line cap=round
                    },
            }
    }

    \definecolor{cTr}{RGB}{0,0,0}            
    \definecolor{cSub}{RGB}{0,114,178}      
    \definecolor{cBlk}{RGB}{213,94,0}        
    \definecolor{cVal}{RGB}{0,158,115}       

    \pgfmathsetmacro{\XMAX}{300}
    \pgfmathsetmacro{\Vtr}{1.15}
    \pgfmathsetmacro{\Tval}{180}

    \begin{axis}[
            name=ax1,
            spectrumaxis,
            at={(0,0)},
            anchor=south west,
            xlabel={$m$},
            ylabel={$\MSE$},
            title={$u_t^{\mathrm{tr}} \to u_t^{\mathrm{glob}} \to u_t^{\mathrm{grp}} \to u_t^{\star}$},
            legend to name=speclegend,
            legend columns=1,
            legend cell align=left,
            legend style={
                    draw=none,
                    font=\small,
                    row sep=1pt,
                    /tikz/every odd column/.append style={column sep=2pt},
                },
        ]

        \def\VtrA{0.95}          
        \def\egA{0.72}
        \def\AgA{1.29}
        \def\elA{0.62}
        \def\AlA{2.13}

        \def\mTG{30}
        \def\mGL{70}
        \def\mLV{236}

        \path[fill=cTr,   fill opacity=0.025] (axis cs:0,-0.05)       rectangle (axis cs:\mTG,2.6);
        \path[fill=cSub, fill opacity=0.030] (axis cs:\mTG,-0.05)    rectangle (axis cs:\mGL,2.6);
        \path[fill=cBlk,  fill opacity=0.030] (axis cs:\mGL,-0.05)    rectangle (axis cs:\mLV,2.6);
        \path[fill=cVal,  fill opacity=0.030] (axis cs:\mLV,-0.05)    rectangle (axis cs:\XMAX,2.6);

        \addplot[color=cTr] coordinates {(0,\VtrA) (\XMAX,\VtrA)};
        \addlegendentry{$U_t^{\mathrm{tr}}$}

        \addplot[color=cSub] expression[domain=0:\XMAX, samples=600]
            {\egA + \AgA/sqrt(x+1)};
        \addlegendentry{$U_t^{\mathrm{glob}}$}

        \addplot[color=cBlk]  expression[domain=0:\XMAX, samples=600]
            {\elA + \AlA/sqrt(x+1)};
        \addlegendentry{$U_t^{\mathrm{grp}}$}

        \addplot[color=cVal, restrict y to domain=-0.05:2.6]
        expression[domain=0:\XMAX, samples=800]{\Tval/(x+1)};
        \addlegendentry{$U_t^{\star}$}

        \addplot[color=cTr, ultra thick] coordinates {(0,\VtrA) (\mTG,\VtrA)};
        \addplot[color=cSub, ultra thick] expression[domain=\mTG:\mGL, samples=220]
            {\egA + \AgA/sqrt(x+1)};
        \addplot[color=cBlk, ultra thick] expression[domain=\mGL:\mLV, samples=220]
            {\elA + \AlA/sqrt(x+1)};
        \addplot[color=cVal, ultra thick, restrict y to domain=-0.05:2.6]
        expression[domain=\mLV:\XMAX, samples=220]{\Tval/(x+1)};

        \addplot[color=cTr!60, dashed, very thin] coordinates {(\mTG,-0.05) (\mTG,2.6)};
        \addplot[color=cTr!60, dashed, very thin] coordinates {(\mGL,-0.05) (\mGL,2.6)};
        \addplot[color=cTr!60, dashed, very thin] coordinates {(\mLV,-0.05) (\mLV,2.6)};

    \end{axis}

    \begin{axis}[
            name=ax2,
            spectrumaxis,
            at={(ax1.east)},
            anchor=west,
            xshift=0.025\linewidth,
            xlabel={$m$},
            ylabel={},
            title={$u_t^{\mathrm{tr}} \to u_t^{\mathrm{grp}} \to u_t^{\star}$},
        ]

        \def\egB{0.90}
        \def\AgB{1.29}
        \def\elB{0.666}
        \def\AlB{2.13}

        \def\mTG{18.4}
        \def\mLV{222}

        \path[fill=cTr,  fill opacity=0.025] (axis cs:0,-0.05)     rectangle (axis cs:\mTG,2.6);
        \path[fill=cBlk, fill opacity=0.030] (axis cs:\mTG,-0.05)  rectangle (axis cs:\mLV,2.6);
        \path[fill=cVal, fill opacity=0.030] (axis cs:\mLV,-0.05)  rectangle (axis cs:\XMAX,2.6);

        \addplot[color=cTr] coordinates {(0,\Vtr) (\XMAX,\Vtr)};
        \addplot[color=cSub] expression[domain=0:\XMAX, samples=600]
            {\egB + \AgB/sqrt(x+1)};
        \addplot[color=cBlk] expression[domain=0:\XMAX, samples=600]
            {\elB + \AlB/sqrt(x+1)};
        \addplot[color=cVal, restrict y to domain=-0.05:2.6]
        expression[domain=0:\XMAX, samples=800]{\Tval/(x+1)};

        \addplot[color=cTr, ultra thick] coordinates {(0,\Vtr) (\mTG,\Vtr)};
        \addplot[color=cBlk, ultra thick] expression[domain=\mTG:\mLV, samples=240]
            {\elB + \AlB/sqrt(x+1)};
        \addplot[color=cVal, ultra thick, restrict y to domain=-0.05:2.6]
        expression[domain=\mLV:\XMAX, samples=240]{\Tval/(x+1)};

        \addplot[color=cTr!60, dashed, very thin] coordinates {(\mTG,-0.05) (\mTG,2.6)};
        \addplot[color=cTr!60, dashed, very thin] coordinates {(\mLV,-0.05) (\mLV,2.6)};

    \end{axis}

    \begin{axis}[
            name=ax3,
            spectrumaxis,
            at={(ax2.east)},
            anchor=west,
            xshift=0.025\linewidth,
            xlabel={$m$},
            ylabel={},
            title={$u_t^{\mathrm{grp}} \to u_t^{\star}$},
        ]

        \def\VtrC{2.2}          
        \def\egC{1.40}
        \def\AgC{1.29}
        \def\elC{0.65}
        \def\AlC{2.13}

        \def\mLV{226}

        \path[fill=cBlk, fill opacity=0.030] (axis cs:0,-0.05)     rectangle (axis cs:\mLV,2.6);
        \path[fill=cVal, fill opacity=0.030] (axis cs:\mLV,-0.05)  rectangle (axis cs:\XMAX,2.6);

        \addplot[color=cTr] coordinates {(0,\VtrC) (\XMAX,\VtrC)};
        \addplot[color=cSub] expression[domain=0:\XMAX, samples=600]
            {\egC + \AgC/sqrt(x+1)};
        \addplot[color=cBlk] expression[domain=0:\XMAX, samples=600]
            {\elC + \AlC/sqrt(x+1)};
        \addplot[color=cVal, restrict y to domain=-0.05:2.6]
        expression[domain=0:\XMAX, samples=800]{\Tval/(x+1)};

        \addplot[color=cBlk, ultra thick] expression[domain=0:\mLV, samples=240]
            {\elC + \AlC/sqrt(x+1)};
        \addplot[color=cVal, ultra thick, restrict y to domain=-0.05:2.6]
        expression[domain=\mLV:\XMAX, samples=240]{\Tval/(x+1)};

        \addplot[color=cTr!60, dashed, very thin] coordinates {(\mLV,-0.05) (\mLV,2.6)};

    \end{axis}

    \node[anchor=west] at ([xshift=3mm]ax3.east) {%
        \pgfplotslegendfromname{speclegend}%
    };

\end{tikzpicture}
	\caption{Illustration of the MSE bounds of \Cref{prop:bias-variance,thm:projection-oracle,thm:bias-variance} as a function of the target batch size \(m\): Full-Training is constant, Global and Group-Wise Subset Updates decay as \(1/\sqrt{m}\) (Group-Wise with the larger excess-error constant but the smaller bias), and Target-Only decays as \(1/m\). The excess-error constants are shared across panels; only the bias terms grow with the distribution mismatch (left to right). Shaded regions indicate the method with the smallest MSE; as the mismatch increases, Group-Wise Subset Update dominates over a wider range of \(m\).}
	\label{fig:spectrum-view}
\end{figure}

\subsection{Trajectory Analysis}\label{adxsubsec:trajectory-proofs}
This section proves the trajectory results of \Cref{subsec:bias-variance-tradeoffs}, \Cref{thm:trajectory,cor:exact-convergence}. The arguments are standard descent arguments~\citep{ghadimi2013stochastic,karimi2016linear,ajalloeian2020convergence}; what is specific to our setting is that the required one-step control is supplied by the projection geometry through \Cref{thm:projection-oracle}.

\paragraph{Setting.}
Let \(\mathcal{F}_t\) be the \(\sigma\)-algebra generated by the batches \((B_s, B_s^\star)_{s < t}\), with \(\theta_0\) deterministic, so that \(\theta_t\) is \(\mathcal{F}_t\)-measurable and \((B_t, B_t^\star)\) is independent of \(\mathcal{F}_t\). Write \(\mathbb{E}_t[\cdot] \coloneqq \mathbb{E}[\cdot \mid \mathcal{F}_t]\) and \(g_{\star, t} \coloneqq g_\star(\theta_t)\). Because the batches at step \(t\) are fresh, every conditional quantity of \Cref{subsec:bias-variance-tradeoffs}, written there as conditional on \(\theta_t\), coincides with its \(\mathcal{F}_t\)-conditional version: \(\MSE(u_t) = \mathbb{E}_t[\lVert u_t - g_{\star, t} \rVert^2] = \mathcal{B}(U_t) + \mathcal{E}(u_t)\), and \Cref{lma:majorization,thm:projection-oracle} hold with \(\mathbb{E}_t\) in place of \(\mathbb{E}[\cdot \mid \theta_t]\). Step sizes are deterministic, all second moments are finite, and the projection is selected measurably with a fixed tie-breaking rule. We write \(\Lambda_T \coloneqq \sum_{t < T} \eta_t\) and \(\Delta_t \coloneqq \mathbb{E}[\mathcal{L}_\star(\theta_t)] - \mathcal{L}_\star^{\inf}\), so that \(\Delta_0 = \mathcal{L}_\star(\theta_0) - \mathcal{L}_\star^{\inf}\), as in \Cref{subsec:bias-variance-tradeoffs}. For convenience, we restate the two results.

\begin{theorem*}[Trajectory bound, \Cref{thm:trajectory}]
	Under \Cref{ass:smooth,ass:coverage} with \(q < 1\) and deterministic step sizes \(0 < \eta_t \leq 1/\beta\), for every \(T \geq 1\),
	\[
		\frac{\sum_{t < T} \eta_t\, \mathbb{E}[\lVert g_{\star, t} \rVert^2]}{\Lambda_T}
		\leq \frac{4 \Delta_0}{(1 - q^2) \Lambda_T}
		+ \frac{2(1 + q^2)}{(1 - q^2)^2} \cdot \frac{\sum_{t < T} \eta_t (b_t + 2\varsigma_t)^2}{\Lambda_T}.
	\]
\end{theorem*}

\begin{corollary*}[Exact and whole-sequence convergence, \Cref{cor:exact-convergence}]
	Under \Cref{ass:smooth,ass:coverage} with \(b_t = \varsigma_t = 0\), \(q < 1\), and a constant step size \(0 < \eta \leq 1/\beta\), \(\frac{1}{T} \sum_{t < T} \mathbb{E}[\lVert g_{\star, t} \rVert^2] \leq 2 \Delta_0 / (\eta (1 - q^2) T)\) and \(g_{\star, t} \to 0\) almost surely. If in addition \Cref{ass:PL} holds, then \(\Delta_T \leq (1 - \mu \eta (1 - q^2))^T \Delta_0\), and \(\theta_t\) converges almost surely, with finite path length \(\sum_t \lVert \theta_{t+1} - \theta_t \rVert < \infty\), to a random \(\theta_\infty\) with \(\mathcal{L}_\star(\theta_\infty) = \mathcal{L}_\star^{\inf}\) (\Cref{thm:finite-length}).
\end{corollary*}

We first turn \Cref{ass:coverage} into a relative bound on the MSE of the projected update.

\begin{lemma}[Relative MSE from coverage and exposed variance]\label{lma:relative-mse}
	Under \Cref{ass:coverage}, with \(q\) and \(\phi_t\) as defined after it, almost surely
	\(\sqrt{\MSE(u_t)} \leq q \lVert g_{\star, t} \rVert + \phi_t\); in particular \(\MSE(u_t) \leq q^2 \lVert g_{\star, t} \rVert^2\) when \(\phi_t = 0\), and for \(q < 1\)
	\[
		\MSE(u_t)
		\leq \frac{1 + q^2}{2} \lVert g_{\star, t} \rVert^2 + \frac{1 + q^2}{1 - q^2} \phi_t^2 .
	\]
\end{lemma}
\begin{proof}
	\Cref{thm:projection-oracle} (with \(\mathbb{E}_t\)) and \Cref{ass:coverage} give \(\sqrt{\MSE(u_t)} \leq \sqrt{\rho \lVert g_{\star, t} \rVert^2 + b_t^2} + 2\sqrt{\omega \lVert g_{\star, t} \rVert^2 + \varsigma_t^2} \leq q \lVert g_{\star, t} \rVert + \phi_t\) by \(\sqrt{x + y} \leq \sqrt{x} + \sqrt{y}\). If \(q = 0\), the last display is immediate. Otherwise, squaring and applying Young's inequality \(2 q \lVert g_{\star, t} \rVert \phi_t \leq \epsilon q^2 \lVert g_{\star, t} \rVert^2 + \epsilon^{-1} \phi_t^2\) with \(\epsilon = (1 - q^2)/(2 q^2)\) gives \((1 + \epsilon) q^2 = (1 + q^2)/2\) and \(1 + \epsilon^{-1} = (1 + q^2)/(1 - q^2)\).
\end{proof}

\begin{proof}[Proof of \Cref{thm:trajectory}]
	Substituting \Cref{lma:relative-mse} into \Cref{lma:majorization} (with \(\mathbb{E}_t\)) gives, for \(0 < \eta_t \leq 1/\beta\),
	\[
		\mathbb{E}_t[\mathcal{L}_\star(\theta_{t+1})]
		\leq \mathcal{L}_\star(\theta_t) - \frac{\eta_t}{2} \cdot \frac{1 - q^2}{2} \lVert g_{\star, t} \rVert^2 + \frac{\eta_t}{2} \cdot \frac{1 + q^2}{1 - q^2} \phi_t^2 .
	\]
	Taking expectations, summing over \(t < T\), and using \(\mathcal{L}_\star(\theta_T) \geq \mathcal{L}_\star^{\inf}\),
	\(\frac{1 - q^2}{4} \sum_{t < T} \eta_t \mathbb{E}[\lVert g_{\star, t} \rVert^2] \leq \Delta_0 + \frac{1 + q^2}{2 (1 - q^2)} \sum_{t < T} \eta_t \phi_t^2\); multiplying by \(4 / ((1 - q^2) \Lambda_T)\) yields the claim.
\end{proof}

\begin{proof}[Proof of \Cref{cor:exact-convergence}]
	With \(\phi_t = 0\), \Cref{lma:relative-mse} gives \(\MSE(u_t) \leq q^2 \lVert g_{\star, t} \rVert^2\), and \Cref{lma:majorization} gives \(\mathbb{E}_t[\mathcal{L}_\star(\theta_{t+1})] \leq \mathcal{L}_\star(\theta_t) - \frac{\eta (1 - q^2)}{2} \lVert g_{\star, t} \rVert^2\). Telescoping yields \(\sum_{t < T} \mathbb{E}[\lVert g_{\star, t} \rVert^2] \leq 2 \Delta_0 / (\eta (1 - q^2))\) for every \(T\), which is the stated rate; since the right-hand side does not depend on \(T\), \(\mathbb{E}[\sum_t \lVert g_{\star, t} \rVert^2] < \infty\), so by monotone convergence \(\sum_t \lVert g_{\star, t} \rVert^2 < \infty\) and hence \(g_{\star, t} \to 0\) almost surely. Under \Cref{ass:PL}, substituting \(\lVert g_{\star, t} \rVert^2 \geq 2\mu(\mathcal{L}_\star(\theta_t) - \mathcal{L}_\star^{\inf})\) gives \(\Delta_{t+1} \leq \chi \Delta_t\) with \(\chi \coloneqq 1 - \mu \eta (1 - q^2)\), and \(\chi \in [0, 1)\) because \(\eta \leq 1/\beta\) and \(\mu \leq \beta\) (\Cref{ass:smooth,ass:PL} together imply \(\mu \leq \beta\)), so \(\Delta_T \leq \chi^T \Delta_0\). Whole-sequence convergence is \Cref{thm:finite-length} below.
\end{proof}

\begin{theorem}[Almost-sure finite-length convergence]\label{thm:finite-length}
	Under the conditions of \Cref{cor:exact-convergence} and \Cref{ass:PL}, there is a random \(\theta_\infty \in \mathbb{R}^d\) such that \(\sum_{t \geq 0} \lVert \theta_{t+1} - \theta_t \rVert < \infty\) and \(\theta_t \to \theta_\infty\) almost surely, with \(\mathcal{L}_\star(\theta_\infty) = \mathcal{L}_\star^{\inf}\) almost surely, and
	\[
		\mathbb{E}[\lVert \theta_T - \theta_\infty \rVert]
		\leq \frac{\eta (1 + q) \sqrt{2 \beta \Delta_0}}{1 - \sqrt{\chi}}\, \chi^{T/2},
		\qquad \chi = 1 - \mu \eta (1 - q^2) .
	\]
	No uniqueness of the minimizer, isolatedness of stationary points, or convexity of \(\mathcal{L}_\star\) or of \(U_t\) is required.
\end{theorem}
\begin{proof}
	By the \(L^2\) triangle inequality and \Cref{lma:relative-mse}, \(\sqrt{\mathbb{E}_t[\lVert u_t \rVert^2]} \leq \lVert g_{\star, t} \rVert + \sqrt{\MSE(u_t)} \leq (1 + q) \lVert g_{\star, t} \rVert\). Smoothness and boundedness from below give the self-bound \(\lVert g_\star(\theta) \rVert^2 \leq 2\beta (\mathcal{L}_\star(\theta) - \mathcal{L}_\star^{\inf})\) (apply the descent inequality to the step \(\theta - g_\star(\theta)/\beta\)). Hence, by Jensen's inequality and the proof of \Cref{cor:exact-convergence},
	\[
		\mathbb{E}[\lVert \theta_{t+1} - \theta_t \rVert]
		= \eta\, \mathbb{E}[\lVert u_t \rVert]
		\leq \eta (1 + q) \sqrt{\mathbb{E}[\lVert g_{\star, t} \rVert^2]}
		\leq \eta (1 + q) \sqrt{2 \beta \Delta_t}
		\leq \eta (1 + q) \sqrt{2 \beta \Delta_0}\, \chi^{t/2} .
	\]
	The right-hand side is summable, so \(\mathbb{E}[\sum_t \lVert \theta_{t+1} - \theta_t \rVert] < \infty\) by Tonelli's theorem; thus the path length is finite almost surely, \((\theta_t)\) is almost surely Cauchy, and it converges to some \(\theta_\infty\). Since \(\mathcal{L}_\star(\theta_t) - \mathcal{L}_\star^{\inf} \geq 0\) converges to \(\mathcal{L}_\star(\theta_\infty) - \mathcal{L}_\star^{\inf}\) by continuity while \(\Delta_t \to 0\), Fatou's lemma gives \(\mathbb{E}[\mathcal{L}_\star(\theta_\infty) - \mathcal{L}_\star^{\inf}] = 0\). Summing the displayed estimate from \(T\) onward proves the tail bound (for \(\chi = 0\) with the usual convention).
\end{proof}

Without \Cref{ass:PL}, \Cref{cor:exact-convergence} still gives \(\sum_t \mathbb{E}[\lVert u_t \rVert^2] \leq (1 + q)^2 \sum_t \mathbb{E}[\lVert g_{\star, t} \rVert^2] < \infty\), so the parameter jumps \(\lVert \theta_{t+1} - \theta_t \rVert = \eta \lVert u_t \rVert\) vanish almost surely. If, in addition, the iterates lie almost surely in a compact set in which \(\mathcal{L}_\star\) has finitely many stationary points, then \(\theta_t\) converges almost surely to one of them, though not necessarily to a global minimizer: every cluster point of the bounded sequence \((\theta_t)\) is stationary by continuity, since \(g_{\star, t} \to 0\), and a bounded sequence with vanishing jumps and finitely many cluster points converges, as otherwise it would cross between disjoint neighborhoods of two cluster points infinitely often and produce a cluster point outside them.

\paragraph{Instantiating the constants.}
Suppose the target-gradient covariance satisfies \(\Sigma_\star \preceq (\sigma_0^2 + \sigma_1^2 \lVert g_{\star, t} \rVert^2) I\) at every iterate, and \(\dim D(U_t) \leq d_{\mathrm{eff}}\) almost surely (\Cref{thm:projection-oracle}\labelcref{thm:projection-oracle-ii} gives \(d_{\mathrm{eff}} = \sum_p \min\{\lvert G_p \rvert, n-1\}\) for Group-Wise Subset Update and \(d_{\mathrm{eff}} = n - 1\) for Global Subset Update). Then \(\mathcal{V}(U_t) = \mathbb{E}_t[\tr(\Proj_{D(U_t)} \Sigma_\star)] / m \leq d_{\mathrm{eff}} (\sigma_0^2 + \sigma_1^2 \lVert g_{\star, t} \rVert^2) / m\), so the noise part of \Cref{ass:coverage} holds with \(\omega = d_{\mathrm{eff}} \sigma_1^2 / m\) and \(\varsigma_t^2 = d_{\mathrm{eff}} \sigma_0^2 / m\); the floor \(b_t\) accounts for the training--target support mismatch and \(\sigma_0\) for a target-noise floor. Neither is an implementation parameter. In the relative-noise regime \(\sigma_0 = 0\), \(b_t = 0\), the condition \(q < 1\) of \Cref{cor:exact-convergence} reads
\[
	m > \frac{4 d_{\mathrm{eff}} \sigma_1^2}{(1 - \sqrt{\rho})^2} ,
\]
a sufficient certification threshold (not a necessary sample-complexity bound) that makes the tradeoff of \Cref{subsec:bias-variance-tradeoffs} quantitative: refining the partition improves the coverage \(\rho\) while enlarging \(d_{\mathrm{eff}}\), and hence the target batch size needed to certify the finer set. The zero-floor regime is genuinely stronger than bounded variance: it requires compatible training directions and vanishing exposed target noise near target-stationary points, in the spirit of the strong-growth and interpolation conditions of SGD theory~\citep{vaswani2019fast}, and it should not be asserted merely because a model is large.

\paragraph{Idealizations.}
The guarantees of this section concern the exact projected update of \Cref{eq:proj} taken as the SGD step. Our experiments pass \(u_t\) to AdamW (\Cref{rmk:optimizer-aware}) and select with the first-order top-\(k\) rule of \Cref{subsubsec:existing-methods}, and set inclusion orders bias and exposed variance (\Cref{thm:projection-oracle}) but not trajectories pointwise. In our experiments the target batch is resampled at every step from a fixed finite target set (\Cref{sec:experiment}). Conditioning on that set, i.e., including it in \(\mathcal{F}_0\) and replacing \(\mathbb{P}_\star\) by its empirical distribution, leaves every result of this section intact for the empirical target loss, since an independent resample is conditionally unbiased for its gradient; increasing \(m\) then reduces the sampling noise relative to that empirical objective without adding population information. Transferring the guarantees to the population target loss additionally requires a uniform bound on the gap between the empirical and the population target gradients along the trajectory, which is a separate statistical question.

\section{Omitted Details from \texorpdfstring{\Cref{sec:system}}{Section 4}}\label{adxsec:system}
This section provides the details deferred from \Cref{sec:system}. First, \Cref{adxsubsec:general-partition} extends the tensor lifetime scheduling and the scoring of \Cref{subsec:one-pass-tensor-lifetime-scheduling} to general partitions, and \Cref{adxsec:two-pass-implementation} presents a two-pass implementation that avoids the memory management issues at the cost of recomputation. \Cref{adxsubsec:efficient-projection} explains how to compress per-sample gradients efficiently without materializing the full gradient, and \Cref{adxsubsec:scoring-complexity} derives the exact computational and memory complexity of each scoring method in \Cref{tab:scoring-comparison}. \Cref{adxsec:integrating-LoRA-MeSO} then describes how Layer-Wise Subset Update integrates with LoRA and MeSO, covering gradient representations and optimizer state transfer. Finally, \Cref{adxsubsec:beyond-linear} discusses how data regularization extends beyond linear layers, including an efficient scoring algorithm for embedding layers and the treatment of normalization layers.

\paragraph{Notation.}
Unlike in \Cref{sec:system}, where we focus primarily on layer-aligned partitions \(\mathcal{G} = \{G_p\}_{p=1}^{P}\) such that each \(G_p\) contains parameters of layer(s), in \Cref{adxsubsec:general-partition,adxsec:two-pass-implementation} we consider general partitions such that each group might contain a part of each layer and can also span across multiple layers. For the ease of presentation, we overload the notation \(^{(p)}\) and \(^{(l)}\), where the former always refers to group-wise quantities, and the latter refers to layer-wise quantities.

\subsection{General Partition}\label{adxsubsec:general-partition}
The one-pass tensor lifetime schedule of \Cref{subsec:one-pass-tensor-lifetime-scheduling} focuses on layer-aligned partitions, where each group \(G_p\) is a union of whole layers. Here, we discuss the additional complications that arise for arbitrary partitions whose groups may split parameters within a single layer.

\paragraph{Layer-by-layer discipline.}
After the backward pass, all retained pairs \((a^{(l)}, \partial \ell / \partial e^{(l)})\) are available. To stay within the memory budget, score computation and gradient assembly must still proceed layer by layer. At each layer \(l\), we compute per-layer score contributions and accumulate them into the scores for each group. A group \(G_p\) becomes resolvable once all layers intersecting \(G_p\) have been visited, at which point \(S_{t,p}\) is determined. The update direction \(u^{(p)}\) is then assembled from the per-sample gradients of the selected samples at layers intersecting \(G_p\), and the retained quantities at layer \(l\) are released once all groups intersecting that layer have been resolved.

\paragraph{Non-layer-aligned partitions.}
The layer-by-layer discipline introduces a tension between memory and recomputation for partitions that do not align with the layer structure. When a group \(G_p\) spans multiple layers, \(S_{t,p}\) is not known until all layers in the group have been scored. If the scoring strategy materializes per-sample gradients \(g_i^{(l)}\) at each layer, these cannot in general be retained across the group's full span without exceeding the memory budget; they must be discarded after scoring, and recomputed from the retained pairs \((a^{(l)}, \partial \ell / \partial e^{(l)})\) when \(u_t^{(p)}\) is assembled.

\begin{itemize}
	\item \emph{Scoring.} We now show how to compute the per-group scores \(s_i^{(p)}\) during the backward pass. A crucial observation is that the score decomposes linearly over \emph{parameter space}, i.e., \(s_i^{(p)} = \sum_{j \in G_p} g_{i,j} \cdot \hat{g}_{\star, j}\) where \(g_{i,j}\) denotes the \(j^{\text{th}}\) entry of \(g_i\), similarly for \(\hat{g}_{\star, j}\). Hence, computing \(s_i^{(p)}\) reduces to computing the \emph{per-parameter} score \(s_{i,j} \coloneqq g_{i,j} \cdot \hat{g}_{\star, j}\) for each \(i \in [n]\) and \(j \in [d]\). This is realized via layer-wise accumulation: for each \(p\), we initialize \(s_i^{(p)} \gets 0\). Then during the backward pass at layer \(l\), we first materialize the per-sample gradient \(g_i^{(l)}\) and the average target gradient \(\hat{g}_\star^{(l)}\), and obtain all per-parameter scores \(s_{i,j}\) by computing \(\hat{g}_\star^{(l)} \odot g_i^{(l)} \in \mathbb{R}^{d / L}\) for all \(j\) belonging to layer \(l\). Finally, we accumulate \(s_i^{(p)} \gets s_i^{(p)} + \sum_{j \in G_p} s_{i, j}\), and continue the backward pass to layer \(l + 1\).

	      \begin{remark}
		      We highlight that the scoring methods discussed in \Cref{subsec:efficient-scoring}, except for Direct, therefore can not handle general partition, as they do not materialize \(\hat{g}_{\star}^{(l)}\) and \(g_i^{(l)}\) explicitly.
	      \end{remark}

	\item \emph{Update assembly.} Recall that for layer-aligned partitions, we can break down the construction of \(u_t^{(p)}\) into layers. Similarly, in the case of general partitions, we again follow the layer-by-layer discipline, where for each layer, we first identify groups \(G_p\) that have non-empty intersections with it, and for each of these groups, what are the corresponding parameters in the intersection. Then, we materialize all the per-sample gradients \(\{g^{(l)}_i\}_{i=1}^{n}\) of the entire layer, and resolve for each group by first computing the average over samples \(i \in S_{t,p}\) with sub-vectors of \(g_i^{(l)}\) that corresponds to parameters that belongs to the group.
\end{itemize}

\subsection{Two-Pass Implementation}\label{adxsec:two-pass-implementation}
Recall from \Cref{subsec:standard-training} that in standard training, both \(a^{(l)}\) and \(\partial \ell / \partial e^{(l)}\) can be released as soon as the batch gradient \(\hat{g}_{\mathrm{tr}}^{(l)}\) is assembled, that is, once the update direction at layer \(l\) is decided. For Group-Wise Subset Update, however, the update direction at each group \(p\) depends on \(S_{t,p}\), and \(S_{t,p}\) can only be determined after scoring all samples across every layer whose parameters intersect \(G_p\). Standard training releases the per-sample quantities needed for gradient assembly before \(S_{t,p}\) is known. For Layer-Wise Subset Update, each group corresponds to a single layer, so \(S_{t,p}\) can be determined after scoring all samples within that layer, and per-sample quantities can be released immediately. For coarser partitions, however, a two-pass approach trades extra computation for standard memory.

\paragraph{Separate scoring and gradient passes.}
One solution is to use two passes. A first \emph{scoring pass} accumulates per-group scores \(s_i^{(p)}\) layer by layer during a standard forward-backward pass as described in \Cref{adxsubsec:general-partition}. After \(S_{t,p}\) is determined, a second \emph{gradient pass} re-runs forward-backward on the union \(\bigcup_p S_{t,p}\): at each layer \(l\), the gradient for each group \(p\) whose parameters intersect that layer is assembled from the samples in \(S_{t,p}\), which again follow the same strategy described in \Cref{adxsubsec:general-partition} in the case of general partitions.

As a concrete instance, consider Global Subset Update, where all parameters form a single group so that \(S_{t,p} = S_t\) is global. The per-sample score decomposes as \(s_i = \sum_{l=1}^{L} s_i^{(l)}\), so the scoring pass accumulates per-layer contributions into one global score per sample. The gradient pass simplifies to a forward-backward on the \(k\) selected samples, at an additional cost of recomputation. GREATS~\citep{wang2024greats} adopts this two-pass design; \Cref{algo:global-subset-update-two-pass} presents the full pseudocode.

\begin{algorithm}[htpb]\label{algo:global-subset-update-two-pass}
	\DontPrintSemicolon{}
	\caption{Global Subset Update (two-pass)}
	\KwData{Model \(\theta_t\), training data \(\{z_i\}_{i=1}^n\), target data \(\{z^\star_j\}_{j=1}^m\), learning rate \(\eta_t\)}
	\KwResult{Updated model parameters \(\theta_{t+1}\)}
	\SetKwFunction{Forward}{Forward}
	\SetKwFunction{Backward}{Backward}
	\SetKwFunction{Release}{Release}
	\SetKwFunction{Select}{Select-S}
	\SetKwFunction{Score}{Score}

	\BlankLine

	\tcc{Scoring pass}
	(\(\ell\), \(\{a_{i}^{(l)}, e_{i}^{(l)}\}_{l,i} \cup \{a_{j}^{\star(l)}, e_{j}^{\star(l)}\}_{l,j}\))\(\gets\)\Forward{\(\theta_t\), \(\{z_i\}_{i=1}^{n} \cup \{z^\star_j\}_{j=1}^{m}\)}\Comment*[r]{\memcost{\(+ O(2N\seqlen \sqrt{dL})\)}}
	\For{\(l = L, \ldots, 1\)}{
		\(\{\partial \ell / \partial e_{i}^{(l)}\}_{i=1}^{n} \cup \{\partial \ell / \partial e_{j}^{\star(l)}\}_{j=1}^{m} \gets\)\Backward{\(\ell\), \(\theta_t^{(l)}\)}\;
		\(s_i^{(l)} \gets\)\Score{\(a_i^{(l)}\), \(\partial \ell / \partial e_i^{(l)}\), \(\{a^{\star (l)}_j, \partial \ell / \partial e^{\star (l)}_j\}_{j=1}^m\)} for each \(i \in [n]\)\Comment*[r]{\Cref{subsec:efficient-scoring}}
		\Release{\(a_i^{(l)}\), \(\partial \ell / \partial e_i^{(l)}\), \(a_j^{\star(l)}\), \(\partial \ell / \partial e_j^{\star(l)}\)} for all \(i, j\)\Comment*[r]{\memcost{\(- O(2N\seqlen \sqrt{d/L})\)}}
	}
	\(s_i \gets \sum_{l=1}^L s_i^{(l)}\) for all \(i \in [n]\)\Comment*[r]{Global scores}
	\(S_t \gets\)\Select{\(\{s_i\}_{i=1}^n\), \(k\)}\;
	\BlankLine
	\tcc{Gradient pass on subset \(S_t\)}
	(\(\ell_{S_t}\), \(\{a_{i}^{(l)}\}_{l,i}\))\(\gets\)\Forward{\(\theta_t\), \(\{z_i\}_{i \in S_t}\)}\Comment*[r]{\memcost{\(+ O(k\seqlen \sqrt{dL})\)}}
	\For{\(l = L, \ldots, 1\)}{
		\(\{\partial \ell_{S_t} / \partial e_{i}^{(l)}\}_{i \in S_t} \gets\)\Backward{\(\ell_{S_t}\), \(\theta_t^{(l)}\)}\Comment*[r]{\memcost{\(- O(k\seqlen \sqrt{d/L}) + O(k\seqlen \sqrt{d/L})\)}}
		\(u_t^{(l)} \gets \frac{1}{k} \sum_{i \in S_t} \sum_{\tau=1}^{\seqlen} (\partial \ell_{S_t} / \partial e_{i,\tau}^{(l)}) \otimes a_{i,\tau}^{(l)}\)\Comment*[r]{\memcost{\(+ O(d/L)\)}}
		\Release{\(a_i^{(l)}\), \(\partial \ell_{S_t} / \partial e_i^{(l)}\)} for all \(i \in S_t\)\Comment*[r]{\memcost{\(- O(2k\seqlen \sqrt{d/L})\)}}
		\(\theta_{t+1}^{(l)} \gets \theta_t^{(l)} - \eta_t u_t^{(l)}\)\;
		\Release{\(u_t^{(l)}\)}\Comment*[r]{\memcost{\(- O(d/L)\)}}
	}
	\Return{\(\theta_{t+1}\)}\;
\end{algorithm}

\paragraph{Compatibility with memory-saving techniques.}\label{adxsubsubsec:compatability-with-memory-saving-techniques}
The two-pass design is compatible with both activation checkpointing and gradient accumulation, since each pass independently preserves the standard memory profile. For activation checkpointing, when the partition does not align with the checkpointing schedule (i.e., some group \(G_p\) spans multiple checkpoint segments), the scoring pass computes per-group scores under standard checkpointing, releasing activations at segment boundaries, and the gradient pass performs a checkpointed forward-backward over \(\bigcup_p S_{t,p}\). For gradient accumulation with rules that depend on global batch statistics (e.g., top-\(k\)), the scoring pass processes all micro-batches to determine \(S_{t,p}\) for every group, and the gradient pass assembles gradients from \(\bigcup_p S_{t,p}\), itself split into micro-batches if needed.

\subsection{Efficient Per-Sample Gradient Projection}\label{adxsubsec:efficient-projection}
We now provide a self-contained discussion on the state-of-the-art per-sample gradient compression~\citep{hu2025grass} that underlies Compressed Scoring (\Cref{subsubsec:compressed-scoring}); the same compression machinery is also used in the MeSO integration in \Cref{adxsubsec:meso-integration}.

\paragraph{Factorized projection.}
State-of-the-art gradient compressors~\citep{hu2025grass,choe2025your} consider a projection operator \(\Pi \in \mathbb{R}^{\kappa \times d}\) from dimension \(d \coloneqq d_{\text{in}} d_{\text{out}}\) to dimension \(\kappa \ll d\) with a two-stage structure:
\[
	\Pi = P_{\text{final}} \left(P_{\text{in}} \otimes P_{\text{out}}\right),
\]
where \(P_{\text{in}} \in \mathbb{R}^{\kappa_{\text{in}} \times d_{\text{in}}}\) and \(P_{\text{out}} \in \mathbb{R}^{\kappa_{\text{out}} \times d_{\text{out}}}\) are the first-stage projections, \(P_{\text{in}} \otimes P_{\text{out}}\) denotes their Kronecker product, and \(P_{\text{final}} \in \mathbb{R}^{\kappa \times \kappa_{\text{in}} \kappa_{\text{out}}}\) is a small second-stage projection.

Remarkably, for any input vectors that admit such a factorized structure, the factorized projection can be computed efficiently without materializing the full \(\mathbb{R}^{\kappa \times d}\) representation of \(\Pi\), and can directly operate on the input factors in \(\mathbb{R}^{d_{\text{in}}}\) and \(\mathbb{R}^{d_{\text{out}}}\), enabling efficient forward and back-projection. \textit{Forward projection.} For any input with outer-product structure \(b \otimes a\) where \(a \in \mathbb{R}^{d_{\text{in}}}\) and \(b \in \mathbb{R}^{d_{\text{out}}}\), the forward projection evaluates as
\[
	\Pi \vect(b \otimes a) = P_{\text{final}} \vect \big((P_{\text{out}} b) \otimes (P_{\text{in}} a)\big),
\]
requiring only two small matrix-vector products (\(P_{\text{in}} a \in \mathbb{R}^{\kappa_{\text{in}}}\) and \(P_{\text{out}} b \in \mathbb{R}^{\kappa_{\text{out}}}\)) followed by the second-stage projection, without forming \(\vect(b \otimes a) \in \mathbb{R}^{d}\). By linearity, this extends to sums of outer products: for \(X = \sum_{\tau} b_\tau \otimes a_\tau\),
\[
	\widetilde{X}
	= \Pi \vect(X) = P_{\text{final}} \vect \left(\sum_{\tau} (P_{\text{out}} b_\tau) \otimes (P_{\text{in}} a_\tau)\right),
\]
where each term contributes a small outer product in \(\mathbb{R}^{\kappa_{\text{out}} \times \kappa_{\text{in}}}\). \textit{Backward projection.} For any vector \(\widetilde{x} \in \mathbb{R}^{\kappa}\), the transpose map \(\Pi^\top \widetilde{x} \in \mathbb{R}^{d}\) is equally efficient. Let \(X^{\prime} \in \mathbb{R}^{\kappa_{\text{out}} \times \kappa_{\text{in}}}\) be the matrix obtained by reshaping \((P_{\text{final}})^\top \widetilde{x}\), and let \(\Mat\) denote the inverse of \(\vect\), which reshapes a vector in \(\mathbb{R}^{d}\) into a \(d_{\text{out}} \times d_{\text{in}}\) matrix. The Kronecker structure gives
\[
	\Mat(\Pi^\top \widetilde{x})
	= (P_{\text{out}})^\top X ^{\prime} P_{\text{in}},
\]
a pair of small matrix multiplications, without forming the full matrix \(\Pi^\top \in \mathbb{R}^{d \times \kappa}\).

\paragraph{Per-sample gradient compression.}\label{adxsubsubsec:per-sample-gradient-compression}
We now show how to apply this compressor to compress linear layers' per-sample gradients. \textit{Notation.} For generality, we do not assume that each linear layer is square: specifically, the weight at layer \(l\) is \(W^{(l)} \in \mathbb{R}^{d_{\text{out}}^{(l)} \times d_{\text{in}}^{(l)}}\) with \(d^{(l)} \coloneqq d_{\text{in}}^{(l)} d_{\text{out}}^{(l)}\), related to the vectorized form by \(\theta^{(l)} = \vect(W^{(l)})\). Per-training-sample weight gradients are correspondingly denoted \(G_i^{(l)} \in \mathbb{R}^{d_{\text{out}}^{(l)} \times d_{\text{in}}^{(l)}}\), so that \(g_i^{(l)} = \vect(G_i^{(l)})\). Similarly for \(G^{\star}_j\) for per-target-sample gradients. Let \(\langle \cdot, \cdot \rangle\) denote the Frobenius inner product, i.e., the Euclidean inner product of the vectorized forms. \textit{Efficient per-sample gradient compression.} Recall that the per-sample weight gradient for sample \(z_i\) at layer \(l\) is a sum of outer products over tokens: \(g_i^{(l)} = \sum_{\tau=1}^\seqlen (\partial \ell / \partial e_{i,\tau}^{(l)}) \otimes a_{i,\tau}^{(l)} \in \mathbb{R}^{d^{(l)}}\). Consider a compressor \(\Pi^{(l)} \coloneqq P_{\text{final}}^{(l)} (P_{\text{in}}^{(l)} \otimes P_{\text{out}}^{(l)})\), with \(P_{\text{in}}^{(l)} \in \mathbb{R}^{\kappa_{\text{in}}^{(l)} \times d_{\text{in}}^{(l)}}\) and \(P_{\text{out}}^{(l)} \in \mathbb{R}^{\kappa_{\text{out}}^{(l)} \times d_{\text{out}}^{(l)}}\), and \(P_{\text{final}}^{(l)} \in \mathbb{R}^{\kappa^{(l)} \times \kappa_{\text{in}}^{(l)}\kappa_{\text{out}}^{(l)}}\). Applying the forward projection gives the compressed per-sample gradient
\[
	\widetilde{g}_i^{(l)}
	= P^{(l)}_{\text{final}} \vect \left(\sum_{\tau=1}^{\seqlen} \Bigg(P^{(l)}_{\text{out}} \frac{\partial \ell}{\partial e_{i,\tau}^{(l)}}\Bigg) \otimes \left(P^{(l)}_{\text{in}} a_{i,\tau}^{(l)}\right)\right) \in \mathbb{R}^{\kappa^{(l)}}.
\]
We note that this can be computed directly from the cached activations and activation gradients, without materializing \(g_i^{(l)}\). With careful choice of \(P_{\text{in} }^{(l)} \), \(P_{\text{out} }^{(l)}\), and \(P_{\text{final} }^{(l)}\), the per-sample memory is \(O(\kappa^{(l)})\) instead of \(O(d^{(l)})\), and the per-token computational cost is \(O(\kappa^{(l)})\)~\citep{hu2025grass}, so compressing all \(n\) samples across \(\seqlen\) tokens and \(L\) layers costs \(O(n\seqlen \sum_l \kappa^{(l)})\) in total. Later, we write this operation as \(\Pi^{(l)}(a_i^{(l)}, \partial \ell / \partial e_i^{(l)}, \{a_j^{\star(l)}, \partial \ell / \partial e_{j}^{\star(l)}\})\) in \Cref{algo:meso-layer-wise} to emphasize that \(\Pi^{(l)}\) operates on factorized inputs without materializing the full gradient.

\subsection{Scoring Complexity}\label{adxsubsec:scoring-complexity}
We derive the exact computational and memory complexity of each scoring method in \Cref{tab:scoring-comparison}. Every addition and multiplication counts as one operation; memory is measured in scalar entries under fully batched execution. Throughout, we let \(w = \sqrt{d/L}\) denote the layer width for brevity, \(N = n + m\), and \(\kappa \equiv \kappa^{(l)}\) the per-layer compressed dimension of Compressed Scoring; all quantities carry a layer index \({}^{(l)}\) that we suppress for brevity.

\subsubsection{Direct}
Direct materializes the per-sample gradients and the target gradient and then takes their inner products (\Cref{subsubsec:direct-gip}). Its cost decomposes into three steps:

\begin{enumerate}[leftmargin=*]
	\item \emph{Materialize \(n\) training gradient matrices} \(g_i = \sum_{\tau=1}^{\seqlen} (\partial \ell / \partial e_{i,\tau}) \otimes a_{i,\tau}\). Each Kronecker product costs \(w^2\) multiplications; summing \(\seqlen\) outer products element-wise adds \((\seqlen-1) w^2\) additions. Per sample: \((2\seqlen-1)w^2\). Total: \(n(2\seqlen-1)w^2\).
	\item\label{direct-materialization-ii} \emph{Materialize \(\hat{g}_\star = \frac{1}{m}\sum_j g_j^{\star}\).} Same Kronecker-product accumulation for \(m\) target samples: \(m(2\seqlen-1)w^2\). Summing the \(m\) matrices: \((m-1)w^2\) additions. Scaling by \(1/m\): \(w^2\) multiplications. Total: \(m(2\seqlen-1)w^2 + (m-1)w^2 + w^2 = 2m\seqlen{}w^2\).
	\item \emph{Score.} Each Frobenius inner product \(s_i = \langle \hat{g}_\star, g_i \rangle\) (the Euclidean inner product of the vectorized matrices) costs \(w^2\) multiplications \(+\) \((w^2-1)\) additions \(= (2w^2-1)\) per sample. Total: \(n(2w^2-1)\).
\end{enumerate}
Summing the three steps yields total FLOPs of \(n(2\seqlen-1)w^2 + 2m\seqlen{}w^2 + n(2w^2-1) = 2N\seqlen{}w^2 + n(w^2-1)\), and a memory footprint consisting of the \(n\) gradient matrices \(g_i \in \mathbb{R}^{w \times w}\) plus \(\hat{g}_\star\), totaling \((n+1)w^2\) entries.

\subsubsection{Ghost Inner Product (GIP)}
GIP evaluates the inner product through \(\seqlen \times \seqlen\) cross-correlations (\Cref{subsubsec:direct-gip}). For each training-target pair \((i,j)\), forming the activation-gradient cross-correlation \((\partial \ell / \partial e_i)^{\top}(\partial \ell / \partial e_j^{\star}) \in \mathbb{R}^{\seqlen \times \seqlen}\) and the activation cross-correlation \((a_i)^{\top} a_j^{\star} \in \mathbb{R}^{\seqlen \times \seqlen}\) each costs \(\seqlen^2(2w-1)\) (each entry is a dot product of \(w\)-vectors: \(w\) multiplications \(+\) \((w-1)\) additions, giving \(\seqlen^2(2w-1)\) per matrix), so two such matrices cost \(2\seqlen^2(2w-1)\); the subsequent Frobenius inner product over the two \(\seqlen \times \seqlen\) matrices is element-wise multiply and sum, costing \(\seqlen^2\) multiplications \(+\) \((\seqlen^2-1)\) additions \(= (2\seqlen^2-1)\). Per pair: \(2\seqlen^2(2w-1) + (2\seqlen^2-1) = 4\seqlen^2 w - 1\). Accumulating \(m\) pairs per training sample adds \((m-1)\) additions and \(1\) multiplication (the \(1/m\) scaling), giving \(m(4\seqlen^2 w - 1) + m = 4m\seqlen^2 w\) per training sample. Across \(n\) training samples, the total FLOPs is \(4nm\seqlen^2 w\), and the memory consists of two \(\seqlen \times \seqlen\) cross-correlation matrices per \((i,j)\) pair, batched over all pairs, giving \(2nm\seqlen^2\) entries.

\subsubsection{Per-Token Inner Product (PIP)}
PIP factors the score into a matrix-vector product against the aggregated target gradient followed by a token-level dot product (\Cref{subsubsec:per-token-inner-product}), in two steps:

\begin{enumerate}[leftmargin=*]
	\item \emph{Materialize \(\hat{G}_\star\),} identical to Step \labelcref{direct-materialization-ii} of Direct, with cost \(2m\seqlen{}w^2\).
	\item \emph{Score.} For each training token \((i,\tau)\): a matrix-vector product \(\hat{G}_\star  a_{i,\tau} \in \mathbb{R}^w\) costs \(w(2w-1)\) (each of \(w\) entries is a dot product of \(w\)-vectors), then a dot product with \(\partial \ell / \partial e_{i,\tau}\) costs \((2w-1)\). Per token: \(w(2w-1) + (2w-1) = (w+1)(2w-1) = 2w^2 + w - 1\). Accumulating \(\seqlen\) tokens into \(s_i\): \(\seqlen(2w^2 + w - 1) + (\seqlen-1) = 2\seqlen{}w^2 + \seqlen{}w - 1\) per sample. Total: \(n(2\seqlen{}w^2 + \seqlen{}w - 1)\).
\end{enumerate}
Summing yields total FLOPs of \(2m\seqlen{}w^2 + n(2\seqlen{}w^2 + \seqlen{}w - 1) = 2N\seqlen{}w^2 + n(\seqlen{}w - 1)\), and a memory footprint consisting of \(\hat{G}_\star \in \mathbb{R}^{w \times w}\) plus the batched intermediates \(\hat{G}_\star A \in \mathbb{R}^{w \times n\seqlen}\), giving \(w^2 + n\seqlen{}w\) entries.

\subsubsection{Compressed Scoring}
Compressed Scoring computes the score between compressed gradients (\Cref{subsubsec:compressed-scoring}) in three steps:

\begin{enumerate}[leftmargin=*]
	\item \emph{Compress all \(N\) per-sample gradients} to \(\widetilde{g}_i \in \mathbb{R}^{\kappa^{(l)}}\) via the projection \(\Pi\). The per-token cost depends on the Kronecker structure of \(\Pi\) (see \Cref{adxsubsec:efficient-projection}); the total is \(O(N\seqlen\kappa^{(l)})\).
	\item \emph{Aggregate target gradient} \(\widetilde{g}_\star = \frac{1}{m}\sum_j \widetilde{g}_j^{\star}\). Summing \(m\) vectors of dimension \(\kappa^{(l)}\): \((m-1)\kappa^{(l)}\) additions. Scaling by \(1/m\): \(\kappa^{(l)}\) multiplications. Total: \(m\kappa^{(l)}\).
	\item \emph{Score.} Each inner product \(s_i = \langle \widetilde{g}_\star, \widetilde{g}_i \rangle\) costs \(\kappa^{(l)}\) multiplications \(+\) \((\kappa^{(l)}-1)\) additions \(= (2\kappa^{(l)}-1)\) per sample. Total: \(n(2\kappa^{(l)}-1)\).
\end{enumerate}
Summing yields total FLOPs of \(O(N\seqlen\kappa^{(l)}) + (2n+m)\kappa^{(l)}\), and a memory footprint of \((n+1)\kappa^{(l)}\) (the \(n\) compressed training gradients plus \(\widetilde{g}_\star\)).

\subsection{Integrating LoRA and MeSO}\label{adxsec:integrating-LoRA-MeSO}
The Dr.\ Post-Training framework is agnostic to the gradient representation at each layer, composing naturally with parameter-efficient and memory-efficient training methods. We now detail the integration with LoRA and MeSO.

\subsubsection{LoRA Integration}\label{adxsubsec:lora-integration}

We adopt the notation of \Cref{adxsubsec:LoRA}. Since the two LoRA adapters \(A_t^{(l)}\) and \(B_t^{(l)}\) are themselves linear maps, the per-sample gradients admit the same outer-product factorization as full-parameter gradients: writing \(h_{i,\tau}^{(l)} \coloneqq A_t^{(l)} a_{i,\tau}^{(l)} \in \mathbb{R}^{r}\) for the adapter's intermediate activation (and absorbing the LoRA scaling factor),
\[
	G_i^{(A,l)}
	= \sum_{\tau=1}^{\seqlen} \frac{\partial \ell}{\partial h_{i,\tau}^{(l)}} \otimes a_{i,\tau}^{(l)} \in \mathbb{R}^{r \times d_{\text{in}}^{(l)}}, \qquad
	G_i^{(B,l)}
	= \sum_{\tau=1}^{\seqlen} \frac{\partial \ell}{\partial e_{i,\tau}^{(l)}} \otimes h_{i,\tau}^{(l)} \in \mathbb{R}^{d_{\text{out}}^{(l)} \times r}.
\]
The layer-wise scoring \(s_i^{(l)} = \langle \hat{G}_\star^{(l)}, G_i^{(l)} \rangle\) and gradient computation proceed identically to the full-parameter case, but with the adapter dimension \(r(d_{\text{in}}^{(l)} + d_{\text{out}}^{(l)})\) replacing the full dimension \(d_{\text{in}}^{(l)} d_{\text{out}}^{(l)}\). This yields proportional reductions in both scoring and gradient computation cost. No modification to the Layer-Wise Subset Update algorithm (\Cref{algo:layer-wise-subset-update}) is needed; LoRA simply changes the gradient representation that the algorithm operates on.

\subsubsection{MeSO Integration}\label{adxsubsec:meso-integration}
We adopt the notation of \Cref{adxsubsec:MeSO}. Since MeSO only changes how the optimizer update is applied, not how per-sample gradients are computed, data regularization integrates directly: any scoring method from \Cref{subsec:efficient-scoring} (including the exact methods) can be used to determine \(S_{t,p}\), and the optimizer then updates in the projected subspace using only the samples in \(S_{t,p}\).

\paragraph{Compression.}

We adapt \Cref{adxsubsec:efficient-projection} for the computation of the compression required by MeSO. Note that the compressed gradients are trivial to compute by directly applying the forward projection in the factorized form. The tricky part is that when computing the momentum states transfer during subspace refreshing, some of the input does not have the factorized structure. Taking the first moment \(\hat{m}_{\text{old}}\) as an example, we see that the transformation requires computing \(\Pi_{\text{new}} \Pi_{\text{old}}^{\top} \hat{m}_{\text{old}}\). We see that for the backward projection, we have an efficient implementation to compute \(\Pi_{\text{old}}^{\top} \hat{m}_{\text{old}}\). However, for the forward projection, it is unclear how to forward project this general input \(\Pi_{\text{old}}^{\top} \hat{m}_{\text{old}}\) under \(\Pi_{\text{new}}\). However, a simple application of the same trick (formally known as the \emph{vec-trick}), as discussed in the backward projection, resolves this problem and provides a way to compute the forward projection on general input. We omit the details here for simplicity.

\paragraph{Scoring.}
For exact scoring, this is quite straightforward: we first compute the scores using one of the exact scoring methods discussed in \Cref{subsec:efficient-scoring}, discard all the intermediate quantities, solve for \(S_{t,p}\), and finally compute the compressed gradients for \(S_{t,p}\). Since the exact scoring and compressed gradient computations generally do not have overlapping computations, we do not need to consider retaining any intermediate quantities to avoid recomputation, simplifying the memory management issue.

On the other hand, for approximate scoring, we see that the compressed per-sample gradients that MeSO computes for parameter updates can also be used for compressed scoring (\Cref{subsubsec:compressed-scoring}) as an additional optimization, so that the \emph{same} compressed gradients \(\widetilde{g}_{i}^{(l)}, \widetilde{g}_\star^{(l)}\) drive both alignment scoring (\(s_i^{(l)} = \langle \widetilde{g}_{\star}^{(l)}, \widetilde{g}_i^{(l)} \rangle\)) and the optimizer update, avoiding redundant computation.

However, in this case, we note that managing memory for the compressed per-sample gradients needed by the optimizer update and compressed scoring is again non-trivial: when there is a group \(G_p\) that contains many layers, materializing all the per-sample compressed gradients across all layers may still be memory-intensive and infeasible under a tight memory budget.

\paragraph{Memory management for compressed scoring.}
We next discuss the memory-management aspect for MeSO with compressed gradients reused. Specifically, assume that we materialize all the \(O(n)\) compressed per-sample gradients \(\widetilde{g}_i^{(l)}\) and also \(\widetilde{g}_{\star}^{(l)}\) during the backward pass. Once \(\widetilde{g}_i^{(l)}\) is computed, the activations \(a_i^{(l)}\) and activation gradients \(\partial \ell / \partial e_i^{(l)}\) can be released immediately, since the optimizer update only requires averaging the compressed gradients in \(S_{t,l}\). This replaces \(O(2 n\seqlen\sqrt{d/L})\) memory (for \(a_i^{(l)}\) and \(\partial \ell / \partial e_i^{(l)}\)) with \(O(n\kappa^{(l)})\) (for \(\widetilde{g}_i^{(l)}\)), a net savings when \(\kappa^{(l)} < \seqlen\sqrt{d/L}\), which holds in typical settings. Hence, it is generally possible to directly retain all the compressed gradients to avoid recomputation of the per-sample compressed gradients for the MeSO updates.

For illustration, the pseudocode for this strategy in the case of Layer-Wise Subset Update is provided in \Cref{algo:meso-layer-wise}, where the compressed gradients are further released at each layer once \(S_{t,l}\) is determined, so only one layer's worth of per-sample compressed gradients is held at a time.

\begin{algorithm}[H]\label{algo:meso-layer-wise}
	\DontPrintSemicolon{}
	\caption{Layer-Wise Subset Update with MeSO}
	\KwData{Model \(\theta_t\), training data \(\{z_i\}_{i=1}^n\), target data \(\{z^\star_j\}_{j=1}^m\), learning rate \(\eta_t\), projections \(\Pi\)}
	\KwResult{Updated model parameters \(\theta_{t+1}\)}
	\SetKwFunction{Forward}{Forward}
	\SetKwFunction{Backward}{Backward}
	\SetKwFunction{Release}{Release}
	\SetKwFunction{Select}{Select-S}
	\BlankLine
	(\(\ell\), \(\{a_{i}^{(l)}\}_{l,i}\), \(\{a_{j}^{\star(l)}\}_{l,j}\))\(\gets\)\Forward{\(\theta_t\), \(\{z_i\}_{i=1}^{n} \cup \{z^\star_j\}_{j=1}^{m}\)}\;
	\For{\(l = L, \ldots, 1\)}{
		\(\{\partial \ell / \partial e_{i}^{(l)}\}_{i=1}^{n}, \{\partial \ell / \partial e_{j}^{\star(l)}\}_{j=1}^{m} \gets\)\Backward{\(\ell\), \(\theta_t^{(l)}\)}\;
		\(\widetilde{g}_i^{(l)} \gets \Pi(a_i^{(l)}, \partial \ell / \partial e_i^{(l)}) \in \mathbb{R}^{\kappa^{(l)}}\) for all \(i \in [n]\)\Comment*[r]{Compress training}
		\(\widetilde{g}_j^{\star(l)} \gets \Pi(a_j^{\star(l)}, \partial \ell / \partial e_j^{\star(l)}) \in \mathbb{R}^{\kappa^{(l)}}\) for all \(j \in [m]\)\Comment*[r]{Compress target}
		\Release{\(a_i^{(l)}\), \(\partial \ell / \partial e_i^{(l)}\), \(a_j^{\star(l)}\), \(\partial \ell / \partial e_j^{\star(l)}\)} for all \(i, j\)\;
		\(\widetilde{g}_\star^{(l)} \gets \frac{1}{m}\sum_{j=1}^m \widetilde{g}_j^{\star(l)}\)\;
		\(s^{(l)}_i \gets \langle \widetilde{g}_\star^{(l)}, \widetilde{g}_i^{(l)} \rangle\) for each \(i \in [n]\)\Comment*[r]{Score via compressed gradients}
		\(S_{t,l} \gets\)\Select{\(\{s^{(l)}_i\}_{i=1}^n\), \(k\)}\;
		\(\widetilde{u}_t^{(l)} \gets \frac{1}{k} \sum_{i \in S_{t,l}} \widetilde{g}_i^{(l)}\)\Comment*[r]{Aggregate selected}
		\Release{\(\widetilde{g}_i^{(l)}\)} for all \(i\)\;
		\(\theta_{t+1}^{(l)} \gets \theta _t^{(l)} - \eta_t \cdot (\Pi^\top \widetilde{u}_t^{(l)})\)\Comment*[r]{Back-project and update}
		\Release{\(\widetilde{u}_t^{(l)}\)}\;
	}
	\Return{\(\theta_{t+1}\)}\;
\end{algorithm}

\subsection{Beyond Linear Layers}\label{adxsubsec:beyond-linear}
The system design of \Cref{sec:system} develops scoring and selection exclusively for linear layers (e.g., \texttt{nn.Linear} in \texttt{PyTorch}), which account for the vast majority of trainable parameters in standard transformer architectures. A transformer model, however, also contains other trainable modules---token embeddings (\texttt{nn.Embedding}) and normalization layers (\texttt{nn.LayerNorm}, \texttt{nn.RMSNorm})---whose treatment we now discuss.

\subsubsection{Embedding Layers}\label{adxsubsubsec:embedding-scoring}
The token embedding layer maps each input token index \(x_{i,\tau} \in \mathcal{V}\) to a dense vector via a lookup: \(\mathrm{output}_{i,\tau} = W[x_{i,\tau}]\), where \(W \in \mathbb{R}^{V \times d_{\mathrm{emb}}}\) is the embedding matrix, \(\mathcal{V}\) is the vocabulary, and \(d_{\mathrm{emb}}\) is the embedding dimension. Unlike a linear layer, whose per-sample gradient is a sum of dense outer products, the embedding gradient is \emph{sparse}: writing \(\delta_{i,\tau} \coloneqq \partial \ell_i / \partial \mathrm{output}_{i,\tau} \in \mathbb{R}^{d_{\mathrm{emb}}}\) for the activation gradient at position \(\tau\), the per-sample weight gradient is
\[
	G_i = \sum_{\tau=1}^{\seqlen} e_{x_{i,\tau}} \otimes \delta_{i,\tau} \in \mathbb{R}^{V \times d_{\mathrm{emb}}},
\]
where \(e_{x_{i,\tau}} \in \mathbb{R}^{\lvert \mathcal{V} \rvert}\) is the one-hot indicator for token \(x_{i,\tau}\). Row \(v\) of \(G_i\) equals \(\sum_{\tau \colon x_{i,\tau} = v} \delta_{i,\tau}\), and all other rows are zero.

\paragraph{Efficient scoring via lookup.}
This sparsity enables a scoring algorithm that avoids materializing any per-sample gradient matrix. The target embedding gradient is
\[
	\hat{G}_\star = \frac{1}{m} \sum_{j=1}^{m} \sum_{\tau=1}^{\seqlen} e_{x_{j,\tau}^\star} \otimes \delta_{j,\tau}^\star \in \mathbb{R}^{V \times d_{\mathrm{emb}}},
\]
computable in \(O(m\seqlen{}d_{\mathrm{emb}})\) via \texttt{index\_add} (scatter-accumulate over token indices). The per-sample alignment score then simplifies as follows:
\[
	s_i^{(\mathrm{emb})}
	= \langle G_i, \hat{G}_\star \rangle
	= \sum_{v \in \mathcal{V}} \langle G_i[v], \hat{G}_\star[v] \rangle
	= \sum_{\tau=1}^{\seqlen} \langle \delta_{i,\tau},\, \hat{G}_\star[x_{i,\tau}] \rangle,
\]
where the last equality follows from \(G_i[v] \neq 0\) only for \(v \in \{x_{i,\tau}\}_{\tau}\). Each term is a table lookup of \(\hat{G}_\star\) at index \(x_{i,\tau}\) followed by a dot product with \(\delta_{i,\tau}\), both \(O(d_{\mathrm{emb}})\) operations. Summing over \(\seqlen\) tokens gives \(O(\seqlen{}d_{\mathrm{emb}})\) per sample, and \(O(n\seqlen{}d_{\mathrm{emb}})\) across all \(n\) training samples. Including the cost of constructing \(\hat{G}_\star\), the total is \(O(N\seqlen{}d_{\mathrm{emb}})\) FLOPs with \(O(VD)\) memory for storing \(\hat{G}_\star\).

\paragraph{Comparison with linear layer scoring.}
The embedding scoring algorithm is the natural analogue of the Per-Token Inner Product (PIP, \Cref{subsubsec:per-token-inner-product}) for linear layers. For a linear layer, PIP evaluates \(s_i^{(l)} = \sum_\tau (\partial \ell / \partial e_{i,\tau}^{(l)})^\top \hat{G}_\star^{(l)} a_{i,\tau}^{(l)}\), which requires a matrix-vector product \(\hat{G}_\star^{(l)} a_{i,\tau}^{(l)} \in \mathbb{R}^{\sqrt{d/L}}\) at each token. For the embedding layer, the ``input'' is a one-hot vector \(e_{x_{i,\tau}}\), so the matrix-vector product \(\hat{G}_\star \, e_{x_{i,\tau}} = \hat{G}_\star[x_{i,\tau}]\) reduces to a single row lookup, eliminating the \(O(d_{\mathrm{emb}})\) factor from the matrix-vector multiply. In practice, the activation gradients \(\delta_{i,\tau}\) are already available during backpropagation (they are the upstream gradient flowing into the embedding layer), so scoring adds only the table-lookup and dot-product cost per token.

\paragraph{Gradient assembly.}
After selection, the weight gradient for the selected subset \(S\) is assembled as \(G^{(\mathrm{emb})} = \frac{1}{\lvert S \rvert} \sum_{i \in S} G_i\), which is a sparse accumulation computable in \(O(\lvert S \rvert \seqlen d_{\mathrm{emb}})\) via \texttt{index\_add}, matching the cost of standard embedding gradient computation on a batch of size \(\lvert S \rvert\).

\subsubsection{Normalization Layers}\label{adxsubsubsec:normalization-layers}
Normalization layers (e.g., \texttt{LayerNorm}~\citep{ba2016layer}, \texttt{RMSNorm}~\citep{zhang2019root}) contain a small number of trainable parameters---a scale vector \(\gamma \in \mathbb{R}^{d_{\mathrm{emb}}}\) and optionally a bias \(b \in \mathbb{R}^{d_{\mathrm{emb}}}\)---contributing \(O(d_{\mathrm{emb}})\) parameters per normalization layer versus \(O(d_{\mathrm{emb}}^2)\) for each linear layer. These parameters are not scored or selected; they receive the gradient of the training batch only, i.e., they are trained under Full-Training Update (\(U = \{\hat{g}_{\mathrm{tr}}\}\)), identically across the compared methods. In the merged-batch one-pass schedule (our SFT experiments) a custom backward restricts their weight gradient to the \(n\) training samples of the merged batch, excluding the \(m\) target samples; in the RL pipelines the target gradient is computed in a separate pass, so the target samples never enter their gradient. In a typical transformer, the total parameter count of all normalization layers is negligible relative to that of the linear layers: in \textsc{Qwen3-1.7B-Base} they amount to about \(1.2 \times 10^{5}\) of \(1.7 \times 10^{9}\) trainable parameters (below \(0.01\%\)), and in the LoRA settings every trainable parameter belongs to an adapter of a hooked linear layer.

\paragraph{Summary.}
Every trainable parameter in our experiments is handled by exactly one of two rules: it is either \emph{hooked}, i.e., scored and updated from its selected subset, or it is updated with the plain training-batch gradient (Full-Training Update); our implementation enumerates the trainable parameters at setup so that none is left unaccounted for. The hooked modules are all \texttt{nn.Linear} modules of the transformer blocks and, in the full-parameter and MeSO settings, the embedding matrix (scored with the lookup-based algorithm above) and the \texttt{lm\_head}. In \textsc{Qwen3-1.7B-Base} the two are weight-tied; we treat the two uses as separate groups, each selecting independently as any other pair of groups, and sum their selected gradients into the shared matrix; the feasible set of the shared weight is then the sum of the two group feasible sets, which contains the single-subset set, so the bias ordering of \Cref{thm:projection-oracle}, which uses only this inclusion and not the per-group separability of \Cref{subsubsec:group-wise-decomposition}, is preserved, and under Global Subset Update both uses share one subset so the update reduces to the standard tied gradient. In the LoRA settings every trainable parameter is an adapter of a hooked linear layer, and the embedding and normalization layers are frozen. Updated with the training-batch gradient are the normalization layers (\Cref{adxsubsubsec:normalization-layers}) and the value head of the RLHF policy, which is not scored and is updated with every rollout (\Cref{adxsubsec:details-of-rlhf}). Together these two classes account for all trainable parameters, and the target samples never enter the gradient of an unhooked parameter.

\section{Omitted Details from \texorpdfstring{\Cref{sec:experiment}}{Section 5}}\label{adxsec:experiment}
In this section, we provide all the omitted details for all experiment settings, including SFT, RLHF, and RLVR. We note that all reported results in the main text are averaged over \(5\) random seeds with error bars given as standard errors of the mean, and evaluation is performed on held-out test sets disjoint from the data used for regularization or training. Unless stated otherwise, every curated arm uses exact scoring (PIP, or GIP in the question-answering setting; the two return the same inner products); compressed scoring appears only in the ablations of \Cref{adxsubsubsec:sft-ablations,adxsubsubsec:rlhf-scoring-backend}.

\paragraph{Compared methods.}
Our experiments ask how a small target set should guide updates supported by a much larger general training pool. We therefore compare Full-Training Update, i.e., standard training on the general pool, with the two main data-regularized updates, Global Subset Update and Layer-Wise Subset Update, under the same models, data, optimizers, schedules, and step budgets; the two curated updates share the selection rule (top-\(k\) in SFT, negative score filtering in RL), the scoring backend, and all hyperparameters, and differ only in the partition.

\Cref{adxtab:experiment-overview} lists the model, fine-tuning method, scoring backend, selection rule, and target set of every setting, with pointers to its results; the details follow in the respective subsections. As baselines we take the state-of-the-art online selection method of each paradigm: GREATS~\citep{wang2024greats}, which outperforms loss-based~\citep{loshchilov2015online}, gradient-norm~\citep{katharopoulos2018not}, reference-model~\citep{mindermann2022prioritized}, and embedding-similarity~\citep{reimers2019sentence} selection on the general/target pairs we reuse, and IIF~\citep{hu2025a} in RLHF. Global Subset Update with top-\(k\) is the first-order variant of GREATS; in the question-answering settings, GREATS' redundancy-aware greedy construction selected nearly the same subsets as top-\(k\), so all curated arms use first-order alignment scores (\Cref{adxsubsec:global-subset-derivation}).\footnote{GREATS itself validates this term through the fidelity of its Taylor approximation rather than through a downstream comparison with first-order selection.}

\begin{table}[htpb]
	\centering
	\caption{Overview of the experimental settings: model and trainer, fine-tuning method, scoring backend, selection rule, target set, and where the results are reported (main text; appendix).}
	\label{adxtab:experiment-overview}
	\small
	\begin{adjustbox}{max width=\linewidth}
\setlength{\tabcolsep}{4pt}\footnotesize
\begin{tabular}{@{}l >{\raggedright\arraybackslash}p{3.5cm} >{\raggedright\arraybackslash}p{3.3cm} >{\raggedright\arraybackslash}p{3.7cm} >{\raggedright\arraybackslash}p{3.0cm}@{}}
	\toprule
	 & \textbf{Question answering} & \textbf{Reasoning and IF} & \textbf{RLHF} & \textbf{RLVR} \\
	 & \Cref{subsubsec:SFT-QA}; \Cref{adxsubsubsec:sft-qa} & \Cref{subsubsec:SFT-capability}; \Cref{adxsubsubsec:capability-setting} & \Cref{subsec:RLHF}; \Cref{adxsubsec:details-of-rlhf} & \Cref{subsec:RLVR}; \Cref{adxsubsec:details-of-rlvr} \\
	\midrule
	Model & \textsc{Llama-3.2-1B} & \textsc{Qwen3-1.7B-Base} & \textsc{GPT-Neo-2.7B} & \textsc{Qwen3-1.7B-Base} \\
	Trainer & SFT & SFT & PPO (\texttt{TRL}) & GRPO (\texttt{Verl}) \\
	Fine-tuning & Full, LoRA, MeSO & Full parameter & LoRA & Full parameter \\
	\midrule
	Scoring & Exact (GIP), compressed & Exact (PIP) & Exact (PIP), compressed & Exact (PIP) \\
	\addlinespace[3pt]
	Selection & Top-\(k\), \(k = 4\) of \(8\) & Top-\(k\), \(k \in \{2, 4, 6\}\) of \(8\) & Negative score filtering & Negative score filtering \\
	\addlinespace[3pt]
	Target set & \(16\) held-out examples & \(32\) held-out examples (\(16\)--\(128\) in the ablation) & Self-referencing rollouts, or \(1024\) held-out prompts; reward or PPO loss as signal & Training prompts; verifiable reward as signal \\
	\addlinespace[3pt]
	Target gradient & \(m = 1\) example, resampled every step & \(m = 1\) example, resampled every step & \(256\) prompts per PPO step, captured once per rollout batch and reused across its mini-batches & \(64\) prompts \(\times\) \(8\) rollouts, refreshed every fourth round and reused in between \\
	\midrule
	Additional baselines & GREATS, Random, Target-Only & Random & IIF & -- \\
	\midrule
	Results & \Cref{tab:SFT-downstream,adxtab:sft-qa-ppl,fig:SFT-lora-dynamics,fig:case-study,adxtab:sft-qa-matrix,adxtab:sft-ablation} & \Cref{tab:SFT-qwen,adxtab:sft-qwen-nval-if,adxtab:sft-qwen-nval-math} & \Cref{fig:RLHF,fig:RLHF-judges,fig:RLHF-compressed,adxtab:rlhf-scoring-backend} & \Cref{fig:RLVR-math,adxtab:rlvr} \\
	\bottomrule
\end{tabular}

	\end{adjustbox}
\end{table}

\paragraph{Computing resources.}
Our experiments are conducted on the NCSA Delta cluster.\footnote{\url{https://docs.ncsa.illinois.edu/systems/delta/en/latest/index.html}} SFT, RLHF, and benchmark experiments run on nodes equipped with a single \texttt{AMD EPYC 7763 CPU @ 2.45GHz} (64 cores) and four \texttt{NVIDIA A40 GPUs} (48 GB GDDR6 each). RLVR experiments run on nodes with dual \texttt{Intel Xeon Platinum 8558 CPUs} (96 cores total) and eight \texttt{NVIDIA H200 GPUs} (141 GB HBM3 each).

\subsection{Supervised Fine-Tuning}\label{adxsubsec:details-of-sft}
We now provide the missing details of the SFT experiments in \Cref{subsec:SFT}, following its two stages: the question-answering experiments with \textsc{Llama-3.2-1B} (\Cref{subsubsec:SFT-QA}), including the ablations of \Cref{subsubsec:SFT-ablations}, and the reasoning and instruction-following experiments with \textsc{Qwen3-1.7B-Base} (\Cref{subsubsec:SFT-capability}).

\subsubsection{Question Answering with \textsc{Llama-3.2-1B}}\label{adxsubsubsec:sft-qa}
\paragraph{Datasets.}
We consider four general/target pairs:
\begin{enumerate}[leftmargin=*]
	\item \texttt{alpaca}\footnote{\url{https://huggingface.co/datasets/tatsu-lab/alpaca}}~\citep{alpaca} (CC BY-NC 4.0 license)/\texttt{samsum}\footnote{\url{https://huggingface.co/datasets/knkarthick/samsum}}~\citep{gliwa2019samsum} (CC BY-NC-ND 4.0 license). Single-turn instruction-following/dialogue summarization. We randomly sample \(40\%\) of the \texttt{alpaca} training pool per seed (\(\sim\!20.8\)k examples).
	\item \texttt{less-mix}~\citep{xia2024less}/\texttt{tydiqa}\footnote{\url{https://huggingface.co/datasets/google-research-datasets/tydiqa}}~\citep{clark2020tydi} (Apache-2.0 license). Heterogeneous instruction mixture/multilingual extractive QA. We randomly sample \(0.5\%\) of \texttt{less-mix} per seed (\(\sim\!9.8\)k examples).
	\item \texttt{triviaqa}\footnote{\url{https://huggingface.co/datasets/mandarjoshi/trivia_qa}}~\citep{joshi2017triviaqa}/\texttt{nq\_open}\footnote{\url{https://huggingface.co/datasets/google-research-datasets/nq_open}}~\citep{kwiatkowski2019natural} (CC BY-SA 3.0 license). Closed-book TriviaQA QA pairs (\texttt{rc.nocontext}, \(\sim\!138\)k examples)/closed-book Natural Questions evaluation. We randomly sample \(5\%\) of the \texttt{triviaqa} training pool per seed (\(\sim\!6.9\)k examples).
	\item \texttt{less-mix}~\citep{xia2024less}/\texttt{squad}\footnote{\url{https://huggingface.co/datasets/rajpurkar/squad}}~\citep{rajpurkar2016squad} (CC BY-SA 4.0 license). Same training pool as the \texttt{tydiqa} setting, with the target task replaced by SQuAD reading-comprehension answers without context. \(0.5\%\) sampling rate as above.
\end{enumerate}

For each target task we hold out \(16\) samples for data regularization and \(500\) held-out samples for in-training perplexity tracking, on which we also report the final-checkpoint downstream metric (Rouge-L for \texttt{samsum}; F1 for \texttt{tydiqa}, \texttt{nq\_open}, and \texttt{squad}). \textsc{Llama-3.2-1B} ships without a chat template, so we install an \texttt{open-instruct}-style fallback (\(\langle\)user\(\rangle\)/\(\langle\)assistant\(\rangle\) plaintext markers) before tokenization, with the loss computed only on assistant-content tokens.

\paragraph{Training configurations.}
All models are trained with standard supervised fine-tuning using next-token prediction on formatted instruction--response pairs. Across the four settings, we report results for
\begin{enumerate*}[label=(\roman*)]
	\item full-parameter fine-tuning,
	\item LoRA fine-tuning, and
	\item MeSO;
\end{enumerate*}
All three fine-tuning methods are reported on all four settings (\Cref{adxsubsubsec:sft-ablations}). All methods use the same number of optimization steps within each dataset pair and the same effective batch size. We use AdamW~\citep{loshchilov2018decoupled} with \(\beta_1 = 0.9\), \(\beta_2 = 0.999\), weight decay \(0\), a linear learning-rate schedule with \(3\%\) warmup, batch size \(8\), maximum sequence length \(512\), and train for \(1\) epoch. Training uses \texttt{bfloat16} precision and FlashAttention-2. For the learning rate, we follow \citet{grattafiori2024llama}, where we use \(10^{-5}\) for full-parameter fine-tuning, \(10^{-4}\) for LoRA (a \(10\times\) scaling~\citep{schulman2025lora}), and \(5 \times 10^{-5}\) for MeSO.

The number of training steps is determined by the \(1\)-epoch budget on the sampled training pool: \(\sim\!2600\) for \texttt{alpaca}/\texttt{samsum}, \(\sim\!1225\) for \texttt{less-mix}/\texttt{tydiqa} and \texttt{less-mix}/\texttt{squad}, and \(\sim\!865\) for \texttt{triviaqa}/\texttt{nq\_open}. We log evaluation perplexity at \(\sim\!100\) equally spaced steps across the trajectory (\(80\) for \texttt{triviaqa}/\texttt{nq\_open}).

For LoRA, we use rank \(r = 8\), scaling factor \(\alpha = 16\), dropout \(0.1\), and apply adapters to all linear layers (the full Q/K/V/O attention projections and the gate/up/down MLP projections). For MeSO, we use the gradient compressor described in \Cref{adxsubsubsec:per-sample-gradient-compression} with a Gaussian projector for both \(P_{\text{out}}^{(l)}\) and \(P_{\text{in}}^{(l)}\) of dimension \(512\) (i.e., the total per-layer compression dimension is \(\kappa^{(l)} = 512 \times 512\)), and no second-stage projector (\(P^{(l)}_{\text{final}} = I\)), refreshing the subspace every \(200\) steps; the tied input/output embedding matrix and the normalization weights are updated by standard AdamW without compression.

\paragraph{Data regularization.}
Data regularization is guided by a small held-out target set of \(16\) samples. At each training step, we form a merged batch of \(n=8\) training samples and \(m=1\) target sample (resampled per step from the \(16\) samples held-out target pool), and score the candidate training examples either exactly (GIP, \Cref{subsubsec:direct-gip}) or with compressed scoring (\Cref{subsubsec:compressed-scoring}, \(\kappa^{(l)} = 64 \times 64\)); both backends are run for every curated arm under all three fine-tuning methods (\Cref{adxsubsubsec:sft-ablations}), and the main-text results use exact scoring.

Both curated updates in \Cref{subsubsec:SFT-QA} solve for \(S_t\) (or \(S_{t,l}\) in the case of Layer-Wise Subset Update) by the top-\(k\) strategy with a retaining rate of \(50\%\), i.e., \(k = n/2 = 4\) of the \(n = 8\) candidate training examples per step. Global Subset Update with top-\(k\) selection is the first-order variant of the online data selection baseline GREATS~\citep{wang2024greats}, which introduced the first two general/target pairs and whose greedy construction additionally accounts for the redundancy penalty (\Cref{adxsubsec:global-subset-derivation}); we use the same top-\(k\) rule for both curated arms so that they differ only in partition granularity. Target-Only Update trains on the \(16\) target examples alone, with the learning rate and schedule of Full-Training, and is reported under two protocols: with the step budget of the other arms (\emph{same budget}) and stopped after \(50\) steps (\emph{early stop}, the protocol of \Cref{tab:SFT-downstream}). For a given seed these are exactly the target examples the curated arms score against.

\paragraph{Results.}
\Cref{tab:SFT-downstream} in the main text reports the LoRA runs with exact scoring, \Cref{fig:SFT-lora-dynamics} shows their evaluation perplexity on the \(500\) held-out examples over training, and \Cref{adxtab:sft-qa-ppl} reports the final evaluation perplexity of the three main update rules and of Target-Only Update under every fine-tuning method.

\newlength{\sftrowh}\newlength{\sftroww}
\begin{figure}[htpb]
	\centering
	\setlength{\sftrowh}{0.155\linewidth}\setlength{\sftroww}{0.2324\linewidth}
	\begin{subfigure}[b]{0.252\linewidth}
		\caption{\texttt{alpaca}/\texttt{samsum}}
		\centering
		\includegraphics[height=\sftrowh,viewport=0 0 41 347.69,clip]{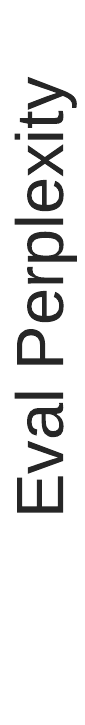}%
		\includegraphics[height=\sftrowh,width=\sftroww,trim=4 0 0 0,clip]{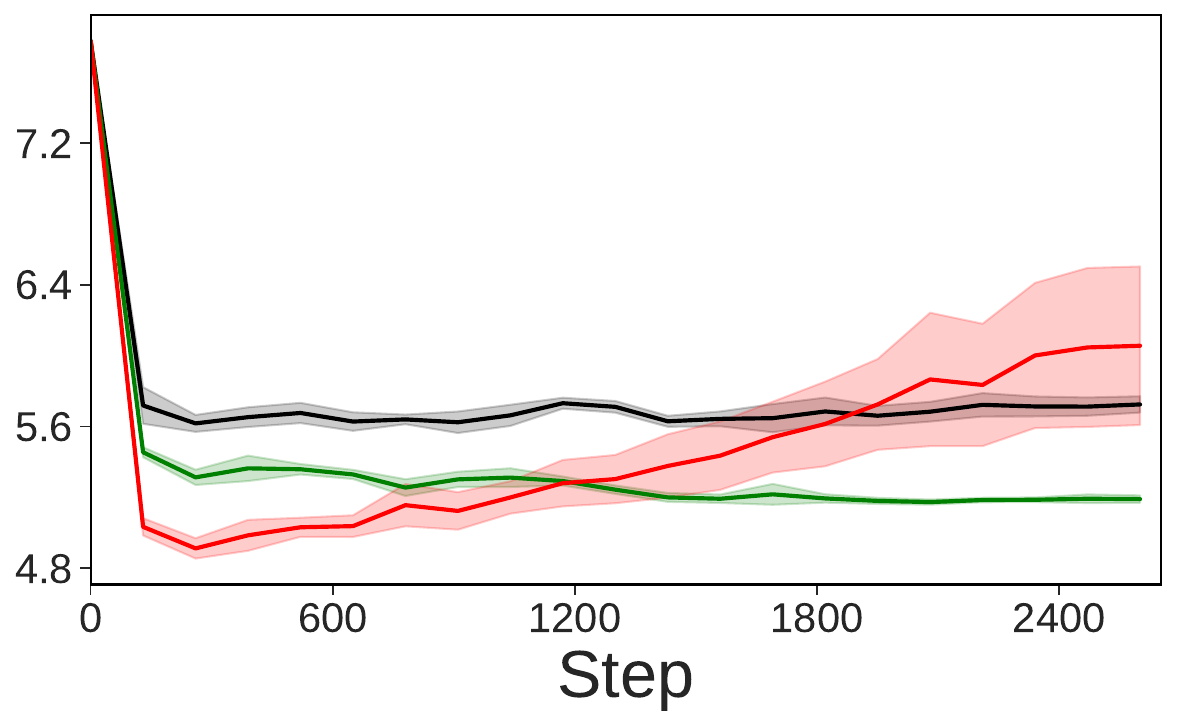}
	\end{subfigure}\hfill
	\begin{subfigure}[b]{0.234\linewidth}
		\caption{\texttt{less-mix}/\texttt{tydiqa}}
		\centering
		\includegraphics[height=\sftrowh,trim=41 0 0 0,clip]{Figures/SFT/less_tydiqa_LoRA.pdf}
	\end{subfigure}\hfill
	\begin{subfigure}[b]{0.25\linewidth}
		\caption{\texttt{triviaqa}/\texttt{nq\_open}}
		\centering
		\includegraphics[height=\sftrowh,trim=41 0 0 0,clip]{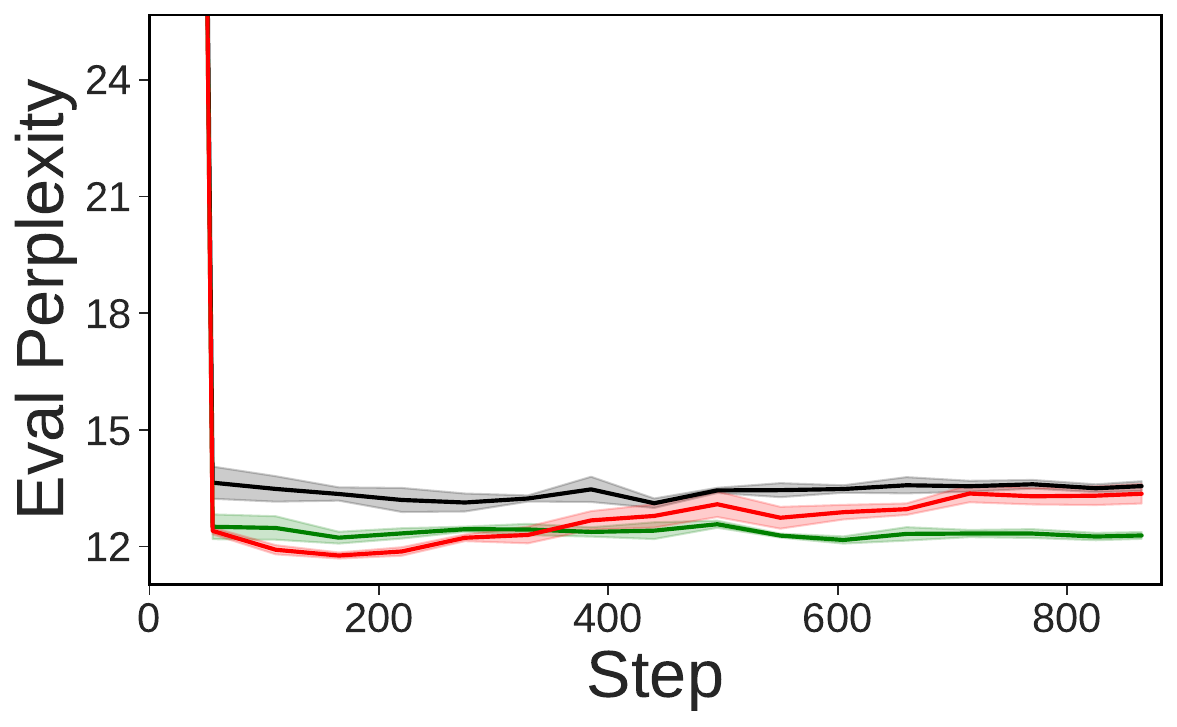}
	\end{subfigure}\hfill
	\begin{subfigure}[b]{0.234\linewidth}
		\caption{\texttt{less-mix}/\texttt{squad}}
		\centering
		\includegraphics[height=\sftrowh,trim=41 0 0 0,clip]{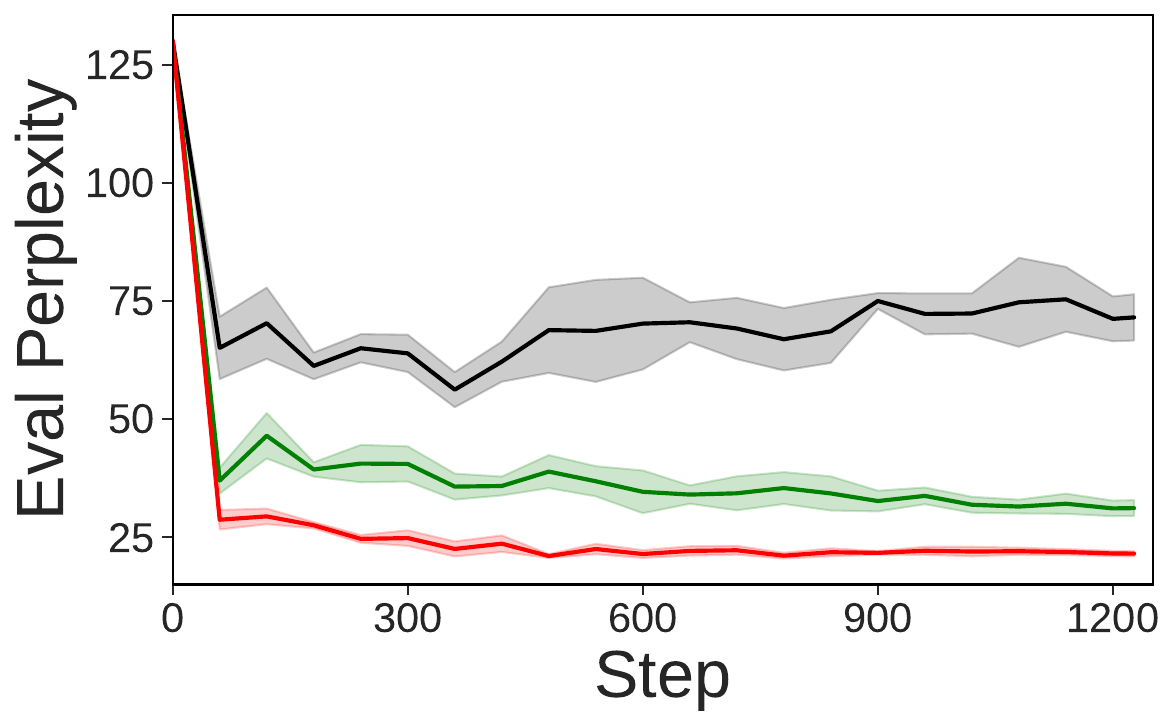}
	\end{subfigure}
	\par\smallskip
	\includegraphics[width=0.5\linewidth]{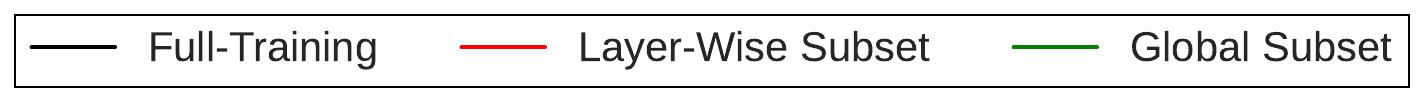}
	\caption{Question-answering SFT (LoRA, exact scoring): evaluation perplexity on the \(500\) held-out examples over training, mean over \(5\) seeds with standard-error bands.}
	\label{fig:SFT-lora-dynamics}
\end{figure}

Two points stand out. First, the ordering Layer-Wise \(<\) Global \(<\) Full-Training in perplexity holds on nearly every setting/fine-tuning pair: Global Subset Update ends below Full-Training on all pairs but one, and Layer-Wise Subset Update ends below Global Subset Update on all pairs but two. The gaps are widest under the broad \texttt{less-mix} pool, whose general distribution is most misaligned with the targets, and narrowest on \texttt{triviaqa}/\texttt{nq\_open}, whose general pool is already close to the target. Second, no method is tuned: every arm runs with the learning rate, schedule, and step budget of Full-Training, although their training dynamics differ. In \Cref{fig:SFT-lora-dynamics}, Layer-Wise Subset Update falls below the other arms within the first few hundred steps on every setting, and on \texttt{alpaca}/\texttt{samsum} and \texttt{triviaqa}/\texttt{nq\_open} under LoRA it then overfits: its held-out perplexity reaches its minimum, below the minimum of every other arm, after about a fifth of the budget and climbs for the rest of training until it ends above Global Subset Update, while its perplexity on the \(16\) target examples keeps falling. These two LoRA pairs are two of the three exceptions to the ordering above; under full-parameter fine-tuning the same arm ends close to its minimum and the ordering is intact. The third exception is \texttt{triviaqa}/\texttt{nq\_open} under MeSO, where every arm ends well above its counterparts under the other two fine-tuning methods and Full-Training ends lowest, the price of the compressed optimizer update at the larger step size. The least regularized arm thus converges fastest and would benefit most from a shorter budget or a smaller step size, so the gains we report for it are conservative.

\paragraph{Target-Only Update.}
Under the shared budget, one pass over the pool and hence hundreds of passes over the target examples, the pool-free arm behaves as \Cref{subsec:bias-variance-tradeoffs} predicts for a small target set: its held-out perplexity is minimal within the first few dozen steps and diverges as the examples are memorized, ending far above every other arm on every target but \texttt{less-mix}/\texttt{tydiqa} (\Cref{adxtab:sft-qa-ppl}). Nothing regularizes an arm without a pool, hence the early-stopped protocol, an oracle stop near that minimum that favors Target-Only Update; even so it ends above both curated updates on three of the four targets. The downstream metric adds what the loss does not isolate, the style of the answer. On \texttt{samsum} and \texttt{tydiqa} the content is in the prompt and the references are short summaries or spans, whereas the pools answer in paragraphs, so style decides the metric; the target examples fix it within the first steps and greedy decoding keeps it after the perplexity has diverged, so Target-Only Update is ahead of or level with the best pool-based arm under LoRA and full-parameter fine-tuning (under MeSO, whose compressed update learns more slowly, the early stop is too short on \texttt{tydiqa}), and the curated updates recover part of the gap by keeping the shortest pool examples. On the closed-book targets the metric depends on knowledge that the target examples cannot supply and that memorizing them erodes, so it never comes first under either protocol: style is what a handful of examples can teach, knowledge is not.

\begin{table}[htpb]
	\centering
	\caption{Final evaluation perplexity on the \(500\) held-out examples of the three main update rules and Target-Only Update in the question-answering experiments.}
	\label{adxtab:sft-qa-ppl}
	\small
	\begin{adjustbox}{max width=\linewidth}
\begin{tabular}{l l c c c c c}
	\toprule
	\textbf{Setting} & \textbf{Fine-tuning} & \textbf{Full-Training} & \textbf{Global Subset} & \textbf{Layer-Wise Subset} & \textbf{Target-Only (same budget)} & \textbf{Target-Only (early stop)} \\
	\midrule
	\multirow{3}{*}{\texttt{alpaca}/\texttt{samsum}} & Full parameter & \(5.67\,{\scriptstyle\pm 0.02}\) & \(5.30\,{\scriptstyle\pm 0.02}\) & \(\mathbf{4.80\,{\scriptstyle\pm 0.03}}\) & \(62.42\,{\scriptstyle\pm 5.16}\) & \(21.36\,{\scriptstyle\pm 5.34}\) \\
	 & LoRA & \(5.72\,{\scriptstyle\pm 0.05}\) & \(\mathbf{5.19\,{\scriptstyle\pm 0.02}}\) & \(6.06\,{\scriptstyle\pm 0.45}\) & \(289.66\,{\scriptstyle\pm 21.36}\) & \(12.21\,{\scriptstyle\pm 0.66}\) \\
	 & MeSO & \(6.69\,{\scriptstyle\pm 0.08}\) & \(6.53\,{\scriptstyle\pm 0.09}\) & \(\mathbf{6.20\,{\scriptstyle\pm 0.15}}\) & \(89.51\,{\scriptstyle\pm 8.35}\) & \(64.01\,{\scriptstyle\pm 11.58}\) \\
	\midrule
	\multirow{3}{*}{\texttt{less-mix}/\texttt{tydiqa}} & Full parameter & \(3.30\,{\scriptstyle\pm 0.03}\) & \(2.31\,{\scriptstyle\pm 0.04}\) & \(\mathbf{1.72\,{\scriptstyle\pm 0.03}}\) & \(1.68\,{\scriptstyle\pm 0.04}\) & \(1.78\,{\scriptstyle\pm 0.03}\) \\
	 & LoRA & \(3.50\,{\scriptstyle\pm 0.10}\) & \(2.19\,{\scriptstyle\pm 0.04}\) & \(\mathbf{1.51\,{\scriptstyle\pm 0.03}}\) & \(1.85\,{\scriptstyle\pm 0.06}\) & \(1.75\,{\scriptstyle\pm 0.07}\) \\
	 & MeSO & \(2.83\,{\scriptstyle\pm 0.06}\) & \(2.55\,{\scriptstyle\pm 0.38}\) & \(\mathbf{1.63\,{\scriptstyle\pm 0.02}}\) & \(1.98\,{\scriptstyle\pm 0.07}\) & \(2.84\,{\scriptstyle\pm 0.15}\) \\
	\midrule
	\multirow{3}{*}{\texttt{triviaqa}/\texttt{nq\_open}} & Full parameter & \(13.40\,{\scriptstyle\pm 0.06}\) & \(12.37\,{\scriptstyle\pm 0.07}\) & \(\mathbf{11.98\,{\scriptstyle\pm 0.09}}\) & \(77.55\,{\scriptstyle\pm 8.35}\) & \(36.98\,{\scriptstyle\pm 3.01}\) \\
	 & LoRA & \(13.55\,{\scriptstyle\pm 0.13}\) & \(\mathbf{12.29\,{\scriptstyle\pm 0.09}}\) & \(13.36\,{\scriptstyle\pm 0.25}\) & \(168.40\,{\scriptstyle\pm 24.93}\) & \(32.92\,{\scriptstyle\pm 2.51}\) \\
	 & MeSO & \(\mathbf{18.91\,{\scriptstyle\pm 0.18}}\) & \(20.19\,{\scriptstyle\pm 0.22}\) & \(19.40\,{\scriptstyle\pm 0.37}\) & \(246.82\,{\scriptstyle\pm 39.75}\) & \(160.66\,{\scriptstyle\pm 21.85}\) \\
	\midrule
	\multirow{3}{*}{\texttt{less-mix}/\texttt{squad}} & Full parameter & \(81.78\,{\scriptstyle\pm 3.40}\) & \(42.53\,{\scriptstyle\pm 2.29}\) & \(\mathbf{27.06\,{\scriptstyle\pm 1.23}}\) & \(109.41\,{\scriptstyle\pm 5.82}\) & \(55.45\,{\scriptstyle\pm 6.05}\) \\
	 & LoRA & \(71.49\,{\scriptstyle\pm 4.90}\) & \(31.11\,{\scriptstyle\pm 1.69}\) & \(\mathbf{21.47\,{\scriptstyle\pm 0.57}}\) & \(463.25\,{\scriptstyle\pm 84.54}\) & \(48.10\,{\scriptstyle\pm 5.45}\) \\
	 & MeSO & \(69.99\,{\scriptstyle\pm 5.42}\) & \(48.25\,{\scriptstyle\pm 5.47}\) & \(\mathbf{35.27\,{\scriptstyle\pm 1.38}}\) & \(434.70\,{\scriptstyle\pm 76.23}\) & \(231.68\,{\scriptstyle\pm 13.29}\) \\
	\bottomrule
\end{tabular}

	\end{adjustbox}
\end{table}

\subsubsection{Ablations on the Question-Answering Setting}\label{adxsubsubsec:sft-ablations}
We now provide the omitted details of the ablations in \Cref{subsubsec:SFT-ablations}. Unless stated otherwise, every arm follows the protocol of \Cref{adxsubsubsec:sft-qa} unchanged: the same pools, target sets of \(16\) examples, \(n = 8\), \(m = 1\), \(k = n/2\), learning rate, schedule, and step budget, evaluated at the final checkpoint on the \(500\) held-out examples over \(5\) seeds. \Cref{adxtab:sft-qa-matrix-lora,adxtab:sft-qa-matrix} report every partition under both scoring backends for LoRA, full-parameter fine-tuning, and MeSO, and \Cref{adxtab:sft-ablation} the controls.
\begin{itemize}[leftmargin=*]
	\item \emph{Random Subset controls.} Random Subset Update draws a uniformly random subset of \(k = 4\) candidates at every step and trains on it through the same update path as the curated arm it mirrors, so the two differ only in how the subset is chosen: Global Random Subset draws one subset per step that all layers share, and Layer-Wise Random Subset draws an independent subset for every layer.
	\item \emph{Partition granularity.} \Cref{adxtab:partition-groups} lists the four partitions and their group counts, which depend on which layers are scored. Under LoRA the scored layers are the \(224\) adapter factors, the \(A\) and \(B\) matrices of each of the \(112\) projections, so Layer-Wise Subset Update selects a separate subset for the two factors of one projection, and the coarser partitions group all factors of a transformer block or of an attention/MLP sublayer. Under full-parameter fine-tuning and MeSO the scored layers are the \(112\) projections together with the token embedding and the output head, which lie outside the transformer blocks and are singleton groups under every partition except the global one. Each group sums the scores of its layers and keeps its own top-\(k\) subset through the same single-pass update as the layer-wise arm (\Cref{subsubsec:group-wise-decomposition}).
	\item \emph{Layer-normalized scores.} The case study (\Cref{adxsubsubsec:case-study}) admits two readings of why the shared ranking fails: a few layers dominate it by the magnitude of their scores, or the layers' rankings disagree. The first reading suggests a cheaper fix than the group-wise relaxation, so we test a layer-normalized variant of Global Subset Update that divides each layer's score vector by its mean absolute value before the vectors are summed into the shared score, so that no layer can dominate the shared ranking by magnitude alone. Under the second reading, rescaling only changes which layers decide the shared ranking and still keeps one subset for all layers, so it cannot serve layers whose rankings disagree, whereas the group-wise update follows each layer's own ranking. This variant and the controls are scored with compressed scoring, so \Cref{adxtab:sft-ablation} lists the compressed-scoring rows of the two endpoints next to them.
	\item \emph{Fine-tuning methods.} Every arm, including Target-Only Update, is run with full-parameter fine-tuning, LoRA, and MeSO on all four settings, with the learning rates of \Cref{adxsubsubsec:sft-qa}; the fine-tuning method changes which parameters are updated and how, whereas scoring and selection are unchanged.
	\item \emph{Scoring backend.} Every curated arm is scored both exactly and with compressed scoring at \(\kappa^{(l)} = 64 \times 64\), under all three fine-tuning methods.
\end{itemize}

\begin{table}[htpb]
	\centering
	\caption{Partitions of the question-answering ablation and their number of groups \(P\) on \textsc{Llama-3.2-1B} under each fine-tuning method. Under LoRA the scored layers are the adapter factors, so the two factors of one projection are separate layers; under full-parameter fine-tuning and MeSO the token embedding and the output head are singleton groups under every partition except the global one.}
	\label{adxtab:partition-groups}
	\small
	\begin{adjustbox}{max width=\linewidth}
\begin{tabular}{l l c c}
	\toprule
	\multirow{3}{*}{\textbf{Partition}} & \multirow{3}{*}{\textbf{One group per}} & \multicolumn{2}{c}{\textbf{Number of groups \(P\)}} \\
	\cmidrule(lr){3-4}
	 & & LoRA & Full parameter, MeSO \\
	 & & (\(224\) adapter factors) & (\(112\) projections, embedding, output head) \\
	\midrule
	Global Subset & all scored layers & \(1\) & \(1\) \\
	Block-Wise Subset & transformer block & \(16\) (\(14\) factors each) & \(18\) (\(16\) blocks, \(2\) singletons) \\
	Sublayer-Wise Subset & attention or MLP sublayer & \(32\) (\(8\) or \(6\) factors each) & \(34\) (\(32\) sublayers, \(2\) singletons) \\
	Layer-Wise Subset & scored layer & \(224\) & \(114\) \\
	\bottomrule
\end{tabular}

	\end{adjustbox}
\end{table}

\begin{table}[htpb]
	\centering
	\caption{Partition granularity and scoring backend on the question-answering settings under full-parameter fine-tuning and MeSO: downstream metric (\%) at the final checkpoint, in the layout of \Cref{adxtab:sft-qa-matrix-lora}, which reports the LoRA runs. Exact scoring is GIP, compressed scoring uses \(\kappa^{(l)} = 64 \times 64\), and Full-Training and Target-Only Update do not score.}
	\label{adxtab:sft-qa-matrix}
	\small
	\begin{subtable}{\linewidth}
		\centering
		\subcaption{Full-parameter fine-tuning.}
		\label{adxtab:sft-qa-matrix-full}
		\begin{adjustbox}{max width=\linewidth}
\begin{tabular}{l c c c c c c c c}
	\toprule
	\multirow{2}{*}{\textbf{Method}} & \multicolumn{2}{c}{\texttt{alpaca}/\texttt{samsum}} & \multicolumn{2}{c}{\texttt{less-mix}/\texttt{tydiqa}} & \multicolumn{2}{c}{\texttt{triviaqa}/\texttt{nq\_open}} & \multicolumn{2}{c}{\texttt{less-mix}/\texttt{squad}} \\
	\cmidrule(lr){2-3}\cmidrule(lr){4-5}\cmidrule(lr){6-7}\cmidrule(lr){8-9}
	 & Exact & Compressed & Exact & Compressed & Exact & Compressed & Exact & Compressed \\
	\midrule
	Full-Training & \multicolumn{2}{c}{\(20.4\,{\scriptstyle\pm 0.2}\)} & \multicolumn{2}{c}{\(3.5\,{\scriptstyle\pm 0.2}\)} & \multicolumn{2}{c}{\(13.2\,{\scriptstyle\pm 0.1}\)} & \multicolumn{2}{c}{\(8.4\,{\scriptstyle\pm 0.1}\)} \\
	\midrule
	Global Subset & \(20.9\,{\scriptstyle\pm 0.7}\) & \(20.1\,{\scriptstyle\pm 0.7}\) & \(8.8\,{\scriptstyle\pm 1.6}\) & \(6.4\,{\scriptstyle\pm 0.8}\) & \(13.6\,{\scriptstyle\pm 0.3}\) & \(13.1\,{\scriptstyle\pm 0.2}\) & \(8.5\,{\scriptstyle\pm 0.2}\) & \(8.4\,{\scriptstyle\pm 0.1}\) \\
	Block-Wise Subset & \(24.6\,{\scriptstyle\pm 1.2}\) & \(23.7\,{\scriptstyle\pm 0.8}\) & \(18.7\,{\scriptstyle\pm 1.7}\) & \(11.5\,{\scriptstyle\pm 0.6}\) & \(13.7\,{\scriptstyle\pm 0.3}\) & \(13.7\,{\scriptstyle\pm 0.2}\) & \(8.5\,{\scriptstyle\pm 0.1}\) & \(8.6\,{\scriptstyle\pm 0.1}\) \\
	Sublayer-Wise Subset & \(26.3\,{\scriptstyle\pm 0.9}\) & \(23.6\,{\scriptstyle\pm 1.0}\) & \(22.6\,{\scriptstyle\pm 1.4}\) & \(\mathbf{16.9\,{\scriptstyle\pm 1.8}}\) & \(14.1\,{\scriptstyle\pm 0.2}\) & \(13.5\,{\scriptstyle\pm 0.1}\) & \(8.6\,{\scriptstyle\pm 0.2}\) & \(8.4\,{\scriptstyle\pm 0.1}\) \\
	Layer-Wise Subset & \(\mathbf{26.9\,{\scriptstyle\pm 0.6}}\) & \(\mathbf{24.1\,{\scriptstyle\pm 1.0}}\) & \(\mathbf{32.7\,{\scriptstyle\pm 1.8}}\) & \(16.3\,{\scriptstyle\pm 2.2}\) & \(\mathbf{14.2\,{\scriptstyle\pm 0.3}}\) & \(\mathbf{13.8\,{\scriptstyle\pm 0.2}}\) & \(\mathbf{9.0\,{\scriptstyle\pm 0.2}}\) & \(\mathbf{8.9\,{\scriptstyle\pm 0.1}}\) \\
	\midrule
	Target-Only (same budget) & \multicolumn{2}{c}{\(30.8\,{\scriptstyle\pm 0.3}\)} & \multicolumn{2}{c}{\(56.2\,{\scriptstyle\pm 2.2}\)} & \multicolumn{2}{c}{\(8.2\,{\scriptstyle\pm 0.7}\)} & \multicolumn{2}{c}{\(8.6\,{\scriptstyle\pm 0.3}\)} \\
	Target-Only (early stop) & \multicolumn{2}{c}{\(32.1\,{\scriptstyle\pm 0.3}\)} & \multicolumn{2}{c}{\(49.1\,{\scriptstyle\pm 1.8}\)} & \multicolumn{2}{c}{\(10.1\,{\scriptstyle\pm 0.5}\)} & \multicolumn{2}{c}{\(8.8\,{\scriptstyle\pm 0.3}\)} \\
	\bottomrule
\end{tabular}

		\end{adjustbox}
	\end{subtable}
	\par\medskip
	\begin{subtable}{\linewidth}
		\centering
		\subcaption{MeSO.}
		\label{adxtab:sft-qa-matrix-meso}
		\begin{adjustbox}{max width=\linewidth}
\begin{tabular}{l c c c c c c c c}
	\toprule
	\multirow{2}{*}{\textbf{Method}} & \multicolumn{2}{c}{\texttt{alpaca}/\texttt{samsum}} & \multicolumn{2}{c}{\texttt{less-mix}/\texttt{tydiqa}} & \multicolumn{2}{c}{\texttt{triviaqa}/\texttt{nq\_open}} & \multicolumn{2}{c}{\texttt{less-mix}/\texttt{squad}} \\
	\cmidrule(lr){2-3}\cmidrule(lr){4-5}\cmidrule(lr){6-7}\cmidrule(lr){8-9}
	 & Exact & Compressed & Exact & Compressed & Exact & Compressed & Exact & Compressed \\
	\midrule
	Full-Training & \multicolumn{2}{c}{\(19.3\,{\scriptstyle\pm 0.7}\)} & \multicolumn{2}{c}{\(4.2\,{\scriptstyle\pm 0.2}\)} & \multicolumn{2}{c}{\(\mathbf{11.7\,{\scriptstyle\pm 0.3}}\)} & \multicolumn{2}{c}{\(\mathbf{7.4\,{\scriptstyle\pm 0.2}}\)} \\
	\midrule
	Global Subset & \(20.1\,{\scriptstyle\pm 0.9}\) & \(19.6\,{\scriptstyle\pm 0.6}\) & \(11.1\,{\scriptstyle\pm 4.0}\) & \(8.9\,{\scriptstyle\pm 2.3}\) & \(9.7\,{\scriptstyle\pm 0.5}\) & \(9.0\,{\scriptstyle\pm 1.7}\) & \(6.7\,{\scriptstyle\pm 0.2}\) & \(6.9\,{\scriptstyle\pm 0.2}\) \\
	Block-Wise Subset & \(25.0\,{\scriptstyle\pm 0.4}\) & \(22.1\,{\scriptstyle\pm 0.5}\) & \(22.4\,{\scriptstyle\pm 1.6}\) & \(15.6\,{\scriptstyle\pm 3.2}\) & \(10.1\,{\scriptstyle\pm 0.4}\) & \(10.4\,{\scriptstyle\pm 0.6}\) & \(7.1\,{\scriptstyle\pm 0.3}\) & \(6.9\,{\scriptstyle\pm 0.2}\) \\
	Sublayer-Wise Subset & \(24.2\,{\scriptstyle\pm 0.6}\) & \(\mathbf{24.4\,{\scriptstyle\pm 0.6}}\) & \(30.6\,{\scriptstyle\pm 2.0}\) & \(17.8\,{\scriptstyle\pm 2.2}\) & \(10.7\,{\scriptstyle\pm 0.3}\) & \(10.7\,{\scriptstyle\pm 0.6}\) & \(6.9\,{\scriptstyle\pm 0.2}\) & \(6.9\,{\scriptstyle\pm 0.2}\) \\
	Layer-Wise Subset & \(\mathbf{25.8\,{\scriptstyle\pm 0.6}}\) & \(\mathbf{24.4\,{\scriptstyle\pm 1.3}}\) & \(\mathbf{32.7\,{\scriptstyle\pm 2.8}}\) & \(\mathbf{28.5\,{\scriptstyle\pm 1.5}}\) & \(10.9\,{\scriptstyle\pm 0.5}\) & \(10.4\,{\scriptstyle\pm 0.2}\) & \(6.8\,{\scriptstyle\pm 0.3}\) & \(6.7\,{\scriptstyle\pm 0.1}\) \\
	\midrule
	Target-Only (same budget) & \multicolumn{2}{c}{\(31.7\,{\scriptstyle\pm 0.4}\)} & \multicolumn{2}{c}{\(36.0\,{\scriptstyle\pm 1.8}\)} & \multicolumn{2}{c}{\(4.0\,{\scriptstyle\pm 0.4}\)} & \multicolumn{2}{c}{\(4.1\,{\scriptstyle\pm 0.4}\)} \\
	Target-Only (early stop) & \multicolumn{2}{c}{\(30.5\,{\scriptstyle\pm 0.4}\)} & \multicolumn{2}{c}{\(21.8\,{\scriptstyle\pm 1.5}\)} & \multicolumn{2}{c}{\(3.3\,{\scriptstyle\pm 0.4}\)} & \multicolumn{2}{c}{\(2.8\,{\scriptstyle\pm 0.4}\)} \\
	\bottomrule
\end{tabular}

		\end{adjustbox}
	\end{subtable}
\end{table}

\begin{table}[htpb]
	\centering
	\caption{Random Subset and layer-normalized controls under the three fine-tuning methods: downstream metric (\%) at the final checkpoint. The controls do not score; the Global Subset, layer-normalized, and Layer-Wise Subset rows use compressed scoring, and the two endpoints repeat the compressed columns of \Cref{adxtab:sft-qa-matrix-lora,adxtab:sft-qa-matrix}.}
	\label{adxtab:sft-ablation}
	\small
	\begin{adjustbox}{max width=\linewidth}
\begin{tabular}{l ccc ccc ccc ccc}
	\toprule
	\multirow{2}{*}{\textbf{Method}} & \multicolumn{3}{c}{\texttt{alpaca}/\texttt{samsum}} & \multicolumn{3}{c}{\texttt{less-mix}/\texttt{tydiqa}} & \multicolumn{3}{c}{\texttt{triviaqa}/\texttt{nq\_open}} & \multicolumn{3}{c}{\texttt{less-mix}/\texttt{squad}} \\
	\cmidrule(lr){2-4}\cmidrule(lr){5-7}\cmidrule(lr){8-10}\cmidrule(lr){11-13}
	 & Full param. & LoRA & MeSO & Full param. & LoRA & MeSO & Full param. & LoRA & MeSO & Full param. & LoRA & MeSO \\
	\midrule
	Full-Training & \(20.4\,{\scriptstyle\pm 0.2}\) & \(22.4\,{\scriptstyle\pm 0.3}\) & \(19.3\,{\scriptstyle\pm 0.7}\) & \(3.5\,{\scriptstyle\pm 0.2}\) & \(4.9\,{\scriptstyle\pm 0.3}\) & \(4.2\,{\scriptstyle\pm 0.2}\) & \(13.2\,{\scriptstyle\pm 0.1}\) & \(12.7\,{\scriptstyle\pm 0.2}\) & \(\mathbf{11.7\,{\scriptstyle\pm 0.3}}\) & \(8.4\,{\scriptstyle\pm 0.1}\) & \(9.0\,{\scriptstyle\pm 0.2}\) & \(\mathbf{7.4\,{\scriptstyle\pm 0.2}}\) \\
	Global Random Subset & \(20.1\,{\scriptstyle\pm 0.4}\) & \(22.2\,{\scriptstyle\pm 0.3}\) & \(18.3\,{\scriptstyle\pm 0.7}\) & \(3.5\,{\scriptstyle\pm 0.1}\) & \(3.9\,{\scriptstyle\pm 0.2}\) & \(6.6\,{\scriptstyle\pm 2.6}\) & \(12.8\,{\scriptstyle\pm 0.2}\) & \(13.0\,{\scriptstyle\pm 0.3}\) & \(10.4\,{\scriptstyle\pm 0.5}\) & \(8.1\,{\scriptstyle\pm 0.2}\) & \(8.3\,{\scriptstyle\pm 0.1}\) & \(6.7\,{\scriptstyle\pm 0.3}\) \\
	Layer-Wise Random Subset & \(20.6\,{\scriptstyle\pm 0.1}\) & \(21.9\,{\scriptstyle\pm 0.3}\) & \(18.0\,{\scriptstyle\pm 0.7}\) & \(3.2\,{\scriptstyle\pm 0.1}\) & \(6.0\,{\scriptstyle\pm 0.2}\) & \(4.5\,{\scriptstyle\pm 0.4}\) & \(13.1\,{\scriptstyle\pm 0.2}\) & \(12.7\,{\scriptstyle\pm 0.3}\) & \(10.2\,{\scriptstyle\pm 0.4}\) & \(8.2\,{\scriptstyle\pm 0.1}\) & \(8.9\,{\scriptstyle\pm 0.2}\) & \(7.1\,{\scriptstyle\pm 0.2}\) \\
	\midrule
	Global Subset & \(20.1\,{\scriptstyle\pm 0.7}\) & \(22.3\,{\scriptstyle\pm 0.2}\) & \(19.6\,{\scriptstyle\pm 0.6}\) & \(6.4\,{\scriptstyle\pm 0.8}\) & \(12.8\,{\scriptstyle\pm 1.5}\) & \(8.9\,{\scriptstyle\pm 2.3}\) & \(13.1\,{\scriptstyle\pm 0.2}\) & \(\mathbf{13.8\,{\scriptstyle\pm 0.2}}\) & \(9.0\,{\scriptstyle\pm 1.7}\) & \(8.4\,{\scriptstyle\pm 0.1}\) & \(9.0\,{\scriptstyle\pm 0.2}\) & \(6.9\,{\scriptstyle\pm 0.2}\) \\
	Global Subset, layer-normalized scores & \(22.0\,{\scriptstyle\pm 0.5}\) & \(24.0\,{\scriptstyle\pm 0.5}\) & \(19.7\,{\scriptstyle\pm 0.5}\) & \(9.5\,{\scriptstyle\pm 1.1}\) & \(10.5\,{\scriptstyle\pm 0.9}\) & \(15.6\,{\scriptstyle\pm 3.7}\) & \(13.5\,{\scriptstyle\pm 0.3}\) & \(13.4\,{\scriptstyle\pm 0.4}\) & \(10.3\,{\scriptstyle\pm 0.2}\) & \(8.6\,{\scriptstyle\pm 0.1}\) & \(8.9\,{\scriptstyle\pm 0.2}\) & \(7.0\,{\scriptstyle\pm 0.3}\) \\
	Layer-Wise Subset & \(\mathbf{24.1\,{\scriptstyle\pm 1.0}}\) & \(\mathbf{29.2\,{\scriptstyle\pm 0.7}}\) & \(\mathbf{24.4\,{\scriptstyle\pm 1.3}}\) & \(\mathbf{16.3\,{\scriptstyle\pm 2.2}}\) & \(\mathbf{46.9\,{\scriptstyle\pm 0.6}}\) & \(\mathbf{28.5\,{\scriptstyle\pm 1.5}}\) & \(\mathbf{13.8\,{\scriptstyle\pm 0.2}}\) & \(13.3\,{\scriptstyle\pm 0.4}\) & \(10.4\,{\scriptstyle\pm 0.2}\) & \(\mathbf{8.9\,{\scriptstyle\pm 0.1}}\) & \(\mathbf{9.5\,{\scriptstyle\pm 0.3}}\) & \(6.7\,{\scriptstyle\pm 0.1}\) \\
	\midrule
	Target-Only (same budget) & \(30.8\,{\scriptstyle\pm 0.3}\) & \(30.9\,{\scriptstyle\pm 0.3}\) & \(31.7\,{\scriptstyle\pm 0.4}\) & \(56.2\,{\scriptstyle\pm 2.2}\) & \(51.8\,{\scriptstyle\pm 3.7}\) & \(36.0\,{\scriptstyle\pm 1.8}\) & \(8.2\,{\scriptstyle\pm 0.7}\) & \(7.0\,{\scriptstyle\pm 0.5}\) & \(4.0\,{\scriptstyle\pm 0.4}\) & \(8.6\,{\scriptstyle\pm 0.3}\) & \(5.5\,{\scriptstyle\pm 0.4}\) & \(4.1\,{\scriptstyle\pm 0.4}\) \\
	Target-Only (early stop) & \(32.1\,{\scriptstyle\pm 0.3}\) & \(32.8\,{\scriptstyle\pm 0.6}\) & \(30.5\,{\scriptstyle\pm 0.4}\) & \(49.1\,{\scriptstyle\pm 1.8}\) & \(51.1\,{\scriptstyle\pm 2.9}\) & \(21.8\,{\scriptstyle\pm 1.5}\) & \(10.1\,{\scriptstyle\pm 0.5}\) & \(9.3\,{\scriptstyle\pm 0.4}\) & \(3.3\,{\scriptstyle\pm 0.4}\) & \(8.8\,{\scriptstyle\pm 0.3}\) & \(8.7\,{\scriptstyle\pm 0.4}\) & \(2.8\,{\scriptstyle\pm 0.4}\) \\
	\bottomrule
\end{tabular}

	\end{adjustbox}
\end{table}

\paragraph{Results.}
\Cref{adxtab:sft-qa-matrix,adxtab:sft-ablation} complete the observations of \Cref{subsubsec:SFT-ablations}: the controls, the fine-tuning methods, and the scoring backend in detail.
\begin{itemize}[leftmargin=*]
	\item \emph{Random Subset controls.} Discarding half of every batch at random does not reproduce the gains: both controls stay within a couple of points of Full-Training on every setting and fine-tuning method, above it as often as below, whether the subset is shared or drawn per layer. The improvements of the curated arms come from the steering by the target scores, not from the smaller effective batch.
	\item \emph{Layer-normalized scores.} Rescaling the layers' scores does not close the gap: layer normalization moves Global Subset Update by several points at most, closing part but never most of its distance to the layer-wise endpoint on the misaligned settings, and changes the closed-book settings little. This supports the second reading above: rebalancing the magnitudes changes which layers decide the shared ranking, but one shared subset still cannot serve layers whose rankings disagree.
	\item \emph{Fine-tuning methods.} The ordering does not depend on the fine-tuning method: under full-parameter fine-tuning as under LoRA, Layer-Wise Subset Update with exact scoring is the best pool-based arm on every setting, and Global Subset Update improves over Full-Training on \texttt{less-mix}/\texttt{tydiqa} and stays level with it elsewhere. Under MeSO the picture is mixed (\Cref{adxtab:sft-qa-matrix-meso}): Layer-Wise Subset Update with exact scoring remains the best arm on \texttt{alpaca}/\texttt{samsum} and \texttt{less-mix}/\texttt{tydiqa}, but on the two closed-book settings no curated arm improves over Full-Training, and every MeSO arm trails its full-parameter counterpart there. MeSO applies each layer's update through a \(512 \times 512\) sketch of its gradient, so a curated subset steers a noisier update, and the downstream cost is largest where the gains of curation are smallest; we report MeSO as one of the fine-tuning methods the framework composes with, not as a recommendation.
	\item \emph{Scoring backend.} The finer the partition, the more it depends on exact scores. Global Subset Update is insensitive to the backend (exact and compressed agree within one to two standard errors everywhere), whereas Layer-Wise Subset Update loses up to several points under compressed scoring on the misaligned settings, most under full-parameter fine-tuning on \texttt{less-mix}/\texttt{tydiqa}, where the loss reaches double digits, and little elsewhere, with the intermediate partitions in between. The cause is per-layer sketch noise rather than bias: on \(26\) identical \texttt{alpaca}/\texttt{samsum} batches, compressed per-layer scores correlate with exact ones at \(0.89\) (top-\(4\) overlap \(0.77\)) at \(64 \times 64\) and \(0.99\) (\(0.89\)) at \(512 \times 512\), and summing the layers as Global Subset Update does raises the \(64 \times 64\) agreement to \(0.96\) (\(0.87\)). Exact scoring is also affordable in this setting: at \(\seqlen = 512\) and \(m = 1\), GIP is the cheapest exact method and costs at most \(34\)~ms per step more than compressed scoring (\Cref{tab:score-cost}), hence its use in the main text; the advantage of compressed scoring grows with \(\seqlen\) and \(m\). Assessed jointly, the two backends are not interchangeable: exact GIP costs \(3\)--\(5\) points of step time more than compressed scoring at \(n = 8\), \(\seqlen = 512\) (\Cref{adxtab:timing-grid} against the compressed totals quoted in \Cref{adxsubsec:system-efficiency}) and none of the accuracy, whereas compressed scoring returns those points at no accuracy cost for Global Subset Update but at an accuracy cost for Layer-Wise Subset Update that reaches double digits under full-parameter fine-tuning on the misaligned pool. Compressed scoring is therefore the choice only when \(\seqlen\) or \(m\) makes exact scoring the dominant cost (\Cref{tab:score-cost}), and then preferably with Global Subset Update or a \(512 \times 512\) sketch, which restores the per-layer ranking agreement above; all accuracy results in this paper, and the benchmarks of \Cref{subsec:system-efficiency}, use exact scoring.
\end{itemize}

\subsubsection{Reasoning and Instruction Following with \textsc{Qwen3-1.7B-Base}}\label{adxsubsubsec:capability-setting}
\paragraph{Datasets.}
Next, we consider reasoning and instruction following tasks, where we sample three \(32{,}000\)-example pools uniformly from \texttt{Dolci-Instruct-SFT}\footnote{\url{https://huggingface.co/datasets/allenai/Dolci-Instruct-SFT}}~\citep{olmo3}, restricted to plain alternating user/assistant conversations of at most \(16\)K characters and excluding the Tool Use and Hardcoded Data domains: an \emph{instruction} pool (Chat, Precise IF, Other, Multilingual, Safety), a \emph{reasoning} pool (Math, Coding, Reasoning, Science), and a \emph{mixed} pool over all nine domains. A fourth pool of \(32{,}000\) examples is sampled from the \texttt{Tulu 3} SFT mixture\footnote{\url{https://huggingface.co/datasets/allenai/tulu-3-sft-mixture}}~\citep{lambert2025tulu}, in which Precise-IF data makes up about \(3\%\) and math about \(36\%\) of the examples. Every pool is decontaminated against all benchmark and target prompts.

\paragraph{Targets and benchmarks.}
Each target set has \(128\) examples; the main-text runs use the first \(32\) of the set shuffled with the training seed, and the target-set size is varied below. The benchmarks are separate datasets that share the target skill, scored by generation from the final checkpoint (greedy decoding, at most \(2048\) new tokens): \texttt{IFEval}~\citep{zhou2023instruction} and \texttt{IFBench}~\citep{pyatkin2025generalizing} (both prompt-level strict accuracy) for the Precise-IF target, and \texttt{MATH500}~\citep{lightman2024lets} and \texttt{GSM8K}~\citep{cobbe2021training} for the Math targets, whose final answers are checked against the reference with Math-Verify~\citep{kydlicek2025mathverify}. The three targets are:
\begin{enumerate*}[label=(\roman*)]
	\item \emph{Precise IF (\texttt{Dolci})}: single-turn examples from the Precise-IF source of \texttt{Dolci-Instruct-SFT}, removed from every pool by id;
	\item \emph{Math (\texttt{Dolci})}: held-out \texttt{Dolci} examples from the pools' own math sources (\texttt{Tulu 3 Persona MATH}, \texttt{Algebra}, and \texttt{GSM}), so that the target set has the pool's math style;
	\item \emph{Math (curated)}: \texttt{MATH}-train~\citep{hendrycks2021measuring} problems re-solved by the post-trained \textsc{Qwen3-32B} in non-thinking mode (\(8\) samples at temperature \(0.7\); the shortest response whose final answer passes the same Math-Verify check against the reference is kept, and the \(4\) unsolved problems are dropped, leaving \(124\) examples). \emph{Curated} thus means re-solved and verified, mimicking the situation where one can curate a small, high-quality target set; since \texttt{MATH}-train is disjoint from \texttt{MATH500} and \texttt{GSM8K}, it does not contaminate the benchmarks.
\end{enumerate*}

\paragraph{Training configurations.}
All arms use full-parameter fine-tuning of \textsc{Qwen3-1.7B-Base}~\citep{qwen3technicalreport}, following prior work~\citep{kopiczko2026data,nayak2026critical,shen2026geometry}, for one epoch (\(4000\) steps) with batch size \(8\), maximum sequence length \(2048\), AdamW at learning rate \(10^{-5}\) with a linear schedule and \(3\%\) warmup, no weight decay, \texttt{bfloat16} precision and FlashAttention-2. The training loss is the mean over examples of per-example token-mean losses, so that per-sample gradients are gradients of per-example losses, and the embedding row of the chat end-of-turn token \texttt{<|im\_end|>}, which is untrained in the base model, is initialised from \texttt{<|endoftext|>} before fine-tuning so that generations terminate. Every arm is run with \(5\) seeds, and we report the mean \(\pm\) standard error over seeds.

\paragraph{Data regularization.}
The protocol is that of the question-answering setting (\Cref{adxsubsubsec:sft-qa}), with \(n = 8\) candidates and \(m = 1\) target example per step, except that the target example is drawn from the target sets above, scores are exact PIP scores (\Cref{subsubsec:per-token-inner-product}) in every run, with no compressed-scoring arm, and top-\(k\) selection is run at \(k \in \{2, 4, 6\}\), with \(k = 4\) in the main text. Full-Training ignores the target set, so its arm is shared between the two Math targets on the reasoning pool. The curated updates are compared against Full-Training only: the question here is how much of a pool to use at a given mismatch, which the pool-free Target-Only Update does not address (\Cref{adxsubsubsec:sft-qa}).

\paragraph{Results.}
\Cref{tab:SFT-qwen} reports both curated updates at a single shared configuration, \(k = n/2\) with \(32\) target examples, used for all seven settings and both methods. \Cref{adxtab:sft-qwen-nval-if,adxtab:sft-qwen-nval-math} sweep \(k \in \{n/4, n/2, 3n/4\}\) and the target-set size: every curated arm is rerun with \(128\), \(64\), \(32\), and \(16\) target examples, obtained for each seed by shuffling the \(128\) examples with the training seed and keeping a prefix, so that the sets are nested within a seed; Full-Training does not depend on the target set and is shared. The qualitative picture of \Cref{subsubsec:SFT-capability} holds at every \(k\) and every target-set size down to \(16\) examples: selection helps most where the target skill is diluted in the pool, with the largest gains on \texttt{Tulu 3}; on the Precise-IF target both curated methods match or beat Full-Training on \texttt{IFEval} in every setting; and on the Math targets the gains are largest on \texttt{Tulu 3} and, with the Math (curated) target, for Layer-Wise Subset Update, which improves both \texttt{MATH500} and \texttt{GSM8K}. \texttt{IFEval}'s constraint templates appear verbatim in \texttt{Dolci} Precise-IF prompts (the prompts themselves are decontaminated), whereas \texttt{IFBench} uses new constraint types, so the two benchmarks separate fitting the target's constraint types from generalizing beyond them.

In a smaller set of additional runs, we observe that Random Subset Update at \(k = n/2\) on the \texttt{Dolci} pools stays at or below Full-Training on \texttt{IFEval} and within a point of it elsewhere, so the gains of the curated arms over Full-Training are not an effect of training on fewer examples; the \texttt{IFBench} decrease on the mixed pool, which random subsets share, reflects the smaller number of trained examples rather than the selection.

\begin{table}[t]
	\centering
	\caption{Precise-IF settings of the reasoning and instruction-following experiment (greedy decoding): every arm at \(k \in \{2, 4, 6\}\) of the \(n = 8\) candidates, in one block per target-set size in \(\{128, 64, 32, 16\}\), where every arm uses that many leading examples of the seed-shuffled \(128\)-example target set; Full-Training does not depend on the target set and is given once in the first row. \texttt{IFEval} and \texttt{IFBench}: prompt-level strict accuracy. Best curated arm per column within each target-size block in bold.}
	\label{adxtab:sft-qwen-nval-if}
	\small\setlength{\tabcolsep}{3pt}\renewcommand{\arraystretch}{0.95}
	\begin{adjustbox}{max width=\linewidth}
\begin{tabular}{l l c c c c c c c}
\toprule
\multirow{2}{*}[-0.5\dimexpr\aboverulesep+\belowrulesep+\cmidrulewidth\relax]{\textbf{Target size}} & \multirow{2}{*}[-0.5\dimexpr\aboverulesep+\belowrulesep+\cmidrulewidth\relax]{\textbf{Method}} & \multirow{2}{*}[-0.5\dimexpr\aboverulesep+\belowrulesep+\cmidrulewidth\relax]{\(k\)} & \multicolumn{2}{c}{\texttt{Dolci} instruction/Precise IF (\texttt{Dolci})} & \multicolumn{2}{c}{\texttt{Dolci} mixed/Precise IF (\texttt{Dolci})} & \multicolumn{2}{c}{\texttt{Tulu 3}/Precise IF (\texttt{Dolci})} \\
\cmidrule(lr){4-5} \cmidrule(lr){6-7} \cmidrule(lr){8-9}
 & & & \texttt{IFEval} & \texttt{IFBench} & \texttt{IFEval} & \texttt{IFBench} & \texttt{IFEval} & \texttt{IFBench} \\
\midrule
-- & Full-Training & 8 & \(42.7\,{\scriptstyle\pm 0.6}\) & \(17.9\,{\scriptstyle\pm 0.2}\) & \(40.9\,{\scriptstyle\pm 0.6}\) & \(17.9\,{\scriptstyle\pm 0.3}\) & \(32.5\,{\scriptstyle\pm 0.2}\) & \(16.2\,{\scriptstyle\pm 0.3}\) \\
\midrule
\multirow{6}{*}[-0.5\dimexpr\aboverulesep+\belowrulesep+\cmidrulewidth\relax]{\(128\)} & \multirow{3}{*}{Global Subset} & 2 & \(42.7\,{\scriptstyle\pm 0.4}\) & \(18.7\,{\scriptstyle\pm 0.4}\) & \(42.6\,{\scriptstyle\pm 0.6}\) & \(\mathbf{16.9\,{\scriptstyle\pm 0.5}}\) & \(35.0\,{\scriptstyle\pm 0.3}\) & \(\mathbf{17.0\,{\scriptstyle\pm 0.3}}\) \\
 &  & 4 & \(44.0\,{\scriptstyle\pm 0.7}\) & \(18.8\,{\scriptstyle\pm 0.3}\) & \(41.3\,{\scriptstyle\pm 0.4}\) & \(16.5\,{\scriptstyle\pm 0.3}\) & \(35.3\,{\scriptstyle\pm 0.4}\) & \(16.0\,{\scriptstyle\pm 0.4}\) \\
 &  & 6 & \(\mathbf{44.1\,{\scriptstyle\pm 0.4}}\) & \(18.3\,{\scriptstyle\pm 0.5}\) & \(42.3\,{\scriptstyle\pm 0.3}\) & \(16.1\,{\scriptstyle\pm 0.2}\) & \(35.6\,{\scriptstyle\pm 0.2}\) & \(16.1\,{\scriptstyle\pm 0.6}\) \\
\cmidrule(lr){2-9}
 & \multirow{3}{*}{Layer-Wise Subset} & 2 & \(42.6\,{\scriptstyle\pm 0.1}\) & \(\mathbf{19.9\,{\scriptstyle\pm 0.7}}\) & \(42.3\,{\scriptstyle\pm 0.4}\) & \(16.7\,{\scriptstyle\pm 0.1}\) & \(35.1\,{\scriptstyle\pm 0.3}\) & \(16.6\,{\scriptstyle\pm 0.4}\) \\
 &  & 4 & \(43.3\,{\scriptstyle\pm 0.5}\) & \(18.1\,{\scriptstyle\pm 0.5}\) & \(42.3\,{\scriptstyle\pm 0.3}\) & \(16.2\,{\scriptstyle\pm 0.3}\) & \(\mathbf{36.2\,{\scriptstyle\pm 0.6}}\) & \(16.3\,{\scriptstyle\pm 0.5}\) \\
 &  & 6 & \(43.8\,{\scriptstyle\pm 0.4}\) & \(18.0\,{\scriptstyle\pm 0.4}\) & \(\mathbf{42.7\,{\scriptstyle\pm 0.2}}\) & \(16.3\,{\scriptstyle\pm 0.1}\) & \(35.3\,{\scriptstyle\pm 0.5}\) & \(16.7\,{\scriptstyle\pm 0.3}\) \\
\midrule
\multirow{6}{*}[-0.5\dimexpr\aboverulesep+\belowrulesep+\cmidrulewidth\relax]{\(64\)} & \multirow{3}{*}{Global Subset} & 2 & \(43.0\,{\scriptstyle\pm 0.8}\) & \(18.9\,{\scriptstyle\pm 0.4}\) & \(41.1\,{\scriptstyle\pm 0.5}\) & \(16.3\,{\scriptstyle\pm 0.4}\) & \(33.7\,{\scriptstyle\pm 0.7}\) & \(16.8\,{\scriptstyle\pm 0.3}\) \\
 &  & 4 & \(43.5\,{\scriptstyle\pm 0.7}\) & \(18.7\,{\scriptstyle\pm 0.6}\) & \(41.8\,{\scriptstyle\pm 0.4}\) & \(16.3\,{\scriptstyle\pm 0.5}\) & \(34.6\,{\scriptstyle\pm 0.4}\) & \(16.7\,{\scriptstyle\pm 0.2}\) \\
 &  & 6 & \(\mathbf{44.1\,{\scriptstyle\pm 0.3}}\) & \(17.9\,{\scriptstyle\pm 0.4}\) & \(\mathbf{42.4\,{\scriptstyle\pm 0.3}}\) & \(16.2\,{\scriptstyle\pm 0.3}\) & \(34.9\,{\scriptstyle\pm 0.6}\) & \(16.4\,{\scriptstyle\pm 0.2}\) \\
\cmidrule(lr){2-9}
 & \multirow{3}{*}{Layer-Wise Subset} & 2 & \(43.6\,{\scriptstyle\pm 0.5}\) & \(18.7\,{\scriptstyle\pm 0.4}\) & \(41.1\,{\scriptstyle\pm 0.7}\) & \(\mathbf{16.7\,{\scriptstyle\pm 0.4}}\) & \(35.3\,{\scriptstyle\pm 0.6}\) & \(\mathbf{17.1\,{\scriptstyle\pm 0.5}}\) \\
 &  & 4 & \(43.8\,{\scriptstyle\pm 0.4}\) & \(\mathbf{19.3\,{\scriptstyle\pm 0.5}}\) & \(42.1\,{\scriptstyle\pm 0.3}\) & \(16.0\,{\scriptstyle\pm 0.3}\) & \(\mathbf{35.8\,{\scriptstyle\pm 0.6}}\) & \(\mathbf{17.1\,{\scriptstyle\pm 0.5}}\) \\
 &  & 6 & \(\mathbf{44.1\,{\scriptstyle\pm 0.3}}\) & \(18.2\,{\scriptstyle\pm 0.3}\) & \(41.7\,{\scriptstyle\pm 0.5}\) & \(16.1\,{\scriptstyle\pm 0.5}\) & \(35.7\,{\scriptstyle\pm 0.3}\) & \(16.5\,{\scriptstyle\pm 0.2}\) \\
\midrule
\multirow{6}{*}[-0.5\dimexpr\aboverulesep+\belowrulesep+\cmidrulewidth\relax]{\(32\)} & \multirow{3}{*}{Global Subset} & 2 & \(42.9\,{\scriptstyle\pm 0.5}\) & \(18.1\,{\scriptstyle\pm 0.3}\) & \(41.5\,{\scriptstyle\pm 0.4}\) & \(16.6\,{\scriptstyle\pm 0.4}\) & \(34.8\,{\scriptstyle\pm 0.5}\) & \(16.6\,{\scriptstyle\pm 0.4}\) \\
 &  & 4 & \(\mathbf{44.5\,{\scriptstyle\pm 0.2}}\) & \(18.5\,{\scriptstyle\pm 0.3}\) & \(41.8\,{\scriptstyle\pm 0.3}\) & \(16.5\,{\scriptstyle\pm 0.4}\) & \(35.3\,{\scriptstyle\pm 0.3}\) & \(16.0\,{\scriptstyle\pm 0.3}\) \\
 &  & 6 & \(43.7\,{\scriptstyle\pm 0.3}\) & \(17.8\,{\scriptstyle\pm 0.2}\) & \(41.8\,{\scriptstyle\pm 0.4}\) & \(16.2\,{\scriptstyle\pm 0.3}\) & \(34.5\,{\scriptstyle\pm 0.3}\) & \(16.1\,{\scriptstyle\pm 0.2}\) \\
\cmidrule(lr){2-9}
 & \multirow{3}{*}{Layer-Wise Subset} & 2 & \(43.4\,{\scriptstyle\pm 0.1}\) & \(\mathbf{19.7\,{\scriptstyle\pm 0.8}}\) & \(41.7\,{\scriptstyle\pm 0.4}\) & \(\mathbf{17.1\,{\scriptstyle\pm 0.2}}\) & \(35.6\,{\scriptstyle\pm 0.7}\) & \(\mathbf{17.1\,{\scriptstyle\pm 0.8}}\) \\
 &  & 4 & \(44.1\,{\scriptstyle\pm 0.4}\) & \(18.7\,{\scriptstyle\pm 0.5}\) & \(41.7\,{\scriptstyle\pm 0.4}\) & \(\mathbf{17.1\,{\scriptstyle\pm 0.3}}\) & \(35.3\,{\scriptstyle\pm 0.4}\) & \(16.4\,{\scriptstyle\pm 0.9}\) \\
 &  & 6 & \(43.8\,{\scriptstyle\pm 0.3}\) & \(18.0\,{\scriptstyle\pm 0.6}\) & \(\mathbf{42.2\,{\scriptstyle\pm 0.3}}\) & \(15.8\,{\scriptstyle\pm 0.2}\) & \(\mathbf{36.1\,{\scriptstyle\pm 0.4}}\) & \(16.5\,{\scriptstyle\pm 0.5}\) \\
\midrule
\multirow{6}{*}[-0.5\dimexpr\aboverulesep+\belowrulesep+\cmidrulewidth\relax]{\(16\)} & \multirow{3}{*}{Global Subset} & 2 & \(42.7\,{\scriptstyle\pm 0.3}\) & \(18.5\,{\scriptstyle\pm 0.6}\) & \(41.4\,{\scriptstyle\pm 0.6}\) & \(16.3\,{\scriptstyle\pm 0.5}\) & \(35.0\,{\scriptstyle\pm 0.5}\) & \(16.2\,{\scriptstyle\pm 0.5}\) \\
 &  & 4 & \(43.7\,{\scriptstyle\pm 0.3}\) & \(18.7\,{\scriptstyle\pm 0.6}\) & \(41.5\,{\scriptstyle\pm 0.3}\) & \(16.5\,{\scriptstyle\pm 0.2}\) & \(34.9\,{\scriptstyle\pm 0.6}\) & \(15.9\,{\scriptstyle\pm 0.5}\) \\
 &  & 6 & \(43.4\,{\scriptstyle\pm 0.5}\) & \(18.8\,{\scriptstyle\pm 0.6}\) & \(\mathbf{42.0\,{\scriptstyle\pm 0.5}}\) & \(15.4\,{\scriptstyle\pm 0.3}\) & \(35.1\,{\scriptstyle\pm 0.4}\) & \(15.8\,{\scriptstyle\pm 0.3}\) \\
\cmidrule(lr){2-9}
 & \multirow{3}{*}{Layer-Wise Subset} & 2 & \(42.7\,{\scriptstyle\pm 0.5}\) & \(\mathbf{19.7\,{\scriptstyle\pm 0.5}}\) & \(40.1\,{\scriptstyle\pm 0.4}\) & \(\mathbf{17.5\,{\scriptstyle\pm 0.6}}\) & \(34.5\,{\scriptstyle\pm 0.4}\) & \(\mathbf{17.2\,{\scriptstyle\pm 0.6}}\) \\
 &  & 4 & \(43.4\,{\scriptstyle\pm 0.4}\) & \(19.6\,{\scriptstyle\pm 0.4}\) & \(40.7\,{\scriptstyle\pm 0.2}\) & \(16.7\,{\scriptstyle\pm 0.4}\) & \(35.0\,{\scriptstyle\pm 0.4}\) & \(\mathbf{17.2\,{\scriptstyle\pm 0.5}}\) \\
 &  & 6 & \(\mathbf{43.8\,{\scriptstyle\pm 0.3}}\) & \(18.7\,{\scriptstyle\pm 0.3}\) & \(41.7\,{\scriptstyle\pm 0.4}\) & \(16.5\,{\scriptstyle\pm 0.4}\) & \(\mathbf{35.8\,{\scriptstyle\pm 0.2}}\) & \(16.9\,{\scriptstyle\pm 0.5}\) \\
\bottomrule
\end{tabular}

	\end{adjustbox}
\end{table}

\begin{table}[t]
	\centering
	\caption{Math settings of the reasoning and instruction-following experiment (greedy decoding), same layout as \Cref{adxtab:sft-qwen-nval-if}: the three settings with the Math (\texttt{Dolci}) target, which uses the pools' own math sources, and the reasoning pool with the Math (curated) target, \texttt{MATH}-train problems re-solved and verified by \textsc{Qwen3-32B}. Full-Training on the reasoning pool is shared between its two targets. Best curated arm per column within each target-size block in bold.}
	\label{adxtab:sft-qwen-nval-math}
	\small\setlength{\tabcolsep}{3pt}\renewcommand{\arraystretch}{0.95}
	\begin{adjustbox}{max width=\linewidth}
\begin{tabular}{l l c c c c c c c c c}
\toprule
\multirow{2}{*}[-0.5\dimexpr\aboverulesep+\belowrulesep+\cmidrulewidth\relax]{\textbf{Target size}} & \multirow{2}{*}[-0.5\dimexpr\aboverulesep+\belowrulesep+\cmidrulewidth\relax]{\textbf{Method}} & \multirow{2}{*}[-0.5\dimexpr\aboverulesep+\belowrulesep+\cmidrulewidth\relax]{\(k\)} & \multicolumn{2}{c}{\texttt{Dolci} reasoning/Math (\texttt{Dolci})} & \multicolumn{2}{c}{\texttt{Dolci} mixed/Math (\texttt{Dolci})} & \multicolumn{2}{c}{\texttt{Tulu 3}/Math (\texttt{Dolci})} & \multicolumn{2}{c}{\texttt{Dolci} reasoning/Math (curated)} \\
\cmidrule(lr){4-5} \cmidrule(lr){6-7} \cmidrule(lr){8-9} \cmidrule(lr){10-11}
 & & & \texttt{MATH500} & \texttt{GSM8K} & \texttt{MATH500} & \texttt{GSM8K} & \texttt{MATH500} & \texttt{GSM8K} & \texttt{MATH500} & \texttt{GSM8K} \\
\midrule
-- & Full-Training & 8 & \(63.6\,{\scriptstyle\pm 0.5}\) & \(79.9\,{\scriptstyle\pm 0.2}\) & \(64.2\,{\scriptstyle\pm 0.3}\) & \(78.8\,{\scriptstyle\pm 0.1}\) & \(57.1\,{\scriptstyle\pm 0.5}\) & \(78.8\,{\scriptstyle\pm 0.2}\) & \(63.6\,{\scriptstyle\pm 0.5}\) & \(79.9\,{\scriptstyle\pm 0.2}\) \\
\midrule
\multirow{6}{*}[-0.5\dimexpr\aboverulesep+\belowrulesep+\cmidrulewidth\relax]{\(128\)} & \multirow{3}{*}{Global Subset} & 2 & \(63.4\,{\scriptstyle\pm 0.4}\) & \(80.3\,{\scriptstyle\pm 0.3}\) & \(\mathbf{63.6\,{\scriptstyle\pm 0.6}}\) & \(\mathbf{80.0\,{\scriptstyle\pm 0.1}}\) & \(60.0\,{\scriptstyle\pm 0.6}\) & \(79.0\,{\scriptstyle\pm 0.3}\) & \(63.1\,{\scriptstyle\pm 0.7}\) & \(79.7\,{\scriptstyle\pm 0.3}\) \\
 &  & 4 & \(62.8\,{\scriptstyle\pm 0.5}\) & \(80.8\,{\scriptstyle\pm 0.2}\) & \(\mathbf{63.6\,{\scriptstyle\pm 0.7}}\) & \(79.7\,{\scriptstyle\pm 0.1}\) & \(58.6\,{\scriptstyle\pm 0.5}\) & \(79.6\,{\scriptstyle\pm 0.2}\) & \(64.0\,{\scriptstyle\pm 0.5}\) & \(80.5\,{\scriptstyle\pm 0.2}\) \\
 &  & 6 & \(\mathbf{64.0\,{\scriptstyle\pm 0.7}}\) & \(\mathbf{80.9\,{\scriptstyle\pm 0.3}}\) & \(63.4\,{\scriptstyle\pm 0.6}\) & \(\mathbf{80.0\,{\scriptstyle\pm 0.2}}\) & \(57.3\,{\scriptstyle\pm 0.8}\) & \(79.5\,{\scriptstyle\pm 0.5}\) & \(64.0\,{\scriptstyle\pm 0.7}\) & \(80.3\,{\scriptstyle\pm 0.2}\) \\
\cmidrule(lr){2-11}
 & \multirow{3}{*}{Layer-Wise Subset} & 2 & \(60.9\,{\scriptstyle\pm 0.7}\) & \(79.0\,{\scriptstyle\pm 0.5}\) & \(\mathbf{63.6\,{\scriptstyle\pm 0.4}}\) & \(79.9\,{\scriptstyle\pm 0.2}\) & \(\mathbf{60.5\,{\scriptstyle\pm 0.5}}\) & \(79.5\,{\scriptstyle\pm 0.3}\) & \(65.0\,{\scriptstyle\pm 0.8}\) & \(80.8\,{\scriptstyle\pm 0.4}\) \\
 &  & 4 & \(63.1\,{\scriptstyle\pm 0.3}\) & \(80.1\,{\scriptstyle\pm 0.4}\) & \(63.0\,{\scriptstyle\pm 0.8}\) & \(79.7\,{\scriptstyle\pm 0.2}\) & \(59.2\,{\scriptstyle\pm 0.4}\) & \(79.4\,{\scriptstyle\pm 0.3}\) & \(\mathbf{65.8\,{\scriptstyle\pm 0.3}}\) & \(\mathbf{80.9\,{\scriptstyle\pm 0.5}}\) \\
 &  & 6 & \(\mathbf{64.0\,{\scriptstyle\pm 0.6}}\) & \(80.4\,{\scriptstyle\pm 0.2}\) & \(63.4\,{\scriptstyle\pm 0.3}\) & \(79.6\,{\scriptstyle\pm 0.1}\) & \(59.2\,{\scriptstyle\pm 0.5}\) & \(\mathbf{80.0\,{\scriptstyle\pm 0.2}}\) & \(63.5\,{\scriptstyle\pm 0.5}\) & \(80.8\,{\scriptstyle\pm 0.3}\) \\
\midrule
\multirow{6}{*}[-0.5\dimexpr\aboverulesep+\belowrulesep+\cmidrulewidth\relax]{\(64\)} & \multirow{3}{*}{Global Subset} & 2 & \(62.6\,{\scriptstyle\pm 0.6}\) & \(80.4\,{\scriptstyle\pm 0.2}\) & \(63.8\,{\scriptstyle\pm 0.5}\) & \(79.9\,{\scriptstyle\pm 0.4}\) & \(\mathbf{60.2\,{\scriptstyle\pm 0.6}}\) & \(79.0\,{\scriptstyle\pm 0.3}\) & \(63.3\,{\scriptstyle\pm 0.2}\) & \(80.0\,{\scriptstyle\pm 0.3}\) \\
 &  & 4 & \(64.0\,{\scriptstyle\pm 0.3}\) & \(80.4\,{\scriptstyle\pm 0.1}\) & \(63.4\,{\scriptstyle\pm 0.5}\) & \(\mathbf{80.0\,{\scriptstyle\pm 0.3}}\) & \(59.3\,{\scriptstyle\pm 0.5}\) & \(79.4\,{\scriptstyle\pm 0.2}\) & \(64.8\,{\scriptstyle\pm 0.5}\) & \(80.2\,{\scriptstyle\pm 0.3}\) \\
 &  & 6 & \(\mathbf{64.2\,{\scriptstyle\pm 0.6}}\) & \(\mathbf{80.6\,{\scriptstyle\pm 0.2}}\) & \(\mathbf{64.3\,{\scriptstyle\pm 0.5}}\) & \(79.4\,{\scriptstyle\pm 0.3}\) & \(57.0\,{\scriptstyle\pm 0.6}\) & \(79.5\,{\scriptstyle\pm 0.2}\) & \(64.3\,{\scriptstyle\pm 0.4}\) & \(80.3\,{\scriptstyle\pm 0.2}\) \\
\cmidrule(lr){2-11}
 & \multirow{3}{*}{Layer-Wise Subset} & 2 & \(61.8\,{\scriptstyle\pm 0.7}\) & \(78.8\,{\scriptstyle\pm 0.4}\) & \(62.7\,{\scriptstyle\pm 0.3}\) & \(79.7\,{\scriptstyle\pm 0.2}\) & \(59.4\,{\scriptstyle\pm 0.8}\) & \(79.7\,{\scriptstyle\pm 0.4}\) & \(\mathbf{65.6\,{\scriptstyle\pm 0.6}}\) & \(\mathbf{81.0\,{\scriptstyle\pm 0.2}}\) \\
 &  & 4 & \(62.9\,{\scriptstyle\pm 0.2}\) & \(79.7\,{\scriptstyle\pm 0.4}\) & \(62.6\,{\scriptstyle\pm 0.6}\) & \(79.6\,{\scriptstyle\pm 0.2}\) & \(59.2\,{\scriptstyle\pm 0.1}\) & \(\mathbf{79.9\,{\scriptstyle\pm 0.1}}\) & \(65.1\,{\scriptstyle\pm 0.2}\) & \(80.8\,{\scriptstyle\pm 0.3}\) \\
 &  & 6 & \(63.8\,{\scriptstyle\pm 0.6}\) & \(\mathbf{80.6\,{\scriptstyle\pm 0.1}}\) & \(62.8\,{\scriptstyle\pm 0.4}\) & \(79.3\,{\scriptstyle\pm 0.2}\) & \(58.6\,{\scriptstyle\pm 0.3}\) & \(79.4\,{\scriptstyle\pm 0.3}\) & \(64.2\,{\scriptstyle\pm 0.3}\) & \(80.7\,{\scriptstyle\pm 0.3}\) \\
\midrule
\multirow{6}{*}[-0.5\dimexpr\aboverulesep+\belowrulesep+\cmidrulewidth\relax]{\(32\)} & \multirow{3}{*}{Global Subset} & 2 & \(64.2\,{\scriptstyle\pm 0.3}\) & \(80.2\,{\scriptstyle\pm 0.2}\) & \(63.9\,{\scriptstyle\pm 0.5}\) & \(\mathbf{80.4\,{\scriptstyle\pm 0.2}}\) & \(60.2\,{\scriptstyle\pm 0.5}\) & \(79.4\,{\scriptstyle\pm 0.3}\) & \(62.9\,{\scriptstyle\pm 0.2}\) & \(79.8\,{\scriptstyle\pm 0.4}\) \\
 &  & 4 & \(\mathbf{64.4\,{\scriptstyle\pm 0.3}}\) & \(80.7\,{\scriptstyle\pm 0.2}\) & \(\mathbf{64.2\,{\scriptstyle\pm 0.4}}\) & \(\mathbf{80.4\,{\scriptstyle\pm 0.3}}\) & \(58.3\,{\scriptstyle\pm 0.4}\) & \(79.4\,{\scriptstyle\pm 0.2}\) & \(64.0\,{\scriptstyle\pm 0.2}\) & \(79.9\,{\scriptstyle\pm 0.1}\) \\
 &  & 6 & \(63.8\,{\scriptstyle\pm 0.4}\) & \(\mathbf{80.9\,{\scriptstyle\pm 0.2}}\) & \(63.7\,{\scriptstyle\pm 0.1}\) & \(\mathbf{80.4\,{\scriptstyle\pm 0.1}}\) & \(56.4\,{\scriptstyle\pm 0.7}\) & \(79.3\,{\scriptstyle\pm 0.3}\) & \(63.9\,{\scriptstyle\pm 0.4}\) & \(79.9\,{\scriptstyle\pm 0.2}\) \\
\cmidrule(lr){2-11}
 & \multirow{3}{*}{Layer-Wise Subset} & 2 & \(62.1\,{\scriptstyle\pm 0.8}\) & \(79.4\,{\scriptstyle\pm 0.5}\) & \(63.5\,{\scriptstyle\pm 0.7}\) & \(80.2\,{\scriptstyle\pm 0.4}\) & \(60.4\,{\scriptstyle\pm 0.4}\) & \(79.4\,{\scriptstyle\pm 0.2}\) & \(\mathbf{66.0\,{\scriptstyle\pm 0.5}}\) & \(80.3\,{\scriptstyle\pm 0.4}\) \\
 &  & 4 & \(62.8\,{\scriptstyle\pm 0.2}\) & \(80.3\,{\scriptstyle\pm 0.2}\) & \(63.9\,{\scriptstyle\pm 0.4}\) & \(\mathbf{80.4\,{\scriptstyle\pm 0.3}}\) & \(\mathbf{60.6\,{\scriptstyle\pm 0.6}}\) & \(79.8\,{\scriptstyle\pm 0.3}\) & \(65.1\,{\scriptstyle\pm 0.5}\) & \(\mathbf{81.0\,{\scriptstyle\pm 0.2}}\) \\
 &  & 6 & \(63.6\,{\scriptstyle\pm 0.3}\) & \(80.6\,{\scriptstyle\pm 0.3}\) & \(63.1\,{\scriptstyle\pm 0.4}\) & \(79.8\,{\scriptstyle\pm 0.4}\) & \(58.8\,{\scriptstyle\pm 0.3}\) & \(\mathbf{80.0\,{\scriptstyle\pm 0.2}}\) & \(64.6\,{\scriptstyle\pm 0.7}\) & \(80.9\,{\scriptstyle\pm 0.2}\) \\
\midrule
\multirow{6}{*}[-0.5\dimexpr\aboverulesep+\belowrulesep+\cmidrulewidth\relax]{\(16\)} & \multirow{3}{*}{Global Subset} & 2 & \(63.8\,{\scriptstyle\pm 0.4}\) & \(80.1\,{\scriptstyle\pm 0.1}\) & \(\mathbf{64.4\,{\scriptstyle\pm 0.3}}\) & \(\mathbf{80.2\,{\scriptstyle\pm 0.2}}\) & \(60.0\,{\scriptstyle\pm 0.6}\) & \(78.8\,{\scriptstyle\pm 0.5}\) & \(62.5\,{\scriptstyle\pm 0.3}\) & \(80.0\,{\scriptstyle\pm 0.1}\) \\
 &  & 4 & \(63.4\,{\scriptstyle\pm 0.3}\) & \(80.7\,{\scriptstyle\pm 0.2}\) & \(63.5\,{\scriptstyle\pm 0.3}\) & \(79.9\,{\scriptstyle\pm 0.3}\) & \(59.4\,{\scriptstyle\pm 0.5}\) & \(79.8\,{\scriptstyle\pm 0.3}\) & \(63.1\,{\scriptstyle\pm 0.2}\) & \(79.9\,{\scriptstyle\pm 0.2}\) \\
 &  & 6 & \(\mathbf{64.8\,{\scriptstyle\pm 0.6}}\) & \(\mathbf{81.0\,{\scriptstyle\pm 0.1}}\) & \(64.0\,{\scriptstyle\pm 0.4}\) & \(79.8\,{\scriptstyle\pm 0.3}\) & \(57.6\,{\scriptstyle\pm 0.2}\) & \(78.8\,{\scriptstyle\pm 0.4}\) & \(64.4\,{\scriptstyle\pm 0.5}\) & \(80.2\,{\scriptstyle\pm 0.2}\) \\
\cmidrule(lr){2-11}
 & \multirow{3}{*}{Layer-Wise Subset} & 2 & \(62.6\,{\scriptstyle\pm 0.8}\) & \(79.7\,{\scriptstyle\pm 0.3}\) & \(63.8\,{\scriptstyle\pm 0.6}\) & \(80.0\,{\scriptstyle\pm 0.3}\) & \(\mathbf{61.2\,{\scriptstyle\pm 0.6}}\) & \(79.7\,{\scriptstyle\pm 0.3}\) & \(64.7\,{\scriptstyle\pm 0.4}\) & \(81.0\,{\scriptstyle\pm 0.5}\) \\
 &  & 4 & \(64.0\,{\scriptstyle\pm 0.8}\) & \(80.4\,{\scriptstyle\pm 0.3}\) & \(63.3\,{\scriptstyle\pm 0.3}\) & \(80.0\,{\scriptstyle\pm 0.2}\) & \(60.0\,{\scriptstyle\pm 0.5}\) & \(80.0\,{\scriptstyle\pm 0.4}\) & \(\mathbf{65.4\,{\scriptstyle\pm 0.7}}\) & \(\mathbf{81.1\,{\scriptstyle\pm 0.2}}\) \\
 &  & 6 & \(63.4\,{\scriptstyle\pm 0.5}\) & \(80.6\,{\scriptstyle\pm 0.3}\) & \(63.3\,{\scriptstyle\pm 0.5}\) & \(79.9\,{\scriptstyle\pm 0.3}\) & \(59.4\,{\scriptstyle\pm 0.5}\) & \(\mathbf{80.3\,{\scriptstyle\pm 0.2}}\) & \(64.8\,{\scriptstyle\pm 0.7}\) & \(\mathbf{81.1\,{\scriptstyle\pm 0.2}}\) \\
\bottomrule
\end{tabular}

	\end{adjustbox}
\end{table}

\subsection{Reinforcement Learning from Human Feedback}\label{adxsubsec:details-of-rlhf}
We now provide the missing details of the RLHF experiments in \Cref{subsec:RLHF}, including details on datasets, libraries, training configurations, and data regularization.

\subsubsection{Libraries and Datasets}
We implement PPO-based RLHF on top of \texttt{TRL}\footnote{\url{https://huggingface.co/docs/trl/index}}~\citep{vonwerra2022trl} (Apache-2.0 license, version \(0.26.1\)), using its value-head policy model and re-implementing its PPO loop in our trainer so that per-sample scores can be computed within each mini-batch, and sample prompts from \texttt{real-toxicity-prompts}\footnote{\url{https://huggingface.co/datasets/allenai/real-toxicity-prompts}}~\citep{gehman2020realtoxicityprompts} (Apache-2.0 license). For rewards, we use the \texttt{LFTW R4 Target} toxicity detector\footnote{\url{https://huggingface.co/facebook/roberta-hate-speech-dynabench-r4-target}}~\citep{vidgen2021lftw}, and we evaluate toxicity using \texttt{da-electra-hatespeech-detection}\footnote{\url{https://huggingface.co/DaNLP/da-electra-hatespeech-detection}}. We note that this experimental setup largely follows the official example provided from \texttt{TRL}~\citep{huggingface2023detoxifying}\footnote{\url{https://huggingface.co/docs/trl/v0.4.7/en/detoxifying_a_lm}} (which is also adopted by \citet{hu2025a}), but with a newer \texttt{TRL} version. Reward model, evaluation classifier, decoding, and continuation length follow the \texttt{TRL} detoxification example~\citep{huggingface2023detoxifying} and IIF~\citep{hu2025a}. The evaluation classifier is independent of the reward model, so the toxicity row of \Cref{fig:RLHF} is not the signal the policy is optimized against, while the reward row reports that signal itself. \citet{hu2025a} score \(400\) \texttt{Wiki-Toxic} prompts after training; we score \(500\) held-out \texttt{real-toxicity-prompts} prompts at every PPO step, which yields the curves of \Cref{fig:RLHF}.

\subsubsection{Training Configurations}
Training prompts are the first \(80\%\) (unshuffled) of the \(35{,}109\) \texttt{real-toxicity-prompts} entries with prompt toxicity above \(0.3\), i.e., \(28{,}087\) prompts, each truncated to a random length of \(10\)--\(15\) tokens; the remaining \(20\%\) is split in half, one half providing the held-out target prompts of \Cref{subsec:RLHF} and the other the evaluation prompts, so the two never overlap. For each prompt, the policy samples a continuation of up to \(30\) tokens (temperature \(1\), no top-\(k\) or top-\(p\) truncation), scored by the reward model; PPO updates maximize reward while penalizing deviation from a frozen reference policy via an adaptive KL penalty. One epoch over the training prompts with rollout batch size \(256\) gives \(110\) PPO steps. We use \(4\) PPO epochs with mini-batch size \(4\), learning rate \(10^{-5}\) (policy) and \(5 \times 10^{-4}\) (value head), a linear schedule with \(3\%\) warmup, and no weight decay. LoRA is applied to the query and value projections with rank \(16\), \(\alpha = 32\), and no dropout. Curation acts on the LoRA parameters; the value head is not scored and is updated with every rollout of the mini-batch (for IIF, with the pre-filtered rollout batch). The PPO surrogate uses clipping range \(0.2\) for both policy and value, value function coefficient \(0.1\), discount factor \(\gamma = 1.0\), and generalized advantage estimation (GAE) with \(\lambda = 0.95\). The KL coefficient is initialized at \(0.02\) and adapted by \texttt{TRL}'s adaptive KL controller (target \(70\), horizon \(10{,}000\)) using the \(k_1\) estimator; as in \texttt{TRL}, the remaining mini-batches of a rollout are skipped once the mini-batch policy KL exceeds \(0.45\) (\(1.5\) times a target of \(0.3\)). Evaluation is performed every PPO step by sampling continuations (temperature \(1\), top-\(k\) \(50\), no top-\(p\) truncation) for \(500\) held-out prompts with prompt toxicity above \(0.5\) (batch size \(256\)) and reporting the mean predicted toxicity.

\subsubsection{Data Regularization}
In this setting, scores are exact PIP scores (\Cref{subsubsec:per-token-inner-product}); \Cref{adxsubsubsec:rlhf-scoring-backend} repeats the same configurations with compressed scoring. Two types of ``target set'' are considered:
\begin{enumerate*}[label=(\roman*)]
	\item the same rollout batch (self-referencing), or
	\item new rollouts from the current policy on a batch of \(256\) prompts drawn at each step from a held-out target set of \(1024\) prompts (disjoint from the evaluation prompts, see above).
\end{enumerate*}
Furthermore, we consider two target loss functions for the target signals:
\begin{enumerate*}[label=(\roman*)]
	\item the \emph{expected reward}, estimated on the target rollouts \(\{(x_i, y_i)\}_{i=1}^{m}\) by the REINFORCE surrogate \(-\frac{1}{m}\sum_{i} \tilde{R}_i \log \pi_\theta(y_i \mid x_i)\), where \(\tilde{R}_i\) is the raw reward standardized over the target batch (no KL penalty, value baseline, or clipping) and \(\log \pi_\theta(y_i \mid x_i)\) is the sequence log-probability, i.e., the sum over response tokens, and
	\item the \emph{PPO training loss}, i.e., the clipped surrogate plus the value loss that the policy is trained with.
\end{enumerate*}
In both cases the target gradient is captured once per rollout batch at the parameters that generated the rollouts, so for the training-loss target the importance ratio equals one and clipping is inactive: it reduces to the token-level policy gradient with GAE advantages plus the value-loss term.

Finally, we solve \(S_t\) (or \(S_{t,l}\) in the case of Layer-Wise Subset Update) via \emph{negative score filtering}, where at each PPO step, samples (i.e., rollouts) with negative scores are dropped from the update. The retained rollouts of a group are averaged, i.e., the group receives the gradient of the mini-batch loss restricted to its retained samples, so the effective update is normalized by the number of retained rollouts rather than by the mini-batch size; a group whose rollouts all score negative receives a zero gradient at that step (for IIF, an empty filtered rollout batch skips the PPO update), and the value head is updated with every rollout as described above. In every configuration the retained rollouts are trained with the full PPO objective, including the KL penalty; the target loss enters only through the alignment scores.

\subsubsection{Toxicity Under Different Judges}\label{adxsubsubsec:rlhf-judges}
The evaluation toxicity in \Cref{fig:RLHF} is scored by \texttt{da-electra-hatespeech-detection}, the judge of the prior work we follow~\citep{hu2025a}. To show that the comparison does not hinge on this judge, we regenerate continuations from the final checkpoint of every run on the same \(500\) evaluation prompts, with the same decoding as the in-training evaluation, and score them with four classifiers: the paper judge (\texttt{DaNLP}), the reward model (\texttt{LFTW}), and two English toxicity classifiers trained on the Jigsaw toxic-comment data, \texttt{roberta\_toxicity\_classifier}\footnote{\url{https://huggingface.co/s-nlp/roberta_toxicity_classifier}} (\texttt{s-nlp}) and \texttt{toxic-bert}\footnote{\url{https://huggingface.co/unitary/toxic-bert}}. \Cref{fig:RLHF-judges} reports the mean toxicity probability over \(5\) seeds, with the untuned \textsc{GPT-Neo-2.7B} as a reference. The picture of \Cref{fig:RLHF} holds under every judge: with the expected-reward target, every curated update ends far below standard PPO under all four judges (Global and Layer-Wise Subset Updates by roughly an order of magnitude under every judge, IIF by a factor of about four under the paper judge and by an order of magnitude under the two English judges), and with the training-loss target they remain close to standard PPO. Under the English judges, the three curated updates are within noise of one another at the final checkpoint.

\begin{figure}[htbp]
	\centering
	\begin{subfigure}[t]{0.49\linewidth}
		\caption{Self-ref, reward}
		\includegraphics[width=\linewidth]{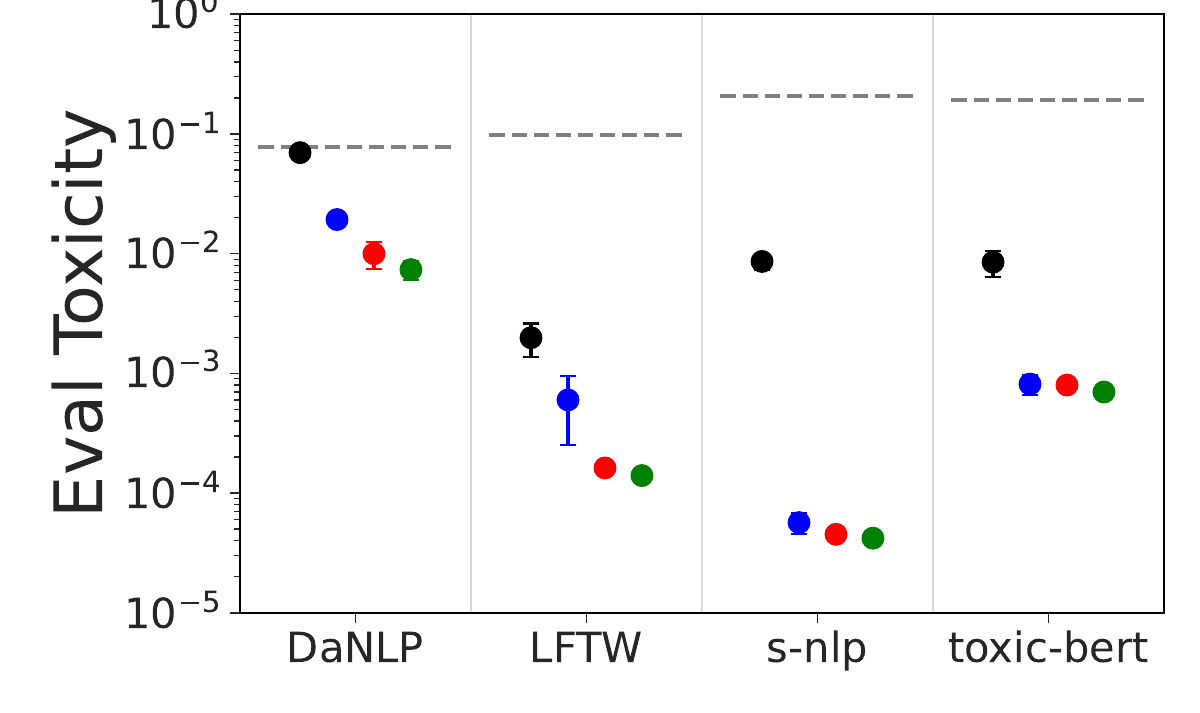}
	\end{subfigure}\hfill
	\begin{subfigure}[t]{0.49\linewidth}
		\caption{Self-ref, train loss}
		\includegraphics[width=\linewidth]{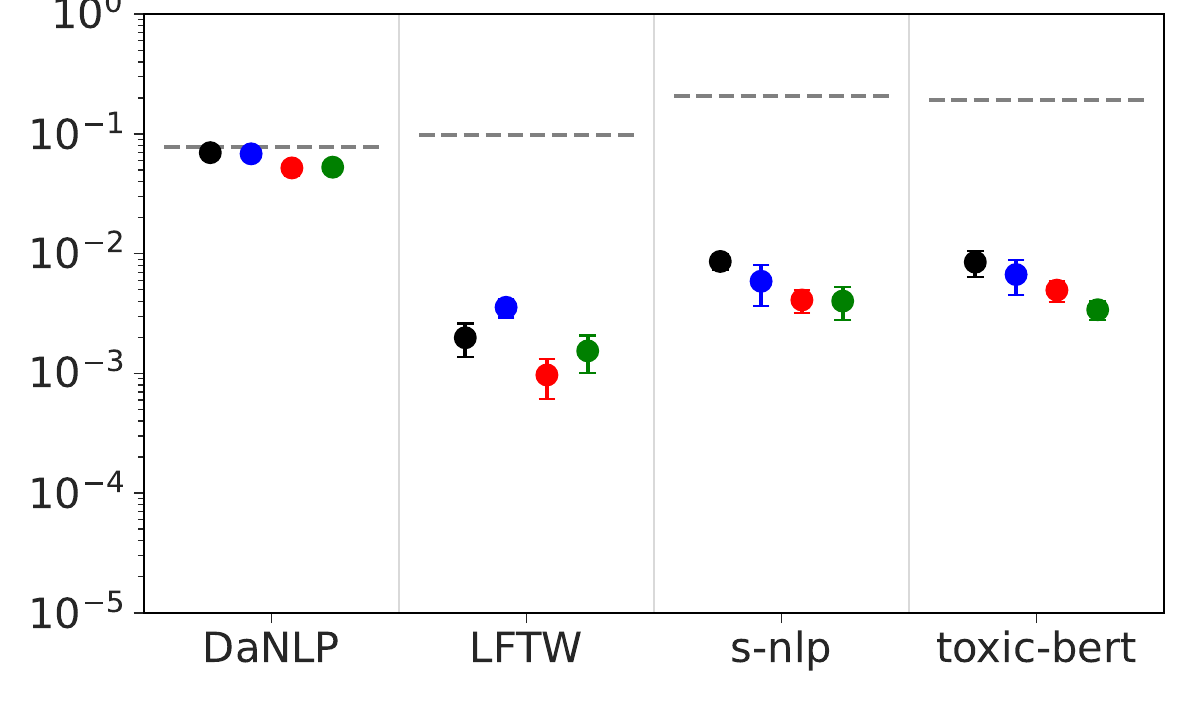}
	\end{subfigure}
	\par\smallskip
	\begin{subfigure}[t]{0.49\linewidth}
		\caption{Held-out, reward}
		\includegraphics[width=\linewidth]{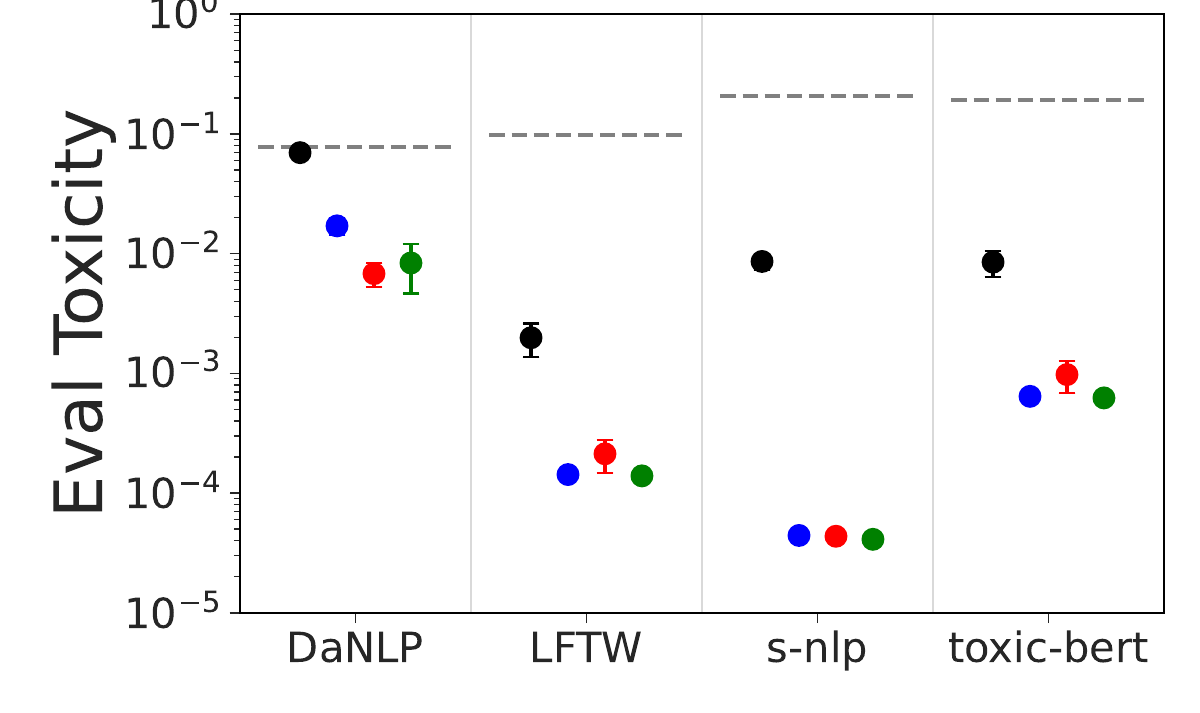}
	\end{subfigure}\hfill
	\begin{subfigure}[t]{0.49\linewidth}
		\caption{Held-out, train loss}
		\includegraphics[width=\linewidth]{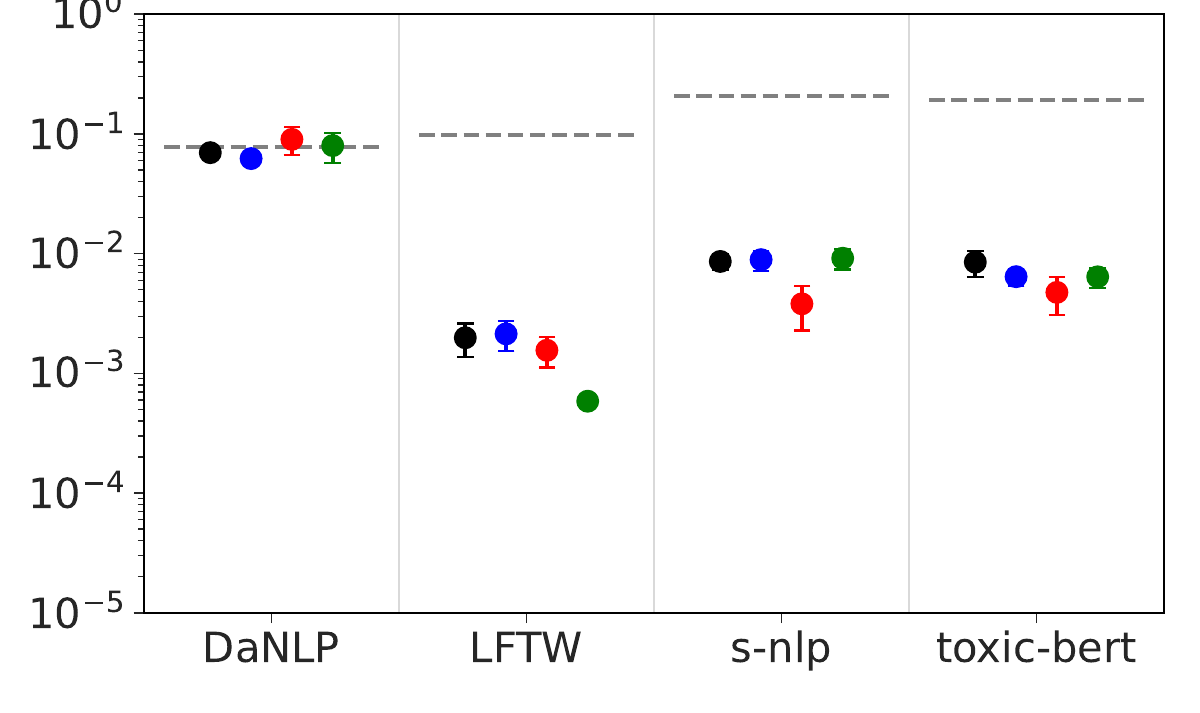}
	\end{subfigure}
	\par\smallskip
	\includegraphics[width=0.8\linewidth]{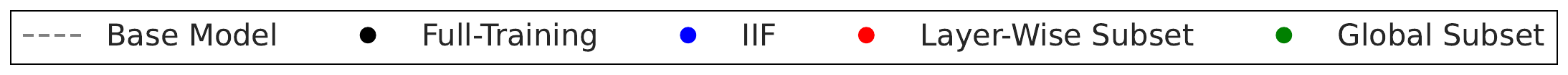}
	\caption{RLHF toxicity (mean probability, log scale) of the final checkpoints under four judges: the paper judge (\texttt{DaNLP}), the reward model (\texttt{LFTW}), and two English classifiers (\texttt{s-nlp}, \texttt{toxic-bert}); dashed lines mark the untuned base model. Panels follow \Cref{fig:RLHF}.}
	\label{fig:RLHF-judges}
\end{figure}

\subsubsection{Compressed Scoring}\label{adxsubsubsec:rlhf-scoring-backend}
The RLHF results of \Cref{subsec:RLHF} use exact PIP scoring (\Cref{subsubsec:per-token-inner-product}), which is cheap under LoRA, where the per-sample adapter gradients are small. Here we ask whether the RL comparison is robust to the scoring approximation, since the target signal is itself a noisy reward estimate and negative score filtering keys on the sign of every score: we repeated every curated configuration of \Cref{fig:RLHF} with compressed scoring (\Cref{subsubsec:compressed-scoring}) at \(\kappa^{(l)} = 64 \times 64\) under the same seeds and hyperparameters; standard PPO does no scoring and is shared. With LoRA rank \(r = 16\), the rank-\(16\) side of each adapter factor is already smaller than \(64\), so its compressor is the identity and only the \(2560\)-dimensional side is projected.

\Cref{fig:RLHF-compressed} shows the compressed-scoring training curves on the axes of \Cref{fig:RLHF}, and \Cref{adxtab:rlhf-scoring-backend} compares the final checkpoints under the four judges of \Cref{adxsubsubsec:rlhf-judges}. Compressed scoring reproduces the main comparison: with the expected-reward target every curated update again ends far below standard PPO under all four judges (Layer-Wise Subset Update by an order of magnitude under the paper judge, Global Subset Update and IIF by a factor of two to six, and all three by an order of magnitude under the two English judges), and with the training-loss target the curated updates stay close to standard PPO. With the expected-reward target, final toxicities move by \(0.1\)--\(0.8\) points, within one to two standard errors and in both directions (Layer-Wise Subset Update slightly lower with compressed scoring, Global Subset Update and self-referencing IIF slightly higher), so compressed scoring orders the endpoints as Layer-Wise below Global below IIF, whereas exact scoring leaves Global and Layer-Wise Subset Updates tied within error.

Two finer differences remain. Under the reward target, Layer-Wise and Global Subset Updates reach \(3\%\) evaluation toxicity after a similar number of PPO steps with compressed scoring (\(22\)--\(27\) on average, against \(10\)--\(18\) with exact scoring and \(38\)--\(45\) for IIF), so the approximation costs some of the early speed-up. Under the training-loss target the backend matters more: the self-referencing Global and Layer-Wise arms, which end below standard PPO with exact scoring (\(5.3\) and \(5.2\) against \(7.0\) under the paper judge), end at \(6.2\) and \(7.4\) with compressed scoring.

\begin{figure}[htbp]
	\centering
	\begin{subfigure}[t]{0.24\linewidth}
		\caption{Self-ref, reward}
		\includegraphics[width=\linewidth]{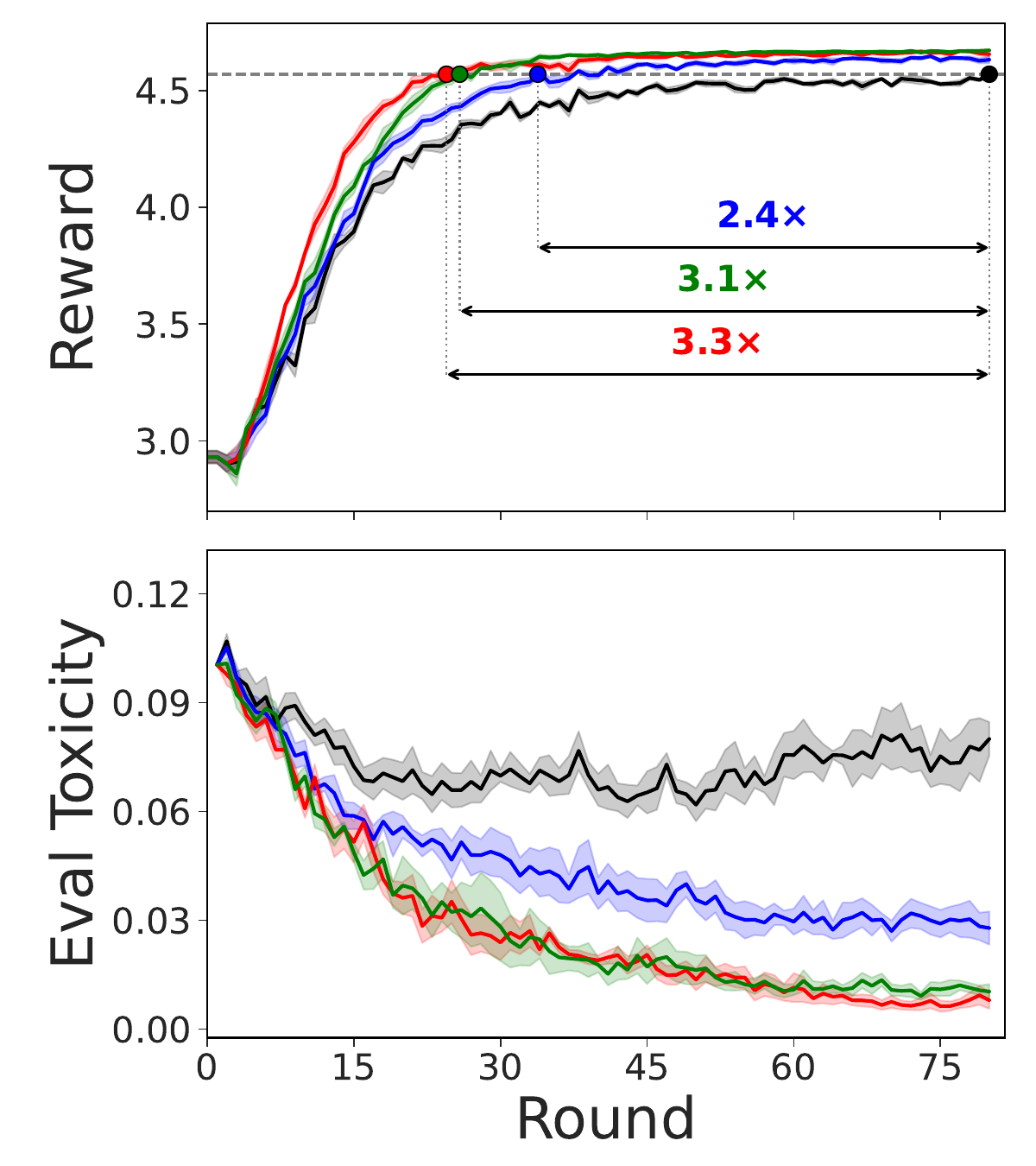}
	\end{subfigure}\hfill
	\begin{subfigure}[t]{0.24\linewidth}
		\caption{Self-ref, train loss}
		\includegraphics[width=\linewidth]{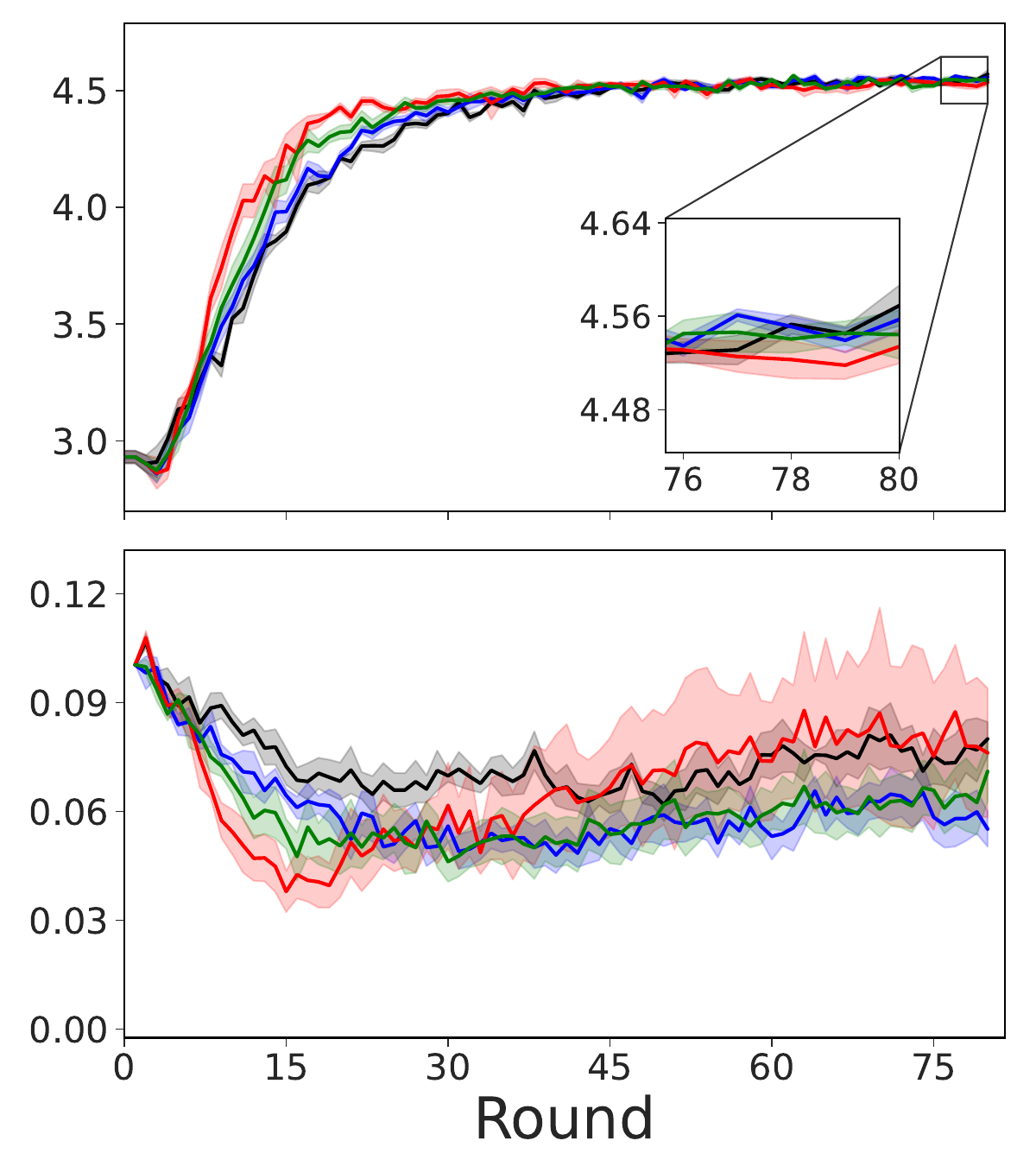}
	\end{subfigure}\hfill
	\begin{subfigure}[t]{0.24\linewidth}
		\caption{Held-out, reward}
		\includegraphics[width=\linewidth]{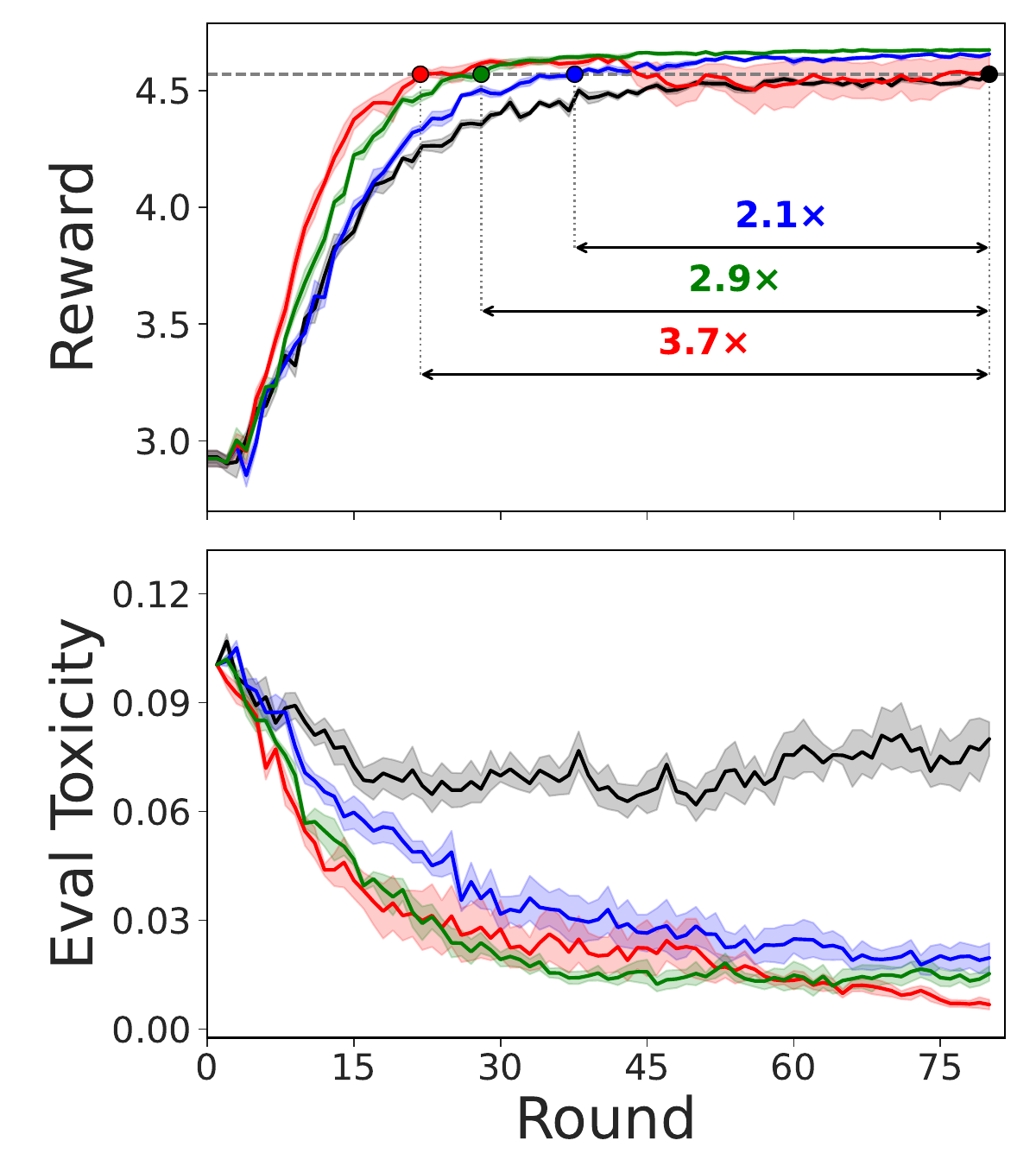}
	\end{subfigure}\hfill
	\begin{subfigure}[t]{0.24\linewidth}
		\caption{Held-out, train loss}
		\includegraphics[width=\linewidth]{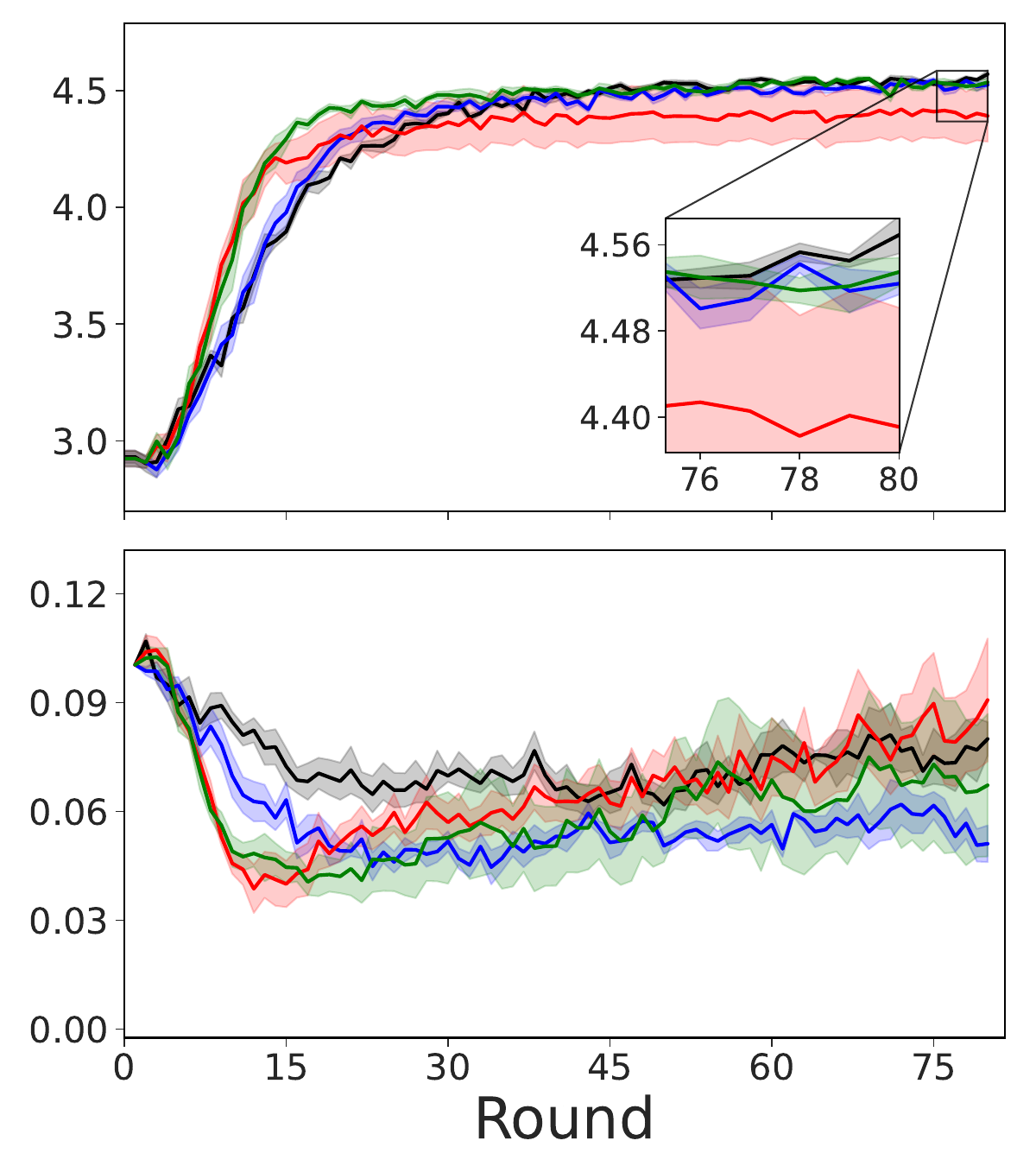}
	\end{subfigure}
	\par\smallskip
	\includegraphics[width=0.5\linewidth]{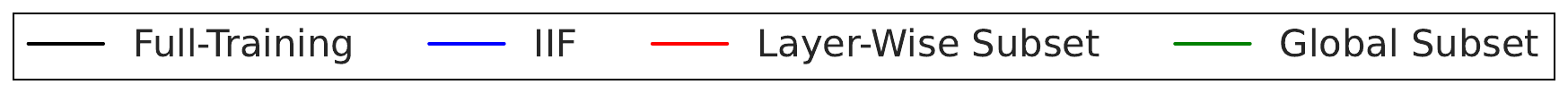}
	\caption{RLHF training curves with compressed scoring, same layout and axes as \Cref{fig:RLHF}. \textbf{Top:} training reward. \textbf{Bottom:} evaluation toxicity (mean probability).}
	\label{fig:RLHF-compressed}
\end{figure}

\begin{table}[tbp]
	\centering
	\caption{Final-checkpoint toxicity under four judges (mean toxicity probability over \(500\) evaluation prompts, in \%) for exact PIP and compressed scoring. Standard PPO does no scoring and is shared.}
	\label{adxtab:rlhf-scoring-backend}
	\small\setlength{\tabcolsep}{3pt}
	\begin{adjustbox}{max width=\linewidth}
		\begin{tabular}{llcccccccc}
\toprule
\multirow{2}{*}{Target} & \multirow{2}{*}{Method} & \multicolumn{2}{c}{\texttt{DaNLP}} & \multicolumn{2}{c}{\texttt{LFTW}} & \multicolumn{2}{c}{\texttt{s-nlp}} & \multicolumn{2}{c}{\texttt{toxic-bert}} \\
\cmidrule(lr){3-4}\cmidrule(lr){5-6}\cmidrule(lr){7-8}\cmidrule(lr){9-10}
 & & Exact & Compressed & Exact & Compressed & Exact & Compressed & Exact & Compressed \\
\midrule
\multicolumn{2}{l}{Full-Training (standard PPO)} & \multicolumn{2}{c}{6.98 \(\pm\) 0.45} & \multicolumn{2}{c}{0.20 \(\pm\) 0.06} & \multicolumn{2}{c}{0.86 \(\pm\) 0.13} & \multicolumn{2}{c}{0.85 \(\pm\) 0.21} \\
\midrule
\multirow{3}{*}{Self-ref., reward} & IIF & 1.93 \(\pm\) 0.25 & 2.77 \(\pm\) 0.41 & 0.06 \(\pm\) 0.03 & 0.02 \(\pm\) 0.00 & 0.01 \(\pm\) 0.00 & 0.01 \(\pm\) 0.00 & 0.08 \(\pm\) 0.01 & 0.07 \(\pm\) 0.00 \\
 & Global Subset & 0.74 \(\pm\) 0.13 & 1.22 \(\pm\) 0.26 & 0.01 \(\pm\) 0.00 & 0.01 \(\pm\) 0.00 & 0.00 \(\pm\) 0.00 & 0.00 \(\pm\) 0.00 & 0.07 \(\pm\) 0.00 & 0.06 \(\pm\) 0.00 \\
 & Layer-Wise Subset & 1.00 \(\pm\) 0.25 & 0.72 \(\pm\) 0.18 & 0.02 \(\pm\) 0.00 & 0.01 \(\pm\) 0.00 & 0.00 \(\pm\) 0.00 & 0.00 \(\pm\) 0.00 & 0.08 \(\pm\) 0.01 & 0.06 \(\pm\) 0.00 \\
\midrule
\multirow{3}{*}{Self-ref., train loss} & IIF & 6.82 \(\pm\) 0.79 & 5.98 \(\pm\) 0.40 & 0.36 \(\pm\) 0.06 & 0.12 \(\pm\) 0.04 & 0.59 \(\pm\) 0.22 & 0.45 \(\pm\) 0.07 & 0.67 \(\pm\) 0.22 & 0.44 \(\pm\) 0.08 \\
 & Global Subset & 5.28 \(\pm\) 0.28 & 6.20 \(\pm\) 0.68 & 0.15 \(\pm\) 0.05 & 0.21 \(\pm\) 0.04 & 0.40 \(\pm\) 0.12 & 1.05 \(\pm\) 0.15 & 0.34 \(\pm\) 0.06 & 0.68 \(\pm\) 0.09 \\
 & Layer-Wise Subset & 5.20 \(\pm\) 0.76 & 7.38 \(\pm\) 1.80 & 0.10 \(\pm\) 0.04 & 0.15 \(\pm\) 0.04 & 0.41 \(\pm\) 0.09 & 0.50 \(\pm\) 0.07 & 0.50 \(\pm\) 0.10 & 0.65 \(\pm\) 0.13 \\
\midrule
\multirow{3}{*}{Held-out, reward} & IIF & 1.71 \(\pm\) 0.28 & 1.80 \(\pm\) 0.29 & 0.01 \(\pm\) 0.00 & 0.01 \(\pm\) 0.00 & 0.00 \(\pm\) 0.00 & 0.01 \(\pm\) 0.00 & 0.06 \(\pm\) 0.00 & 0.06 \(\pm\) 0.00 \\
 & Global Subset & 0.84 \(\pm\) 0.37 & 1.55 \(\pm\) 0.25 & 0.01 \(\pm\) 0.00 & 0.01 \(\pm\) 0.00 & 0.00 \(\pm\) 0.00 & 0.00 \(\pm\) 0.00 & 0.06 \(\pm\) 0.00 & 0.06 \(\pm\) 0.00 \\
 & Layer-Wise Subset & 0.68 \(\pm\) 0.15 & 0.56 \(\pm\) 0.10 & 0.02 \(\pm\) 0.01 & 0.05 \(\pm\) 0.02 & 0.00 \(\pm\) 0.00 & 0.01 \(\pm\) 0.00 & 0.10 \(\pm\) 0.03 & 0.09 \(\pm\) 0.02 \\
\midrule
\multirow{3}{*}{Held-out, train loss} & IIF & 6.23 \(\pm\) 0.54 & 5.40 \(\pm\) 0.34 & 0.21 \(\pm\) 0.06 & 0.11 \(\pm\) 0.03 & 0.89 \(\pm\) 0.17 & 0.85 \(\pm\) 0.15 & 0.64 \(\pm\) 0.11 & 0.63 \(\pm\) 0.06 \\
 & Global Subset & 7.99 \(\pm\) 2.26 & 7.13 \(\pm\) 1.82 & 0.06 \(\pm\) 0.01 & 0.24 \(\pm\) 0.04 & 0.92 \(\pm\) 0.18 & 0.64 \(\pm\) 0.23 & 0.64 \(\pm\) 0.12 & 0.58 \(\pm\) 0.13 \\
 & Layer-Wise Subset & 9.01 \(\pm\) 2.33 & 8.89 \(\pm\) 1.31 & 0.16 \(\pm\) 0.04 & 0.22 \(\pm\) 0.07 & 0.38 \(\pm\) 0.15 & 0.76 \(\pm\) 0.23 & 0.47 \(\pm\) 0.17 & 0.67 \(\pm\) 0.17 \\
\bottomrule
\end{tabular}

	\end{adjustbox}
\end{table}

\subsection{Reinforcement Learning with Verifiable Rewards}\label{adxsubsec:details-of-rlvr}
We now provide the missing details of the RLVR experiments in \Cref{subsec:RLVR}: libraries and datasets, training configurations, and data regularization.

\subsubsection{Libraries and Datasets}
We use the \texttt{Verl} framework\footnote{\url{https://github.com/verl-project/verl}}~\citep{sheng2024hybridflow} (Apache-2.0 license) to train on the \texttt{MATH} dataset\footnote{\url{https://github.com/hendrycks/math}}~\citep{hendrycks2021measuring} (MIT license). We note that this experimental setup largely follows the official example provided from \texttt{Verl},\footnote{\url{https://github.com/verl-project/verl/blob/65eb5a15058d73ef54064092222b7af67d658241/examples/grpo_trainer/run_qwen3-8b.sh}} with different models and datasets.

\subsubsection{Training Configurations}
The policy is initialized from the pretrained base model \textsc{Qwen3-1.7B-Base}~\citep{qwen3technicalreport} rather than its post-trained counterpart and trained with full-parameter updates; each run uses four \texttt{H200} GPUs on one node. Each training sample is a math problem; the model generates \(8\) candidate solutions and receives a verifiable reward from exact-match checking against the ground-truth answer. The policy is optimized using GRPO~\citep{shao2024deepseekmath} with training batch size \(128\), mini-batch size \(32\), micro-batch size \(2\) per GPU for the curated arms and \(4\) for standard GRPO, KL coefficient \(0.001\) (low-variance estimator), learning rate \(10^{-6}\), maximum prompt length \(1024\), and maximum response length \(2048\). Gradient checkpointing is enabled. We train for \(80\) rounds and evaluate every \(3\) rounds on about \(4{,}500\) \texttt{MATH} test problems with greedy decoding and one sample per prompt (\texttt{Verl}'s default validation setting); \Cref{fig:RLVR-math} reports this accuracy. All other settings follow the default \texttt{Verl} configuration.

\subsubsection{Data Regularization}
We instantiate the framework as in-run cross-validation with exact PIP scores; an RLVR round is dominated by rollout generation, so the scoring backend is immaterial to its cost. Every fourth round, \(64\) prompts drawn from the training set itself serve as the target batch (no prompts are held aside), each is rolled out \(8\) times with the current policy, GRPO's own group size, and the REINFORCE surrogate of the expected verifier reward on these \(512\) target rollouts (as in \Cref{adxsubsec:details-of-rlhf}) provides the target signal, on top of the \(1024\) training rollouts (\(128\) prompts \(\times\) \(8\)) of every round. As in RLHF, the target gradient is captured in a separate pass, so the target rollouts never enter the gradient of any parameter, and is reused across mini-batch updates; since the target gradient changes slowly along the trajectory, we amortize it over four rounds rather than refreshing it every round, which brings the target cost to \(128\) target rollouts per round, an eighth of the training rollouts. Negative score filtering then drops the training rollouts with negative alignment scores. Retained rollouts are aggregated with \texttt{Verl}'s default token-mean loss, so each layer's update is the policy-gradient loss averaged over the response tokens of the rollouts it retains, and a layer that retains no rollout receives a zero gradient at that round; selection and this normalization act per micro-batch of two rollouts per GPU, and the micro-batch gradients are averaged across micro-batches and data-parallel ranks exactly as in standard GRPO. Scoring and selection act on all linear modules, the embedding matrix, and the tied output head (\Cref{adxsubsec:beyond-linear}).

\paragraph{Results.}
\Cref{adxtab:rlvr} quantifies \Cref{fig:RLVR-math}. After the first \(20\) rounds, Layer-Wise Subset Update leads at every evaluation round and reaches a final accuracy of \(63.8\%\), against \(63.0\%\) for both standard GRPO and Global Subset Update. The target rollouts are the price of the target signal, which in RL, unlike SFT, has to be re-sampled from the current policy and is standard in influence-based data selection in RL~\citep{wang2025nice}; amortized over the four rounds between refreshes, they add an eighth to the rollout budget of every round. This overhead does not account for the gain: continuing standard GRPO past the \(80\)-round budget, it first reaches the round-\(78\) accuracy of Layer-Wise Subset Update at about round \(147\), after roughly \(1.7\) times as many rollouts as Layer-Wise Subset Update has generated by round \(78\), target rollouts included; matched on rollouts instead, i.e., at about \(88\) rounds, standard GRPO remains at \(63.0\%\).

\begin{table}[htpb]
	\centering
	\caption{RLVR evaluation accuracy (\%) of \Cref{fig:RLVR-math} at selected rounds and at the end of training.}
	\label{adxtab:rlvr}
	\small
\begin{tabular}{lccc}
	\toprule
	 & \textbf{Full-Training} & \textbf{Global Subset} & \textbf{Layer-Wise Subset} \\
	\midrule
	Round 30 & \(61.6 \pm 0.2\) & \(61.7 \pm 0.2\) & \(62.6 \pm 0.1\) \\
	Round 60 & \(62.8 \pm 0.1\) & \(62.7 \pm 0.3\) & \(63.3 \pm 0.3\) \\
	Final & \(63.0 \pm 0.1\) & \(63.0 \pm 0.1\) & \(63.8 \pm 0.2\) \\
	\bottomrule
\end{tabular}

\end{table}

\subsection{System Efficiency}\label{adxsubsec:system-efficiency}
This section expands the system benchmarks of \Cref{subsec:system-efficiency} in three directions: the complete per-step sweep, which adds the small-batch configuration \(n=2\), \(\seqlen=1024\) and the two-pass implementation of Global Subset Update (Global Subset (2P), \Cref{adxsec:two-pass-implementation}) to \Cref{tab:system-overhead} (\Cref{adxtab:timing-grid}); the same sweep and the peak GPU memory under activation checkpointing (\Cref{adxtab:timing-grid-ckpt,adxtab:peak-memory}); the scoring-cost sweep of \Cref{tab:score-cost} is discussed last. The synthetic benchmarks run on a single \texttt{A40} GPU in \texttt{bfloat16} on three models from different architecture families (\textsc{SmolLM2-360M}, \textsc{TinyLlama-1.1B}, \textsc{Llama-3.2-3B}) with exact GIP scoring, the backend of the question-answering experiments (\Cref{adxsubsubsec:sft-qa}), and report timings on fixed-length batches averaged over \(20\) timed iterations after \(10\) warm-up iterations.

\paragraph{Per-step breakdown.}
\Cref{tab:system-overhead,adxtab:timing-grid} split one training step into forward, backward, and optimizer time, and the backward pass into four components: \texttt{a.grad} computes the input-activation gradient \(\partial \ell / \partial a^{(l)}\) with a single matrix multiplication against the weight matrix (the output gradient \(\partial \ell / \partial e^{(l)}\) is propagated from the layer above and costs nothing at this layer); \texttt{scoring} computes the per-sample alignment scores; \texttt{w.grad} assembles the weight gradient from the selected samples; and \texttt{autograd} collects the remaining overhead, namely graph traversal, memory allocation, hook dispatch, and the backward passes of the non-linear modules (normalization layers, activation functions, etc.).

The overhead of the curated updates is the balance of two opposing parts. The extra target example and the scoring add a nearly fixed cost: \texttt{Forward} and \texttt{a.grad} grow with the \((n+m)/n = 9/8\) larger batch, and \texttt{scoring} costs \(32\)--\(86\)~ms on the three models. Against this, \texttt{w.grad} shrinks, because only the \(k = 4\) kept samples enter the weight gradient and the fused kernel reads them in place: by \(1.3\times\) on \textsc{SmolLM2-360M} (\(27.5\) to \(21.6\)~ms), \(1.7\times\) on \textsc{TinyLlama-1.1B} (\(75.9\) to \(44.4\)~ms), and \(2.7\times\) on \textsc{Llama-3.2-3B} (\(338\) to \(126\)~ms). On the smallest model the step is launch-bound, so the saving is small and the fixed costs dominate (\(+18\)--\(20\%\)); as the model grows the weight gradient becomes compute-bound and the saving grows with it, until on \textsc{Llama-3.2-3B} it outweighs the fixed costs and the curated step is on par with Full-Training (\(-1\%\) to \(-2\%\)). The relative overhead is therefore configuration-dependent. In the small-batch configuration (\(n = 2\), \(\seqlen = 1024\)) the single target example is half the candidate batch, so \texttt{Forward} and \texttt{a.grad} grow by up to \(1.6\times\) and the fixed costs are amortized over only two candidates; the overhead of Layer-Wise Subset Update is then \(+47\%\) on \textsc{SmolLM2-360M}, \(+35\%\) on \textsc{TinyLlama-1.1B}, and \(+21\%\) on \textsc{Llama-3.2-3B}, although the absolute overhead stays within about \(180\)~ms per step. The favourable regime is the one that matters for LLM post-training, namely larger models and larger candidate batches. Finally, the two-pass Global Subset (2P) repeats the forward pass to separate scoring from update assembly, and costs \(+33\%\) to \(+60\%\) over Full-Training at \(n=8\); it is needed only when a group spans several checkpoint segments (below). With compressed scoring (\(\kappa^{(l)} = 64 \times 64\), \Cref{subsubsec:compressed-scoring}) in place of GIP, the same steps cost \(+15\%\), \(+4\%\), and \(-5\%\) at \(n=8\): the cheaper backend saves \(3\)--\(5\) points of step time in this configuration, at the accuracy cost quantified in \Cref{adxsubsubsec:sft-ablations}, and its advantage grows with \(\seqlen\) and \(m\) (\Cref{tab:score-cost}).

\begin{table}[htbp]
	\centering
	\caption{Per-step component breakdown (ms) on a single \texttt{A40} GPU across three models and two batch configurations (\(m=1\), \(k=n/2\), exact GIP scoring).}
	\label{adxtab:timing-grid}
	\begin{subtable}[t]{0.49\textwidth}
		\centering
		\subcaption{\textsc{SmolLM2-360M}, \(n=8\), \(\seqlen=512\).}
		\begin{adjustbox}{max width=\linewidth}
\begin{tabular}{l ccc cc}
\toprule
	\multirow{2}{*}{\textbf{Component}} & \multirow{2}{*}{\textbf{Full-Training}} & \multirow{2}{*}{\textbf{Layer-Wise Subset}} & \multirow{2}{*}{\textbf{Global Subset (1P)}} & \multicolumn{2}{c}{\textbf{Global Subset (2P)}} \\
\cmidrule(lr){5-6}
	& & & & Score & Grad \\
\midrule
	\textbf{Forward} & 78.7 & 90.5 & 89.1 & 89.0 & 49.9 \\
	\textbf{Backward} & 155.2 & 194.1 & 191.6 & 169.7 & 80.8 \\
\rowcolor{black!6}
	\quad \texttt{a.grad} & 32.6 & 33.5 & 33.3 & 33.4 & 16.0 \\
\rowcolor{black!6}
	\quad \texttt{scoring} & --- & 32.2 & 32.5 & 32.1 & --- \\
\rowcolor{black!6}
	\quad \texttt{w.grad} & 27.5 & 21.6 & 21.2 & --- & 15.4 \\
\rowcolor{black!6}
	\quad \texttt{autograd} & 95.1 & 106.8 & 104.5 & 104.2 & 49.4 \\
	\textbf{Optimizer} & 26.6 & 26.7 & 26.6 & --- & 26.6 \\
\midrule
	\textbf{Total} & \textbf{260} & \textbf{311} \scriptsize(+19.5\%) & \textbf{307} \scriptsize(+18.0\%) & \multicolumn{2}{c}{\textbf{416} \scriptsize(+59.7\%)} \\
\bottomrule
\end{tabular}

		\end{adjustbox}
	\end{subtable}
	\hfill
	\begin{subtable}[t]{0.49\textwidth}
		\centering
		\subcaption{\textsc{SmolLM2-360M}, \(n=2\), \(\seqlen=1024\).}
		\begin{adjustbox}{max width=\linewidth}
\begin{tabular}{l ccc cc}
\toprule
	\multirow{2}{*}{\textbf{Component}} & \multirow{2}{*}{\textbf{Full-Training}} & \multirow{2}{*}{\textbf{Layer-Wise Subset}} & \multirow{2}{*}{\textbf{Global Subset (1P)}} & \multicolumn{2}{c}{\textbf{Global Subset (2P)}} \\
\cmidrule(lr){5-6}
	& & & & Score & Grad \\
\midrule
	\textbf{Forward} & 49.4 & 61.6 & 62.0 & 61.8 & 51.2 \\
	\textbf{Backward} & 82.2 & 144.8 & 143.7 & 130.9 & 81.6 \\
\rowcolor{black!6}
	\quad \texttt{a.grad} & 15.9 & 25.5 & 25.5 & 25.5 & 14.9 \\
\rowcolor{black!6}
	\quad \texttt{scoring} & --- & 31.6 & 32.5 & 32.3 & --- \\
\rowcolor{black!6}
	\quad \texttt{w.grad} & 15.2 & 12.4 & 12.2 & --- & 12.9 \\
\rowcolor{black!6}
	\quad \texttt{autograd} & 51.1 & 75.2 & 73.5 & 73.1 & 53.8 \\
	\textbf{Optimizer} & 26.6 & 26.6 & 26.6 & --- & 28.2 \\
\midrule
	\textbf{Total} & \textbf{158} & \textbf{233} \scriptsize(+47.3\%) & \textbf{232} \scriptsize(+46.8\%) & \multicolumn{2}{c}{\textbf{354} \scriptsize(+123.5\%)} \\
\bottomrule
\end{tabular}

		\end{adjustbox}
	\end{subtable}
	\\[1.5em]
	\begin{subtable}[t]{0.49\textwidth}
		\centering
		\subcaption{\textsc{TinyLlama-1.1B}, \(n=8\), \(\seqlen=512\).}
		\begin{adjustbox}{max width=\linewidth}
\begin{tabular}{l ccc cc}
\toprule
	\multirow{2}{*}{\textbf{Component}} & \multirow{2}{*}{\textbf{Full-Training}} & \multirow{2}{*}{\textbf{Layer-Wise Subset}} & \multirow{2}{*}{\textbf{Global Subset (1P)}} & \multicolumn{2}{c}{\textbf{Global Subset (2P)}} \\
\cmidrule(lr){5-6}
	& & & & Score & Grad \\
\midrule
	\textbf{Forward} & 134.0 & 145.7 & 146.2 & 146.6 & 71.1 \\
	\textbf{Backward} & 278.5 & 306.6 & 306.3 & 261.6 & 145.8 \\
\rowcolor{black!6}
	\quad \texttt{a.grad} & 78.3 & 81.2 & 81.4 & 81.5 & 41.7 \\
\rowcolor{black!6}
	\quad \texttt{scoring} & --- & 42.0 & 42.3 & 42.1 & --- \\
\rowcolor{black!6}
	\quad \texttt{w.grad} & 75.9 & 44.4 & 44.4 & --- & 39.5 \\
\rowcolor{black!6}
	\quad \texttt{autograd} & 124.3 & 139.1 & 138.1 & 137.9 & 64.6 \\
	\textbf{Optimizer} & 80.7 & 80.7 & 80.8 & --- & 80.7 \\
\midrule
	\textbf{Total} & \textbf{493} & \textbf{533} \scriptsize(+8.1\%) & \textbf{533} \scriptsize(+8.1\%) & \multicolumn{2}{c}{\textbf{706} \scriptsize(+43.1\%)} \\
\bottomrule
\end{tabular}

		\end{adjustbox}
	\end{subtable}
	\hfill
	\begin{subtable}[t]{0.49\textwidth}
		\centering
		\subcaption{\textsc{TinyLlama-1.1B}, \(n=2\), \(\seqlen=1024\).}
		\begin{adjustbox}{max width=\linewidth}
\begin{tabular}{l ccc cc}
\toprule
	\multirow{2}{*}{\textbf{Component}} & \multirow{2}{*}{\textbf{Full-Training}} & \multirow{2}{*}{\textbf{Layer-Wise Subset}} & \multirow{2}{*}{\textbf{Global Subset (1P)}} & \multicolumn{2}{c}{\textbf{Global Subset (2P)}} \\
\cmidrule(lr){5-6}
	& & & & Score & Grad \\
\midrule
	\textbf{Forward} & 71.6 & 103.8 & 103.9 & 104.2 & 54.5 \\
	\textbf{Backward} & 147.6 & 220.5 & 219.9 & 195.5 & 79.6 \\
\rowcolor{black!6}
	\quad \texttt{a.grad} & 41.5 & 57.0 & 57.1 & 57.2 & 21.6 \\
\rowcolor{black!6}
	\quad \texttt{scoring} & --- & 41.4 & 42.0 & 41.9 & --- \\
\rowcolor{black!6}
	\quad \texttt{w.grad} & 39.2 & 24.5 & 24.1 & --- & 21.3 \\
\rowcolor{black!6}
	\quad \texttt{autograd} & 66.9 & 97.6 & 96.7 & 96.5 & 36.7 \\
	\textbf{Optimizer} & 80.7 & 80.7 & 80.7 & --- & 80.8 \\
\midrule
	\textbf{Total} & \textbf{300} & \textbf{405} \scriptsize(+35.0\%) & \textbf{405} \scriptsize(+34.9\%) & \multicolumn{2}{c}{\textbf{515} \scriptsize(+71.6\%)} \\
\bottomrule
\end{tabular}

		\end{adjustbox}
	\end{subtable}
	\\[1.5em]
	\begin{subtable}[t]{0.49\textwidth}
		\centering
		\subcaption{\textsc{Llama-3.2-3B}, \(n=8\), \(\seqlen=512\).}
		\begin{adjustbox}{max width=\linewidth}
\begin{tabular}{l ccc cc}
\toprule
	\multirow{2}{*}{\textbf{Component}} & \multirow{2}{*}{\textbf{Full-Training}} & \multirow{2}{*}{\textbf{Layer-Wise Subset}} & \multirow{2}{*}{\textbf{Global Subset (1P)}} & \multicolumn{2}{c}{\textbf{Global Subset (2P)}} \\
\cmidrule(lr){5-6}
	& & & & Score & Grad \\
\midrule
	\textbf{Forward} & 340.0 & 387.6 & 387.4 & 388.8 & 188.8 \\
	\textbf{Backward} & 821.0 & 752.4 & 758.8 & 621.5 & 424.7 \\
\rowcolor{black!6}
	\quad \texttt{a.grad} & 226.2 & 254.7 & 252.4 & 256.2 & 116.9 \\
\rowcolor{black!6}
	\quad \texttt{scoring} & --- & 85.5 & 85.1 & 85.1 & --- \\
\rowcolor{black!6}
	\quad \texttt{w.grad} & 338.1 & 126.2 & 136.0 & --- & 173.5 \\
\rowcolor{black!6}
	\quad \texttt{autograd} & 256.7 & 286.0 & 285.3 & 280.3 & 134.3 \\
	\textbf{Optimizer} & 235.5 & 235.4 & 235.5 & --- & 235.5 \\
\midrule
	\textbf{Total} & \textbf{1397} & \textbf{1375} \scriptsize(-1.5\%) & \textbf{1382} \scriptsize(-1.1\%) & \multicolumn{2}{c}{\textbf{1859} \scriptsize(+33.1\%)} \\
\bottomrule
\end{tabular}

		\end{adjustbox}
	\end{subtable}
	\hfill
	\begin{subtable}[t]{0.49\textwidth}
		\centering
		\subcaption{\textsc{Llama-3.2-3B}, \(n=2\), \(\seqlen=1024\).}
		\begin{adjustbox}{max width=\linewidth}
\begin{tabular}{l ccc cc}
\toprule
	\multirow{2}{*}{\textbf{Component}} & \multirow{2}{*}{\textbf{Full-Training}} & \multirow{2}{*}{\textbf{Layer-Wise Subset}} & \multirow{2}{*}{\textbf{Global Subset (1P)}} & \multicolumn{2}{c}{\textbf{Global Subset (2P)}} \\
\cmidrule(lr){5-6}
	& & & & Score & Grad \\
\midrule
	\textbf{Forward} & 183.5 & 263.7 & 263.1 & 264.3 & 113.7 \\
	\textbf{Backward} & 426.8 & 523.1 & 526.3 & 444.9 & 200.6 \\
\rowcolor{black!6}
	\quad \texttt{a.grad} & 115.5 & 165.6 & 163.7 & 164.4 & 62.4 \\
\rowcolor{black!6}
	\quad \texttt{scoring} & --- & 85.9 & 85.6 & 85.5 & --- \\
\rowcolor{black!6}
	\quad \texttt{w.grad} & 172.4 & 70.6 & 75.0 & --- & 63.0 \\
\rowcolor{black!6}
	\quad \texttt{autograd} & 138.9 & 201.0 & 202.0 & 194.9 & 75.3 \\
	\textbf{Optimizer} & 235.6 & 235.5 & 235.6 & --- & 235.5 \\
\midrule
	\textbf{Total} & \textbf{846} & \textbf{1022} \scriptsize(+20.8\%) & \textbf{1025} \scriptsize(+21.2\%) & \multicolumn{2}{c}{\textbf{1259} \scriptsize(+48.8\%)} \\
\bottomrule
\end{tabular}

		\end{adjustbox}
	\end{subtable}
\end{table}

\paragraph{Activation checkpointing.}
As noted in \Cref{sec:system}, the partition granularity determines the memory footprint and the compatibility with activation checkpointing. \Cref{adxtab:peak-memory} reports peak GPU memory with and without checkpointing, and \Cref{adxtab:timing-grid-ckpt} the per-step breakdown with checkpointing.

\begin{table}[htbp]
	\centering
	\caption{Peak GPU memory (GB) across three models and two batch configurations (\(m=1\), \(k=n/2\), exact GIP scoring) on a single \texttt{A40} GPU, with and without activation checkpointing (ckpt).}
	\label{adxtab:peak-memory}
	\begin{adjustbox}{max width=\linewidth}
\begin{tabular}{cc cccc cccc}
\toprule
	& & \multicolumn{4}{c}{\(n=8,\; T=512\)} & \multicolumn{4}{c}{\(n=2,\; T=1024\)} \\
\cmidrule(lr){3-6}\cmidrule(lr){7-10}
	\textbf{Model} & \textbf{Setting} & Full-Training & Layer-Wise Subset & Global Subset (1P) & Global Subset (2P) & Full-Training & Layer-Wise Subset & Global Subset (1P) & Global Subset (2P) \\
\midrule
	\multirow{2}{*}{\textsc{SmolLM2-360M}} & No ckpt & 9.4 & 10.3 & 10.3 & 10.3 & 5.7 & 7.6 & 7.6 & 7.6 \\
	 & With ckpt & 4.6 & 4.9 & 6.5 & 4.9 & 3.4 & 3.9 & 5.0 & 3.9 \\
\midrule
	\multirow{2}{*}{\textsc{TinyLlama-1.1B}} & No ckpt & 15.0 & 16.1 & 16.1 & 16.1 & 10.6 & 12.9 & 12.9 & 12.9 \\
	 & With ckpt & 10.3 & 10.3 & 12.4 & 10.3 & 10.3 & 10.3 & 10.4 & 10.3 \\
\midrule
	\multirow{2}{*}{\textsc{Llama-3.2-3B}} & No ckpt & 37.9 & 40.5 & 40.5 & 40.5 & 29.9 & 33.0 & 33.0 & 33.0 \\
	 & With ckpt & 29.9 & 29.9 & 30.7 & 29.9 & 29.9 & 29.9 & 29.9 & 29.9 \\
\bottomrule
\end{tabular}

	\end{adjustbox}
\end{table}

Without checkpointing, peak memory is nearly identical across methods, and the difference to Full-Training is the \(m\) extra target samples. With checkpointing, the group of Global Subset Update spans many layers across the checkpointing schedule and breaks the one-pass scheduling, so its memory overhead grows visibly (\(6.5\) vs.\ \(4.6\)~GB for Full-Training on \textsc{SmolLM2-360M}, \(12.4\) vs.\ \(10.3\)~GB on \textsc{TinyLlama-1.1B}, and \(30.7\) vs.\ \(29.9\)~GB on \textsc{Llama-3.2-3B} at \(n=8\), \(\seqlen=512\)), whereas Layer-Wise Subset Update stays within \(0.3\)~GB of Full-Training. Global Subset (2P), which separates the scoring and gradient passes (\Cref{adxsec:two-pass-implementation}), matches Layer-Wise's footprint at a compute cost of \(+38\%\) to \(+85\%\) over Full-Training. For the one-pass methods, the compute overhead under checkpointing stays between \(+1\%\) and \(+18\%\) at \(n=8\), \(\seqlen=512\) (\textsc{Llama-3.2-3B} is again at parity with Full-Training), since checkpointing inflates the baseline and the scoring-augmented backward pass roughly in proportion.

\begin{table}[htbp]
	\centering
	\caption{Same settings as \Cref{adxtab:timing-grid} with activation checkpointing.}
	\label{adxtab:timing-grid-ckpt}
	\begin{subtable}[t]{0.49\textwidth}
		\centering
		\subcaption{\textsc{SmolLM2-360M}, \(n=8\), \(\seqlen=512\).}
		\begin{adjustbox}{max width=\linewidth}
\begin{tabular}{l ccc cc}
\toprule
	\multirow{2}{*}{\textbf{Component}} & \multirow{2}{*}{\textbf{Full-Training}} & \multirow{2}{*}{\textbf{Layer-Wise Subset}} & \multirow{2}{*}{\textbf{Global Subset (1P)}} & \multicolumn{2}{c}{\textbf{Global Subset (2P)}} \\
\cmidrule(lr){5-6}
	& & & & Score & Grad \\
\midrule
	\textbf{Forward} & 78.0 & 88.8 & 89.4 & 88.9 & 82.2 \\
	\textbf{Backward} & 218.3 & 264.9 & 263.4 & 241.6 & 155.6 \\
\rowcolor{black!6}
	\quad \texttt{a.grad} & 32.6 & 33.4 & 33.3 & 33.4 & 19.5 \\
\rowcolor{black!6}
	\quad \texttt{scoring} & --- & 31.9 & 32.3 & 32.1 & --- \\
\rowcolor{black!6}
	\quad \texttt{w.grad} & 27.5 & 21.5 & 21.3 & --- & 17.3 \\
\rowcolor{black!6}
	\quad \texttt{autograd} & 158.2 & 178.2 & 176.4 & 176.1 & 118.8 \\
	\textbf{Optimizer} & 26.6 & 26.6 & 26.6 & --- & 28.5 \\
\midrule
	\textbf{Total} & \textbf{323} & \textbf{380} \scriptsize(+17.8\%) & \textbf{379} \scriptsize(+17.5\%) & \multicolumn{2}{c}{\textbf{597} \scriptsize(+84.8\%)} \\
\bottomrule
\end{tabular}

		\end{adjustbox}
	\end{subtable}
	\hfill
	\begin{subtable}[t]{0.49\textwidth}
		\centering
		\subcaption{\textsc{SmolLM2-360M}, \(n=2\), \(\seqlen=1024\).}
		\begin{adjustbox}{max width=\linewidth}
\begin{tabular}{l ccc cc}
\toprule
	\multirow{2}{*}{\textbf{Component}} & \multirow{2}{*}{\textbf{Full-Training}} & \multirow{2}{*}{\textbf{Layer-Wise Subset}} & \multirow{2}{*}{\textbf{Global Subset (1P)}} & \multicolumn{2}{c}{\textbf{Global Subset (2P)}} \\
\cmidrule(lr){5-6}
	& & & & Score & Grad \\
\midrule
	\textbf{Forward} & 70.8 & 73.7 & 72.4 & 74.0 & 79.2 \\
	\textbf{Backward} & 150.4 & 203.0 & 201.5 & 180.8 & 163.5 \\
\rowcolor{black!6}
	\quad \texttt{a.grad} & 19.2 & 26.2 & 26.2 & 25.6 & 17.0 \\
\rowcolor{black!6}
	\quad \texttt{scoring} & --- & 32.8 & 34.2 & 32.3 & --- \\
\rowcolor{black!6}
	\quad \texttt{w.grad} & 17.1 & 13.6 & 12.4 & --- & 14.3 \\
\rowcolor{black!6}
	\quad \texttt{autograd} & 114.1 & 130.3 & 128.6 & 123.0 & 132.2 \\
	\textbf{Optimizer} & 27.8 & 26.8 & 26.8 & --- & 28.3 \\
\midrule
	\textbf{Total} & \textbf{249} & \textbf{303} \scriptsize(+21.9\%) & \textbf{301} \scriptsize(+20.8\%) & \multicolumn{2}{c}{\textbf{526} \scriptsize(+111.2\%)} \\
\bottomrule
\end{tabular}

		\end{adjustbox}
	\end{subtable}
	\\[1.5em]
	\begin{subtable}[t]{0.49\textwidth}
		\centering
		\subcaption{\textsc{TinyLlama-1.1B}, \(n=8\), \(\seqlen=512\).}
		\begin{adjustbox}{max width=\linewidth}
\begin{tabular}{l ccc cc}
\toprule
	\multirow{2}{*}{\textbf{Component}} & \multirow{2}{*}{\textbf{Full-Training}} & \multirow{2}{*}{\textbf{Layer-Wise Subset}} & \multirow{2}{*}{\textbf{Global Subset (1P)}} & \multicolumn{2}{c}{\textbf{Global Subset (2P)}} \\
\cmidrule(lr){5-6}
	& & & & Score & Grad \\
\midrule
	\textbf{Forward} & 143.4 & 145.3 & 145.1 & 146.0 & 71.0 \\
	\textbf{Backward} & 399.1 & 438.5 & 436.6 & 393.6 & 209.5 \\
\rowcolor{black!6}
	\quad \texttt{a.grad} & 77.9 & 81.4 & 81.0 & 81.7 & 41.6 \\
\rowcolor{black!6}
	\quad \texttt{scoring} & --- & 42.1 & 42.2 & 42.3 & --- \\
\rowcolor{black!6}
	\quad \texttt{w.grad} & 75.5 & 44.5 & 44.5 & --- & 39.4 \\
\rowcolor{black!6}
	\quad \texttt{autograd} & 245.6 & 270.5 & 268.9 & 269.5 & 128.5 \\
	\textbf{Optimizer} & 80.7 & 80.7 & 80.8 & --- & 80.7 \\
\midrule
	\textbf{Total} & \textbf{623} & \textbf{665} \scriptsize(+6.6\%) & \textbf{662} \scriptsize(+6.3\%) & \multicolumn{2}{c}{\textbf{901} \scriptsize(+44.5\%)} \\
\bottomrule
\end{tabular}

		\end{adjustbox}
	\end{subtable}
	\hfill
	\begin{subtable}[t]{0.49\textwidth}
		\centering
		\subcaption{\textsc{TinyLlama-1.1B}, \(n=2\), \(\seqlen=1024\).}
		\begin{adjustbox}{max width=\linewidth}
\begin{tabular}{l ccc cc}
\toprule
	\multirow{2}{*}{\textbf{Component}} & \multirow{2}{*}{\textbf{Full-Training}} & \multirow{2}{*}{\textbf{Layer-Wise Subset}} & \multirow{2}{*}{\textbf{Global Subset (1P)}} & \multicolumn{2}{c}{\textbf{Global Subset (2P)}} \\
\cmidrule(lr){5-6}
	& & & & Score & Grad \\
\midrule
	\textbf{Forward} & 81.8 & 103.3 & 103.7 & 103.8 & 51.1 \\
	\textbf{Backward} & 211.5 & 314.0 & 313.6 & 288.6 & 118.1 \\
\rowcolor{black!6}
	\quad \texttt{a.grad} & 41.3 & 56.9 & 57.0 & 57.0 & 21.2 \\
\rowcolor{black!6}
	\quad \texttt{scoring} & --- & 41.3 & 41.9 & 41.5 & --- \\
\rowcolor{black!6}
	\quad \texttt{w.grad} & 39.0 & 24.4 & 24.1 & --- & 20.9 \\
\rowcolor{black!6}
	\quad \texttt{autograd} & 131.2 & 191.4 & 190.5 & 190.1 & 76.0 \\
	\textbf{Optimizer} & 80.7 & 80.7 & 80.7 & --- & 81.8 \\
\midrule
	\textbf{Total} & \textbf{374} & \textbf{498} \scriptsize(+33.1\%) & \textbf{498} \scriptsize(+33.1\%) & \multicolumn{2}{c}{\textbf{643} \scriptsize(+72.0\%)} \\
\bottomrule
\end{tabular}

		\end{adjustbox}
	\end{subtable}
	\\[1.5em]
	\begin{subtable}[t]{0.49\textwidth}
		\centering
		\subcaption{\textsc{Llama-3.2-3B}, \(n=8\), \(\seqlen=512\).}
		\begin{adjustbox}{max width=\linewidth}
\begin{tabular}{l ccc cc}
\toprule
	\multirow{2}{*}{\textbf{Component}} & \multirow{2}{*}{\textbf{Full-Training}} & \multirow{2}{*}{\textbf{Layer-Wise Subset}} & \multirow{2}{*}{\textbf{Global Subset (1P)}} & \multicolumn{2}{c}{\textbf{Global Subset (2P)}} \\
\cmidrule(lr){5-6}
	& & & & Score & Grad \\
\midrule
	\textbf{Forward} & 345.3 & 386.5 & 384.2 & 386.3 & 181.8 \\
	\textbf{Backward} & 1108.2 & 1090.9 & 1084.7 & 947.5 & 581.2 \\
\rowcolor{black!6}
	\quad \texttt{a.grad} & 225.8 & 261.0 & 256.2 & 257.3 & 115.4 \\
\rowcolor{black!6}
	\quad \texttt{scoring} & --- & 86.7 & 85.4 & 85.3 & --- \\
\rowcolor{black!6}
	\quad \texttt{w.grad} & 337.8 & 127.6 & 137.7 & --- & 172.5 \\
\rowcolor{black!6}
	\quad \texttt{autograd} & 544.5 & 615.6 & 605.3 & 604.9 & 293.2 \\
	\textbf{Optimizer} & 235.4 & 235.5 & 235.5 & --- & 235.4 \\
\midrule
	\textbf{Total} & \textbf{1689} & \textbf{1713} \scriptsize(+1.4\%) & \textbf{1704} \scriptsize(+0.9\%) & \multicolumn{2}{c}{\textbf{2332} \scriptsize(+38.1\%)} \\
\bottomrule
\end{tabular}

		\end{adjustbox}
	\end{subtable}
	\hfill
	\begin{subtable}[t]{0.49\textwidth}
		\centering
		\subcaption{\textsc{Llama-3.2-3B}, \(n=2\), \(\seqlen=1024\).}
		\begin{adjustbox}{max width=\linewidth}
\begin{tabular}{l ccc cc}
\toprule
	\multirow{2}{*}{\textbf{Component}} & \multirow{2}{*}{\textbf{Full-Training}} & \multirow{2}{*}{\textbf{Layer-Wise Subset}} & \multirow{2}{*}{\textbf{Global Subset (1P)}} & \multicolumn{2}{c}{\textbf{Global Subset (2P)}} \\
\cmidrule(lr){5-6}
	& & & & Score & Grad \\
\midrule
	\textbf{Forward} & 190.4 & 262.5 & 261.6 & 268.2 & 103.7 \\
	\textbf{Backward} & 588.4 & 750.8 & 745.7 & 671.3 & 290.0 \\
\rowcolor{black!6}
	\quad \texttt{a.grad} & 115.7 & 166.6 & 165.0 & 165.6 & 62.5 \\
\rowcolor{black!6}
	\quad \texttt{scoring} & --- & 86.3 & 86.1 & 86.1 & --- \\
\rowcolor{black!6}
	\quad \texttt{w.grad} & 172.6 & 70.8 & 75.7 & --- & 62.8 \\
\rowcolor{black!6}
	\quad \texttt{autograd} & 300.1 & 427.1 & 418.9 & 419.5 & 164.8 \\
	\textbf{Optimizer} & 235.4 & 235.5 & 235.5 & --- & 235.6 \\
\midrule
	\textbf{Total} & \textbf{1014} & \textbf{1249} \scriptsize(+23.1\%) & \textbf{1243} \scriptsize(+22.5\%) & \multicolumn{2}{c}{\textbf{1569} \scriptsize(+54.7\%)} \\
\bottomrule
\end{tabular}

		\end{adjustbox}
	\end{subtable}
\end{table}

\paragraph{Scoring cost.}
\Cref{tab:score-cost} reports the standalone cost of the four scoring mechanisms of \Cref{subsec:efficient-scoring} at \(n=8\), sweeping \(\seqlen\) at fixed \(m=1\) and \(m\) at fixed \(\seqlen=512\). The measurements follow the complexity analysis of \Cref{tab:scoring-comparison}. Among the exact methods, GIP is the cheapest only at short sequences with a single target example (\(\seqlen=512\), \(m=1\): \(29\)--\(73\)~ms), as its leading term predicts (\(m\seqlen < \sqrt{d/L}/2\)). Along \(m\), PIP overtakes GIP at \(m=4\) on \textsc{SmolLM2-360M} (\(52\) vs.\ \(95\)~ms) and \textsc{TinyLlama-1.1B} (\(121\) vs.\ \(138\)~ms) and at \(m=8\) on \textsc{Llama-3.2-3B} (\(429\) vs.\ \(508\)~ms), consistent with the predicted crossover at \(m \gtrsim \sqrt{d/L}/(2\seqlen)\). Along \(\seqlen\), GIP's quadratic cost takes over quickly: at \(\seqlen=2048\) PIP is already the cheapest exact method on all three models (\(122\), \(335\), and \(961\)~ms against GIP's \(356\), \(519\), and \(1001\)~ms), and by \(\seqlen=8192\) GIP costs \(5\)--\(13\times\) PIP. The predicted crossover from PIP to Direct at \(\seqlen \gtrsim \sqrt{d/L}\) (\Cref{subsubsec:per-token-inner-product}), in contrast, does not appear within \(\seqlen \le 8192\): PIP stays ahead of Direct on \textsc{SmolLM2-360M} (\(463\) vs.\ \(575\)~ms) and \textsc{Llama-3.2-3B} (\(3880\) vs.\ \(4091\)~ms) and ties on \textsc{TinyLlama-1.1B} (\(1351\) vs.\ \(1353\)~ms). The two methods share the leading term \(2N\seqlen(d/L)\); PIP's lower-order term \(n\seqlen\sqrt{d/L}\) is the \([n, \seqlen, \sqrt{d/L}]\) intermediate, which our fused kernel keeps in registers, whereas Direct's lower-order term \(n(d/L)\) materializes \(n\) gradient matrices, a memory-bound write and read that no kernel removes. Compressed scoring is the cheapest method at every \(\seqlen\) and \(m\) (\(25\)--\(412\)~ms), and its cost is independent of \(m\).

\end{document}